\documentclass[12pt]{article}

\usepackage[utf8]{inputenc}
\usepackage[english]{babel}
\usepackage[margin=1in]{geometry}

\usepackage{blindtext}
\usepackage{appendix}
\usepackage{amssymb}
\usepackage{hyperref}
\usepackage{mathrsfs}
\hypersetup{
    colorlinks=true,
    linkcolor=blue,
    filecolor=magenta,      
    urlcolor=cyan,
}
 
\usepackage{amsthm}
\usepackage{amsmath}
\usepackage{graphicx}
\usepackage{subfig}
\usepackage{float}
\usepackage{algorithm, algorithmic}
\newtheorem{theorem}{Theorem}[section]
\newtheorem{proposition}{Proposition}[section]

\newtheorem{lemma}[theorem]{Lemma}
\newtheorem{assumption}{Assumption}[section]
\theoremstyle{remark}
\newtheorem{remark}{Remark}[section]

\theoremstyle{definition}
\newtheorem{definition}{Definition}[section]

\def\b0{\mathbf{0}}

\def \mx {\mathbf{x}}
\def \my {\mathbf{y}}
\def \mv {\mathbf{v}}
\def \mz {\mathbf{z}}

\title{\textbf{
Latent Schr{\"o}dinger Bridge  Diffusion Model for Generative Learning
}
}
\author{
Yuling Jiao
\thanks{School of Mathematics and Statistics, Wuhan University, Wuhan, China.
Email: yulingjiaomath@whu.edu.cn}
\and
Lican Kang
\thanks{School of Mathematics and Statistics, Wuhan University, Wuhan, China.
Email: kanglican@whu.edu.cn}
\and
Huazhen Lin
\thanks{ Center of Statistical Research and School of Statistics, Southwestern University of Finance and Economics, Chengdu, China.
Email: linhz@swufe.edu.cn}
\and
Jin Liu 
\thanks{School of Data Science, Chinese, University of Hong Kong (Shenzhen), Shenzhen, 
China.
Email: liujinlab@cuhk.edu.cn}
\and
Heng Zuo
\thanks{School of Mathematics and Statistics, Wuhan University, Wuhan, China.
Email: zuoheng@whu.edu.cn}
}

\date{}

\begin{document}

\maketitle
\begin{abstract}
This paper aims to conduct a comprehensive theoretical analysis of current diffusion models. We introduce a novel generative learning methodology utilizing the Schr{\"o}dinger  bridge diffusion model in latent space as the framework for theoretical exploration in this domain. Our approach commences with the pre-training of an encoder-decoder architecture using data originating from a distribution that may diverge from the target distribution, thus facilitating the accommodation of a large sample size through the utilization of pre-existing large-scale models. Subsequently, we develop a diffusion model within the latent space utilizing the Schr{\"o}dinger  bridge framework.
Our theoretical analysis encompasses the establishment of end-to-end error analysis for learning distributions via the latent  Schr{\"o}dinger   bridge diffusion model. Specifically, we control the second-order Wasserstein distance between the generated distribution and the target distribution. Furthermore, our obtained convergence rates are sharp and effectively mitigate the curse of dimensionality, offering robust theoretical support for prevailing diffusion models.

\vspace{0.5cm} \noindent{\bf KEY WORDS}:
Diffusion models,  Schr{\"o}dinger bridge, Encoder-decoder, Curse of dimensionality,
End-to-end error analysis.
\end{abstract}

\section{Introduction}
Generative models, designed to generate data approximately distributed according to the target distribution, have emerged as a cornerstone in the realm of machine learning. In recent years, their adoption has proliferated across diverse sectors, catalyzing the evolution of various methodologies, including generative adversarial networks (GANs) \cite{goodfellow2014generative}, variational autoencoders (VAEs) \cite{kingma2013auto}, normalizing flows \cite{dinh2016density,rezende2015variational,papamakarios2021normalizing}, and notably, diffusion models
\cite{ho2020denoising,sohl2015deep,nichol2021improved,song2019generative,song2020improved,song2020score}. Distinctly, diffusion models, leveraging stochastic differential equations
(SDEs), have established themselves as a pivotal and mainstream approach in the generative model domain. 
In applications, diffusion models have been instrumental in driving progress in fields like image and audio synthesis, text-to-image generation, natural language processing, 
and text-to video generation
\cite{dhariwal2021diffusion,ho2022cascaded,rombach2022high,saharia2022photorealistic,zhang2023adding,han2022card,austin2021structured,li2022diffusion,liu2024sora}. 
Their proficiency in modeling and generating complex, high-dimensional data distributions 
sets them apart, particularly in scenarios where conventional models are inadequate.

Recently, there has been a significant focus on exploring the theoretical guarantees of diffusion models, with the aim of demystifying their underlying principles from a theoretical standpoint. 
Specifically, this primarily centers on error analysis  with accurate score estimators or end-to-end aspects. The study of known estimated errors in score estimation has prompted investigations \cite{chen2023improved,conforti2023score,lee2022convergence,lee2023convergence,benton2023linear,li2023towards,gao2023wasserstein}.
Additionally, \cite{oko2023diffusion,chen2023score} have provided an end-to-end  error analysis, incorporating the theory of score estimation into their considerations.
While both \cite{oko2023diffusion} and \cite{chen2023score}  delineated  
end-to-end convergence rates and attempted to address the formidable challenge posed by the curse of dimensionality  by presuming a 
low-dimensional space for the target distribution, their theoretical findings fail to fully capture the characteristics of the prevailing stable diffusion model \cite{rombach2022high}  and
Sora \cite{liu2024sora}, which  
incorporate  
encoder-decoder structures in 
pre-training within diffusion models. However, 
The existing theoretical analyses remain inadequate in explaining this augmentation from encoder-decoder structures. 
 This naturally raises critical questions: 

{\it 
How do errors introduced by the encoder-decoder architecture propagate through the diffusion models? and What is the influence of pre-training sample size on the accuracy and stability of diffusion models? 
}

Addressing these questions is essential for advancing a rigorous understanding of the complex interplay between pre-training methodologies, encoder-decoder architectures, and the convergence properties of diffusion models--an area that remains theoretically and practically challenging.

To bridge this gap, we undertake a novel analytical framework for latent diffusion models that leverages the 
Schr{\"o}dinger bridge problem
\cite{schrodinger1932theorie,jamison1975markov,dai1991stochastic,leonard2014survey}   in conjunction with 
encoder-decoder structures. 
This approach aims to establish a rigorous theoretical framework to analyze the propagation and impact of error dynamics arising from 
encoder-decoder architectures, systematically evaluate the role of pre-training sample size in determining the accuracy of diffusion models, and ultimately conduct a comprehensive error analysis within the context of latent diffusion models.

The Schr{\"o}dinger bridge \cite{schrodinger1932theorie,jamison1975markov,dai1991stochastic,leonard2014survey} is  substantiated by pertinent literature references \cite{wang2021deep,de2021diffusion,hamdouche2023generative,liu20232,chen2023schrodinger}, wherein
the SDEs are defined in a finite time  horizon.
This  differs from  diffusion models predominantly 
based on \cite{song2020score}, wherein
 the adopted SDEs, such as Ornstein-Uhlenbeck (OU) and Langevin  SDEs, 
 are defined over an infinite time horizon $[0,\infty)$.
The diffusion models grounded in the Schr{\"o}dinger bridge, as evidenced by \cite{wang2021deep,de2021diffusion,hamdouche2023generative,liu20232,chen2023schrodinger}, have demonstrated exceptional performance in practical applications, surpassing existing diffusion models.  
Additionally, from both sampling and optimization perspectives
\cite{huang2021schrodinger,jiao2021convergence,dai2023global},
the Schr{\"o}dinger bridge has demonstrated its superiority over Langevin SDE methods.
Furthermore, in \cite{wang2021deep}, a convergence analysis is conducted for the Schr{\"o}dinger bridge generative model,
and \cite{de2021diffusion}  provided a theoretical analysis under an accurate score estimation and some regularity assumptions. 
However, these assumptions warrant further  verification, 
and their frameworks do not consider the pre-training encoder-decoder structure.
In our method, we first 
train an encoder-decoder structure separately,  a process which can be conducted with domain shift.
Then, we construct an SDE defined over the time interval $[0,1]$ in the latent space, where the initial distribution is the convolution of the encoder target distribution and a Gaussian distribution, and the terminal distribution is the encoder target distribution. Thirdly, we derive the deep score estimation of the score function corresponding to the initial convolution distribution using score matching techniques \cite{hyvarinen2005estimation,vincent2011connection}. Finally, we employ the  Euler–Maruyama (EM) approach to discretize the SDE corresponding to the estimated score, thereby obtaining the desired samples by implementing the early stopping technique and  the trained decoder. See Section \ref{sec:ps} for a detailed description of our proposed latent diffusion models. 
It is imperative to acknowledge that within our method,  a salient observation arises regarding the decoupling of the pre-training procedure from the formulation of the latent diffusion model. Moreover, the pre-training process is endowed with the capacity to accommodate distributional shifts diverging from the target distribution. Consequently,  this approach facilitates the seamless incorporation of pre-existing large-scale models into the pre-training phase, thus guaranteeing an abundant supply of samples for training endeavors.
Theoretically, we systematically establish rigorous 
end-to-end Wasserstein-distance error bounds between the distribution of generated data and the target distribution, which are minimax optimal.

Our main contributions  can be summarized as follows:
\begin{itemize}
\item
Building upon the Schr{\"o}dinger bridge, 
we introduce a novel latent diffusion model for generative learning.  Our method incorporates a separate encoder-decoder structure, which can be utilized to accommodate domain shift.
This can be accomplished by  integrating pre-existing large-scale models into the pre-training phase, ensuring the availability of ample  samples essential for comprehensive pre-training procedures.
We formulate an  SDE defined on the time interval $[0,1]$, transitioning from the 
Gaussian-convoluted distribution of the encoder  target distribution  to the encoder target distribution. 
Leveraging score matching techniques \cite{hyvarinen2005estimation,vincent2011connection}, EM discretization, 
early stopping technique,
and the trained decoder, we derive  samples that are approximately distributed according to the target distribution.
\item
We rigorously establish the end-to-end theoretical analysis for latent diffusion models. Specifically,  through a combination of   
pre-training,  score estimation,
numerical analysis of SDEs,  and time truncation techniques, we deduce the optimal Wasserstein-distance error bounds between the distribution of generated samples and the target distribution,  
quantified as
$\widetilde{\mathcal{O}}\left(
n^{-\frac{\beta}{d^{*}+2\beta}} 
\right)$. 
Here, $n$ denotes the sample size used to train latent diffusion model,  $d^*$
represents the dimension of latent space,
and $\beta$ refers to the smoothness index of the latent density function.
This rate mitigates the curse of dimensionality inherent in raw data, thereby illustrating the effectiveness of  latent diffusion models from a theoretical perspective.
Moreover, the framework of our theoretical analysis  is also applicable for analyzing other diffusion models. 
\item
Our theoretical results have improved the existing theoretical findings 
of Schr{\"o}dinger bridge diffusion models
\cite{wang2021deep,de2021diffusion}, as it extends beyond mere convergence analysis or a theoretical analysis with an accurate score estimation. 
To the best of our knowledge, this represents the first derivation of an 
end-to-end convergence rate for latent diffusion models based on the Schr{\"o}dinger bridge.
In contrast to the end-to-end analysis in \cite{oko2023diffusion,chen2023score},  
our results integrate the error analysis of the pre-training encoder-decoder structure rather than directly assuming a low-dimensional space for the target distribution.   Consequently, our theoretical findings provide a more comprehensive explanation for mainstream diffusion models \cite{rombach2022high,liu2024sora}.

\end{itemize}

\subsection{Related work}
In this section, we undertake a comprehensive review of pertinent literature, with a primary focus on existing generative models, including GANs, VAEs, normalizing flows, diffusion models, and  Schr{\"o}dinger bridge diffusion.

GANs \cite{goodfellow2014generative}  consist of two neural networks—the generator and the discriminator—trained simultaneously through a competitive process. The generator aims to produce data indistinguishable from real-world data, while the discriminator tries to distinguish between real and generated data. This adversarial process leads to the generation of highly realistic and diverse outputs. Despite their commendable capabilities, GANs confront challenges that warrant careful consideration. Predominant among these challenges are training instability, mode collapse, and hallucination.
In a concerted effort to gain a comprehensive understanding of the intricate dynamics and limitations of GANs, theoretical analyses have been conducted in recent works \cite{huang2022error, jiao2023approximation, hasan2023error}.

VAEs  \cite{kingma2013auto}  are based on an autoencoder architecture, comprising an encoder that compresses data into a latent space and a decoder that reconstructs data from this latent representation. The key feature of VAEs is their use of probabilistic latent variables and Bayesian inference, enabling them to not just replicate input data but to also generate new data that is similar to the input.

Normalizing flows
\cite{dinh2016density,rezende2015variational,papamakarios2021normalizing}, especially those guided by ordinary differential equations (ODEs),  provide a robust framework for transforming a simple probability distribution, such as a standard normal distribution, into a more complex form through a sequence of invertible and differentiable mappings.
The essence of normalizing flows lies in their ability to capture complex dependencies and structures within data by iteratively applying invertible transformations. These transformations facilitate the modeling of sophisticated probability distributions, enabling practitioners to navigate through the intricacies of high-dimensional data and capture latent features.  Recently,  \cite{chen2023restoration,chen2023probability} undertook a theoretical analysis of ODE flows, specifically introducing the  accurate score estimation.  However, this assumption warrants further  verification.
Very recently,  \cite{chang2024deep,gao2024convergence2,jiao2024convergence}
 provided end-to-end  convergence rates for ODE-based (conditional) generative models.

In  diffusion  models, the fundamental architecture comprises both forward and backward processes. The primary objective of the forward process is to systematically add Gaussian noises to the dataset, thereby resulting in the limiting distribution being a Gaussian  distribution.  
Conversely, the backward process assumes the role of a denoising mechanism, undertaking the formidable task of reconstructing data in a manner that closely approximates the characteristics of the original samples. 
Within the denoising step, a critical step emerges in the form of training a score neural network. This neural network is meticulously crafted to approximate the underlying score function,  the gradient of the log density function of the  underlying distribution.  This synthesis of forward and backward processes, coupled with the sophisticated training of a score neural network, encapsulates the essence of diffusion models.
Furthermore, within practical applications, datasets frequently exhibit inherent low-dimensional structures. 
The derivation of their low-dimensional representations within latent spaces is consequently pursued.
A noteworthy advancement within this domain is the stable diffusion model, introduced by \cite{rombach2022high}. The integration of an autoencoder architecture within the stable diffusion model emerges as a pivotal mechanism, facilitating the discernment of the intrinsic low-dimensional essence of raw data.
By executing diffusion processes within the latent spaces, the model excels in encapsulating the underlying distribution. This capability is particularly advantageous within complex data distributions.
In addition, analogous studies have investigated the effective capture of data distributions through the utilization of latent spaces. Reference can be made to \cite{vahdat2021score,ramesh2022hierarchical}.
In summary, these diffusion models can be generally classified into three categories: denoising diffusion probabilistic models, score-based generative models, and score SDEs. 
A thorough examination of diffusion models can be found in the specialized review paper \cite{yang2023diffusion}.
The theoretical analysis of diffusion models primarily focuses on error analysis with accurate score estimators
\cite{chen2023improved,conforti2023score,lee2022convergence,lee2023convergence,benton2023linear,li2023towards,gao2023wasserstein}
or end-to-end error analysis \cite{oko2023diffusion,chen2023score}.  
Recently, an end-to-end theoretical analysis of conditional diffusion models was explored  in
\cite{Fu2024UnveilCD}.
However, these theoretical results did not consider the theoretical analysis of 
pre-training encoder-decoder structures, a crucial step in contemporary mainstream diffusion models. Consequently, these findings are unable to provide further insights into the underlying mechanisms of diffusion models. In contrast, our theoretical results include the theoretical analysis of pre-training encoder-decoder architectures, thereby offering a comprehensive explanation.

Recently,  generative learning based on the 
Schr{\"o}dinger bridge has attracted much attention
\cite{wang2021deep,de2021diffusion,liu20232,hamdouche2023generative,chen2023schrodinger}.
In \cite{wang2021deep}, the authors developed two Schr\"odinger bridges, namely, one spanning from the Dirac measure at the origin to the convolution, and the other advancing from the convolution to the target distribution. Subsequently, they constructed a generative modeling framework by incorporating the drift term, estimated through a deep score estimator and a deep density ratio estimator, into the EM  discretization method. Theoretical analysis was then employed to establish the consistency of the proposed approach.
\cite{de2021diffusion}
presented diffusion Schr\"odinger bridge, an original approximation of the iterative proportional fitting  procedure to solve the Schr\"odinger bridge problem,
and provided theoretical analysis under an accurate score estimation.   
\cite{hamdouche2023generative}
developed a  generative model based on the Schr\"odinger bridge  that captures the temporal dynamics of the time series. They showed some numerical illustrations of  this method in high dimension for the generation of sequential images.
\cite{liu20232}  proposed image-to-image Schr\"odinger bridge, a  class of conditional diffusion models that directly learn the nonlinear diffusion processes between two given distributions.  These diffusion bridges are particularly useful for image restoration, as the degraded images are structurally informative priors for reconstructing the
clean images.
In  \cite{chen2023schrodinger}, the Schr{\"o}dinger bridge diffusion models were introduced within the specific context of  text-to-speech (TTS) synthesis, denoted as Bridge-TTS.  This novel framework, Bridge-TTS, surpasses conventional diffusion models when applied to the domain of TTS synthesis.   Specifically, 
the effectiveness of Bridge-TTS is empirically substantiated through experimental evaluations conducted on the LJ-Speech dataset. The results underscore its prowess in both synthesis quality and sampling efficiency, positioning Bridge-TTS as a notable advancement in the domain of TTS synthesis.
As discussed above, a multitude of generative learning methods,
implementing the Schr{\"o}dinger bridge,
have been proposed. Nonetheless, these works frequently suffer from a deficiency in rigorous theoretical underpinning necessary to support practical utility. In this work, we bridge this gap by furnishing an exhaustive theoretical framework, incorporating the encoder-decoder paradigm.

\subsection{Notations}
We introduce the notations used throughout this paper.
Let $[N]:=\{0,1,\cdots,N\}$ represent the set of integers ranging from 0 to $N$. 
Let $\mathbb{N}^+$ denote the set of positive integers.
For matrices $A, B \in \mathbb{R}^{d\times d}$, we assert $A\preccurlyeq B$ when the matrix $B - A$ is positive semi-definite. The identity matrix in $\mathbb{R}^{d\times d}$ is denoted as $\mathbf{I}_d$.
The $\ell^2$-norm of a vector $\mathbf{x}=\{x_1,\ldots,x_d\}^{\top}\in\mathbb{R}^d$ is defined by $\Vert\mathbf{x}\Vert:=\sqrt{\sum_{i=1}^{d}x_i^2}$. Simultaneously, the operator norm of a matrix $A$ is articulated as $\Vert{A}\Vert:=\sup_{\Vert\mathbf{x}\Vert\leq{1}}\Vert{A\mathbf{x}}\Vert$. The function space $C^2(\mathbb{R}^d)$ encompasses functions that are twice continuously differentiable from $\mathbb{R}^{d}$ to $\mathbb{R}$. For any $f\in C^2(\mathbb{R}^{d})$, the symbols $\nabla{f}$, $\nabla^2{f}$, and ${\rm{\Delta}}f$ signify its gradient, Hessian matrix, and Laplacian, respectively.
The $L^{\infty}(K)$-norm, denoted as 
$\Vert{f}\Vert_{L^{\infty}(K)}:=\sup_{\mathbf{x}\in K}|f(\mathbf{x})|$,
captures the supremum of the absolute values of a function over a  set $K \subset \mathbb{R}^d$. For a vector function $\mathbf{v}:\mathbb{R}^{d}\rightarrow\mathbb{R}^{d}$, the $L^{\infty}(K)$-norm is defined as $\Vert \mathbf{v} \Vert_{L^{\infty}(K)} := \sup_{\mathbf{x}\in K}\Vert \mathbf{v}(\mathbf{x})\Vert$. The asymptotic notation $f(\mathbf{x}) = \mathcal{O}\left(g(\mathbf{x})\right)$ is employed to signify that $f(\mathbf{x})\leq Cg(\mathbf{x})$ for some constant $C > 0$. Additionally, the notation $\widetilde{\mathcal{O}}(\cdot)$ is utilized to discount logarithmic factors in the asymptotic analysis.

\subsection{Outlines}
In Section \ref{sec:prel}, we provide an exposition of the background related to Schr{\"o}dinger bridge problems, along with definitions of deep neural networks, Wasserstein distance, and covering number. 
Section \ref{sec:ps} elaborates on the detailed formulation of our proposed latent diffusion models. The end-to-end theoretical analysis for the generated samples is presented in Section \ref{sec:ta}. 
Concluding remarks  are provided in Section \ref{sec:con}. In Appendix, detailed proofs of all lemmas and theorems in this paper are presented.

\section{Preliminaries}\label{sec:prel}
\textbf{
Schr{\"o}dinger Bridge Problems:
}
To commence,   we introduce some notations. We denote by $\Omega = C([0,1],\mathbb{R}^d)$  the space of $\mathbb{R}^d$-valued continuous functions on the time interval $[0, 1]$.  
We further denote $\mathcal{P}(\Omega)$ as the space of probability measures on the path space $\Omega$.
Within this context, let $\mathbf{W}_{\mx}\in\mathcal{P}(\Omega)$ represent the Wiener measure, characterized by its initial marginal distribution $\delta_{\mx}$. Then,
the law of the  Brownian motion
is defined as $\mathbf{P}= \int \mathbf{W}_{\mx}\mathrm{d}\mx$.
The Schr\"{o}dinger bridge problem (SBP), originally introduced by   \cite{schrodinger1932theorie},  
addresses the task of determining a probability distribution within the path space  $\mathcal{P}(\Omega)$ that smoothly interpolates between two given probability measures  $\nu, \mu\in \mathbb{R}^d$.
The objective is to identify a path measure $\mathbf{Q}^* \in \mathcal{P}(\Omega)$ that approximates the path measure of Brownian motion, with respect to the relative entropy
\cite{jamison1975markov,leonard2014survey}. 
Specifically,  the path measure $\mathbf{Q}^*$ minimizes the 
KL divergence 
$$
\mathbf{Q}^* \in \arg \min \mathbb{D}_{\mathrm{KL}}(\mathbf{Q}||\mathbf{P}),$$ 
subject to  
$$\mathbf{Q}_0 =\nu, \mathbf{Q}_1 = \mu.
$$ 
Here, 
the relative entropy $\mathbb{D}_{\mathrm{KL}}(\mathbf{Q}||\mathbf{P}) = \int \log(\frac{\mathrm{d} \mathbf{Q}}{\mathrm{d} \mathbf{P}}) \mathrm{d} \mathbf{Q} $ if $\mathbf{Q}\ll \mathbf{P}$ (i.e. $\mathbf{Q}$ is absolutely continuous w.r.t. $\mathbf{P}$), and $\mathbb{D}_{\mathrm{KL}}(\mathbf{Q}||\mathbf{P}) = \infty$ otherwise.
$\mathbf{Q}_{t} = (Z_t)_{\#}\mathbf{Q}=\mathbf{Q}\circ Z_t^{-1}, t\in [0,1]$, denotes the marginal measure with $Z = (Z_t)_{t\in [0,1]}$ being the canonical process on $\Omega$.
Additionally, this SBP
can be formulated as 
a  Schr\"{o}dinger system \cite{leonard2014survey},
as shown in the following proposition.
\begin{proposition}\label{th01}\cite{leonard2014survey}
Let $\mathscr{L}$ be the Lebesgue measure.
If $\nu, \mu \ll \mathscr{L}$,  then SBP admits a unique solution $\mathbf{Q}^* = f^*(Z_0)g^*(Z_1)\mathbf{P}$, where
$f^*$ and $g^*$ are $\mathscr{L}$-measurable nonnegative  functions satisfying the  Schr\"{o}dinger system
$$\left\{\begin{array}{l}
f^*(\mx) \mathbb{E}_{\mathbf{P}}\left[g^*\left(Z_{1}\right) \mid Z_{0}=\mx\right]= \frac{\mathrm{d} \nu}{\mathrm{d}\mathscr{L}}(\mx), \quad \mathscr{L}-a . e . \\
g^*(\my)  \mathbb{E}_{\mathbf{P}}\left[f^{*}\left(Z_{0}\right) \mid Z_{1}=\my\right]=\frac{\mathrm{d} \mu}{\mathrm{d}\mathscr{L}}(\my),  \quad \mathscr{L}-a . e .
\end{array}\right.$$
Furthermore, the pair $(\mathbf{Q}^*_{t},\mv^*_{t})$ with $$\mv^*_{t}(\mx) = \nabla_{\mx} \log \mathbb{E}_{\mathbf{P}}\left[g^{*}\left(Z_{1}\right) \mid Z_{t}= \mx\right]$$ solves the minimum action problem
$$\min_{\mu_t,\mv_t} \int_{0}^{1} \mathbb{E}_{\mz\sim \mu_t}[\|\mv_t(\mz)\|^2] \mathrm{d} t$$ s.t.
$$\left\{\begin{array}{l}
\partial_{t}\mu_t = -\nabla \cdot(\mu_t \mv_t) +  \frac{\Delta \mu_t}{2}, \quad \text { on }(0,1) \times \mathbb{R}^{d} \\
\mu_{0}=\nu, \mu_{1}=\mu.
\end{array}\right.
$$
\end{proposition}
Let  $K(s, \mx, t, \my) = [2\pi(t-s)]^{-d/2}\exp\left(-\frac{\|\mx - \my\|^2}{2(t-s)}\right)$ represents the transition density of the Wiener process. Furthermore, let 
$q(\mx)$ and $p(\my)$ denote
the densities of $\nu$ and $\mu$, respectively.  We denote 
$$f_{0}(\mx) = f^*(\mx), \ \ g_{1}(\my) = g^*(\my),$$
$${f_{1}}(\my) = \mathbb{E}_{\mathbf{P}}\left[f^{*}\left(Z_{0}\right) \mid Z_{1}=\my\right] = \int K(0, \mx, 1, \my)f_{0}(\mx) \mathrm{d} \mx,$$
$${g_{0}}(\mx)= \mathbb{E}_{\mathbf{P}}\left[g^*\left(Z_{1}\right) \mid Z_{0}=\mx\right] = \int K(0, \mx, 1, \my)g_{1}(\my) \mathrm{d} \my.$$
Then, the  Schr\"{o}dinger system  in Proposition \ref{th01} can also be characterized by
\begin{equation*}
q(\mx) = f_0(\mx) {g_{0}}(\mx), \ \  p(\my)=  {f_{1}}(\my)g_1(\my)
\end{equation*}
with the following forward and backward time harmonic equations  \cite{chen2021stochastic},
 $$\left\{\begin{array}{l}
\partial_t f_t(\mx) = \frac{\Delta}{2} f_t(\mx),  \\
\partial_t g_t(\mx) = -\frac{\Delta}{2} g_t(\mx),
\end{array}\right. \quad \text { on }(0,1) \times \mathbb{R}^{d}.
$$

Furthermore, SBP can be characterized as a stochastic control problem \cite{dai1991stochastic}. Specifically,  the vector field
\begin{equation}\label{drift}
\begin{aligned}
\mv^*_{t}=\nabla_{\mx}\log g_t(\mx)
= \nabla_{\mx}\log  \int K(t, \mx, 1, \my)g_1(\my) \mathrm{d} \my
\end{aligned}
\end{equation}
solves the following stochastic control problem.
\begin{proposition}\label{th02}\cite{dai1991stochastic}
Let $\mathcal{V}$ consist of admissible Markov controls with finite energy. Then,
$$\mathbf{v}^*_{t}(\mx)\in \arg\min_{\mathbf{v} \in \mathcal{V}}\mathbb{E}\left[\int_0^1\frac{1}{2}\|\mathbf{v}_t\|^2\mathrm{d}t\right]$$
s.t.
\begin{equation}\label{sdeb}
\left\{\begin{array}{l}
\mathrm{d}\mx_t = \mathbf{v}_t \mathrm{d}t+\mathrm{d} B_t, \\
\mx_0\sim q(\mx),\quad \mx_1\sim p(\mx).
\end{array}\right.
\end{equation}
\end{proposition}
With Proposition \ref{th02}, 
the SDE  \eqref{sdeb} with a 
time-varying drift term $\mathbf{v}^*_{t}$ in \eqref{drift} can transport  the initial distribution $\nu$ to  the target distribution $\mu$ on the unit time interval. This enables us to devise a diffusion model  with an estimated drift term.
Roughly speaking, if we designate the initial distribution $\nu$ as $q(\sigma,\mathbf{x})$, as defined in \eqref{aaa}, then the  diffusion process \eqref{model} outlined below serves as a solution to \eqref{sdeb},  transporting  $q(\sigma,\mathbf{x})$ to the target one.

\begin{definition}[ReLU FNNs]\label{relufnns}
A class of neural networks NN$(L,M,J,\kappa)$ with depth $L$, width $M$,  
sparsity level $J$,  
weight bound $\kappa$,
  is defined as
    \begin{equation*}
\begin{aligned}
            {\rm{NN}}(L,M,J,\kappa) = \Big\{&\mathbf{s}(t,\mathbf{x}):= (\mathbf{W}_L{\rm{ReLU}}(\cdot) + \mathbf{b}_L)\circ\cdots\circ(\mathbf{W}_1{\rm{ReLU}}(\cdot) + \mathbf{b}_1)(
            [t, \mathbf{x}^{\top}]^{\top}):\\
&
 \mathbf{W}_i \in \mathbb{R}^{d_{i+1}\times d_i},  \mathbf{b}_i \in \mathbb{R}^{d_{i+1}}, i=0,1,\ldots,L-1,
M:=\max\{d_0,\ldots,d_L\},
\\
&\mathop{\max}_{1\leq{i}\leq{L}}\{\Vert{\mathbf{b}}_i\Vert_{\infty}, \Vert{\mathbf{W}_i}\Vert_{\infty}\}\leq{\kappa},
~ \sum_{i=1}^L\left(\Vert{\mathbf{W}_i}\Vert_0 + \Vert{\mathbf{b}_i}\Vert_0\right)\leq J\Big\}.
\end{aligned}
\end{equation*}
\end{definition}

\begin{definition}[Wasserstein distance]
Let $\mu$ and $\nu$ be two  probability measures defined on $\mathbb{R}^d$ with finite second moments, the 
second-order Wasserstein distance  is defined as:
\begin{equation*}
W_2(\mu,\nu) := 
\left(\inf_{\gamma \in 
\mathcal D(\mu,\nu)} 
\int_{\mathbb R^d} \int_{\mathbb R^d} \|\mx-\my\|^2 \, \gamma(\mathrm{d}\mx, \mathrm{d}\my)  \right)^{1/2},
\end{equation*}
where $\mathcal{D}(\mu, \nu)$ denotes the set of probability measures $\gamma$ on $\mathbb{R}^{2d}$ such that their respective marginal distributions are $\mu$ and $\nu$.
\end{definition}

\begin{definition}[Covering number]
    Let $\rho$ be a pseudo-metric on $\mathcal{U}$ and $S\subseteq\mathcal{U}$. For any $\delta > 0$, a set $A\subseteq\mathcal{U}$ is called a $\delta$-covering number of $S$ if for any $\mathbf{x}\in S$ there exists $\mathbf{y}\in A$ such that $\rho(\mathbf{x},\mathbf{y})\leq\delta$. The $\delta$-covering number of $S$, denoted by $\mathcal{N}(\delta,S,\rho)$, is the minimum cardinality of any $\delta$-covering of $S$.
\end{definition}

\begin{definition}[($\beta$, $R$)-H{\"o}lder Class] Let $\beta = r + \gamma > 0$ be a degree of smoothness, where $r = \lfloor \beta \rfloor$ is an integer and $\gamma\in [0, 1)$. For a function $f: \Omega\rightarrow \mathbb{R}$, its H{\"o}lder norm is defined as
$$
\Vert f \Vert_{\mathcal{H}^{\beta}} := \max_{\Vert\boldsymbol{\alpha}\Vert_1\leq r}\Vert \partial^{\boldsymbol{\alpha}}f \Vert_{\infty} + \max_{\Vert\boldsymbol{\alpha}\Vert_1 = r}\sup_{\mathbf{x}\neq\mathbf{y}}\frac{|\partial^{\boldsymbol{\alpha}} f(\mathbf{x}) - \partial^{\boldsymbol{\alpha}}f(\mathbf{y})|}{\Vert \mathbf{x} - \mathbf{y} \Vert^{\gamma}},
$$
where $\boldsymbol{\alpha}$ is a multi-index. We say a function $f$ is $\beta$-H{\"o}lder, if and only if $\Vert f \Vert_{\mathcal{H}^{\beta}} < \infty$. The ($\beta$, $R$)-H{\"o}lder class for some constant $R > 0$ is defined as
$$
\mathcal{H}^{\beta}(\Omega, R) := \Big\{f: \Omega\rightarrow\mathbb{R} \Big| \Vert f \Vert_{\mathcal{H}^{\beta}} \leq R
 \Big\}.
$$
\end{definition}
\section{Method}\label{sec:ps}
In this section, we articulate our proposed latent diffusion model grounded in the Schr{\"o}dinger bridge.  
Firstly, our method  commences with the pre-training of an encoder-decoder architecture in 
Section \ref{sec:pre}, which utilizes a dataset characterized by a distribution that may differ from the target distribution under consideration.
Secondly, we formulate  the latent diffusion model in Section \ref{sec:ldm}.
Specifically, we introduce an SDE in the latent space, and utilize the score matching method \cite{hyvarinen2005estimation,vincent2011connection} to estimate the score function associated with the convolution distribution. We then proceed to define the SDE corresponding to the estimated score, and employ the EM discretization technique to numerically solve this equation in the latent space. This computational approach facilitates the generation of desired samples through the utilization of the decoder component of our trained encoder-decoder architecture, culminating in the acquisition of data aligned with the target distribution.

\subsection{Pre-training}\label{sec:pre}  
In this section, we construct a framework for pre-training an encoder-decoder structure. This foundational step serves as a cornerstone, delineating our proficiency in data compression and the extraction of its low-dimensional structure. Such strategic advancement aligns with the principles advocated in \cite{ramesh2022hierarchical,rombach2022high,liu2024sora}.

Let $\widetilde{p}_{data}$ denote the distribution of the pre-trained data, which may diverge from the target distribution.
Then, on a population level, the encoder-decoder pair $(\boldsymbol{E}^*,\boldsymbol{D}^*)$ can be derived by minimizing the reconstruction loss, defined as:
\begin{align*}
(\boldsymbol{E}^*,\boldsymbol{D}^*) \in \arg\min_{\boldsymbol{E}, \boldsymbol{D} \text{ measurable}} \mathcal{H} (\boldsymbol{E}, \boldsymbol{D}) := \int_{\mathbb{R}^d} \| (\boldsymbol{D}\circ\boldsymbol{E})(\my) - \my \|^2 \widetilde{p}_{data}(\my)\mathrm{d}\my,
\end{align*}
where the minimum is taken over all measurable functions $\boldsymbol{E}: \mathbb{R}^d \rightarrow \mathbb{R}^{d^*}$ and $\boldsymbol{D}: \mathbb{R}^{d^*}\rightarrow \mathbb{R}^d$ with $d^*<d$ being a predetermined integer. 
It is observed that the existence of the pair $(\boldsymbol{E}^*,\boldsymbol{D}^*)$  within the set of measurable functions is guaranteed,  as there exists a pair for which their composition equals the identity operator. Subsequently, in Section \ref{sec:eb_pt}, we confine  this pair to specific classes of Lipschitz continuous functions, as delineated in Assumption \ref{ass:bounded_support}.

In practice, our access is limited to the i.i.d. pre-training data  $\mathcal{Y} := \{\my_i\}_{i=1}^{\mathcal{M}}$, sampled from the distribution $\widetilde{p}_{data}$, 
where $\mathcal{M}$ denotes the sample size utilized during the 
pre-training phase. Subsequently, we can establish an estimator through empirical risk minimization:
\begin{align}
\label{eq: autoencoder erm}
(\widehat{\boldsymbol{E}}, \widehat{\boldsymbol{D}}) \in \arg\min_{\boldsymbol{E} \in \mathcal{E}, \boldsymbol{D} \in \mathcal{D}} \widehat{\mathcal{H}} (\boldsymbol{E}, \boldsymbol{D}) := \frac{1}{\mathcal{M}} \sum_{i=1}^{\mathcal{M}} \| (\boldsymbol{D}\circ\boldsymbol{E})(\my_i) - \my_i \|^2,
\end{align}
where $\mathcal{E}$  and  $\mathcal{D}$
represent two neural network classes mapping  $\mathbb{R}^d$ to $\mathbb{R}^{d^*}$ and $ \mathbb{R}^{d^*}$ to $\mathbb{R}^d$, respectively.

We observe that the derivation of $(\widehat{\boldsymbol{E}}, \widehat{\boldsymbol{D}})$
as defined in  \eqref{eq: autoencoder erm}  can be facilitated through the utilization of pre-existing large-scale models. This approach permits accommodation for distributional shifts that may deviate from the target distribution. As a result, 
 it guarantees an abundant supply of samples essential for pre-training endeavors, 
enabling us to obtain convergence rates that circumvent the curse of dimensionality, as discussed in our theoretical development in   Subsection \ref{sec:oi} and
 Section \ref{sec:eb_mr}.
\subsection{Latent Diffusion Modeling}\label{sec:ldm}
With the trained encoder-decoder structure established in Section \ref{sec:pre}, we can subsequently construct the diffusion model in the latent space.
We denote by 
$
\widehat{\mathbf{x}}
:=\widehat{\boldsymbol{E}}\mathbf{x}
$ with $\mathbf{x} \sim p_{data}(\cdot)$. 
Let ${p}^{*}_{data}(\cdot)$  and $\widehat{p}^{*}_{data}(\cdot)$ represent the density functions of  $\boldsymbol{E}^*\mathbf{x}$ and 
$\widehat{\mathbf{x}}$, respectively. Additionally, let $\Phi_{\sigma}(\cdot)$ denote the density of the normal distribution 
$\mathcal{N}
(\mathbf{0},\sigma^2\mathbf{I}_{d^*})$.
Let $\mathcal{X}:=\{\mathbf{x}_i\}_{i=1}^n$ be i.i.d. from $p_{data}$, then we have access to the encoder data
$
\widehat{\mathcal{X}}:=
\{\widehat{\mathbf{x}}_{i}\}_{i=1}^{n}
=\{\widehat{\boldsymbol{E}}\mathbf{x}_i\}_{i=1}^n
$.
We denote the convolution of $\widehat{p}^{*}_{data}$ and $\Phi_{\sigma}(\cdot)$ as $q(\cdot,\cdot)$, defined as follows:
    \begin{equation}\label{aaa}
        q(\sigma,\mathbf{x}):= \int{\widehat{p}^{*}_{data}}(\mathbf{y})\Phi_{\sigma}(\mathbf{x} - \mathbf{y})\mathrm{d}\mathbf{y}.
    \end{equation}
Subsequently, the diffusion process can be formulated as  a forward SDE:
    \begin{equation}\label{model}
        \mathrm{d}\mathbf{x}_t = \sigma^2\nabla\log{q}_{1-t}(\mathbf{x}_t)\mathrm{d}t + \sigma\mathrm{d}\mathbf{w}_t,~ t \in [0,1],~\mathbf{x}_0\sim{q}({\sigma}, \mathbf{x}),
    \end{equation}
where $q_{1-t}(\cdot):=q(\sqrt{1-t}\sigma, \cdot)$, $\mathbf{w}_t$ represents a 
$d^*$-dimensional Brownian motion.
Consequently, $\mathbf{x}_1$ follows the distribution $\widehat{p}^{*}_{data}$ as established in Proposition \ref{th02} and  \cite[Theorem 3]{wang2021deep}. This formalization provides a coherent and rigorous foundation for our generative framework.

\noindent
\textbf{Score Matching.}
The SDE provided in  \eqref{model} offers a pathway for constructing a generative model through the application of EM discretization, contingent upon access to the target score denoted as 
$\mathbf{s}^*(t,\mathbf{x}):=\nabla\log{q}_{t}(\mathbf{x})$. However, in practical scenarios, this  score function remains elusive due to its intricate connection with the target density $\widehat{p}^*(\cdot)$. Consequently, we resort to score matching techniques \cite{hyvarinen2005estimation,vincent2011connection} employing deep ReLU FNNs.

Verification reveals that $\mathbf{s}^*$ minimizes the loss function $\mathcal{L}(\mathbf{s})$ across all measurable functions, with $\mathcal{L}(\mathbf{s})$ expressed as
    \begin{equation*}
        \mathcal{L}(\mathbf{s}) := \frac{1}{T}\int_{0}^T\left(\mathbb{E}_{\mathbf{x}_t\sim q_{1-t}(\mathbf{x})}\Vert{\mathbf{s}(1-t, \mathbf{x}_t) - \nabla\log{q_{1-t}(\mathbf{x}_t)}}\Vert^2\right)\mathrm{d}t,
    \end{equation*}
where $0<T<1$.
We notice that the deliberate selection of  $T$ is motivated by the necessity to preclude the score function from exhibiting a blow-up at  $T=1$, concomitantly facilitating the stabilization of the training process for the model. A meticulous exposition regarding the determination of 
$T$ is provided in Subsection \ref{sec:oi}.
Furthermore, the application of this time truncation technique extends beyond our work and has found utility in both model training and theoretical analyses in diffusion models \cite{song2020improved,vahdat2021score,chen2023score,chen2023improved,oko2023diffusion}.

In alignment with \cite{hyvarinen2005estimation,vincent2011connection}, we alternatively consider a denoising score matching objective. With slight notational abuse, we denote this objective as
\begin{equation*} 
        \begin{aligned}
            \mathcal{L}(\mathbf{s}) &= \frac{1}{T}\int_{0}^{T}\left(\mathbb{E}_{{\widehat{p}}^*_{data}(\mathbf{x})}\mathbb{E}_{\mathcal{N}(\mathbf{x}_t;\mathbf{x},(1-t)\sigma^2\mathbf{I})}\left\Vert{\mathbf{s}(1-t, \mathbf{x}_t) + \frac{\mathbf{x}_t - \mathbf{x}}{(1-t)\sigma^2}}\right\Vert^2\right)\mathrm{d}t\\
            &= \frac{1}{T}\int_{0}^{T}\left(\mathbb{E}_{{\widehat{p}}^*_{data}(\mathbf{x})}\mathbb{E}_{\mathcal{N}(\mathbf{z};\mathbf{0},\mathbf{I})}\left\Vert{\mathbf{s}(1-t, \mathbf{x} + \sigma\sqrt{1-t}\mathbf{z}) + \frac{\mathbf{z}}{\sqrt{1-t}\sigma}}\right\Vert^2\right)\mathrm{d}t.
        \end{aligned}
    \end{equation*}
Given $n$ i.i.d.  samples 
$\{\widehat{\mathbf{x}}_{i}\}_{i=1}^{n}$ from ${\widehat{p}}^*_{data}(\cdot)$, 
and $m$ i.i.d. samples $\{(t_j, \mathbf{z}_j)\}_{j=1}^{m}$  from $\mathrm{U}[0,T]$ and $\mathcal{N}(\mathbf{z};\mathbf{0},\mathbf{I}_{d^*})$, we can employ the empirical risk minimizer (ERM) to estimate the score $\mathbf{s}^*$. This ERM, denoted as $\widehat{\mathbf{s}}$, is determined by
\begin{align}\label{escore}
\widehat{\mathbf{s}}\in{\mathop{\rm{argmin}}_{\mathbf{s}\in\mathrm{NN}}}
~\widehat{\mathcal{L}}(\mathbf{s}):= \frac{1}{mn}\sum_{j=1}^{m}\sum_{i=1}^{n}\left\Vert{\mathbf{s}(1-t_j, \widehat{\mathbf{x}}_{i} + \sigma\sqrt{1-t_j}\mathbf{z}_{j}) + \frac{\mathbf{z}_{j}}{\sigma\sqrt{1-t_j}}}\right\Vert^2,
\end{align}
where $\text{NN}$ refers to ReLU FNNs defined in 
Definition \ref{relufnns}.

\noindent
\textbf{EM Discretization.} 
Given the estimated score function $\widehat{\mathbf{s}}$, as defined in \eqref{escore}, we can formulate an SDE  initializing from the prior distribution: 
\begin{equation}\label{sampling_SDE}
        \mathrm{d}\widehat{\mathbf{x}}_t = \sigma^2\widehat{\mathbf{s}}(t,\widehat{\mathbf{x}}_t)\mathrm{d}t + \sigma \mathrm{d}\mathbf{w}_t,~ 
        \widehat{\mathbf{x}}_0 = \mathbf{x}_0\sim{q}(\sigma,\mathbf{x}),~ 0\leq{t}\leq{T}. 
\end{equation}
Now, we employ a discrete-time approximation for the sampling dynamics \eqref{sampling_SDE}.  
Let 
$$0=t_0<t_1<\cdots<t_K=T,~ K \in \mathbb{N}^+,$$ 
be the discretization points on $[0,T]$. 
We consider the explicit 
EM scheme:
\begin{equation}\label{EM_scheme}
    \mathrm{d}\widetilde{\mathbf{x}}_t = \sigma^2\widehat{\mathbf{s}}(t_i, \widetilde{\mathbf{x}}_{t_i})\mathrm{d}t + \sigma \mathrm{d}\mathbf{w}_t,~ \widetilde{\mathbf{x}}_0 = \mathbf{x}_0\sim{q}(\sigma,\mathbf{x}),~ t\in[t_i, t_{i+1}),
\end{equation}
for $i=0,1,\cdots,K-1$. 
Subsequently, we can utilize dynamics \eqref{EM_scheme} to generate new samples
 by using the trained decoder in Section \ref{sec:pre}.

In summary, the specific structure of our proposed method is outlined in Algorithm \ref{sampling_algorithm}.
\begin{algorithm}[H]
\caption{
The Proposed Latent Diffusion Model
}
\label{sampling_algorithm}
\begin{algorithmic}
\STATE 
1. {\bf Input:}
$\sigma$, $N$, $d^*$,
 $\{\mathbf{x}_{i}\}_{i=1}^{n} \sim p_{data}$, 
 $\{\my_i\}_{i=1}^{\mathcal{M}} \sim \widetilde{p}_{data}$, $\{(t_j, \mathbf{z}_j)\}_{j=1}^{m} \sim \text{U}[0,T] ~\mbox{and}~\mathcal{N}(\mathbf{z};\mathbf{0},\mathbf{I}_{d^*})$.
\STATE
2. 
Obtain encoder $\widehat{\boldsymbol{E}}$ and decoder
$\widehat{\boldsymbol{D}}$
 via
\eqref{eq: autoencoder erm}.
\STATE 
3. {\bf Score estimation:} 
Obtain  $\widehat{\mathbf{s}}(\cdot,\cdot)$ by  \eqref{escore}.
\STATE 
4. {\bf Sampling procedure:} 
\STATE 
Sample 
$\mathbf{\epsilon}\sim {\mathcal{N}}(\mathbf{0},\mathbf{I}_{d^*})$, 
$\mathbf{y}\in  \{\widehat{\mathbf{x}}_{i}\}_{i=1}^{n}$;
\STATE $\widetilde{\mathbf{x}}_0 = \mathbf{y} + \sigma\mathbf{\epsilon}$;
        \FORALL{$i=0,1,\ldots,K-1$}{
            \STATE Sample $\mathbf{\epsilon}_{i}\sim{\mathcal{N}}(\mathbf{0},\mathbf{I}_{d^*})$;
            \STATE $\mathbf{b}(\widetilde{\mathbf{x}}_i) = \widehat{\mathbf{s}}\left(\sqrt{1-\frac{i}{K}}\sigma, \widetilde{\mathbf{x}}_i\right)$;
            \STATE $\widetilde{\mathbf{x}}_{i+1} = \widetilde{\mathbf{x}}_i + \frac{\sigma^2}{K}\mathbf{b}(\widetilde{\mathbf{x}}_i) + \frac{\sigma}{K}\mathbf{\epsilon}_i$;
        }
        \ENDFOR
\STATE
5. {\bf Output:}
$\widehat{\boldsymbol{D}}(\widetilde{\mathbf{x}}_K)$.        
    \end{algorithmic}
\end{algorithm}

In Algorithm \ref{sampling_algorithm}, the process involves three primary steps: the initial pre-training, score estimation, and the subsequent sampling procedure.
We separately train the 
encoder-decoder structure and execute the diffusion model in the latent space. This procedure  also mirrors the approach described in \cite{ramesh2022hierarchical,rombach2022high,liu2024sora}.
Therefore, this pre-training phase can be effectively completed by implementing pre-existing large-scale models, owing to its inherent capability to accommodate distributional shifts.
This  enhances the flexibility of our method, empowering it to navigate diverse data landscapes with agility and adaptability.
In contrast to existing diffusion models \cite{song2020score}, our approach eliminates the necessity of an additional step involving the introduction of Gaussian noises, particularly in the forward process. This efficiency stems from the inherent nature of our algorithm, where the input directly represents the convolution of the encoder target distribution and a Gaussian noise. Consequently, there is no requirement to construct a separate forward process for the addition of Gaussian noise, as is common in traditional diffusion models. 
This characteristic and the latent space highlight the computational efficiency inherent in our approach.
Furthermore, in comparison to the diffusion model based on the Schr{\"o}dinger bridge presented in \cite{wang2021deep}, our method does not necessitate the construction of a bridge extending from the origin to the convolution distribution,
and we introduce a  
pre-training procedure to further enhance flexibility.

\section{Theoretical Analysis}\label{sec:ta}

In this section, we establish comprehensive error bounds for the samples generated by Algorithm \ref{sampling_algorithm}. Our focus lies in quantifying the end-to-end performance through the second-order Wasserstein distance, measuring the dissimilarity between the distributions of the generated samples and the target distribution.
This entails conducting error analyses for both the pre-training and latent diffusion model. Therefore, we commence by elucidating the pre-training theory in Section \ref{sec:eb_pt}, followed by the exploration of error analysis in the latent diffusion model in Section \ref{sec:eb_ldm}. Subsequently, we integrate insights from these two analyses to establish our desired error bounds, as delineated in Section \ref{sec:eb_mr}.
\subsection{Error Bounds in Pre-training}\label{sec:eb_pt}
In this section, we derive the error bound for the trained encoder-decoder established in Section \ref{sec:pre}. 
To this end, we  introduce two fundamental assumptions.  
\begin{assumption}[Bounded Support]\label{ass:bounded_support}
    The pre-trained data distribution $\widetilde{p}_{data}$ is supported on $[-1,1]^{d}$.
\end{assumption}

\begin{assumption}[Compressibility]
\label{ass: compressibility}
There exist continuously differential functions $\boldsymbol{E}^{*}:[-1,1]^{d}\rightarrow[-1,1]^{d^{*}}$ and $\boldsymbol{D}^{*}:[-1,1]^{d^{*}}\rightarrow\mathbb{R}^{d}$ such that $\mathcal{H}(\boldsymbol{E},\boldsymbol{D})$ attains its minimum. The minimum value is denoted by
$\varepsilon_{\mathrm{compress}-{\mathrm{noise}}} := \mathcal{H} (\boldsymbol{E}^*, \boldsymbol{D}^*)\leq \delta_0$,  where $\delta_0\geq 0$. 
Furthermore, $\boldsymbol{E}^*\in \mathrm{Lip}(\xi_{\boldsymbol{E}}), \boldsymbol{D}^* \in \mathrm{Lip}(\xi_{\boldsymbol{D}})$ respectively.
\end{assumption}
\begin{remark}
These two assumptions  assume a paramount role in providing a solid underpinning for the pre-training framework. 
Additionally, the theoretical analysis of diffusion models elucidates the bounded support assumption on the target distribution, as presented in \cite[Assumption 2]{lee2023convergence}, \cite{li2023towards}, and \cite[Assumption 2.4]{oko2023diffusion},  establishing it as a foundational pillar in the study of diffusion models.
In Assumption \ref{ass: compressibility},
when the data is perfectly compressible, implying the existence of   $\boldsymbol{D}^*, \boldsymbol{E}^*$
    such that $\boldsymbol{D}^*\circ\boldsymbol{E}^*=\mathbf{I}_d$,  then $\delta_0 = 0.$
\end{remark}
\begin{lemma}\label{lem:pretraining_approximation_error} 
Suppose that Assumptions \ref{ass:bounded_support}-\ref{ass: compressibility} hold. Then,  for any $\epsilon > 0$, there exist two ReLU neural networks $\boldsymbol{E}\in\mathcal{E}$ and $\boldsymbol{D}\in\mathcal{D}$ with respective parameters
$$
L_{\boldsymbol{E}} = \mathcal{O}\left(\log\frac{1}{\epsilon} + d\right), M_{\boldsymbol{E}}=\mathcal{O}\left(d^{*}\xi_{\boldsymbol{E}}^{d}\epsilon^{-d}\right), J_{\boldsymbol{E}}=\mathcal{O}\left(d^{*}\xi_{\boldsymbol{E}}^{d}\epsilon^{-d}\left(\log\frac{1}{\epsilon} + d\right)\right),
\kappa_{\boldsymbol{E}} = \mathcal{O}\left(\max\{1, \xi_{\boldsymbol{E}}\}\right),
$$
$$
L_{\boldsymbol{D}} = \mathcal{O}\left(\log\frac{1}{\epsilon} + d^{*}\right), M_{\boldsymbol{D}}=\mathcal{O}\left(d\xi_{\boldsymbol{D}}^{d^{*}}\epsilon^{-d^{*}}\right), J_{\boldsymbol{D}}=\mathcal{O}\left(d\xi_{\boldsymbol{D}}^{d^{*}}\epsilon^{-d^{*}}\left(\log\frac{1}{\epsilon} + d^{*}\right)\right), \kappa_{\boldsymbol{D}} = \mathcal{O}\left(\max\{1, \xi_{\boldsymbol{D}}\}\right),
$$
such that
$$
\mathcal{H}(\boldsymbol{E},\boldsymbol{D}) - \mathcal{H}(\boldsymbol{E}^{*}, \boldsymbol{D}^{*}) \lesssim \epsilon.
$$
Moreover, $\boldsymbol{E}$ and $\boldsymbol{D}$ are Lipschitz continuous with Lipschitz constants $\gamma_{\boldsymbol{E}}:= 10d\xi_{\boldsymbol{E}}$ and $\gamma_{\boldsymbol{D}}:= 10d^*\xi_{\boldsymbol{D}}$, respectively. That is, for any $\mathbf{y}_1, \mathbf{y}_2\in [-1,1]^d$, $\boldsymbol{E}$ satisfies
$$
\Vert \boldsymbol{E}(\mathbf{y}_1) - \boldsymbol{E}(\mathbf{y}_2)\Vert_\infty \leq \gamma_{\boldsymbol{E}} \Vert \mathbf{y}_1 - \mathbf{y}_2\Vert,
$$
for any $\mathbf{y}_1, \mathbf{y}_2 \in [-1,1]^{d^*}$, $\boldsymbol{D}$ satisfies
$$
\Vert \boldsymbol{D}(\mathbf{y}_1) - \boldsymbol{D}(\mathbf{y}_2)\Vert_\infty \leq \gamma_{\boldsymbol{D}} \Vert \mathbf{y}_1 - \mathbf{y}_2\Vert.
$$
And there exist two constants $K_{\boldsymbol{D}}=\mathcal{O}(1)$, $K_{\boldsymbol{E}}=\mathcal{O}(1)$ such that
$$
\sup_{\mathbf{y}\in [-1,1]^d}\|\boldsymbol{E}(\mathbf{y})\| \leq K_{\boldsymbol{E}}, ~ \sup_{\mathbf{y}\in [-1,1]^{d^*}}\|\boldsymbol{D}(\mathbf{y})\| \leq K_{\boldsymbol{D}}.
$$
\end{lemma}

Now, we restrict $\boldsymbol{E} \in \mathcal{E}$ and $\boldsymbol{D} \in \mathcal{D}$ to the following two classes:
$$
\begin{aligned}
\mathcal{E}_0:= \Big\{\boldsymbol{E}\in \mathcal{E}\Big|&
\text{The network configurations satisfy}\\
& L_{\boldsymbol{E}} = \mathcal{O}\left(\log\frac{1}{\epsilon} + d\right), M_{\boldsymbol{E}}=\mathcal{O}\left(d^{*}\xi_{\boldsymbol{E}}^{d}\epsilon^{-d}\right), \\ & J_{\boldsymbol{E}}=\mathcal{O}\left(d^{*}\xi_{\boldsymbol{E}}^{d}\epsilon^{-d}\left(\log\frac{1}{\epsilon} + d\right)\right), 
\kappa_{\boldsymbol{E}} = \mathcal{O}\left(\max\{1, \xi_{\boldsymbol{E}}\}\right), \sup_{\mathbf{y}\in [-1,1]^d}\Vert\boldsymbol{E}(\mathbf{y})\Vert \leq K_{\boldsymbol{E}},\\
& 
\forall \mathbf{y}_1, \mathbf{y}_2\in [-1,1]^d,
\Vert \boldsymbol{E}(\mathbf{y}_1) - \boldsymbol{E}(\mathbf{y}_2)\Vert_\infty \leq \gamma_{\boldsymbol{E}} \Vert \mathbf{y}_1 - \mathbf{y}_2\Vert
\Big\},
\end{aligned}
$$
$$
\begin{aligned}
\mathcal{D}_0:= \Big\{ \boldsymbol{D}\in\mathcal{D} \Big| & \text{The network configurations satisfy} \\
& L_{\boldsymbol{D}} = \mathcal{O}\left(\log\frac{1}{\epsilon} + d^{*}\right), M_{\boldsymbol{D}}=\mathcal{O}\left(d\xi_{\boldsymbol{D}}^{d^{*}}\epsilon^{-d^{*}}\right), \\ & J_{\boldsymbol{D}}=\mathcal{O}\left(d\xi_{\boldsymbol{D}}^{d^{*}}\epsilon^{-d^{*}}\left(\log\frac{1}{\epsilon} + d^{*}\right)\right),
\kappa_{\boldsymbol{D}} = \mathcal{O}\left(\max\{1, \xi_{\boldsymbol{D}}\}\right), 
\sup_{\mathbf{y}\in [-1,1]^{d^*}} \Vert\boldsymbol{D}(\mathbf{y})\Vert \leq K_{\boldsymbol{D}},\\
& 
\forall \mathbf{y}_1, \mathbf{y}_2\in [-1,1]^{d^*},
\Vert \boldsymbol{D}(\mathbf{y}_1) - \boldsymbol{D}(\mathbf{y}_2)\Vert_\infty \leq \gamma_{\boldsymbol{D}} \Vert \mathbf{y}_1 - \mathbf{y}_2\Vert
\Big\}.
\end{aligned}
$$
Consequently,  we can derive the error bound for the pre-training procedure, as shown in the following theorem.
\begin{theorem}\label{thm:pretraining}
Suppose that Assumptions \ref{ass:bounded_support}-\ref{ass: compressibility} hold. By taking $\epsilon = \mathcal{M}^{-\frac{1}{d + 2}}$ in Lemma \ref{lem:pretraining_approximation_error}, the encoder estimator $\widehat{\boldsymbol{E}}$ with the neural network structures belonging to class $\mathcal{E}_0$ and the decoder estimator $\widehat{\boldsymbol{D}}$ with neural network structures belonging to class $\mathcal{D}_0$ satisfy
$$
\mathbb{E}_{\mathcal{Y}}
\mathbb{E}_{\my \sim \widetilde{p}_{data}}
\| (\widehat{\boldsymbol{D}}\circ\widehat{\boldsymbol{E}})(\my) - \my \|^2  = \widetilde{\mathcal{O}}({\mathcal{M}}^{-\frac{1}{d+2}}+\delta_0).
$$
\end{theorem}
\subsection{Error Bounds in Latent Diffusion Modeling}\label{sec:eb_ldm}
In this section, we obtain the error bounds in latent diffusion modeling. We denote $ \pi_t$, $\widehat{\pi}_t$,  and $\widetilde{\pi}_t$ as the distributions of $ \mathbf{x}_t$, $\widehat{\mathbf{x}}_t$,
and $\widetilde{\mathbf{x}}_t$ correspondingly.  
Since the diffusion process \eqref{model} converges to $\widehat{p}_{data}^*$ as $t \rightarrow 1$ and the domain is compact, we define a truncated estimator $\pi_T^L$ for $\widehat{p}_{data}^*$ as the distribution of $\mathbf{x}_T\cdot\mathbf{I}_{\{\Vert\mathbf{x}_T\Vert_\infty \leq L\}}$, for a large constant $L > 0$ such that $\pi_T^L \approx \widehat{p}_{data}^*$. 
Similarly, we define $\widehat{\pi}_T^L$ and  $\widetilde{\pi}_T^L$ as the distributions of $\widehat{\mathbf{x}}_T\cdot\mathbf{I}_{\{\Vert\widehat{\mathbf{x}}_T\Vert_\infty \leq L\}}$ and $\widetilde{\mathbf{x}}_T\cdot\mathbf{I}_{\{\Vert\widetilde{\mathbf{x}}_T\Vert_\infty \leq L\}}$, respectively. Here, we truncate the supports of the distributions $\pi_T, \widehat{\pi}_T$,
and $\widetilde{\pi}_T$ to  ensure a bounded support for the distribution estimators, which is done for technical reasons.

Let $\mathcal{T}:=\{t_j\}_{j=1}^{m}, \mathcal{Z} := \{\mathbf{z}_j\}_{j=1}^{m}$. 
Denote
\begin{align*}
\overline{\mathcal{L}}_{\widehat{\mathcal{X}}}(\mathbf{s}):=\frac{1}{n}\sum_{i=1}^{n}\ell_{\mathbf{s}}(\widehat{\mathbf{x}}_i), 
~
\text{and}~
\widehat{\mathcal{L}}_{\widehat{\mathcal{X}},\mathcal{T}, \mathcal{Z}}(\mathbf{s}):= \frac{1}{n}\sum_{i=1}^{n}\widehat{\ell}_{\mathbf{s}}(\widehat{\mathbf{x}}_i),
\end{align*}
where
$$
\ell_{\mathbf{s}}(\widehat{\mathbf{x}}):= \frac{1}{T}\int_{0}^{T}\mathbb{E}_{\mathbf{z}}\left\Vert \mathbf{s}(1-t,\mathbf{x} + \sigma\sqrt{1-t}\mathbf{z}) + \frac{\mathbf{z}}{\sigma\sqrt{1-t}}\right\Vert^2 \mathrm{d}t,
$$
and
$$
\widehat{\ell}_{\mathbf{s}}(\widehat{\mathbf{x}}):= \frac{1}{m}\sum_{j=1}^{m}\left\Vert \mathbf{s}(1-t_j, \widehat{\mathbf{x}} + \sigma\sqrt{1-t_j}\mathbf{z}_j) + \frac{\mathbf{z}_j}{\sigma\sqrt{1-t_j}}\right\Vert^2.
$$
Then,  for any $\mathbf{s}\in\mathrm{NN}$,  it yields that
\begin{equation*}
\begin{aligned}
&\mathcal{L}(\widehat{\mathbf{s}}) - \mathcal{L}(\mathbf{s}^{*}) \\
&= \mathcal{L}(\widehat{\mathbf{s}}) - 2\overline{\mathcal{L}}_{\widehat{\mathcal{X}}}(\widehat{\mathbf{s}}) + \mathcal{L}(\mathbf{s}^{*}) + 2\left(\overline{\mathcal{L}}_{\widehat{\mathcal{X}}}(\widehat{\mathbf{s}}) - \mathcal{L}(\mathbf{s}^{*})\right)\\
        &=\mathcal{L}(\widehat{\mathbf{s}}) - 2\overline{\mathcal{L}}_{\mathbf{\widehat{\mathcal{X}}}}(\widehat{\mathbf{s}}) + \mathcal{L}(\mathbf{s}^{*}) + 2\left(\overline{\mathcal{L}}_{\widehat{\mathcal{X}}}(\widehat{\mathbf{s}}) - \widehat{\mathcal{L}}_{\widehat{\mathcal{X}},\mathcal{T},\mathcal{Z}}(\widehat{\mathbf{s}})\right)
         + 2\left(\widehat{\mathcal{L}}_{\widehat{\mathcal{X}},\mathcal{T},\mathcal{Z}}(\widehat{\mathbf{s}}) - \mathcal{L}(\mathbf{s}^{*})\right)\\
         &\leq \mathcal{L}(\widehat{\mathbf{s}}) - 2\overline{\mathcal{L}}_{\widehat{\mathcal{X}}}(\widehat{\mathbf{s}}) + \mathcal{L}(\mathbf{s}^{*}) + 2\left(\overline{\mathcal{L}}_{\widehat{\mathcal{X}}}(\widehat{\mathbf{s}}) - \widehat{\mathcal{L}}_{\widehat{\mathcal{X}},\mathcal{T},\mathcal{Z}}(\widehat{\mathbf{s}})\right)
         + 2\left(\widehat{\mathcal{L}}_{\widehat{\mathcal{X}},\mathcal{T},\mathcal{Z}}(\mathbf{s}) - \mathcal{L}(\mathbf{s}^{*})\right).
    \end{aligned}
\end{equation*}
Taking expectations, followed by taking the  infimum over $\mathbf{s} \in \mathrm{NN}$ on both sides of the above inequality,  it holds that
\begin{equation*}
\begin{aligned}
&\mathbb{E}_{\widehat{\mathcal{X}},\mathcal{T},\mathcal{Z}}\left(\frac{1}{T}\int_{0}^{T}\mathbb{E}_{\mathbf{x}_t}\Vert\widehat{\mathbf{s}}(1-t,\mathbf{x}_t) - \nabla\log q_{1-t}(\mathbf{x}_t)\Vert^2 \mathrm{d}t\right)\\
&=\mathbb{E}_{\widehat{\mathcal{X}},\mathcal{T},\mathcal{Z}}\mathcal{L}(\widehat{\mathbf{s}}) - \mathcal{L}(\mathbf{s}^{*})\\
& \leq \mathbb{E}_{\widehat{\mathcal{X}},\mathcal{T},\mathcal{Z}}\left(\mathcal{L}(\widehat{\mathbf{s}}) - 2\overline{\mathcal{L}}_{\widehat{\mathcal{X}}}(\widehat{\mathbf{s}}) + \mathcal{L}(\mathbf{s}^{*})\right) + 2\mathbb{E}_{\widehat{\mathcal{X}},\mathcal{T},\mathcal{Z}}\left(\overline{\mathcal{L}}_{\widehat{\mathcal{X}}}(\widehat{\mathbf{s}}) - \widehat{\mathcal{L}}_{\widehat{\mathcal{X}},\mathcal{T},\mathcal{Z}}(\widehat{\mathbf{s}})\right)\\
&~~~~+ 2\mathop{\mathrm{inf}}_{\mathbf{s}\in{\mathrm{NN}}}(\mathcal{L}(\mathbf{s}) - \mathcal{L}(\mathbf{s}^{*})).
\end{aligned}
\end{equation*}
In the above inequality, 
the terms $$
\mathbb{E}_{\widehat{\mathcal{X}},\mathcal{T},\mathcal{Z}}\left(\mathcal{L}(\widehat{\mathbf{s}}) - 2\overline{\mathcal{L}}_{\widehat{\mathcal{X}}}(\widehat{\mathbf{s}}) + \mathcal{L}(\mathbf{s}^{*})\right) + 2\mathbb{E}_{\widehat{\mathcal{X}},\mathcal{T},\mathcal{Z}}\left(\overline{\mathcal{L}}_{\widehat{\mathcal{X}}}(\widehat{\mathbf{s}}) - \widehat{\mathcal{L}}_{\widehat{\mathcal{X}},\mathcal{T},\mathcal{Z}}(\widehat{\mathbf{s}})\right)
$$  and $$
\mathop{\mathrm{inf}}_{\mathbf{s}\in{\mathrm{NN}}}(\mathcal{L}(\mathbf{s}) - \mathcal{L}(\mathbf{s}^{*}))
$$  denote the statistical error and approximation error, respectively. 

Using the tools of empirical process theory 
and deep approximation theory,   
we can bound these errors. We first introduce three necessary assumptions.

\begin{assumption}[Bounded Density]\label{ass:bounded_density}
The density $\widehat{p}_{data}^{*}$ in the latent space is supported on $[-1,1]^{d^*}$  and is bounded above and below by  $C_u$ and  $C_{l}$, respectively,  
on $[-1,1]^{d^{*}}$.
\end{assumption}
\begin{assumption}[H{\"o}lder Continuous]\label{ass:holder_density}
The density $\widehat{p}_{data}^{*}$ in the latent space belongs to $\mathcal{H}^{\beta}([-1,1]^{d^*}, R)$ for a H{\"o}lder index $\beta > 0$ and constant $R > 0$.
\end{assumption}
\begin{assumption}[Boundary Smoothness]\label{ass:boundary_smoothness} 
There exists $0 < a_0 < 1$ such that $\widehat{p}_{data}^{*}\in C^{\infty}([-1,1]^{d^*}\backslash [-1 + a_0, 1 - a_0]^{d^*})$.
\end{assumption}

Assumption   \ref{ass:bounded_density}  introduces the compact support for the  target distribution in the latent space, which is similarly adopted in \cite{lee2023convergence,li2023towards,oko2023diffusion}.  Assumption \ref{ass:holder_density}
enforces the H\"older continuity of the density, a fundamental condition in nonparametric estimation.
Assumption \ref{ass:boundary_smoothness}  imposes a regularity condition on the density $\widehat{p}_{data}^{*}$, a technical assumption that is also introduced in \cite{oko2023diffusion}.
In the existing literature, certain regularity conditions on the target distribution have also been introduced. For instance, it is required that the gradient of the score function be bounded both from above and below. Such conditions are outlined in \cite[Assumption 1]{lee2022convergence}, \cite[Assumptions 3-4]{chen2023improved}, \cite[Assumptions 3-4]{lee2023convergence}, \cite[Assumption 1]{gao2023wasserstein}, \cite[Assumptions 2.4, 2.6]{oko2023diffusion}, and \cite[Assumption 3]{chen2023score}.
Moreover, \cite[Assumptions 1, 2]{chen2023score} introduced analogous conditions. Their Assumption 1 delineates the representation of data points $\mathbf{x} \sim p_{data}$ as a product of an unknown matrix $A$ with orthonormal columns and a latent variable $\mathbf{z}$ following a distribution $P_z$ with a density function $p_z$.  Meanwhile, their Assumption 2 supposes
the density function $p_z>0$ is twice continuously differentiable. Moreover, there exist positive constants $B, C_1, C_2$ such that when $\|\mathbf{z}\|_2 \geq B$, the density function $p_z(\mathbf{z}) \leq$ $(2 \pi)^{-d / 2} C_1 \exp \left(-C_2\|\mathbf{z}\|_2^2 / 2\right)$. 

Under Assumptions \ref{ass:bounded_density}-\ref{ass:boundary_smoothness}, 
 we provide the following 
 lemma,  which bounds the approximation error. 
\begin{lemma}[Approximation Error]
\label{lem:approximation}
Suppose that Assumptions \ref{ass:bounded_density}-\ref{ass:boundary_smoothness} hold. 
Let $\mathcal{R} \gg 1 $, $C_T > 0$,
and  $T = 1-\mathcal{R}^{-C_T}$. We can then choose a ReLU neural network $\mathbf{s}\in\mathrm{NN}(L,M,J,\kappa)$ that satisfies
$$
\mathbf{s}(t,\mathbf{x}) = \mathbf{s}(t, \mathbf{s}_{\mathrm{clip}}(\mathbf{x}, -C_0, C_0)) ~\text{for}~ \Vert\mathbf{x}\Vert_\infty > C_0, ~\text{where}~ C_0 = \mathcal{O}(\sqrt{\log \mathcal{R}}),
$$
and has the following structure:  
$$
L = \mathcal{O}(\log^4 \mathcal{R}), M = \mathcal{O}(\mathcal{R}^{d^*}\log^7 \mathcal{R}), J = \mathcal{O}(\mathcal{R}^{d^*}\log^9 \mathcal{R}), \kappa = \exp\left(\mathcal{O}(\log^4 \mathcal{R})\right).
$$
Then, for any $t\in[\mathcal{R}^{-C_T},1]$, we have
$$
\int_{\mathbf{x}\sim q_t(\mathbf{x})}\Vert \mathbf{s}(t,\mathbf{x}) - \nabla\log q_t(\mathbf{x}) \Vert^2 \mathrm{d}\mathbf{x} \lesssim \frac{\mathcal{R}^{-2\beta}\log \mathcal{R}}{\sigma_t^2}.
$$
Moreover, we can take $\mathbf{s}$ satisfying $\Vert\mathbf{s}(t,\cdot)\Vert_{\infty} \lesssim \frac{\sqrt{\log \mathcal{R}}}{\sigma_t}$.
\end{lemma}
Now, we restrict  the ReLU neural network class $\mathrm{NN}(L,M,J,\kappa)$ to 
$$
\begin{aligned}
\mathcal{C}:= \Big\{\mathbf{s}\in &\mathrm{NN}(L,M,J,\kappa) \Big| \Vert\mathbf{s}(t,\cdot)\Vert_\infty \lesssim \frac{\sqrt{\log \mathcal{R}}}{\sigma_t}, \\
& \mathbf{s}(t,\mathbf{x}) = \mathbf{s}(t, \mathbf{s}_{\mathrm{clip}}(\mathbf{x}, -C_0, C_0) ~\text{for}~ \Vert\mathbf{x}\Vert_\infty > C_0, ~\text{where}~ C_0 = \mathcal{O}(\sqrt{\log \mathcal{R}}), \\
& L = \mathcal{O}(\log^4 \mathcal{R}), M = \mathcal{O}(\mathcal{R}^{d^*}\log^7 \mathcal{R}), J = \mathcal{O}(\mathcal{R}^{d^*}\log^9 \mathcal{R}), \kappa = \exp\left(\mathcal{O}(\log^4 \mathcal{R})\right)
\Big\}.
\end{aligned}
$$

We note that the method for deriving the upper bound on the approximation error in Lemma \ref{lem:approximation} is inspired by \cite{oko2023diffusion,Fu2024UnveilCD}.
To elaborate, we express the score function  $\nabla\log q_t(\mathbf{x})$  as 
$$
\nabla\log q_t(\mathbf{x}) = \frac{\nabla q_t(\mathbf{x})}{q_t(\mathbf{x})}.
$$
To obtain the approximation, we treat the numerator and denominator separately.  Since the approximation techniques for both terms are analogous, we focus on approximating  $q_t(\mathbf{x})$ as a representative example.   
Further  details on the full approximation procedure are provided in
 Section \ref{sec:ae} of the appendix.
The first step involves truncating the integral domain in
  $q_t(\mathbf{x})$. Recall that
$$
q_t(\mathbf{x}) = \int_{\mathbb{R}^{d^*}}\widehat{p}_{data}^{*}(\mathbf{y})\mathbf{I}_{\{\Vert\mathbf{y}\Vert_{\infty}\leq 1\}}\cdot
\underbrace{\frac{1}{\sigma_t^{d^*}(2\pi)^{d^*/2}}\exp\left(-\frac{\Vert\mathbf{x} - \mathbf{y}\Vert^2}{2\sigma_t^2}\right)}_{\text{Transition Kernel}} \mathrm{d}\mathbf{y}.
$$
There exists a constant 
$C>0$ such that for any $\mathbf{x}\in\mathbb{R}^{d^*}$,
$$
\left|q_t(\mathbf{x}) - \int_{A_{\mathbf{x}}}\widehat{p}_{data}^{*}(\mathbf{y})\mathbf{I}_{\{\Vert\mathbf{y}\Vert_{\infty}\leq 1\}}\cdot
\frac{1}{\sigma_t^{d^*}(2\pi)^{d^*/2}}\exp\left(-\frac{\Vert\mathbf{x} - \mathbf{y}\Vert^2}{2\sigma_t^2}\right) \mathrm{d}\mathbf{y}\right| \lesssim \epsilon,
$$
where $A_{\mathbf{x}} = \prod_{i=1}^{d^*}a_{i,\mathbf{x}}$ with $a_{i,\mathbf{x}} = \left[x_i - C\sigma_t\sqrt{\log\epsilon^{-1}}, x_i + C\sigma_t\sqrt{\log\epsilon^{-1}}\right]$. 
Next, we approximate the integral over  $A_{\mathbf{x}}$.
To do so, we introduce a Taylor polynomial  $f_{Taylor}^{*}(\mathbf{y})$ to  approximate $\widehat{p}_{data}^{*}(\mathbf{y})$.
As a result, $q_t(\mathbf{x})$  is approximated  as  
$$
\int_{A_{\mathbf{x}}}f_{Taylor}^{*}(\mathbf{y})\mathbf{I}_{\{\Vert\mathbf{y}\Vert_{\infty}\leq 1\}}\cdot\frac{1}{\sigma_t^{d^*}(2\pi)^{d^*/2}}\exp
\left(-\frac{\Vert\mathbf{x} - \mathbf{y}\Vert^2}{2\sigma_t^2}\right) \mathrm{d}\mathbf{y}.
$$
In the third step, we approximate the exponential transition kernel
$\frac{1}{\sigma_t^{d^*}(2\pi)^{d^*/2}}\exp\left(-\frac{\Vert\mathbf{x} - \mathbf{y}\Vert^2}{2\sigma_t^2}\right) 
$ using a Taylor polynomial denoted as  $f_{Taylor}^{kernel}(t,\mathbf{x}, \mathbf{y})$. 
This yields the approximation
\begin{equation*}
\int_{A_{\mathbf{x}}} f_{Taylor}^{*}(\mathbf{y}) f_{Taylor}^{kernel}(t,\mathbf{x},\mathbf{y})\mathbf{I}_{\{\Vert\mathbf{y}\Vert_{\infty}\leq 1\}} \mathrm{d}\mathbf{y}.
\end{equation*}
Since the  product $f_{Taylor}^{*}(\mathbf{y})f_{Taylor}^{kernel}(t,\mathbf{x},\mathbf{y})$ is  a polynomial,   its integration can be evaluated explicitly.
In the fourth step, we utilize a deep neural network to approximate the resulting integral.
Integrating these analyses, we can ultimately construct a ReLU neural network that approximates the score function.

\begin{lemma}[Statistical Error]\label{lem:statistical_error}
Suppose that Assumptions \ref{ass:bounded_density}-\ref{ass:boundary_smoothness} hold. Let $\mathcal{R} \gg 1$, $C_T > 0$,
and $T = 1-\mathcal{R}^{-C_T}$. 
The score estimator $\widehat{\mathbf{s}}$ defined in  \eqref{escore}, with the neural network structures belonging to class $\mathcal{C}$, satisfies

\begin{equation*}
\mathbb{E}_{\widehat{\mathcal{X}},\mathcal{T},\mathcal{Z}}\left(\mathcal{L}(\widehat{\mathbf{s}}) - 2\overline{\mathcal{L}}_{\widehat{\mathcal{X}}}(\widehat{\mathbf{s}}) + \mathcal{L}(\mathbf{s}^{*})\right) \lesssim
\frac{\mathcal{R}^{d^*}\log^{15} \mathcal{R}\left(\log^4 \mathcal{R} + \log n\right)}{n}
\end{equation*}
and
\begin{equation*}
\mathbb{E}_{\widehat{\mathcal{X}},\mathcal{T},\mathcal{Z}}\left(\overline{\mathcal{L}}_{\widehat{\mathcal{X}}}(\widehat{\mathbf{s}}) - \widehat{\mathcal{L}}_{\widehat{\mathcal{X}},\mathcal{T},\mathcal{Z}}(\widehat{\mathbf{s}})\right) \lesssim 
 \mathcal{R}^{C_T}(\log m + \log \mathcal{R}) \cdot \frac{\mathcal{R}^{\frac{d^*}{2}}\log^{\frac{13}{2}}\mathcal{R}(\log^2 \mathcal{R} + \sqrt{\log m})}{\sqrt{m}}.
\end{equation*}
\end{lemma}

By combining the approximation and statistical errors presented in Lemmas \ref{lem:approximation}-\ref{lem:statistical_error}, we can now derive the upper bound for the score estimation, as detailed in the following theorem.
\begin{theorem}[Error Bound for Score Estimation]
\label{thm:generalization}
Suppose that Assumptions \ref{ass:bounded_density}-\ref{ass:boundary_smoothness} hold. 
By choosing $\mathcal{R} = \lfloor n^{\frac{1}{d^* + 2\beta}} \rfloor + 1 \lesssim n^{\frac{1}{d^* + 2\beta}}$, $m = n^{\frac{d^* + 8\beta}{d^* + 2\beta}}$,
and $C_T = 2\beta$ in 
Lemmas \ref{lem:approximation}-\ref{lem:statistical_error}, the score estimator $\widehat{\mathbf{s}}$ defined in  \eqref{escore},
with the neural network structure belonging to class $\mathcal{C}$, satisfies 
$$
\mathbb{E}_{\widehat{\mathcal{X}},\mathcal{T},\mathcal{Z}}\left(\frac{1}{T}\int_{0}^{T}\mathbb{E}_{\mathbf{x}_t}\Vert\widehat{\mathbf{s}}(1-t,\mathbf{x}_t) - \nabla\log q_{1-t}(\mathbf{x}_t)\Vert^2 \mathrm{d}t\right) \lesssim n^{-\frac{2\beta}{d^* + 2\beta}} \log^{19}n.
$$
\end{theorem}
\begin{remark}
In the theoretical analysis of existing diffusion models, a substantial number of inquiries prominently espouse the $L_2$-accuracy of the score estimator. This is evident in a spectrum of assumptions articulated by the literature, including \cite[Assumption 1]{chen2023improved}, \cite[Assumption H1]{conforti2023score}, \cite[Assumption 2]{lee2022convergence}, \cite[Assumption 1]{lee2023convergence}, \cite[Assumption 1]{benton2023linear}, \cite[Assumption 1]{li2023towards}, and \cite[Assumption 3]{gao2023wasserstein}. Despite this pervasive reliance on such assumptions, there exists a dearth of research endeavors that have undertaken a rigorous theoretical consideration of score estimation.
Noteworthy among the limited investigations in this domain is \cite[Theorem 5]{wang2021deep}, which scrutinizes the convergence consistency of score estimation. Furthermore, both \cite[Theorem 3.1]{oko2023diffusion} and \cite[Theorem 2]{chen2023score} contribute to the field by providing the convergence rate of score estimation.
It is imperative to underscore that \cite{oko2023diffusion,chen2023score} operated  within a continuous-time paradigm, neglecting considerations of time discretization and presuming the feasibility of resolving integrals over the time interval $[0, T]$. In contrast, our results meticulously incorporate considerations of time discretization, offering a convergence rate within the discretized time.
Importantly, by ignoring the logarithmic term, this bound achieves the minimax optimal rate  in nonparametric statistics \cite{stone1980optimal,AB2009}.
\end{remark}
Now, we consider bounding   $\mathbb{E}_{\mathcal{X},\mathcal{Y},\mathcal{T},\mathcal{Z}}[W_2({\widetilde{\pi}_T^L}, \widehat{p}^*_{data})]$. 
We systematically decompose this error term into two distinct components, as outlined by the following inequality:
\begin{align}\label{errdec}
\mathbb{E}_{\mathcal{X},\mathcal{Y},\mathcal{T},\mathcal{Z}}[W_2({\widetilde{\pi}_T^L}, \widehat{p}^*_{data})]
 \leq \mathbb{E}_{\mathcal{X},\mathcal{Y},\mathcal{T},\mathcal{Z}}[W_2({\widetilde{\pi}_T^L}, \pi_T^L)]
  +  \mathbb{E}_{\mathcal{Y}}[W_2(\pi_T^L, \widehat{p}^{*}_{data})].
\end{align}
Subsequently, we proceed to establish individual bounds for the two terms on the right-hand side of \eqref{errdec} in Subsections \ref{bound1}-\ref{bound2}.  
The integration of these two analyses culminates in the oracle inequality for 
$\mathbb{E}_{\mathcal{X},\mathcal{Y}}[W_2({\widetilde{\pi}_T^L}, \widehat{p}_{data}^*)]$,
which is expounded upon in Subsection \ref{sec:oi}. 
\subsubsection{Bound $\mathbb{E}_{\mathcal{X},\mathcal{Y},\mathcal{T},\mathcal{Z}}[W_2({\widetilde{\pi}_T^L}, \pi_T^L)]$}
\label{bound1}
In this subsection, our aim is to establish an upper bound for $\mathbb{E}_{\mathcal{X},\mathcal{Y},\mathcal{T}, \mathcal{Z}}[W_2({\widetilde{\pi}_T^L}, \pi_T^L)]$. With Theorem \ref{thm:generalization}, for any $\epsilon > 0$, there exists $\Delta_\epsilon > 0$ such that if $\max_{0\leq i \leq K - 1}
(t_{i+1}-t_i) \leq \Delta_\epsilon$,
then
$$
\begin{aligned}
&\mathbb{E}_{\mathcal{X}, \mathcal{Y},\mathcal{T},\mathcal{Z}} \left[\sum_{i=0}^{K-1}\int_{t_i}^{t_{i + 1}} \mathbb{E}_{\mathbf{x}_{t_i}} \Vert \widehat{\mathbf{s}}(1-t_i, \mathbf{x}_{t_i}) - \nabla\log q_{1-t_i}(\mathbf{x}_{t_i})\Vert^2 \mathrm{d}t  \right] \\
& = \mathbb{E}_{\widehat{\mathcal{X}}, \mathcal{T},\mathcal{Z}} \left[\sum_{i=0}^{K-1}\int_{t_i}^{t_{i + 1}} \mathbb{E}_{\mathbf{x}_{t_i}} \Vert \widehat{\mathbf{s}}(1-t_i, \mathbf{x}_{t_i}) - \nabla\log q_{1-t_i}(\mathbf{x}_{t_i})\Vert^2 \mathrm{d}t  \right] \\
& \leq T\cdot \mathbb{E}_{\widehat{\mathcal{X}},\mathcal{T},\mathcal{Z}}\left[\frac{1}{T}\int_{0}^{T}\mathbb{E}_{\mathbf{x}_t}\Vert\widehat{\mathbf{s}}(1-t,\mathbf{x}_t) - \nabla\log q_{1-t}(\mathbf{x}_t)\Vert^2 \mathrm{d}t\right] + \epsilon \lesssim n^{-\frac{2\beta}{d^* + 2\beta}}\log^{19}n + \epsilon.
\end{aligned}
$$
Specially, taking $\epsilon = n^{-\frac{2\beta}{d^* + 2\beta}}\log^{19}n$, there exists $\Delta_\epsilon = \Delta_n$ such that if $\max_{0\leq i \leq K-1}(t_{i+1}-t_i) \leq \Delta_n$, then we have
\begin{equation} \label{eq:discrete_generalization}
\mathbb{E}_{\mathcal{X}, \mathcal{Y},\mathcal{T},\mathcal{Z}} \left[\sum_{i=0}^{K-1}\int_{t_i}^{t_{i + 1}} \mathbb{E}_{\mathbf{x}_{t_i}} \Vert \widehat{\mathbf{s}}(1-t_i, \mathbf{x}_{t_i}) - \nabla\log q_{1-t_i}(\mathbf{x}_{t_i})\Vert^2 \mathrm{d}t  \right] \lesssim n^{-\frac{2\beta}{d^* + 2\beta}}\log^{19}n.
\end{equation}
Using \eqref{eq:discrete_generalization}, we can bound $\mathbb{E}_{\mathcal{X},\mathcal{Y},\mathcal{T},\mathcal{Z}}[W_2(\widetilde{\pi}_T^L,\pi_T^L)]$ as the following theorem.
\begin{theorem}\label{thm:sampling_error} 
Suppose that assumptions of Theorem \ref{thm:generalization} hold. By choosing $\max_{0\leq i \leq K-1}
(t_{i+1}-t_i)= \mathcal{O}\left(\min\{\Delta_n, n^{-\frac{6\beta}{d^* + 2\beta}}\}\right)$, then we have
\begin{equation*}
\mathbb{E}_{\mathcal{X},\mathcal{Y},\mathcal{T},\mathcal{Z}}[W_2(\widetilde{\pi}_T^L,\pi_T^L)] \lesssim n^{-\frac{\beta}{d^* + 2\beta}}\log^{\frac{19}{2}}n.
\end{equation*}
\end{theorem}


\subsubsection{Bound $\mathbb{E}_{\mathcal{Y}}[W_2(\pi_T^L, {\widehat{p}}^{*}_{data})]$}
\label{bound2}
In this subsection, we present the upper bound for $\mathbb{E}_{\mathcal{Y}}[W_2(\pi_T^L, {\widehat{p}}^{*}_{data})]$. This error originates from the process of early stopping. We have the following lemma.
\begin{lemma}\label{lem:early_stopping}
Let $\mathcal{R} = \lfloor n^{\frac{1}{d^* + 2\beta}} \rfloor + 1 \lesssim n^{\frac{1}{d^* + 2\beta}}$, $T = 1-\mathcal{R}^{-C_T}$,
and $C_T = 2\beta$, we have
\begin{equation*}
    \mathbb{E}_{\mathcal{Y}}[W_2(\pi_T^L, {\widehat{p}}^{*}_{data})]\lesssim n^{-\frac{\beta}{d^* + 2\beta}}.
\end{equation*}
\end{lemma}
\subsubsection{ 
Bound 
$\mathbb{E}_{\mathcal{X},\mathcal{Y}, \mathcal{T}, \mathcal{Z}}[W_2({\widetilde{\pi}_T^L}, {\widehat{p}}^{*}_{data})]
$
}\label{sec:oi}
By amalgamating the analyses articulated in Theorem \ref{thm:sampling_error} 
and Lemma \ref{lem:early_stopping},  we derive  an upper bound for
$\mathbb{E}_{\mathcal{X},\mathcal{Y}, \mathcal{T}, \mathcal{Z}}[W_2({\widetilde{\pi}_T^L}, {\widehat{p}}^{*}_{data})]$
through the meticulous determination of a suitable score neural network  and the optimal stopping time 
$T$.  The  end-to-end convergence rate is shown in the following theorem.
\begin{theorem} \label{thm:oracle_inequality}
Suppose that Assumptions \ref{ass:bounded_density}-\ref{ass:boundary_smoothness} hold, 
and the score estimator $\widehat{\mathbf{s}}$  defined in \eqref{escore} is structured as introduced in Theorem 
\ref{thm:generalization}.
By choosing $\mathcal{R}$, $m$, $C_T$, $\max_{0\leq i \leq K-1}(t_{i+1}-t_i)$ introduced in Theorem \ref{thm:generalization} and Theorem \ref{thm:sampling_error},
we have
\begin{equation*}
\mathbb{E}_{\mathcal{X},\mathcal{Y}, \mathcal{T}, \mathcal{Z}}[W_2({\widetilde{\pi}_T^L}, {\widehat{p}}^{*}_{data})] \lesssim n^{-\frac{\beta}{d^* + 2\beta}} \log^{\frac{19}{2}}n.
\end{equation*}
\end{theorem}

\begin{remark}
Theorem \ref{thm:oracle_inequality} rigorously establishes the end-to-end theoretical guarantees for the  diffusion model within latent space. This is achieved through the seamless integration of score estimation, EM discretization, and early stop techniques. 
Furthermore, this analysis is independent  of the pre-training theory, rendering our theoretical framework applicable to  other diffusion models, irrespective of considerations regarding pre-training encoder-decoder structures. This characteristic underscores the independent validity of our theoretical framework. 
Importantly, the convergence rate
$n^{-\frac{\beta}{d^* + 2\beta}}\log^{\frac{19}{2}}n$ in Theorem \ref{thm:oracle_inequality} mitigates the curse of dimensionality of raw data, primarily attributed to the implementation of the encoder component.
Additionally, this rate is minimax optimal by ignoring the logarithmic factor.
\end{remark}

\subsection{Main Result}\label{sec:eb_mr}
In this section, we present our fundamental theoretical findings, elucidating  
end-to-end convergence rates. To facilitate our analysis, we introduce Assumption \ref{ass: distribution drift} to characterize the distribution shift between the pre-training and target distributions. Specifically, in Assumption \ref{ass: distribution drift}, $\epsilon_{p_{data}, \widetilde{p}_{data}}=0$ denotes the absence of any shift.

\begin{assumption}\label{ass: distribution drift} There exists $\epsilon_{p_{data}, \widetilde{p}_{data}} \geq 0$ such that  $W_2(p_{data}, \widetilde{p}_{data})\leq\epsilon_{p_{data}, \widetilde{p}_{data}}$.
\end{assumption}
Then, by combining Theorem \ref{thm:pretraining}, Theorem \ref{thm:oracle_inequality} and Assumption \ref{ass: distribution drift}, we obtain our main result.
\begin{theorem}\label{thm:main_result}

Suppose that Assumptions \ref{ass:bounded_support}-\ref{ass: distribution drift} hold, and the score estimator $\widehat{\mathbf{s}}$  defined in \eqref{escore} is structured as introduced in Theorem 
\ref{thm:generalization}.
By choosing $\mathcal{R}$, $m$, $C_T$,
and $\max_{0\leq i \leq K-1}(t_{i+1}-t_i)$ introduced in Theorem \ref{thm:generalization} and Theorem \ref{thm:sampling_error}, we have
\begin{equation*}
\mathbb{E}_{\mathcal{X},\mathcal{Y},\mathcal{T},\mathcal{Z}} [W_2(\widehat{\boldsymbol{D}}_{\#}\widetilde{\pi}_T^L, p_{data})]
= \widetilde{\mathcal{O}}\left(
n^{-\frac{\beta}{d^* + 2\beta}}+ \epsilon_{p_{data},\widetilde{p}_{data}} + \mathcal{M}^{-\frac{1}{2(d + 2)}} + \delta_0^{1/2}
\right),
\end{equation*}
where $\widehat{\boldsymbol{D}}_{\#}\widetilde{\pi}_T^L$ denotes the distribution of
$\widehat{\boldsymbol{D}}(\widetilde{\mathbf{x}}_K)$ in Algorithm \ref{sampling_algorithm}.
Moreover, if $\mathcal{M}>n^{\frac{2\beta(d+2)}{d^* + 2\beta}} $, then
$$
\mathbb{E}_{\mathcal{X},\mathcal{Y},\mathcal{T}, \mathcal{Z}} [W_2(\widehat{\boldsymbol{D}}_{\#}\widetilde{\pi}_T^L, p_{data})] = \widetilde{\mathcal{O}}\left(
n^{-\frac{\beta}{d^{*} + 2\beta}} +  \epsilon_{p_{data},\widetilde{p}_{data}} + \delta_0^{1/2}
\right).
$$ 
\end{theorem}

\begin{remark}
In Theorem \ref{thm:main_result}, 
the convergence rates of $\widehat{\boldsymbol{D}}_{\#}\widetilde{\pi}_T^L$ are established by simultaneously considering pre-training encoder-decoder structures  and latent diffusion models. This represents an end-to-end result.
Recall that $\mathcal{M}$ represents the sample size of pre-training. This term can be chosen sufficiently large, as the pre-training phase can be manipulated separately by resorting to 
pre-existing large-scale models.
Therefore, we can set $\mathcal{M}>n^{\frac{2\beta(d+2)}{d^*+2\beta}} $ to obtain the convergence rate of 
$\widetilde{\mathcal{O}}\left(
n^{-\frac{\beta}{d^{*} + 2\beta}} +  \epsilon_{p_{data},\widetilde{p}_{data}} + \delta_0^{1/2}
\right)$.
Additionally, if the distribution shift 
$\epsilon_{p_{data},\widetilde{p}_{data}} $
and pre-training quantity
$\delta_0^{1/2}$ are both smaller than $n^{-\frac{\beta}{d^{*} + 2\beta}}$,
then the convergence rate reduces to
the optimal order $\widetilde{\mathcal{O}}(
n^{-\frac{\beta}{d^{*} + 2\beta}})$. 
This finding circumvents the curse of dimensionality inherent in raw data, thereby demonstrating the effectiveness of  latent diffusion models from a theoretical perspective.
\end{remark}
\begin{remark}
In Theorem \ref{thm:main_result}, 
the theoretical analysis framework is universally applicable to the broader spectrum of general diffusion models, thereby asserting its independent validity. In the existing theoretical analyses of diffusion models, \cite{chen2023improved,conforti2023score,lee2022convergence,lee2023convergence,benton2023linear,li2023towards,gao2023wasserstein} furnish theoretical guarantees  concerning total Variation, Wasserstein distance, or KL divergence, contingent upon predetermined bounds on score estimation errors.
Within the  end-to-end convergence analyses, \cite{wang2021deep} only delivered a  convergence, while both \cite{oko2023diffusion} and \cite{chen2023score}  delineated convergence rates. 
Furthermore, \cite{oko2023diffusion} and \cite{chen2023score}  mitigated the   curse of dimensionality by imposing  a low-dimensional structure assumption on the target distribution. In contrast, our contribution lies in the introduction of a pre-training encoder-decoder framework. This   framework not only aligns with prevailing generative learning methods but also demonstrates operational practicality. As a result, it advances the landscape of diffusion models with a balance of theoretical rigor and real-world feasibility.
\end{remark}

\section{Conclusion}\label{sec:con}
In this work,  we introduce a novel latent diffusion model rooted in the Schr{\"o}dinger bridge, facilitated by 
pre-training an encoder-decoder structure.
An  SDE, defined over the time interval $[0,1]$, is formulated to effectuate the transformation of the convolution distribution into the encoder target distribution within the latent space.  The synergy of pre-training, score estimation,  EM discretization, and early stopping  enables the derivation of desired samples, manifesting an approximate distribution in accordance with the target distribution. Theoretical analysis encompasses  pre-training, 
score estimation theory, 
and  comprehensive end-to-end convergence rates for the generated samples, quantified in terms of the second-order Wasserstein distance.
Our theoretical findings provide solid theoretical guarantees for mainstream diffusion models. The obtained convergence rates predominantly hinge upon the dimension $d^*$ of the latent space, achieving optimality and mitigating the curse of dimensionality inherent in raw data.
However, the dimension $d^*$ is prespecified in practical applications.  Therefore, the development of an adaptive approach for selecting it  can be regarded as future work.
Additionally, extending our theoretical framework to analyze diffusion models based on SDEs, such as OU processes and Langevin SDEs,  defined over an infinite time horizon  $[0,\infty)$ can also be considered as future work.

\section*{Appendix}\label{sec:ap}
\appendix
In this appendix, we provide the detailed proofs of all lemmas and theorems in this paper.

In Section \ref{sec:appB}, we derive the error bound associated with the pre-training procedure.
In Section \ref{sec:ae}, we provide a comprehensive proof of the approximation error.  This proof draws inspiration from the approaches outlined in \cite{oko2023diffusion, Fu2024UnveilCD}. To complete the proof, in line with the methodology of \cite{oko2023diffusion,Fu2024UnveilCD}, we  present a series of 
high-probability bounds for  $q_t(\mathbf{x})$,  its derivatives, and $\nabla\log q_t(\mathbf{x})$ in Section  \ref{sec:shb}, and then introduce auxiliary lemmas in Section \ref{sec:app} to support the approximation of the H\"older class using ReLU networks.
In Section \ref{sec:se}, we establish the statistical error. 
Furthermore, Section \ref{sec:bb} provides the upper bound for
$\mathbb{E}_{\mathcal{X}, \mathcal{Y},\mathcal{T},\mathcal{Z}}\left[W_2(\widetilde{\pi}_T^L, \widehat{p}_{data}^*)\right]$. 
Finally, in Section \ref{sec:appG}, we present the proof of our main results.

\section{Pre-training}\label{sec:appB}
In this section, 
we  prove Lemma \ref{lem:pretraining_approximation_error} and Theorem \ref{thm:pretraining}.

\begin{proof}[Proof of Lemma \ref{lem:pretraining_approximation_error}]
We approximate $[\boldsymbol{E}^*(\mathbf{y})]_i$, $1 \leq i \leq d^*$, separately. Then, we can obtain an approximation of $\boldsymbol{E}^*(\mathbf{y})$ by concatenation. We rescale the input by $\mathbf{y}^{\prime} = \frac{\mathbf{y} + \mathbf{1}}{2} \in [0,1]^{d}$. Such a transformation can be exactly implemented by a single ReLU layer. We define the rescaled function on the transformed input space as $\boldsymbol{E}^{*,\prime}(\mathbf{y}^{\prime}):=\boldsymbol{E}^{*}(2\mathbf{y}^{\prime} - \mathbf{1})$, so that $\boldsymbol{E}^{*,\prime}$ is 
$2\xi_{\boldsymbol{E}}$-Lipschitz continuous. Since $\boldsymbol{E}^*$ is continuous differentiable on $[-1,1]^{d}$, then there exist a constant $B_{\boldsymbol{E}}$ such that
$\sup_{\mathbf{y}\in [-1,1]^d}\Vert\boldsymbol{E}^*(\mathbf{y})\Vert \leq B_{\boldsymbol{E}}$, which implies that $\sup_{\my^{\prime}\in[0,1]^d}\|\boldsymbol{E}^{*,\prime}(\my^{\prime})\| \leq B_{\boldsymbol{E}}$.

We then partition $[0,1]^{d}$ into non-overlapping hypercubes with equal edge length $e_1$. $e_1$ will be chosen depending on the desired approximation error. We denote $N_1 = \lceil \frac{1}{e_1} \rceil$ and let $\mathbf{v}=[v_1,v_2,\cdots,v_d]^{\top} \in [N_1]^d$ be a multi-index. We define $[\overline{\boldsymbol{E}}(\mathbf{y})]_i$, $1 \leq i \leq d^*$,
as
$$
[\overline{\boldsymbol{E}}(\mathbf{y}^{\prime})]_i:= \sum_{\mathbf{v}\in [N_1]^d}\left[\boldsymbol{E}^{*,\prime}\left(\frac{\mathbf{v}}{N_1}\right)\right]_i \Psi_{\mathbf{v}}(\mathbf{y}^{\prime}), ~ 1 \leq i \leq d^*,
$$
where $\Psi_{\mathbf{v}}(\mathbf{y}^{\prime})$ is a partition of unity function which satisfies that
$$
\sum_{\mathbf{v}\in [N_1]^d} \Psi_{\mathbf{v}}(\mathbf{y}^{\prime}) \equiv 1, \text{~on~} [0,1]^d.
$$
We choose $\Psi_{\mathbf{v}}$ as a product of coordinate-wise trapezoid functions, i.e.,
$$
\Psi_{\mathbf{v}}(\mathbf{y}^{\prime}):= \prod_{i=1}^{d} \psi\left(3N_1\left(
y_i^{\prime} - \frac{v_i}{N_1}
\right)\right),
$$
where $\psi$ is a trapezoid function, i.e.,
$$
\psi(x):= \left\{
\begin{array}{ll}
    1, &  |x| < 1 \\
    2 - |x|, &  |x|\in [1,2], \\
    0, & |x| > 2.
\end{array}
\right.
$$
We claim that for $1 \leq i \leq d^*$,
\begin{itemize}
    \item $[\overline{\boldsymbol{E}}(\mathbf{y}^{\prime})]_i$ is an approximation of $[\boldsymbol{E}^{*,\prime}(\mathbf{y}^{\prime})]_i$;
    \item $[\overline{\boldsymbol{E}}(\mathbf{y}^{\prime})]_i$ can be implemented by a ReLU neural network $[\boldsymbol{E}^{\prime}(\mathbf{y}^{\prime})]_i$ with small error.
\end{itemize}
Both claims are verified in 
\cite[Lemma 10]{chen2022distribution}, where we only need to substitute the Lipschitz constant $2\xi_{\boldsymbol{E}}$ into the error analysis. By concatenating $[\boldsymbol{E}^{\prime}(\mathbf{y}^{\prime})]_i$, $1 \leq i \leq d^*$ together, we construct $\boldsymbol{E}^{\prime}$. Given $\epsilon > 0$, we can achieve 
$$
\sup_{\mathbf{y}^{\prime}\in [0,1]^d} \Vert \boldsymbol{E}^{\prime}(\mathbf{y}^{\prime}) - \boldsymbol{E}^{*,\prime}(\mathbf{y}^{\prime})\Vert_\infty \leq \epsilon,
$$
where the neural network configuration satisfies
$$
L_{\boldsymbol{E}} = \mathcal{O}\left(\log\frac{1}{\epsilon} + d\right), M_{\boldsymbol{E}}=\mathcal{O}\left(d^{*}\xi_{\boldsymbol{E}}^{d}\epsilon^{-d}\right), J_{\boldsymbol{E}}=\mathcal{O}\left(d^{*}\xi_{\boldsymbol{E}}^{d}\epsilon^{-d}\left(\log\frac{1}{\epsilon} + d\right)\right),
$$
$$
K_{\boldsymbol{E}} = \mathcal{O}(B_{\boldsymbol{E}}), \kappa_{\boldsymbol{E}} = \mathcal{O}\left(\max\{1, \xi_{\boldsymbol{E}}\}\right).
$$
Here, we set $e_1 = \mathcal{O}\big(\frac{\epsilon}{\xi_{\boldsymbol{E}}}\big)$. Moreover, the neural network $\boldsymbol{E}^{\prime}(\mathbf{y}^{\prime})$ is Lipschitz continuous. For any $\mathbf{y}^{\prime}_1, \mathbf{y}^{\prime}_2 \in [0,1]^d$, it holds that
$$
\Vert \boldsymbol{E}^{\prime}(\mathbf{y}^{\prime}_1) - \boldsymbol{E}^{\prime}(\mathbf{y}^{\prime}_2)\Vert_\infty \leq 2\gamma_{\boldsymbol{E}}\Vert \mathbf{y}^{\prime}_1 - \mathbf{y}^{\prime}_2 \Vert,
$$ 
where $\gamma_{\boldsymbol{E}}:= 10d\xi_{\boldsymbol{E}}$. 
Combining with the single input transformation layer, we denote
$$
\boldsymbol{E}(\mathbf{y}):= \boldsymbol{E}^{\prime} \left( 
\frac{\mathbf{y} + \mathbf{1}}{2} \right) \in \mathcal{E},
$$
then the configuration of $\boldsymbol{E}$ is the same as $\boldsymbol{E}^{\prime}$, and for any $\mathbf{y}_1, \mathbf{y}_2 \in [-1,1]^d$, it holds that
$$
\Vert \boldsymbol{E}(\mathbf{y}_1) - \boldsymbol{E}(\mathbf{y}_2)\Vert_\infty \leq \gamma_{\boldsymbol{E}}\Vert \mathbf{y}_1 - \mathbf{y}_2 \Vert.
$$

Similarly, there exist a neural network $\boldsymbol{D}\in\mathcal{D}$ with configuration
$$
L_{\boldsymbol{D}} = \mathcal{O}\left(\log\frac{1}{\epsilon} + d^{*}\right), M_{\boldsymbol{D}}=\mathcal{O}\left(d\xi_{\boldsymbol{D}}^{d^{*}}\epsilon^{-d^{*}}\right), J_{\boldsymbol{D}}=\mathcal{O}\left(d\xi_{\boldsymbol{D}}^{d^{*}}\epsilon^{-d^{*}}\left(\log\frac{1}{\epsilon} + d^{*}\right)\right),
$$
$$
K_{\boldsymbol{D}} = \mathcal{O}(B_{\boldsymbol{D}}), \kappa_{\boldsymbol{D}} = \mathcal{O}\left(\max\{1, \xi_{\boldsymbol{D}}\}\right),
$$
such that
$$
\mathop{\sup}_{\my\in[-1,1]^{d^{*}}}\Vert\boldsymbol{D}(\my) - \boldsymbol{D}^{*}(\my)\Vert_{\infty}\leq \epsilon,
$$
where $B_{\boldsymbol{D}}$ is the upper bound of $\boldsymbol{D}^*$, i.e., $\sup_{\my\in[-1,1]^{d^*}}\|\boldsymbol{D}^*(\my)\| \leq B_{\boldsymbol{D}}$.
Moreover, the network $\boldsymbol{D}$ is Lipschitz continuous with Lipschitz constant $ \gamma_{\boldsymbol{D}}:=10d^{*}\xi_{\boldsymbol{D}}$, i.e., for any $\my_1$, $\my_2\in[-1,1]^{d^{*}}$, it satisfies
$$
\Vert\boldsymbol{D}(\my_1) - \boldsymbol{D}(\my_2)\Vert_{\infty}\leq \gamma_{\boldsymbol{D}}\Vert\my_1 - \my_2\Vert.
$$

Now, we derive the upper bound of the approximation error. For any $\my\in [-1,1]^{d}$, we have
$$
\begin{aligned}
    &\Vert(\boldsymbol{D}\circ\boldsymbol{E})(\my)-\my\Vert^2 - \Vert(\boldsymbol{D}^{*}\circ\boldsymbol{E}^{*})(\my)-\my\Vert^2\\
    =&\langle(\boldsymbol{D}\circ\boldsymbol{E})(\my) + (\boldsymbol{D}^{*}\circ\boldsymbol{E}^{*})(\my) - 2\my, (\boldsymbol{D}\circ\boldsymbol{E})(\my) - (\boldsymbol{D}^{*}\circ\boldsymbol{E}^{*})(\my)\rangle\\
    \leq&(K_{\boldsymbol{D}} + B_{\boldsymbol{D}} + 2\sqrt{d})\Vert(\boldsymbol{D}\circ\boldsymbol{E})(\my) - (\boldsymbol{D}^{*}\circ\boldsymbol{E}^{*})(\my)\Vert\\
    \leq&(K_{\boldsymbol{D}} + B_{\boldsymbol{D}} + 2\sqrt{d})\left(\Vert(\boldsymbol{D}\circ\boldsymbol{E})(\my) - (\boldsymbol{D}^{*}\circ\boldsymbol{E})(\my)\Vert + \Vert(\boldsymbol{D}^{*}\circ\boldsymbol{E})(\my) - (\boldsymbol{D}^{*}\circ\boldsymbol{E}^{*})(\my)\Vert\right)\\
    \leq&(K_{\boldsymbol{D}} + B_{\boldsymbol{D}} + 2\sqrt{d})(\sqrt{d}\epsilon + \xi_{\boldsymbol{D}}\Vert\boldsymbol{E}(\my) - \boldsymbol{E}^{*}(\my)\Vert)\\
    \leq&(K_{\boldsymbol{D}} + B_{\boldsymbol{D}} + 2\sqrt{d})(\sqrt{d}\epsilon + \sqrt{d}\epsilon\xi_{\boldsymbol{D}})\\
    =&\sqrt{d}\epsilon(K_{\boldsymbol{D}} + B_{\boldsymbol{D}} + 2\sqrt{d})(1 + \xi_{\boldsymbol{D}}).
\end{aligned}
$$
This  implies 
$$
\begin{aligned}    
\mathcal{H}(\boldsymbol{E},\boldsymbol{D}) - \mathcal{H}(\boldsymbol{E}^{*}, \boldsymbol{D}^{*})
=&\int_{[0,1]^{d}}\left(\Vert(\boldsymbol{D}\circ\boldsymbol{E})(\my)-\my\Vert^2 - \Vert(\boldsymbol{D}^{*}\circ\boldsymbol{E}^{*})(\my)-\my\Vert^2\right)\widetilde{p}_{data}(\my)\mathrm{d}\my\\
\leq&\sqrt{d}\epsilon(K_{\boldsymbol{D}} + B_{\boldsymbol{D}} + 2\sqrt{d})(1 + \xi_{\boldsymbol{D}}) \lesssim \epsilon.
\end{aligned}
$$
The proof is complete.
\end{proof}

\begin{proof}[Proof of Theorem \ref{thm:pretraining}] The proof can be divided into three parts. We first decompose the excess risk error into approximation error and statistical error. Then, we bound these two terms separately. Finally, we balance these two terms by choosing an appropriate $\epsilon$.
\\\\
\noindent\textbf{Error Decomposition.} 
For each $\boldsymbol{E}_\theta \in \mathcal{E}_0$ and $\boldsymbol{D}_\theta \in \mathcal{D}_0$, we have
\begin{align*}
&\mathcal{H} (\widehat{\boldsymbol{E}}, \widehat{\boldsymbol{D}}) - \mathcal{H} (\boldsymbol{E}^*, \boldsymbol{D}^*) \\
&= \mathcal{H} (\widehat{\boldsymbol{E}}, \widehat{\boldsymbol{D}}) - \widehat{\mathcal{H}} (\widehat{\boldsymbol{E}}, \widehat{\boldsymbol{D}}) + \widehat{\mathcal{H}} (\widehat{\boldsymbol{E}}, \widehat{\boldsymbol{D}}) - \widehat{\mathcal{H}} (\boldsymbol{E}_\theta, \boldsymbol{D}_\theta) + \widehat{\mathcal{H}} (\boldsymbol{E}_\theta, \boldsymbol{D}_\theta) - \mathcal{H} (\boldsymbol{E}_\theta, \boldsymbol{D}_\theta) \\
&~~~ + \mathcal{H} (\boldsymbol{E}_\theta, \boldsymbol{D}_\theta) - \mathcal{H} (\boldsymbol{E}^*, \boldsymbol{D}^*) \\
&\leq \sup_{\boldsymbol{E} \in \mathcal{E}_0,\boldsymbol{D} \in \mathcal{D}_0} \mathcal{H} (\boldsymbol{E}, \boldsymbol{D}) - \widehat{\mathcal{H}} (\boldsymbol{E}, \boldsymbol{D}) + \sup_{ \boldsymbol{E} \in \mathcal{E}_0,\boldsymbol{D} \in \mathcal{D}_0} \widehat{\mathcal{H}} (\boldsymbol{E}, \boldsymbol{D}) - \mathcal{H} (\boldsymbol{E}, \boldsymbol{D}) \\
&~~~ + \mathcal{H} (\boldsymbol{E}_\theta, \boldsymbol{D}_\theta) - \mathcal{H} (\boldsymbol{E}^*, \boldsymbol{D}^*),
\end{align*}
where the inequality is due to the fact that $\widehat{\mathcal{H}} (\widehat{\boldsymbol{E}}, \widehat{\boldsymbol{D}}) \leq \widehat{\mathcal{H}} (\boldsymbol{E}_\theta, \boldsymbol{D}_\theta)$. Then taking infimum over 
$\boldsymbol{E}_\theta \in \mathcal{E}_0$ and
$\boldsymbol{D}_\theta \in \mathcal{D}_0$   yields
\begin{align}
\label{eq: ae 1}
\begin{aligned}
&\mathcal{H} (\widehat{\boldsymbol{E}}, \widehat{\boldsymbol{D}}) - \mathcal{H} (\boldsymbol{E}^*, \boldsymbol{D}^*) \\
&\leq 2 \sup_{\boldsymbol{E} \in \mathcal{E}_0,\boldsymbol{D} \in \mathcal{D}_0} |\mathcal{H} (\boldsymbol{E}, \boldsymbol{D}) - \widehat{\mathcal{H}} (\boldsymbol{E}, \boldsymbol{D})|+ \inf_{\boldsymbol{E} \in \mathcal{E}_0,\boldsymbol{D} \in \mathcal{D}_0} \mathcal{H} (\boldsymbol{E}, \boldsymbol{D}) - \mathcal{H} (\boldsymbol{E}^*, \boldsymbol{D}^*).
\end{aligned}
\end{align}
The first and second terms in \eqref{eq: ae 1} are called the 
statistical error and  approximation error respectively. By Lemma \ref{lem:pretraining_approximation_error}, the approximation error satisfies
$$
\inf_{\boldsymbol{E} \in \mathcal{E}_0,\boldsymbol{D} \in \mathcal{D}_0} \mathcal{H} (\boldsymbol{E}, \boldsymbol{D}) - \mathcal{H} (\boldsymbol{E}^*, \boldsymbol{D}^*) \lesssim \epsilon.
$$

\noindent\textbf{Statistical Error.} 
For convenience, we denote $\ell_{\boldsymbol{E},\boldsymbol{D}}(\my):=\Vert(\boldsymbol{D}\circ\boldsymbol{E})(\my) - \my\Vert^2$. We denote the $\delta$-covering of $\mathcal{D}_0$ with minimum cardinality $\mathcal{N}(\delta, \mathcal{D}_0, \Vert\cdot\Vert_{L^{\infty}([-1,1]^{d^{*}})})$ as $\mathcal{D}_{\delta}$ and the $\delta$-covering of $\mathcal{E}_0$ with minimum cardinality $\mathcal{N}(\delta, \mathcal{E}_0,\Vert\cdot\Vert_{L^{\infty}[-1,1]^{d}})$ as $\mathcal{E}_{\delta}$. $\mathcal{N}(\delta, \mathcal{D}_0, \Vert\cdot\Vert_{L^{\infty}([-1,1]^{d^{*}})})$ and $\mathcal{N}(\delta, \mathcal{E}_0, \Vert\cdot\Vert_{L^{\infty}([-1,1]^{d})})$ are bounded as follows:
$$
\log\mathcal{N}(\delta, \mathcal{D}_0, \Vert\cdot\Vert_{L^{\infty}([-1,1]^{d^{*}})})\lesssim J_{\boldsymbol{D}}L_{\boldsymbol{D}}\log\left(\frac{L_{\boldsymbol{D}}M_{\boldsymbol{D}}\kappa_{\boldsymbol{D}}}{\delta}\right),
$$
and
$$
\log\mathcal{N}(\delta, \mathcal{E}_0, \Vert\cdot\Vert_{L^{\infty}([-1,1]^{d})})\lesssim J_{\boldsymbol{E}}L_{\boldsymbol{E}}\log\left(\frac{L_{\boldsymbol{E}}M_{\boldsymbol{E}}\kappa_{\boldsymbol{E}}}{\delta}\right).
$$
For any $\boldsymbol{E}\in\mathcal{E}_0$, $\boldsymbol{D}\in\mathcal{D}_0$, there exist $\boldsymbol{E}_{\delta}\in \mathcal{E}_\delta$, $\boldsymbol{D}_{\delta}\in\mathcal{D}_\delta$ such that
$$
\Vert\boldsymbol{E}(\my) - \boldsymbol{E}_{\delta}(\my)\Vert_{L^{\infty}([-1,1]^{d})}\leq\delta,
$$
and
$$
\Vert\boldsymbol{D}(\my) - \boldsymbol{D}_{\delta}(\my)\Vert_{L^{\infty}([-1,1]^{d^{*}})}\leq\delta.
$$
Therefore, we have
$$
\begin{aligned}
    &~~~~|\ell_{\boldsymbol{E},\boldsymbol{D}}(\my) - \ell_{\boldsymbol{E}_{\delta}, \boldsymbol{D}_{\delta}}(\my)|
    \\ & \leq (2K_{\boldsymbol{D}} + 2\sqrt{d})\left(\Vert(\boldsymbol{D}\circ\boldsymbol{E})(\my) - (\boldsymbol{D}_{\delta}\circ\boldsymbol{E})(\my)\Vert + \Vert(\boldsymbol{D}_{\delta}\circ\boldsymbol{E})(\my) - (\boldsymbol{D}_{\delta}\circ\boldsymbol{E}_{\delta})(\my)\Vert\right)\\
    & \leq 2(K_{\boldsymbol{D}} + \sqrt{d})(\delta + \sqrt{d}\gamma_{\boldsymbol{D}}\delta)\\
    & = 2(K_{\boldsymbol{D}} +\sqrt{d})(1 + \sqrt{d}\gamma_{\boldsymbol{D}})\delta.
\end{aligned}
$$
For simplicity, we denote $\mathcal{N}_{\mathcal{D}_0,\delta}:=\mathcal{N}(\delta, \mathcal{D}_0, \Vert\cdot\Vert_{L^{\infty}([-1,1]^{d^{*}})})$, $\mathcal{N}_{\mathcal{E}_0,\delta}:=\mathcal{N}(\delta, \mathcal{E}_0, \Vert\cdot\Vert_{L^{\infty}([-1,1]^{d})})$ and
$C(\delta):=2(K_{\boldsymbol{D}} + \sqrt{d})(1 + \sqrt{d}\gamma_{\boldsymbol{D}})\delta$, then
$$
\begin{aligned}
|\widehat{\mathcal{H}}(\boldsymbol{E},\boldsymbol{D}) - \mathcal{H}(\boldsymbol{E},\boldsymbol{D})|&\leq\left|\frac{1}{\mathcal{M}}\sum_{i=1}^{\mathcal{M}}\left(\ell_{\boldsymbol{E}_{\delta},\boldsymbol{D}_{\delta}}(\my_i) - \mathbb{E}_{\mathcal{Y}}\left[\ell_{\boldsymbol{E}_{\delta},\boldsymbol{D}_{\delta}}(\my_i)\right]\right)\right| + 2C(\delta)\\
&\leq\max_{\boldsymbol{E}_{\delta}\in\mathcal{E}_{\delta},\boldsymbol{D}_{\delta}\in\mathcal{D}_{\delta}}\left|\frac{1}{\mathcal{M}}\sum_{i=1}^{\mathcal{M}}\left(\ell_{\boldsymbol{E}_{\delta},\boldsymbol{D}_{\delta}}(\my_i) - \mathbb{E}_{\mathcal{Y}}\left[\ell_{\boldsymbol{E}_{\delta},\boldsymbol{D}_{\delta}}(\my_i)\right]\right)\right| + 2C({\delta}).
\end{aligned}
$$
Taking supermum over $\boldsymbol{E}$ and $\boldsymbol{D}$ on both sides, we get
$$
\sup_{\boldsymbol{E}\in\mathcal{E}_0,\boldsymbol{D}\in\mathcal{D}_0}|\widehat{\mathcal{H}}(\boldsymbol{E},\boldsymbol{D}) - \mathcal{H}(\boldsymbol{E},\boldsymbol{D})|
\leq\max_{\boldsymbol{E}_{\delta}\in\mathcal{E}_{\delta},\boldsymbol{D}_{\delta}\in\mathcal{D}_{\delta}}\left|\frac{1}{\mathcal{M}}\sum_{i=1}^{\mathcal{M}}\left(\ell_{\boldsymbol{E}_{\delta},\boldsymbol{D}_{\delta}}(\my_i) - \mathbb{E}_{\mathcal{Y}}\left[\ell_{\boldsymbol{E}_{\delta},\boldsymbol{D}_{\delta}}(\my_i)\right]\right)\right| + 2C({\delta}).
$$
Thus, for any $t > 2C({\delta})$, we have
$$
\begin{aligned}
&~~~~\mathbb{P}_{\mathcal{Y}}\left(\sup_{\boldsymbol{E}\in\mathcal{E}_0,\boldsymbol{D}\in\mathcal{D}_0}|\widehat{\mathcal{H}}(\boldsymbol{E},\boldsymbol{D}) - \mathcal{H}(\boldsymbol{E},\boldsymbol{D})| > t\right)\\
& \leq \mathbb{P}_{\mathcal{Y}}\left(\max_{\boldsymbol{E}_{\delta}\in\mathcal{E}_{\delta},\boldsymbol{D}_{\delta}\in\mathcal{D}_{\delta}}\left|\frac{1}{\mathcal{M}}\sum_{i=1}^{\mathcal{M}}\left(\ell_{\boldsymbol{E}_{\delta},\boldsymbol{D}_{\delta}}(\my_i) - \mathbb{E}_{\mathcal{Y}}\left[\ell_{\boldsymbol{E}_{\delta},\boldsymbol{D}_{\delta}}(\my_i)\right]\right)\right| > t-2C({\delta})\right)\\
& \leq
\sum_{\boldsymbol{D}_{\delta}\in\mathcal{D}_{\delta}}\sum_{\boldsymbol{E}_{\delta}\in\mathcal{E}_{\delta}}\mathbb{P}_{\mathcal{Y}}\left(\left|\frac{1}{\mathcal{M}}\sum_{i=1}^{\mathcal{M}}\left(\ell_{\boldsymbol{E}_{\delta},\boldsymbol{D}_{\delta}}(\my_i) - \mathbb{E}_{\mathcal{Y}}\left[\ell_{\boldsymbol{E}_{\delta},\boldsymbol{D}_{\delta}}(\my_i)\right]\right)\right| > t-2C({\delta})\right).
\end{aligned}
$$
Since
$$
\ell_{\boldsymbol{E}_{\delta},\boldsymbol{D}_{\delta}}(\my)\leq 2\left(\Vert(\boldsymbol{D}_{\delta}\circ\boldsymbol{E}_{\delta})(\my)\Vert^2 + \Vert\my\Vert^2\right)\leq 2(K_{\boldsymbol{D}}^2 + d),
$$
by Hoeffding's inequality, we obtain
$$
\mathbb{P}_{\mathcal{Y}}\left(\left|\frac{1}{\mathcal{M}}\sum_{i=1}^{\mathcal{M}}\left(\ell_{\boldsymbol{E}_{\delta},\boldsymbol{D}_{\delta}}(\my_i) - \mathbb{E}_{\mathcal{Y}}
\left[\ell_{\boldsymbol{E}_{\delta},
\boldsymbol{D}_{\delta}}(\my_i)\right]\right)\right| > t-2C_{\delta}\right)\leq2\exp\left(-\frac{\mathcal{M}(t - 2C({\delta}))^2}{2(K_{\boldsymbol{D}}^2 + d)^2}\right).
$$
Therefore, for any $t > 2C({\delta})$, we have
$$
\mathbb{P}_{\mathcal{Y}}\left(\sup_{\boldsymbol{E}\in\mathcal{E}_0,\boldsymbol{D}\in\mathcal{D}_0}|\widehat{\mathcal{H}}(\boldsymbol{E},\boldsymbol{D}) - \mathcal{H}(\boldsymbol{E},\boldsymbol{D})| > t\right)\\
\leq 2\mathcal{N}_{\mathcal{E}_0,\delta} \mathcal{N}_{\mathcal{D}_0, \delta} \exp\left(-\frac{\mathcal{M}(t - 2C({\delta}))^2}{2(K_{\boldsymbol{D}}^2 + d)^2}\right).
$$
Integrating both sides with respect to $t$, for any $a > 0$, we get
$$
\begin{aligned}
& ~~~~ \mathbb{E}_{\mathcal{Y}}\left(\sup_{\boldsymbol{E}\in\mathcal{E}_0,\boldsymbol{D}\in\mathcal{D}_0}|\widehat{\mathcal{H}}(\boldsymbol{E},\boldsymbol{D}) - \mathcal{H}(\boldsymbol{E},\boldsymbol{D})|\right)\\
& = \int_{0}^{+\infty}\mathbb{P}_{\mathcal{Y}}\left(\sup_{\boldsymbol{D}\in\mathcal{D}_0,\boldsymbol{E}\in\mathcal{E}_0}|\widehat{\mathcal{H}}(\boldsymbol{E},\boldsymbol{D}) - \mathcal{H}(\boldsymbol{E},\boldsymbol{D})| > t\right)\mathrm{d}t\\
& \leq  a + 2C({\delta}) + 2\mathcal{N}_{\mathcal{D}_0, \delta}\mathcal{N}_{\mathcal{E}_0, \delta} \int_{a}^{+\infty}\exp\left(-\frac{\mathcal{M}t^2}{2(K_{\boldsymbol{D}}^2 + d)^2}\right)\mathrm{d}t\\
& \leq a + 2C({\delta}) + 2\mathcal{N}_{\mathcal{D}_0,\delta}\mathcal{N}_{\mathcal{E}_0,\delta}\cdot\frac{\sqrt{\pi}}{2}\exp\left(-\frac{\mathcal{M} a^2}{2(K_{\boldsymbol{D}}^2 + d)^2}\right)\frac{2(K_{\boldsymbol{D}}^2 + d)}{\sqrt{2 \mathcal{M}}}.
\end{aligned}
$$
Taking $a = 2(K_{\boldsymbol{D}}^2 + d)\sqrt{\frac{\log\mathcal{N}_{\mathcal{D}_0,\delta}\mathcal{N}_{\mathcal{E}_0,\delta}}{2 \mathcal{M}}}$ and $\delta = \frac{1}{\mathcal{M} }$, we obtain
$$
\begin{aligned}
&~~~~ \mathbb{E}_{\mathcal{Y}}\left(\sup_{\boldsymbol{E}\in\mathcal{E}_0,\boldsymbol{D}\in\mathcal{D}_0}|\widehat{\mathcal{H}}(\boldsymbol{E},\boldsymbol{D}) - \mathcal{H}(\boldsymbol{E},\boldsymbol{D})|\right)\\
& \leq 2(K_{\boldsymbol{D}}^2 + d)\frac{\sqrt{\log\mathcal{N}_{\mathcal{D}_0,\delta}\mathcal{N}_{\mathcal{E}_0, \delta}} + \sqrt{\pi}}{\sqrt{2\mathcal{M}}} + \frac{4(K_{\boldsymbol{D}} + \sqrt{d})(1 + \sqrt{d}\gamma_{\boldsymbol{D}})}{\mathcal{M}}.
\end{aligned}
$$
Substituting the values of $\mathcal{N}_{\mathcal{E}_0,\delta}$ and $\mathcal{N}_{\mathcal{D}_0,\delta}$, we obtain
$$
\mathbb{E}_{\mathcal{Y}}\left(\sup_{\boldsymbol{E}\in\mathcal{E}_0,\boldsymbol{D}\in\mathcal{D}_0}|\widehat{\mathcal{H}}(\boldsymbol{E},\boldsymbol{D}) - \mathcal{H}(\boldsymbol{E},\boldsymbol{D})|\right) = \widetilde{\mathcal{O}}\left(\frac{\epsilon^{-d/2}}{\sqrt{\mathcal{M}}}\right).
$$
\\\\
\textbf{
Balancing Error Terms.} Combining the approximation error and the statistical error, we have
$$
\mathcal{H}(\widehat{\boldsymbol{E}},\widehat{\boldsymbol{D}}) - \mathcal{H}(\boldsymbol{E}^{*}, \boldsymbol{D}^{*}) = \widetilde{\mathcal{O}}\left(\epsilon + \frac{\epsilon^{-d/2}}{\sqrt{\mathcal{M}}}\right).
$$
By choosing $\epsilon = \mathcal{M}^{-\frac{1}{d + 2}}$, we finally obtain
$$
\mathcal{H}(\widehat{\boldsymbol{E}},\widehat{\boldsymbol{D}}) - \mathcal{H}(\boldsymbol{E}^{*}, \boldsymbol{D}^{*}) = \widetilde{\mathcal{O}}\left(\mathcal{M}^{-\frac{1}{d + 2}}\right),
$$
which implies 
$$
\mathcal{H}(\widehat{\boldsymbol{E}},\widehat{\boldsymbol{D}}) =\widetilde{\mathcal{O}}\left(\mathcal{M}^{-\frac{1}{d + 2}} + \delta_0\right).
$$
The proof is complete.
\end{proof}

\section{Approximation Error}\label{sec:ae}
Building on the preliminary lemmas in Sections \ref{sec:shb}-\ref{sec:app}, we now proceed to the proof of Lemma  \ref{lem:approximation} by following \cite{oko2023diffusion,Fu2024UnveilCD}. 
Recall that the score function $\nabla\log q_t(\mathbf{x})$ can be rewritten as 
$$
\nabla\log q_t(\mathbf{x}) = \frac{\nabla q_t(\mathbf{x})}{q_t(\mathbf{x})}.
$$
We approximate the numerator and denominator separately. The construction of the approximations to the numerator and denominator is similar. In the following, we focus on the approximation of $q_t(\mathbf{x})$. The procedure for approximating $q_t(\mathbf{x})$ is outlined as follows:
\paragraph{Clipping the integral interval.} 
Recall that
$$
q_t(\mathbf{x}) = \int_{\mathbb{R}^{d^*}}\widehat{p}_{data}^{*}(\mathbf{y})\mathbf{I}_{\{\Vert\mathbf{y}\Vert_{\infty}\leq 1\}}\cdot
\underbrace{\frac{1}{\sigma_t^{d^*}(2\pi)^{d^*/2}}\exp\left(-\frac{\Vert\mathbf{x} - \mathbf{y}\Vert^2}{2\sigma_t^2}\right)}_{\text{Transition Kernel}} \mathrm{d}\mathbf{y}.
$$
According to Lemma \ref{lem:integral_clipping}, there exists a constant $C > 0$ such that for any $\mathbf{x}\in\mathbb{R}^{d^*}$,
$$
\left|q_t(\mathbf{x}) - \int_{A_{\mathbf{x}}}\widehat{p}_{data}^{*}(\mathbf{y})\mathbf{I}_{\{\Vert\mathbf{y}\Vert_{\infty}\leq 1\}}\cdot
\frac{1}{\sigma_t^{d^*}(2\pi)^{d^*/2}}\exp\left(-\frac{\Vert\mathbf{x} - \mathbf{y}\Vert^2}{2\sigma_t^2}\right) \mathrm{d}\mathbf{y}\right| \lesssim \epsilon,
$$
where $A_{\mathbf{x}} = \prod_{i=1}^{d^*}a_{i,\mathbf{x}}$ with $a_{i,\mathbf{x}} = \left[x_i - C\sigma_t\sqrt{\log\epsilon^{-1}}, x_i + C\sigma_t\sqrt{\log\epsilon^{-1}}\right]$. 
This implies that we only need to approximate the integral over $A_{\mathbf{x}}$ sufficiently.

\paragraph{Approximating $\widehat{p}_{data}^{*}$.} 
Recall that the target density function $\widehat{p}_{data}^{*}(\mathbf{y})$ is assumed to be H\"older continuous, as stated in Assumption \ref{ass:holder_density}.
Therefore, we can utilize the Taylor polynomial  $f_{Taylor}^{*}(\mathbf{y})$ to approximate $\widehat{p}_{data}^{*}(\mathbf{y})$.
Consequently, to approximate $q_t(\mathbf{x})$,  
we derive an approximation in the form of
$$
\int_{A_{\mathbf{x}}}f_{Taylor}^{*}(\mathbf{y})\mathbf{I}_{\{\Vert\mathbf{y}\Vert_{\infty}\leq 1\}}\cdot\frac{1}{\sigma_t^{d^*}(2\pi)^{d^*/2}}\exp\left(-\frac{\Vert\mathbf{x} - \mathbf{y}\Vert^2}{2\sigma_t^2}\right) \mathrm{d}\mathbf{y}.
$$
This analysis is detailed in Section \ref{subsec:c1}.

\paragraph{Approximating the transition kernel.} 
Although the Taylor polynomial $f_{Taylor}^{*}(\mathbf{y})$ can be implemented using a neural network, integrating over  $\mathbf{y}$   remains challenging. To address this, we introduce the Taylor polynomial  $f_{Taylor}^{kernel}(t,\mathbf{x}, \mathbf{y})$ to approximate the exponential transition kernel, resulting in the following approximation:
\begin{equation}\label{eq:diffused_poly}
\int_{A_{\mathbf{x}}} f_{Taylor}^{*}(\mathbf{y}) f_{Taylor}^{kernel}(t,\mathbf{x},\mathbf{y})\mathbf{I}_{\{\Vert\mathbf{y}\Vert_{\infty}\leq 1\}} \mathrm{d}\mathbf{y}.
\end{equation}
See Section \ref{subsec:c2} for more details. 
A similar approximation scheme using diffused local polynomials can also be applied to $\nabla q_t(\mathbf{x})$.

\paragraph{Approximating the integral via ReLU neural networks.} 
In \eqref{eq:diffused_poly},
 the product $f_{Taylor}^{*}(\mathbf{y})$ $\cdot f_{Taylor}^{kernel}(t,\mathbf{x},\mathbf{y})$ is  a polynomial,  thus its integration can be computed explicitly. Consequently, we use ReLU  neural networks to approximate the diffused local polynomials, as outlined in  Section \ref{subsec:c3}.

Finally, combining the above discussions, we derive the error bound for approximating  $\nabla\log q_t(\mathbf{x})$ using a ReLU neural network, as detailed in Section \ref{subsec:c4}.

\subsection{Approximating $\widehat{p}_{data}^*$ via local polynomials}\label{subsec:c1}

In this section, we approximate $\widehat{p}_{data}^*$ via local polynomials, which is a key step for approximating the true score function $\nabla\log q_t(\mathbf{x})$. We give the following lemma.

\begin{lemma}[Approximating $\widehat{p}_{data}^*$ via local polynomials]\label{lem:approximate_pdata1}
Suppose Assumptions \ref{ass:bounded_density}-\ref{ass:holder_density} hold. Let $\mathcal{R} \gg 1$, there exists $p_{\mathcal{R}}(\mathbf{x})$ such that
$$
|\widehat{p}_{data}^*(\mathbf{x}) - p_\mathcal{R}(\mathbf{x})| \lesssim \mathcal{R}^{-\beta},~ \mathbf{x}\in [-1,1]^{d^*}.
$$
\end{lemma}

\begin{proof}
We denote 
$$
f(\mathbf{x}) := \widehat{p}_{data}^{*}(2\mathbf{x} - 1), ~ \mathbf{x} \in [0, 1]^{d^*}.
$$
By Assumption \ref{ass:holder_density}, we know that $\Vert f \Vert_{\mathcal{H}^{\beta}([0,1]^{d^*})} \leq 2^rR$.
For any $\mathbf{n} = (n_1,n_2,\cdots,n_{d^*})^{\top} \in [\mathcal{R}]^{d^*} := \{0,1,\cdots,\mathcal{R}\}^{d^*}$, we define
$$
\psi_{\mathbf{n}}(\mathbf{x}):= \mathbf{I}_{\{\mathbf{x}\in\left(\frac{\mathbf{n}-1}{\mathcal{R}},\frac{\mathbf{n}}{\mathcal{R}}\right]\}} = \prod_{i=1}^{d^*}\mathbf{I}_{\{x_i\in\left(\frac{n_i-1}{\mathcal{R}},\frac{n_i}{\mathcal{R}}\right]\}}, ~ \mathbf{x} = (x_1,x_2,\cdots,x_{d^*})^{\top}\in[0,1]^{d^*}.
$$
Note that the functions $\{\psi_{\mathbf{n}}\}_{\mathbf{n}\in [\mathcal{R}]^{d^*}}$ form a partition of unity of the domain $[0,1]^{d^*}$, i.e.,
$$
\sum_{\mathbf{n}\in[\mathcal{R}]^{d^*}}\psi_{\mathbf{n}}(\mathbf{x}) \equiv 1, ~ \mathbf{x}\in [0,1]^{d^*}.
$$
Denote $p_{\mathbf{n},\boldsymbol{\alpha}}(\mathbf{x}):= \psi_{\mathbf{n}}(\mathbf{x})\left(\mathbf{x} - \frac{\mathbf{n}}{\mathcal{R}}\right)^{\boldsymbol{\alpha}}$ and $c_{\mathbf{n},\boldsymbol{\alpha}}:=\partial^{\boldsymbol{\alpha}}f(\frac{\mathbf{n}}{\mathcal{R}})/\boldsymbol{\alpha}!$. Then $p_{\mathbf{n},\boldsymbol{\alpha}}(\mathbf{x})$ is supported on $\{\mathbf{x}\in[0,1]^{d^*}:\mathbf{x}\in\left(\frac{\mathbf{n}-1}{\mathcal{R} },\frac{\mathbf{n}}{\mathcal{R}}\right]\}$. We define
$$
p(\mathbf{x}):=
\sum_{\mathbf{n}\in[\mathcal{R}]^{d^*}}\sum_{\Vert\boldsymbol{\alpha}\Vert_1\leq r}c_{\mathbf{n},\boldsymbol{\alpha}}p_{\mathbf{n},\boldsymbol{\alpha}}(\mathbf{x}), ~ \mathbf{x}\in[0,1]^{d^*}.
$$
Then, we have
$$
\begin{aligned}
|f(\mathbf{x}) - p(\mathbf{x})| &= \left|\sum_{\mathbf{n}\in[\mathcal{R}]^{d^*}}\psi_{\mathbf{n}}(\mathbf{x})f(\mathbf{x}) - \sum_{\mathbf{n}\in[\mathcal{R}]^{d^*}}\psi_{\mathbf{n}}(\mathbf{x})\sum_{\Vert\boldsymbol{\alpha}\Vert_1\leq r} c_{\mathbf{n},\boldsymbol{\alpha}}\left(\mathbf{x} - \frac{\mathbf{n}}{\mathcal{R}}\right)^{\boldsymbol{\alpha}}\right| \\
&\leq \sum_{\mathbf{n}\in[\mathcal{R}]^{d^*}}\psi_{\mathbf{n}}(\mathbf{x}) \left|f(\mathbf{x}) - \sum_{\Vert\boldsymbol{\alpha}\Vert_1\leq r} c_{\mathbf{n},\boldsymbol{\alpha}}\left(\mathbf{x} - \frac{\mathbf{n}}{\mathcal{R}}\right)^{\boldsymbol{\alpha}}\right|.
\end{aligned}
$$
Using Taylor expansion, there exists $\theta\in[0,1]$ such that
$$
f(\mathbf{x}) = \sum_{\Vert\boldsymbol{\alpha}\Vert_1 < r}c_{\mathbf{n},\boldsymbol{\alpha}}\left(\mathbf{x} - \frac{\mathbf{n}}{\mathcal{R}}\right)^{\boldsymbol{\alpha}} + \sum_{\Vert\boldsymbol{\alpha}\Vert_1 = r} \frac{\partial^{\boldsymbol{\alpha}}f\left((1-\theta)\frac{\mathbf{n}}{\mathcal{R}} + \theta\mathbf{x}\right)}{\boldsymbol{\alpha}!} \left(\mathbf{x} - \frac{\mathbf{n}}{\mathcal{R}}\right)^{\boldsymbol{\alpha}}.
$$
Thus, we have
$$
\begin{aligned}
\psi_{\mathbf{n}}(\mathbf{x})\left|f(\mathbf{x}) - \sum_{\Vert\boldsymbol{\alpha}\Vert_1\leq r} c_{\mathbf{n},\boldsymbol{\alpha}}\left(\mathbf{x} - \frac{\mathbf{n}}{\mathcal{R}}\right)^{\boldsymbol{\alpha}}\right| &\leq 2^rR\psi_{\mathbf{n}}(\mathbf{x}) \sum_{\Vert\boldsymbol{\alpha}\Vert_1 = r}\frac{1}{\boldsymbol{\alpha}!} \left\Vert\mathbf{x} - \frac{\mathbf{n}}{\mathcal{R}}\right\Vert_{\infty}^{\Vert\boldsymbol{\alpha}\Vert_1}\theta^{\gamma}\left\Vert\mathbf{x} - \frac{\mathbf{n}}{\mathcal{R}}\right\Vert_{\infty}^{\gamma} \\
& \leq 2^rR\psi_{\mathbf{n}}(\mathbf{x})\cdot \mathcal{R}^{-\beta}\sum_{\Vert\boldsymbol{\alpha}\Vert_1=r}\frac{1}{\boldsymbol{\alpha}!} \\
& \leq 2^rR\psi_{\mathbf{n}}(\mathbf{x})\cdot \mathcal{R}^{-\beta}\cdot (d^*)^{r},
\end{aligned}
$$
which implies that
$$
|f(\mathbf{x}) - p(\mathbf{x})| \leq (2d^*)^rR \cdot \mathcal{R}^{-\beta} \lesssim \mathcal{R}^{-\beta}.
$$
Finally, we define 
$$
p_{\mathcal{R}}(\mathbf{x}) = p\left(\frac{\mathbf{x} + 1}{2}\right), ~ \mathbf{x} \in [-1,1]^{d^*},
$$
then we have
$$
|\widehat{p}_{data}^*(\mathbf{x}) - p_{\mathcal{R}}(\mathbf{x})| = \left| f\left(\frac{\mathbf{x} + 1}{2}\right) - p\left(\frac{\mathbf{x} + 1}{2}\right)\right| \lesssim \mathcal{R}^{-\beta}.
$$
The proof is complete.
\end{proof}

In Lemma \ref{lem:approximate_pdata1}, we assume that  
Assumptions \ref{ass:bounded_density}-\ref{ass:holder_density} hold. 
By additionally incorporating   Assumption \ref{ass:boundary_smoothness}, we can derive the following results.
\begin{lemma}\label{lem:approximate_pdata2}
Suppose Assumptions \ref{ass:bounded_density}-\ref{ass:boundary_smoothness} hold. Let $\mathcal{R} \gg 1$, there exists $p_{\mathcal{R}}(\mathbf{x})$ that satisfies
$$
\left|\widehat{p}_{data}^*(\mathbf{x}) - p_{\mathcal{R}}(\mathbf{x}) \right| \lesssim \mathcal{R}^{-\beta}, ~ \mathbf{x} \in [-1,1]^{d^*},
$$
and
$$
\left|\widehat{p}_{data}^*(\mathbf{x}) - p_{\mathcal{R}}(\mathbf{x})\right| \lesssim \mathcal{R}^{-(3\beta + 2)}, ~ \mathbf{x} \in [-1,1]^{d^*}\backslash [-1+a_0, 1-a_0]^{d^*}.
$$
Moreover, $p_{\mathcal{R}}(\mathbf{x})$ has the following form:
\begin{align*}
p_{\mathcal{R}}(\mathbf{x}) =
&\sum_{\mathbf{n}\in[\mathcal{R}]^{d^*}}\sum_{\Vert\boldsymbol{\alpha}\Vert_1 < \beta}c_{\mathbf{n},\boldsymbol{\alpha}}^{(0)}p_{\mathbf{n},\boldsymbol{\alpha}}\left(\frac{\mathbf{x} + 1}{2}\right) \mathbf{I}_{\{\Vert\mathbf{x}\Vert_{\infty}\leq 1-a_0\}}\\
&+ \sum_{\mathbf{n}\in[\mathcal{R}]^{d^*}}\sum_{\Vert\boldsymbol{\alpha}\Vert_1 < 3\beta+2}c_{\mathbf{n},\boldsymbol{\alpha}}^{(1)}p_{\mathbf{n},\boldsymbol{\alpha}}\left(\frac{\mathbf{x} + 1}{2}\right) \mathbf{I}_{\{1-a_0 <  \Vert\mathbf{x}\Vert_{\infty}\leq 1\}},
\end{align*}
where $c_{\mathbf{n},\boldsymbol{\alpha}}^{(0)}$ and $c_{\mathbf{n},\boldsymbol{\alpha}}^{(1)}$ satisfy
$$
|c_{\mathbf{n},\boldsymbol{\alpha}}^{(0)}| \lesssim \frac{1}{\boldsymbol{\alpha}!}, ~~ |c_{\mathbf{n},\boldsymbol{\alpha}}^{(1)}| \lesssim \frac{1}{\boldsymbol{\alpha}!}.
$$
\end{lemma}

\begin{proof}
Since $\widehat{p}_{data}^*\in\mathcal{H}^{\beta}([-1,1]^{d^*}, R)$, by Lemma \ref{lem:approximate_pdata1}, we can construct
$$
p_{\mathcal{R},0}(\mathbf{x}) =  \sum_{\mathbf{n}\in[\mathcal{R}]^{d^*}}\sum_{\Vert\boldsymbol{\alpha}\Vert_1 < \beta}c_{\mathbf{n},\boldsymbol{\alpha}}^{(0)}p_{\mathbf{n},\boldsymbol{\alpha}}\left(\frac{\mathbf{x} + 1}{2}\right) \mathbf{I}_{\{  \Vert\mathbf{x}\Vert_{\infty}\leq 1\}}
$$ 
such that
$$
\left|\widehat{p}_{data}^*(\mathbf{x}) - p_{\mathcal{R},0}(\mathbf{x})\right| \lesssim \mathcal{R}^{-\beta}, ~ \mathbf{x} \in [-1,1]^{d^*},
$$
where $|c_{\mathbf{n},\boldsymbol{\alpha}}^{(0)}| \leq \frac{2^{\beta}R}{\boldsymbol{\alpha}!} \lesssim \frac{1}{\boldsymbol{\alpha}!}$.

By Assumption \ref{ass:boundary_smoothness}, we know that $\widehat{p}_{data}^*\in \mathcal{C}^{3\beta + 2}([-1,1]^{d^*}\backslash [-1+a_0, 1-a_0]^{d^*})$.
Consequently, we find that $\widehat{p}_{data}^{*,\prime}$ satisfies
$$
\widehat{p}_{data}^{*,\prime}\in \mathcal{C}^{3\beta + 2}([-1,1]^{d^*}), ~~ \widehat{p}_{data}^* = \widehat{p}_{data}^{*,\prime} \text{ on } [-1,1]^{d^*}\backslash [-1 + a_0, 1 - a_0]^{d^*}.
$$
Therefore, $\widehat{p}_{data}^{*,\prime}\in \mathcal{H}^{3\beta + 2}([-1,1]^{d^*},R_0)$ for some constant $R_0 > 0$. By using Lemma \ref{lem:approximate_pdata1} and replacing $\widehat{p}_{data}^*$ with $\widehat{p}_{data}^{*,\prime}$, we obtain 
$$
p_{\mathcal{R},1}(\mathbf{x}) =  \sum_{\mathbf{n}\in[\mathcal{R}]^{d^*}}\sum_{\Vert\boldsymbol{\alpha}\Vert_1 < 3\beta+2}c_{\mathbf{n},\boldsymbol{\alpha}}^{(1)}p_{\mathbf{n},\boldsymbol{\alpha}}\left(\frac{\mathbf{x} + 1}{2}\right) \mathbf{I}_{\{\Vert\mathbf{x}\Vert_{\infty}\leq 1\}}
$$ 
such that
$$
\left|\widehat{p}_{data}^{*,\prime}(\mathbf{x}) - p_{\mathcal{R},1}(\mathbf{x})\right| \lesssim \mathcal{R}^{-(3\beta + 2)}, ~ \mathbf{x} \in [-1,1]^{d^*},
$$
where $|c_{\mathbf{n},\boldsymbol{\alpha}}^{(1)}|\leq \frac{2^{3\beta + 2}R_0}{\boldsymbol{\alpha}!} \lesssim \frac{1}{\boldsymbol{\alpha}!}$.

Let $p_{\mathcal{R}}(\mathbf{x}) = p_{\mathcal{R},0}(\mathbf{x})\mathbf{I}_{\{\Vert\mathbf{x}\Vert_{\infty} \leq 1 - a_0\}} + p_{\mathcal{R},1}(\mathbf{x})\mathbf{I}_{\{1 - a_0 < \Vert\mathbf{x}\Vert\leq 1\}}$, i.e.,
\begin{align*}
p_\mathcal{R}(\mathbf{x}) = &\sum_{\mathbf{n}\in[\mathcal{R}]^{d^*}}\sum_{\Vert\boldsymbol{\alpha}\Vert_1 < \beta}c_{\mathbf{n},\boldsymbol{\alpha}}^{(0)}p_{\mathbf{n},\boldsymbol{\alpha}}\left(\frac{\mathbf{x} + 1}{2}\right) \mathbf{I}_{\{\Vert\mathbf{x}\Vert_{\infty}\leq 1-a_0\}}\\
&+ \sum_{\mathbf{n}\in[\mathcal{R}]^{d^*}}\sum_{\Vert\boldsymbol{\alpha}\Vert_1 < 3\beta+2}c_{\mathbf{n},\boldsymbol{\alpha}}^{(1)}p_{\mathbf{n},\boldsymbol{\alpha}}\left(\frac{\mathbf{x} + 1}{2}\right) \mathbf{I}_{\{1-a_0 <  \Vert\mathbf{x}\Vert_{\infty}\leq 1\}}.
\end{align*}
Then, $p_{\mathcal{R}}(\mathbf{x})$ satisfies that
$$
\begin{aligned}
\left|\widehat{p}_{data}^*(\mathbf{x}) - p_{\mathcal{R}}(\mathbf{x})\right| & \lesssim 
\left| \widehat{p}_{data}^*(\mathbf{x}) - p_{\mathcal{R},0}(\mathbf{x})\right| \mathbf{I}_{\{\Vert\mathbf{x}\Vert_{\infty} \leq 1- a_0\}} + \left| \widehat{p}_{data}^{*,\prime}(\mathbf{x}) - p_{\mathcal{R},1}(\mathbf{x})\right| \mathbf{I}_{\{1-a_0 < \Vert\mathbf{x}\Vert_{\infty} \leq 1\}}\\
& \lesssim  \mathcal{R}^{-\beta} \mathbf{I}_{\{\Vert\mathbf{x}\Vert_{\infty} \leq 1- a_0\}} + \mathcal{R}^{-(3\beta + 2)} \mathbf{I}_{\{1-a_0 < \Vert\mathbf{x}\Vert_{\infty} \leq 1\}},
\end{aligned}
$$
which implies that
$$
\left|\widehat{p}_{data}^*(\mathbf{x}) - p_{\mathcal{R}}(\mathbf{x})\right| \lesssim \mathcal{R}^{-\beta} + \mathcal{R}^{-(3\beta + 2)} \lesssim \mathcal{R}^{-\beta}, ~ \mathbf{x} \in [-1,1]^{d^*},
$$
and 
$$
\left|\widehat{p}_{data}^*(\mathbf{x}) - p_{\mathcal{R}}(\mathbf{x})\right| \lesssim \mathcal{R}^{-(3\beta + 2)}, ~ \mathbf{x} \in [-1,1]^{d^*} \backslash [-1+a_0, 1-a_0]^{d^*}.
$$
The proof is complete.
\end{proof}

\subsection{
Approximating $q_t(\mathbf{x})$
and $\sigma_t\nabla q_t(\mathbf{x})$
via local polynomial integrals 
}
\label{subsec:c2}
In this section, we approximate $q_t(\mathbf{x})$
and $\nabla q_t(\mathbf{x})$
via local polynomial integrals. If Assumptions \ref{ass:bounded_density}-\ref{ass:boundary_smoothness} hold, then $p_N(\mathbf{y})$ in Lemma \ref{lem:approximate_pdata2} can be rewritten as
$$
\begin{aligned}
p_N(\mathbf{y})=  & \sum_{\mathbf{n}\in[\mathcal{R}]^{d^*}}\sum_{\Vert\boldsymbol{\alpha}\Vert_1 < \beta}c_{\mathbf{n},\boldsymbol{\alpha}}^{(0)}p_{\mathbf{n},\boldsymbol{\alpha}}\left(\frac{\mathbf{y} + 1}{2}\right) \mathbf{I}_{\{\Vert\mathbf{y}\Vert_{\infty}\leq 1-a_0\}}\\
&+ \sum_{\mathbf{n}\in[\mathcal{R}]^{d^*}}\sum_{\Vert\boldsymbol{\alpha}\Vert_1 < 3\beta+2}c_{\mathbf{n},\boldsymbol{\alpha}}^{(1)}p_{\mathbf{n},\boldsymbol{\alpha}}\left(\frac{\mathbf{y} + 1}{2}\right) \mathbf{I}_{\{ \Vert\mathbf{y}\Vert_{\infty}\leq 1\}} \\
& -\sum_{\mathbf{n}\in[\mathcal{R}]^{d^*}}\sum_{\Vert\boldsymbol{\alpha}\Vert_1 < 3\beta+2}c_{\mathbf{n},\boldsymbol{\alpha}}^{(1)}p_{\mathbf{n},\boldsymbol{\alpha}}\left(\frac{\mathbf{y} + 1}{2}\right) \mathbf{I}_{\{ \Vert\mathbf{y}\Vert_{\infty}\leq 1-a_0\}}.
\end{aligned}
$$
Therefore, $p_N(\mathbf{y})$ in either Lemma \ref{lem:approximate_pdata1} or Lemma \ref{lem:approximate_pdata2} can be expressed as the combination of the following functional forms:
$$
\sum_{\mathbf{n}\in[\mathcal{R}]^{d^*}}\sum_{\Vert\boldsymbol{\alpha}\Vert_1 < C_{\boldsymbol{\alpha}}}c_{\mathbf{n},\boldsymbol{\alpha}}p_{\mathbf{n},\boldsymbol{\alpha}}\left(\frac{\mathbf{y} + 1}{2}\right) \mathbf{I}_{\{ \Vert\mathbf{y}\Vert_{\infty}\leq C_{\mathbf{y}}\}},
$$
where $C_{\boldsymbol{\alpha}} \in \{\beta, 3\beta + 2\}$, $C_{\mathbf{y}} \in \{1-a_0, 1\}$ and $c_{\mathbf{n}, \boldsymbol{\alpha}} \in \{c_{\mathbf{n}, \boldsymbol{\alpha}}^{(0)}, c_{\mathbf{n}, \boldsymbol{\alpha}}^{(1)}\}$. 
By replacing $\widehat{p}_{data}^*$ with $p_{\mathcal{R}}$, we denote 
$$
g_1(t,\mathbf{x}) := \int_{\mathbb{R}^{d^*}}p_\mathcal{R}(\mathbf{y})\mathbf{I}_{\{\Vert\mathbf{y}\Vert_{\infty}\leq 1\}}\cdot
\frac{1}{\sigma_t^{d^*}(2\pi)^{d^*/2}}\exp\left(-\frac{\Vert\mathbf{x} - \mathbf{y}\Vert^2}{2\sigma_t^2}\right) \mathrm{d}\mathbf{y}.
$$
The difference between $q_t(\mathbf{x})$ and $g_1(t,\mathbf{x})$ can be bounded as
$$
\begin{aligned}
|g_1(t,\mathbf{x}) - q_t(\mathbf{x})| & \leq \int_{\mathbb{R}^{d^*}}|p_\mathcal{R}(\mathbf{y}) - \widehat{p}_{data}^*(\mathbf{y})|\mathbf{I}_{\{\Vert\mathbf{y}\Vert_{\infty}\leq 1\}}\cdot
\frac{1}{\sigma_t^{d^*}(2\pi)^{d^*/2}}\exp\left(-\frac{\Vert\mathbf{x} - \mathbf{y}\Vert^2}{2\sigma_t^2}\right) \mathrm{d}\mathbf{y} \\
& \lesssim \mathcal{R}^{-\beta}\cdot \int_{\mathbb{R}^{d^*}}
\frac{1}{\sigma_t^{d^*}(2\pi)^{d^*/2}}\exp\left(-\frac{\Vert\mathbf{x} - \mathbf{y}\Vert^2}{2\sigma_t^2}\right) \mathrm{d}\mathbf{y} \lesssim \mathcal{R}^{-\beta}.
\end{aligned}
$$
Therefore, we only need to approximate $g_1(t,\mathbf{x})$. By Lemma \ref{lem:integral_clipping}, for any $\epsilon > 0$, there exists a constant $C > 0$ such that for any $\mathbf{x}\in\mathbb{R}^{d^*}$,
$$
\left|q_t(\mathbf{x}) - \int_{A_{\mathbf{x}}}\widehat{p}_{data}^{*}(\mathbf{y})\mathbf{I}_{\{\Vert\mathbf{y}\Vert_{\infty}\leq 1\}}\cdot
\frac{1}{\sigma_t^{d^*}(2\pi)^{d^*/2}}\exp\left(-\frac{\Vert\mathbf{x} - \mathbf{y}\Vert^2}{2\sigma_t^2}\right) \mathrm{d}\mathbf{y}\right| \lesssim \epsilon,
$$
where $A_{\mathbf{x}} = \prod_{i=1}^{d^*}a_{i,\mathbf{x}}$ with $a_{i,\mathbf{x}} = [x_i - C\sigma_t\sqrt{\log\epsilon^{-1}}, x_i + C\sigma_t\sqrt{\log\epsilon^{-1}}]$. We denote
$$
g_2(t,\mathbf{x}):= \int_{A_{\mathbf{x}}}p_\mathcal{R}(\mathbf{y})\mathbf{I}_{\{\Vert\mathbf{y}\Vert_{\infty}\leq 1\}}\cdot
\frac{1}{\sigma_t^{d^*}(2\pi)^{d^*/2}}\exp\left(-\frac{\Vert\mathbf{x} - \mathbf{y}\Vert^2}{2\sigma_t^2}\right) \mathrm{d}\mathbf{y}.
$$
Since $\widehat{p}_{data}^*$ is bounded, $p_{\mathcal{R}}$ is also bounded. Replacing $\widehat{p}_{data}^*$ with $p_{\mathcal{R}}$ in Lemma \ref{lem:integral_clipping},
the difference between $g_1$ and $g_2$ can be bounded as
$$
|g_1(t,\mathbf{x}) - g_2(t,\mathbf{x})| \lesssim \epsilon.
$$

Note that $g_2(t,\mathbf{x})$ includes an integral involving the exponential function, which is challenging to handle.  To address this difficulty, we use polynomials to approximate the exponential function. 
For any $1\leq i \leq d^*$ and $\mathbf{y}\in A_{\mathbf{x}}$, we know that $\frac{|x_i - y_i|}{\sigma_t} \leq C\sqrt{\log\epsilon^{-1}}$. Thus, by Taylor expansions, we have
$$
\left|
\exp\left(-\frac{(x_i - y_i)^2}{2\sigma_t^2}\right) - \sum_{l=0}^{k-1}\frac{1}{l!}\left(-\frac{(x_i-y_i)^2}{2\sigma_t^2}\right)^l
\right| \leq \frac{C^{2k}\log^k\epsilon^{-1}}{k!2^k}, ~ \forall y_i\in[B_{l}(x_i), B_{u}(x_i)],
$$
where
$$
B_l(x_i) = \max\left\{x_i - C\sigma_t\sqrt{\log\epsilon^{-1}}, -1\right\},
$$
and
$$
B_u(x_i) = \min\left\{x_i + C\sigma_t\sqrt{\log\epsilon^{-1}}, 1\right\}.
$$
By setting $k \geq \frac{3}{2}C^2 u\log\epsilon^{-1}$ and using the inequality $k! \geq (k/3)^k$ when $k \geq 3$, we have
$$
\left|
\exp\left(-\frac{(x_i - y_i)^2}{2\sigma_t^2}\right) - \sum_{l=0}^{k-1}\frac{1}{l!}\left(-\frac{(x_i-y_i)^2}{2\sigma_t^2}\right)^l
\right| \leq \epsilon^{\frac{3}{2}C^2 u\log u}.
$$
Thus, we can set
$$
u = \max\left\{e, \frac{2}{3C^2}\left(1 + \frac{\log d^*}{\log \epsilon^{-1}}\right)\right\}
$$
such that
$$
\epsilon^{\frac{3}{2}C^2u\log u} \leq \frac{\epsilon}{d^*},
$$
where $k = \mathcal{O}\left(\log \epsilon^{-1}\right)$. 
By multiplying over the $d^*$ dimensions indexed by $i$, we have
$$
\left|
\exp\left(-\frac{\Vert\mathbf{x} - \mathbf{y}\Vert^2}{2\sigma_t^2}\right) - \prod_{i=1}^{d^*}\sum_{l=0}^{k-1}\frac{1}{l!}\left(-\frac{(x_i-y_i)^2}{2\sigma_t^2}\right)^l
\right| \leq d^*\left(1 + \frac{\epsilon}{d^*}\right)^{d^*-1}\cdot\frac{\epsilon}{d^*} \lesssim \epsilon.
$$
Therefore, we only need to approximate
$$
g_3(t,\mathbf{x}) := 
\int_{A_{\mathbf{x}}}p_{\mathcal{R}}(\mathbf{y})\mathbf{I}_{\{\Vert\mathbf{y}\Vert_{\infty}\leq 1\}}\cdot
\frac{1}{\sigma_t^{d^*}(2\pi)^{d^*/2}}
\prod_{i=1}^{d^*}\sum_{l=0}^{k-1}\frac{1}{l!}\left(-\frac{(x_i-y_i)^2}{2\sigma_t^2}\right)^l
\mathrm{d}\mathbf{y}.
$$
Notice that the difference between $g_2(t,\mathbf{x})$ and $g_3(t,\mathbf{x})$ can be bounded as
$$
\begin{aligned}
|g_2(t,\mathbf{x}) - g_3(t,\mathbf{x})| &\lesssim \epsilon \cdot \frac{1}{\sigma_t^{d^*}(2\pi)^{d^/2}} \int_{A_{\mathbf{x}}} p_{\mathcal{R}}(\mathbf{y})\mathbf{I}_{\{\Vert\mathbf{y}\Vert_\infty \leq 1\}} \mathrm{d}\mathbf{y} \\
& \lesssim \epsilon \cdot \frac{1}{\sigma_t^{d^*}(2\pi)^{d^/2}} \int_{A_{\mathbf{x}}} \left(\widehat{p}_{data}^*(\mathbf{y}) + {\mathcal{R}}^{-\beta}\right)\mathbf{I}_{\{\Vert\mathbf{y}\Vert_\infty \leq 1\}} \mathrm{d}\mathbf{y} \\
& \lesssim \epsilon \cdot (C_u + \mathcal{R}^{-\beta}) \left(2C\sigma_t\sqrt{\log\epsilon^{-1}}\right)^{d^*} / \sigma_t^{d^*} \\
& \lesssim \epsilon \log^{\frac{d^*}{2}} \epsilon^{-1}.
\end{aligned}
$$
This implies that we only need to approximate $g_3(t,\mathbf{x})$ using a sufficiently accurate ReLU neural network. 

Note that $p_{\mathcal{R}}(\mathbf{y})$ can be expressed as the combination of
$$
\sum_{\mathbf{n}\in[\mathcal{R}]^{d^*}}\sum_{\Vert\boldsymbol{\alpha}\Vert_1 < C_{\boldsymbol{\alpha}}}c_{\mathbf{n},\boldsymbol{\alpha}}p_{\mathbf{n},\boldsymbol{\alpha}}\left(\frac{\mathbf{\mathbf{y}} + 1}{2}\right) \mathbf{I}_{\{ \Vert\mathbf{\mathbf{y}}\Vert_{\infty}\leq C_{\mathbf{\mathbf{y}}}\}},
$$
where $C_{\boldsymbol{\alpha}} \in \{\beta, 3\beta + 2\}$, 
$C_{\mathbf{y}} \in \{1-a_0, 1\}$,
and $c_{\mathbf{n}, \boldsymbol{\alpha}} \in \{c_{\mathbf{n}, \boldsymbol{\alpha}}^{(0)}, c_{\mathbf{n}, \boldsymbol{\alpha}}^{(1)}\}$. Thus, $g_3(t,\mathbf{x})$ can be expressed as the combination of the following functional forms
\begin{equation} \label{eq:diffused_local_poly1}
\begin{aligned}
    &\sum_{\mathbf{n}\in[\mathcal{R}]^{d^*}}\sum_{\Vert\boldsymbol{\alpha}\Vert_1 < C_{\boldsymbol{\alpha}}}c_{\mathbf{n},\boldsymbol{\alpha}}
    \int_{A_{\mathbf{x}}}
    p_{\mathbf{n},\boldsymbol{\alpha}}\left(\frac{\mathbf{\mathbf{y}} + 1}{2}\right) \mathbf{I}_{\{ \Vert\mathbf{\mathbf{y}}\Vert_{\infty}\leq C_{\mathbf{\mathbf{y}}}\}} \cdot \frac{1}{\sigma_t^{d^*}(2\pi)^{d^*/2}}
    \prod_{i=1}^{d^*}\sum_{l=0}^{k-1}\frac{1}{l!}\left(-\frac{(x_i-y_i)^2}{2\sigma_t^2}\right)^l
    \mathrm{d}\mathbf{y} \\
    = &\sum_{\mathbf{n}\in[\mathcal{R}]^{d^*}}\sum_{\Vert\boldsymbol{\alpha}\Vert_1 < C_{\boldsymbol{\alpha}}}c_{\mathbf{n},\boldsymbol{\alpha}}
    \prod_{i=1}^{d^*}\frac{1}{\sigma_t(2\pi)^{1/2}} \Bigg(
    \int_{a_{i,\mathbf{x}}\cap \left(\frac{2(n_i-1)}{\mathcal{R}}-1, \frac{2n_i}{\mathcal{R}}-1 \right]} \mathbf{I}_{\{ \Vert\mathbf{\mathbf{y}}\Vert_{\infty}\leq C_{\mathbf{\mathbf{y}}}\}}
    \left(\frac{y_i + 1}{2} - \frac{n_i}{\mathcal{R}}\right)^{\alpha_i} \\
    & ~~~~~~~~~~~~~~~~~~~~~~~~~~~~~~~~~~~~~~~~~~~~~  \cdot \sum_{l=0}^{k-1}\frac{1}{l!}\left(-\frac{(x_i-y_i)^2}{2\sigma_t^2}\right)^l \mathrm{d}y_i \Bigg).
\end{aligned}
\end{equation}

Similarly, we can define $\mathbf{h}_1(t,\mathbf{x})$ to approximate $\sigma_t\nabla q_t(\mathbf{x})$, where
$$
\mathbf{h}_1(t,\mathbf{x}) :=  \int_{\mathbb{R}^{d^*}}p_{\mathcal{R}}(\mathbf{y})\mathbf{I}_{\{\Vert\mathbf{y}\Vert_{\infty}\leq 1\}}\cdot
\frac{\mathbf{x} - \mathbf{y}}{\sigma_t^{d^*+1}(2\pi)^{d^*/2}}\exp\left(-\frac{\Vert\mathbf{x} - \mathbf{y}\Vert^2}{2\sigma_t^2}\right) \mathrm{d}\mathbf{y}.
$$
The difference between $\mathbf{h}_1(t,\mathbf{x})$ and $\sigma_t\nabla q_t(\mathbf{x})$ can be bounded as
$$
\Vert \mathbf{h}_1(t,\mathbf{x}) - \sigma_t\nabla q_t(\mathbf{x})\Vert \lesssim \mathcal{R}^{-\beta}.
$$
We can also define $\mathbf{h}_2(t,\mathbf{x})$ and $\mathbf{h}_3(t,\mathbf{x})$ as follows:
$$
\mathbf{h}_2(t,\mathbf{x}) :=  \int_{A_{\mathbf{x}}}p_{\mathcal{R}}(\mathbf{y})\mathbf{I}_{\{\Vert\mathbf{y}\Vert_{\infty}\leq 1\}}\cdot
\frac{\mathbf{x} - \mathbf{y}}{\sigma_t^{d^*+1}(2\pi)^{d^*/2}}\exp\left(-\frac{\Vert\mathbf{x} - \mathbf{y}\Vert^2}{2\sigma_t^2}\right) \mathrm{d}\mathbf{y},
$$
$$
\mathbf{h}_3(t,\mathbf{x}) := \int_{A_{\mathbf{x}}}p_{\mathcal{R}}(\mathbf{y})\mathbf{I}_{\{\Vert\mathbf{y}\Vert_{\infty}\leq 1\}}\cdot
\frac{\mathbf{x} - \mathbf{y}}{\sigma_t^{d^*+1}(2\pi)^{d^*/2}}
\prod_{i=1}^{d^*}\sum_{l=0}^{k-1}\frac{1}{l!}\left(-\frac{(x_i-y_i)^2}{2\sigma_t^2}\right)^l
\mathrm{d}\mathbf{y}.
$$
Then,  we have
$$
\Vert \mathbf{h}_2(t,\mathbf{x}) - \mathbf{h}_1(t,\mathbf{x})\Vert \lesssim \epsilon,
$$
and
$$
\Vert \mathbf{h}_3(t,\mathbf{x}) - \mathbf{h}_2(t,\mathbf{x})\Vert \lesssim \epsilon\log^{\frac{d^*+1}{2}}\epsilon^{-1}.
$$
The $j$-th element of $\mathbf{h}_3(t,\mathbf{x})$ can be expressed as the combination of the following functional forms
\begin{equation} \label{eq:diffused_local_poly2}
\begin{aligned}
& \sum_{\mathbf{n}\in[\mathcal{R}]^{d^*}}\sum_{\Vert\boldsymbol{\alpha}\Vert_1 \leq r}c_{\mathbf{n},\boldsymbol{\alpha}} \Bigg(
\prod_{i\neq j}^{d^*}\frac{1}{\sigma_t(2\pi)^{1/2}} \int_{a_{i,\mathbf{x}}\cap \left(\frac{2(n_i-1)}{\mathcal{R}}-1, \frac{2n_i}{\mathcal{R}}-1 \right]}\mathbf{I}_{\{|y_i| \leq C_{\mathbf{y}}\}}
\left(\frac{y_i + 1}{2} - \frac{n_i}{\mathcal{R}}\right)^{\alpha_i} 
\sum_{l=0}^{k-1}\frac{1}{l!}\left(-\frac{(x_i-y_i)^2}{2\sigma_t^2}\right)^l \mathrm{d}y_i \\
& ~~~~ \cdot \frac{1}{\sigma_t(2\pi)^{1/2}} \int_{a_{j,\mathbf{x}}\cap \left(\frac{2(n_j-1)}{\mathcal{R}}-1, \frac{2n_j}{\mathcal{R}}-1 \right]} \mathbf{I}_{\{|y_j| \leq C_{\mathbf{y}}\}}
\left(\frac{y_j + 1}{2} - \frac{n_j}{\mathcal{R}}\right)^{\alpha_j} \left(\frac{x_j-y_j}{\sigma_t}\right)
\sum_{l=0}^{k-1}\frac{1}{l!}\left(\frac{(x_j-y_j)^2}{2\sigma_t^2}\right)^l \mathrm{d}y_j \Bigg).
\end{aligned}
\end{equation}

\subsection{Approximating the local polynomial integrals via ReLU neural networks}\label{subsec:c3}
In this section, we use ReLU neural networks to approximate the local polynomial integrals. 
Specifically, we focus on the approximation of \eqref{eq:diffused_local_poly1}
under the restriction  $\Vert\mathbf{x}\Vert_{\infty} \leq C_0$, where $C_0 > 0$ is a constant.
For convenience, we define
$$
f(t,x,n,\alpha,l):= \frac{1}{\sigma_t(2\pi)^{1/2}} \int_{a_{x}\cap \left(\frac{2(n-1)}{\mathcal{R}}-1, \frac{2n}{\mathcal{R}}-1 \right]} \mathbf{I}_{\{|y| \leq C_{\mathbf{y}}\}}
\left(\frac{y + 1}{2} - \frac{n}{\mathcal{R}}\right)^{\alpha} 
\frac{1}{l!}\left(-\frac{(x-y)^2}{2\sigma_t^2}\right)^l \mathrm{d}y,
$$
where $a_x = [x - C\sigma_t\sqrt{\log\epsilon^{-1}}, x + C\sigma_t\sqrt{\log\epsilon^{-1}}]$. 
Then,
\eqref{eq:diffused_local_poly1} can be expressed as
$$
\eqref{eq:diffused_local_poly1} = \sum_{\mathbf{n}\in[\mathcal{R}]^{d^*}}\sum_{\Vert\boldsymbol{\alpha}\Vert_1 < C_{\boldsymbol{\alpha}}}c_{\mathbf{n},\boldsymbol{\alpha}}\cdot\prod_{i=1}^{d^*}\sum_{l=0}^{k-1}f(t,x_i,n_i,\alpha_i,l).
$$
We first use the ReLU neural network to approximate $f(t,x,n,\alpha,l)$ for $|x|\leq C_0$.

\begin{lemma}\label{lem:approximate_h}
Given $\mathcal{R} \gg 1$, $C_0 > 0$, let $1 - T = \mathcal{R}^{-C_T}$, where $C_T > 0$ is a constant. For any $\epsilon_0 > 0$, $n \leq \mathcal{R}$, $\alpha < C_{\boldsymbol{\alpha}}$, and $l \leq k - 1$, there exists a ReLU neural network $\mathrm{s}_f \in \mathrm{NN}(L, M, J, \kappa)$ with
$$
\begin{aligned}
&L = \mathcal{O}\left(\log^2\epsilon^{-1}_0 + \log^2C_0 + \log^2 \mathcal{R} + \log^4 \epsilon^{-1} \right),\\ &M = \mathcal{O}\left(\log^3\epsilon^{-1}_0 + \log^3C_0 + \log^3 \mathcal{R} + \log^6 \epsilon^{-1} \right), \\
&J = \mathcal{O}\left(\log^4\epsilon^{-1}_0 + \log^4C_0 + \log^4 \mathcal{R} + \log^8 \epsilon^{-1} \right),\\ &\kappa = \exp\left(\mathcal{O}\left(\log^2\epsilon^{-1}_0 + \log^2C_0 + \log^2 \mathcal{R} + \log^4 \epsilon^{-1} \right)\right),
\end{aligned}
$$
such that 
$$
|\mathrm{s}_f(t, x, n, \alpha, l) - f(t, x, n, \alpha, l)| \lesssim \epsilon_0, ~ x\in[-C_0, C_0], ~ t\in[\mathcal{R}^{-C_T}, 1].
$$
\end{lemma}

\begin{proof}
By the definition of $f(t, x,n,\alpha,l)$, we take the transformation $z = \frac{x - y}{\sigma_t}$, then
$$
\begin{aligned}
f(t,x,n,\alpha,l) & = \frac{1}{l!(2\pi)^{1/2}}\int_{D}\sum_{j=0}^{\alpha} C_{\alpha}^{j}\left(\frac{x + 1}{2} - \frac{n}{\mathcal{R}}\right)^{\alpha-j}\left(-\frac{\sigma_t z}{2}\right)^j\left(-\frac{z^2}{2}\right)^l \mathrm{d}z \\
& = \frac{1}{l!(2\pi)^{1/2}} \sum_{j=0}^{\alpha} C_{\alpha}^j \left(\frac{x + 1}{2} - \frac{n}{\mathcal{R}}\right)^{\alpha-j}  \frac{1}{(-2)^{j+l}}\cdot\sigma_t^j \int_D z^{j + 2l} \mathrm{d}z \\
& = \frac{1}{l!(2\pi)^{1/2}} \sum_{j=0}^{\alpha} C_{\alpha}^j \left(x + 1 - \frac{2n}{\mathcal{R}}\right)^{\alpha-j}  \frac{(-1)^{\alpha-j}}{(-2)^{\alpha+l}}\cdot\sigma_t^j \frac{D_U^{j + 2l + 1}(x) - D_L^{j + 2l + 1}(x)}{j + 2l + 1},
\end{aligned}
$$
where $D = [D_L(t,x), D_U(t,x)]$ with
$$
D_L(t,x) = \mathrm{s}_{\mathrm{clip}}\left(\frac{1}{\sigma_t}\cdot\mathrm{s}_{\mathrm{clip}}\left(x - \frac{2n}{\mathcal{R}} + 1, x-C_{\mathbf{y}}, x + C_{\mathbf{y}} \right), -C\sqrt{\log\epsilon^{-1}}, C\sqrt{\log\epsilon^{-1}} \right),
$$
and 
$$
D_U(t,x) = \mathrm{s}_{\mathrm{clip}}\left(\frac{1}{\sigma_t}\cdot\mathrm{s}_{\mathrm{clip}}\left(x - \frac{2(n-1)}{\mathcal{R}} + 1, x-C_{\mathbf{y}}, x + C_{\mathbf{y}} \right), -C\sqrt{\log\epsilon^{-1}}, C\sqrt{\log\epsilon^{-1}} \right).
$$
Thus, we only to approximate the function of the form
$$
f_{n,j,l}(t, x) = \left(x + 1 - \frac{2n}{\mathcal{R}} \right) ^{\alpha-j} \cdot \sigma_t^j \cdot \left(D_{U}^{j + 2l + 1}(t,x) - D_L^{j + 2l + 1}(t,x)\right).
$$
We first consider the approximation of $D_L^{j+2l+1}(t,x)$. We can choose a ReLU neural network
$$
\begin{aligned}
\mathrm{s}_{n,j,l}^{(1)}(t,x) & = \mathrm{s}_{\mathrm{prod},1}(\underbrace{\cdot, \cdots, \cdot}_{j + 2l + 1 ~ \text{times}}) \circ \mathrm{s}_{\mathrm{clip}}(\cdot, -C\sqrt{\log\epsilon^{-1}}, C\sqrt{\log\epsilon^{-1}}) \\
 & ~~~~~~ \circ \mathrm{s}_{\mathrm{prod},2}\left(
\mathrm{s}_{\mathrm{clip}}\left(x - \frac{2n}{\mathcal{R}} + 1, x-C_{\mathbf{y}}, x + C_{\mathbf{y}} \right), 
\mathrm{s}_{\mathrm{rec}}(\cdot) \circ \mathrm{s}_{\mathrm{root}}(\sigma^2 t)
\right).
\end{aligned}
$$
$D_U^{j + 2l + 1}(t,x)$ 
is  approximated by 
$$
\begin{aligned}
\mathrm{s}_{n,j,l}^{(2)}(t,x) & = \mathrm{s}_{\mathrm{prod},1}(\underbrace{\cdot, \cdots, \cdot}_{j + 2l + 1 ~ \text{times}}) \circ \mathrm{s}_{\mathrm{clip}}(\cdot, -C\sqrt{\log\epsilon^{-1}}, C\sqrt{\log\epsilon^{-1}}) \\
 & ~~~~~~ \circ \mathrm{s}_{\mathrm{prod},2}\left(
\mathrm{s}_{\mathrm{clip}}\left(x - \frac{2(n-1)}{\mathcal{R}} + 1, x-C_{\mathbf{y}}, x + C_{\mathbf{y}} \right), 
\mathrm{s}_{\mathrm{rec}}(\cdot) \circ \mathrm{s}_{\mathrm{root}}(\sigma^2 t)
\right).
\end{aligned}
$$
Therefore, $D_U^{j + 2l + 1}(t,x) - D_L^{j + 2l + 1}(t,x)$ can be approximated by
$$
\mathrm{s}_{n,j,l}^{(3)}(t,x) = \mathrm{s}_{n,j,l}^{(1)}(t,x) - \mathrm{s}_{n,j,l}^{(2)}(t,x).
$$
Next, we consider $\left(x + 1 - \frac{2n}{\mathcal{R}}\right)^{\alpha-j}$. Since $\left|x + 1 - \frac{2n}{\mathcal{R}}\right| \leq C_0 + 3$, we can take $C=C_0 + 3$ and $d=\alpha - j$ in Lemma \ref{lem:product}.
Then, there exists a ReLU neural network 
$$
\mathrm{s}_{n,j,l}^{(4)}(x) = \mathrm{s}_{\mathrm{prod},3}(\underbrace{\cdot,\cdots,\cdot}_{\alpha - j ~ \text{times}}) \circ \left(x + 1 - \frac{2n}{\mathcal{R}}\right)
$$
to approximate the term $\left(x + 1 - \frac{2n}{\mathcal{R}}\right)^{\alpha-j}$. 
For $\sigma_t^j$, we choose
$$
\mathrm{s}_{n,j,l}^{(5)}(t) = \mathrm{s}_{\mathrm{prod},4}(\underbrace{\cdot,\cdots,\cdot}_{j ~ \text{times}}) \circ \mathrm{s}_{\mathrm{root}}(\sigma^2 t).
$$
Finally, we can choose a ReLU neural network 
$$
\mathrm{s}_{n,j,l}^{(6)}(t,x) = \mathrm{s}_{\mathrm{prod},5}(\mathrm{s}_{n,j,l}^{(4)}(x), \mathrm{s}_{n,j,l}^{(5)}(t), \mathrm{s}_{n,j,l}^{(3)}(t,x))
$$
to approximate $f_{n,j,l}(t,x)$.

Next, we derive the error bound between $\mathrm{s}_{n,j,l}^{(6)}(t,x)$ and $f_{n,j,l}(t,x)$. For convenience, we denote $\epsilon^{(i)} (i=1,\cdots,6)$ as the approximation error of $\mathrm{s}_{n,j,l}^{(i)} (i=1,\cdots,6)$, and denote $\epsilon_{\mathrm{prod},i} (i=1,\cdots,5)$ as the approximation error of $\mathrm{s}_{\mathrm{prod},i} (i=1,\cdots,5)$. Since $\left|x + 1 - \frac{2n}{\mathcal{R}}\right|^{\alpha-j} \leq \left(C_0 + 3\right)^{\alpha-j}$, $|D_L(t,x)| \leq C\sqrt{\log\epsilon^{-1}}$, $|D_U(t,x)| \leq C\sqrt{\log\epsilon^{-1}}$, and $\sigma_t^j \leq \sigma^j$, we set
$$
C_1 = \max\left\{ \left(C_0 + 3\right)^{\alpha-j}, \sigma^j, (C\sqrt{\log\epsilon^{-1}})^{j + 2l + 1} \right\}.
$$
By Lemma \ref{lem:product}, we have
$$
|\mathrm{s}_{n,j,l}^{(6)}(t, x) - f_{n,j,l}(t,x)| \leq \epsilon_{\mathrm{prod},5} + 3C_1^2\cdot \max\{\epsilon^{(3)}, \epsilon^{(4)}, \epsilon^{(5)}\}.
$$
By Lemmas \ref{lem:reciprocal}-\ref{lem:square_root},
we can bound $\epsilon^{(1)}$ and $\epsilon^{(2)}$. Let $C_2 = \max\left\{C_0 + 1, \frac{1}{\sigma\sqrt{1-T}}\right\}$, then we have
$$
\epsilon^{(1)} \leq \epsilon_{\mathrm{prod},1} + (j + 2l + 1)(C\sqrt{\log\epsilon^{-1}})^{j + 2l}\left[\epsilon_{\mathrm{prod},2} + 2C_2\left(\epsilon_{\mathrm{rec}} + \frac{\sigma\epsilon_{\mathrm{root}}}{\epsilon_{\mathrm{rec}}^2}\right)\right]
$$
and
$$
\epsilon^{(2)} \leq \epsilon_{\mathrm{prod},1} + (j + 2l + 1)(C\sqrt{\log\epsilon^{-1}})^{j + 2l}\left[\epsilon_{\mathrm{prod},2} + 2C_2\left(\epsilon_{\mathrm{rec}} + \frac{\sigma\epsilon_{\mathrm{root}}}{\epsilon_{\mathrm{rec}}^2}\right)\right].
$$
Since
$$
\epsilon^{(3)} \leq \epsilon^{(1)} + \epsilon^{(2)}, 
$$ 
we obtain
$$
\epsilon^{(3)} \leq 2\epsilon_{\mathrm{prod},1} + 2(j + 2l + 1)(C\sqrt{\log\epsilon^{-1}})^{j + 2l}\left[\epsilon_{\mathrm{prod},2} + 2C_2\left(\epsilon_{\mathrm{rec}} + \frac{\sigma\epsilon_{\mathrm{root}}}{\epsilon_{\mathrm{rec}}^2}\right)\right].
$$
We  can  also bound $\epsilon^{(4)}$ and $\epsilon^{(5)}$ as 
$$
\epsilon^{(4)} \leq \epsilon_{\mathrm{prod},3},~ \epsilon^{(5)} \leq \epsilon_{\mathrm{prod},4} + j\sigma^j\epsilon_{\mathrm{root}}.
$$
To ensure that $\epsilon^{(6)} \leq \epsilon_0$, we take $\epsilon_{\mathrm{prod},5} = \frac{\epsilon_0}{2}$ and $\max\{\epsilon^{(3)}, \epsilon^{(4)}, \epsilon^{(5)}\} \leq \frac{\epsilon_0}{6C_1^2}=:\epsilon^*$. In detail, we take
$$
\epsilon_{\mathrm{prod},1} = \frac{\epsilon^*}{4}, \epsilon_{\mathrm{prod},2} = \frac{\epsilon^*}{8(j + 2l + 1)(C\sqrt{\log\epsilon^{-1}})^{j + 2l}}, \epsilon_{\mathrm{prod},3} = \epsilon^*, \epsilon_{\mathrm{prod},4} = \frac{\epsilon^*}{2}.
$$
Moreover, we take
$$
\epsilon_{\mathrm{rec}} = \frac{\epsilon^*}{32C_2(j + 2l + 1)(C\sqrt{\log\epsilon^{-1}})^{j + 2l}}, ~ \epsilon_{\mathrm{root}} = \min\left\{\frac{ \epsilon_{\mathrm{rec}}^3}{\sigma}, \frac{\epsilon^*}{2j\sigma^j}
\right\}.
$$
Then, it is easy to verify that $\max\{\epsilon^{(3)}, \epsilon^{(4)}, \epsilon^{(5)}\} \leq \epsilon^*$. Note that $j \leq \alpha \leq C_{\boldsymbol{\alpha}} = \mathcal{O}(1)$, $l \leq k - 1 = \mathcal{O}(\log \epsilon^{-1})$.  Subsequently, we can obtain the network structures of $\mathrm{s}_{n,j,l}^{(i)} (i=1,\cdots,6)$. For $\mathrm{s}_{n,j,l}^{(1)}, \mathrm{s}_{n,j,l}^{(2)}$,  and $\mathrm{s}_{n,j,l}^{(3)}$, we have
$$
\begin{aligned}
&L = \mathcal{O}\left(\log^2\epsilon^{-1}_0 + \log^2 \mathcal{R} + \log^2C_0 + \log^4 \epsilon^{-1}\right),\\ &M = \mathcal{O}\left(\log^3\epsilon^{-1}_0 + \log^3 \mathcal{R} + \log^3C_0 + \log^6 \epsilon^{-1}\right), \\
&J = \mathcal{O}\left(\log^4\epsilon^{-1}_0 + \log^4C_0 + \log^4 \mathcal{R} + \log^8 \epsilon^{-1}\right), \\ &\kappa = \exp\left(\mathcal{O}\left(\log^2\epsilon^{-1}_0 + \log^2 \mathcal{R} + \log^2C_0 + \log^4 \epsilon^{-1} \right)\right).
\end{aligned}
$$
For $\mathrm{s}_{n,j,l}^{(4)}$, we have
$$
\begin{aligned}
L = \mathcal{O}\left(\log\epsilon^{-1}_0 + \log C_0 + \log \mathcal{R} + \log^2 \epsilon^{-1}\right), ~ M = \mathcal{O}(1), \\
J = \mathcal{O}\left(\log\epsilon^{-1}_0 + \log C_0 + \log \mathcal{R} + \log^2 \epsilon^{-1} \right), ~ \kappa = \exp\left(\mathcal{O}(\log C_0)\right).
\end{aligned}
$$
For $\mathrm{s}_{n,j,l}^{(5)}$, we have
$$
\begin{aligned}
&L = \mathcal{O}\left(\log^2\epsilon^{-1}_0 + \log^2C_0 + \log^2 \mathcal{R} + \log^4 \epsilon^{-1} \right),\\
&M = \mathcal{O}\left(\log^3\epsilon^{-1}_0 + \log^3C_0 + \log^3 \mathcal{R} + \log^6 \epsilon^{-1} \right), \\
&J = \mathcal{O}\left(\log^4\epsilon^{-1}_0 + \log^4C_0 + \log^4 \mathcal{R} + \log^8 \epsilon^{-1} \right),\\ &\kappa = \exp\left(\mathcal{O}\left(\log\epsilon^{-1}_0 + \log C_0 + \log \mathcal{R} + \log^2 \epsilon^{-1} \right)\right).
\end{aligned}
$$
Therefore, by Lemma \ref{lem:concatenation} and Lemma \ref{lem:parallelization}, we finally obtain the network structure of $\mathrm{s}_{n,j,l}^{(6)}$ as follows:
$$
\begin{aligned}
&L = \mathcal{O}\left(\log^2\epsilon^{-1}_0 + \log^2C_0 + \log^2 \mathcal{R} + \log^4 \epsilon^{-1} \right), \\ 
&M = \mathcal{O}\left(\log^3\epsilon^{-1}_0 + \log^3C_0 + \log^3 \mathcal{R} + \log^6 \epsilon^{-1} \right), \\
&J = \mathcal{O}\left(\log^4\epsilon^{-1}_0 + \log^4C_0 + \log^4 \mathcal{R} + \log^8 \epsilon^{-1} \right),\\ &\kappa = \exp\left(\mathcal{O}\left(\log^2\epsilon^{-1}_0 + \log^2C_0 + \log^2 \mathcal{R} + \log^4 \epsilon^{-1} \right)\right).
\end{aligned}
$$

Now, we consider the following ReLU network
$$
\mathrm{s}_{f}(t,x,n,\alpha,l) = \frac{1}{l!(2\pi)^{1/2}}\sum_{j=0}^{\alpha} C_{\alpha}^j \frac{(-1)^{\alpha-j}}{(-2)^{\alpha + l}}\frac{1}{j + 2l + 1} s_{n,j,l}^{(6)}(t,x).
$$
Then, we have
$$
\begin{aligned}
|\mathrm{s}_f(t,x,n,\alpha,l) - f(t,x,n,\alpha,l)| & \leq \frac{1}{l!(2\pi)^{1/2}}\left(\sum_{j=0}^{\alpha} C_{\alpha}^j\frac{1}{2^{\alpha + l}(j + 2l + 1)}
\right) \cdot \epsilon_0 \\
& \leq \frac{\epsilon_0}{l!2^l(2\pi)^{1/2}} \lesssim \epsilon_0.
\end{aligned}
$$
By Lemma \ref{lem:parallelization}, the parameters of the network $\mathrm{s}_f$ satisfy
$$
\begin{aligned}
&L = \mathcal{O}\left(\log^2\epsilon^{-1}_0 + \log^2C_0 + \log^2 \mathcal{R} + \log^4 \epsilon^{-1} \right),\\ &M = \mathcal{O}\left(\log^3\epsilon^{-1}_0 + \log^3C_0 + \log^3 \mathcal{R} + \log^6 \epsilon^{-1} \right), \\
&J = \mathcal{O}\left(\log^4\epsilon^{-1}_0 + \log^4C_0 + \log^4 \mathcal{R} + \log^8 \epsilon^{-1} \right), \\ 
&\kappa = \exp\left(\mathcal{O}\left(\log^2\epsilon^{-1}_0 + \log^2C_0 + \log^2 \mathcal{R} + \log^4 \epsilon^{-1} \right)\right).
\end{aligned}
$$
The proof is complete.
\end{proof}

With Lemma \ref{lem:approximate_h}, we can construct a ReLU neural network to approximate \eqref{eq:diffused_local_poly1}.
\begin{lemma}\label{lem:approximate_poly1}
Given $\mathcal{R} \gg 1$, $C_0 > 0$, let $1 - T = \mathcal{R}^{-C_T}$, where $C_T > 0$ is a constant. For any $\epsilon_0 > 0$, there exists a ReLU neural network $\mathrm{s}_1\in\mathrm{NN}(L,M,J,\kappa)$ with
$$
L = \mathcal{O}\left(\log^2\epsilon^{-1}_0 + \log^2 C_0 + \log^2 \mathcal{R} + \log^4 \epsilon^{-1} \right), $$
$$
M = \mathcal{O}\left(\mathcal{R}^{d^*}\log \epsilon^{-1} (\log^3\epsilon^{-1}_0 + \log^3 C_0 + \log^3 \mathcal{R} + \log^6 \epsilon^{-1})\right),  $$
$$
J = \mathcal{O}\left(\mathcal{R}^{d^*}\log \epsilon^{-1} (\log^4\epsilon^{-1}_0 + \log^4 C_0 + \log^4 \mathcal{R} + \log^8 \epsilon^{-1}) \right), $$
$$
\kappa = \exp\left(\mathcal{O}\left(\log^2\epsilon^{-1}_0 + \log^2C_0 + \log^2 \mathcal{R} + \log^4 \epsilon^{-1} \right)\right),
$$
such that
$$
\left|\mathrm{s}_1(t,\mathbf{x}) - \sum_{\mathbf{n}\in[\mathcal{R}]^{d^*}}\sum_{\Vert\boldsymbol{\alpha}\Vert_1 < C_{\boldsymbol{\alpha}}}c_{\mathbf{n},\boldsymbol{\alpha}}\cdot\prod_{i=1}^{d^*}\sum_{l=0}^{k-1}f(t,x_i,n_i,\alpha_i,l) \right| \lesssim \epsilon_0, ~ \mathbf{x}\in[-C_0,C_0]^{d^*}, t\in[\mathcal{R}^{-C_T}, 1].
$$
\end{lemma}

\begin{proof}
We first construct a ReLU neural network $\mathrm{s}_{\mathrm{sum}}\in\mathrm{NN}(L,M,J,\kappa)$, which  satisfies
$$
\mathrm{s}_{\mathrm{sum}}(t,x,n,\alpha) = \sum_{l=0}^{k-1}\mathrm{s}_f(t,x,n,\alpha,l),
$$
with network parameters
$$
L = \mathcal{O}\left(\log^2\epsilon^{-1}_f + \log^2C_0 + \log^2 \mathcal{R} + \log^4 \epsilon^{-1} \right), 
$$
$$
M = \mathcal{O}\left(\log \epsilon^{-1}(\log^3\epsilon^{-1}_f + \log^3C_0 + \log^3 \mathcal{R} + \log^6 \epsilon^{-1})\right), 
$$
$$
J = \mathcal{O}\left(\log \epsilon^{-1}(\log^4\epsilon^{-1}_f + \log^4C_0 + \log^4 \mathcal{R} + \log^8 \epsilon^{-1} \right)), 
$$
$$
\kappa = \exp\left(\mathcal{O}\left(\log^2\epsilon^{-1}_f + \log^2C_0 + \log^2 \mathcal{R} + \log^4 \epsilon^{-1} \right)\right),
$$
where $\epsilon_f$ is the approximation error of $\mathrm{s}_f$. Then, for any $\mathbf{n}\in [\mathcal{R}]^{d^*}$ and $\boldsymbol{\alpha}\in\mathbb{N}^{d^*}, \Vert\boldsymbol{\alpha}\Vert_1 < C_{\boldsymbol{\alpha}}$, we can construct the following ReLU network:
$$
\mathrm{s}_{\mathbf{n}, \boldsymbol{\alpha}}(t,\mathbf{x}) = \mathrm{s}_{\mathrm{prod}}(\mathrm{s}_{\mathrm{sum}}(t,x_1,n_1,\alpha_1), \cdots, \mathrm{s}_{\mathrm{sum}}(t,x_{d^*}, n_{d^*}, \alpha_{d^*}))
$$
to approximate the term $\prod_{i=1}^{d^*}\sum_{l=0}^{k-1}f(t,x_i,n_i,\alpha_i,l)$.  The approximation error can be written as
$$
\epsilon_{\mathrm{s}_{\mathbf{n},\boldsymbol{\alpha}}} \leq  \epsilon_{\mathrm{prod}} + d^* C_3^{d^*-1} \epsilon_f k,
$$
where
$$
C_3 = \max_{|x_i| \leq C_0, 1\leq i \leq d^*} \sum_{l = 0}^{k-1} |f(t,x_i,n_i,s_i,l)|.
$$
The term $|f(t,x_i,n_i,\alpha_i,l)|$ ($1 \leq i \leq d^*$) can be bounded as follows:
$$
\begin{aligned}
|f(t,x_i,n_i,\alpha_i,l)| & \leq \frac{2}{2^{\alpha_i+l}l!(2\pi)^{1/2}}\sum_{j=0}^{\alpha_i} C_{\alpha_i}^j(C_0 + 3)^{\alpha_i-j} \sigma_t^j (C\sqrt{\log\epsilon^{-1}})^{j + 2l + 1} \\
& \leq \frac{2(C\sqrt{\log\epsilon^{-1}})^{\alpha_i + 2(k-1) + 1}}{2^{\alpha_i + l}l!(2\pi)^{1/2}}\sum_{j=0}^{\alpha_i} C_{\alpha_i}^j(C_0 + 3)^{\alpha_i-j}\sigma^j \\
& \leq \frac{2(C\sqrt{\log\epsilon^{-1}})^{\alpha_i + 2k - 1}}{2^{\alpha_i} l! (2\pi)^{1/2}} \cdot (C_0 + 3 + \sigma)^{\alpha_i}.
\end{aligned}
$$
It implies that
$$
\begin{aligned}
\sum_{l=0}^{k-1} |f(t,x_i,n_i,\alpha_i,l)| & \leq \frac{2(C\sqrt{\log\epsilon^{-1}})^{\alpha_i + 2k -1}}{2^{\alpha_i}(2\pi)^{1/2}} \cdot (C_0 + 3 + \sigma)^{\alpha_i} \cdot \sum_{l=0}^{k-1}\frac{1}{l!} \\
& \leq \frac{2e(C\sqrt{\log\epsilon^{-1}})^{\alpha_i + 2k -1}}{2^{\alpha_i}(2\pi)^{1/2}} \cdot (C_0 + 3 + \sigma)^{\alpha_i}.
\end{aligned}
$$
Therefore, $C_3$ satisfies
$$
\log C_3 = \mathcal{O}(\log C_0 + (\log \epsilon^{-1})\cdot(\log\log \epsilon^{-1})) \leq \mathcal{O}(\log C_0 + \log^2 \epsilon^{-1}),
$$
and the parameters of $\mathrm{s}_{\mathrm{prod}}$ satisfy
$$
L = \mathcal{O}\left(\log\epsilon_{\mathrm{prod}}^{-1} + \log C_0 + \log^2 \epsilon^{-1}\right), ~ M = \mathcal{O}(1),
$$
$$
J = \mathcal{O}\left(\log\epsilon_{\mathrm{prod}}^{-1} + \log C_0 + \log^2 \epsilon^{-1}\right), ~ \kappa = \exp\left(\mathcal{O}(\log C_0 + \log^2 \epsilon^{-1})\right).
$$
We denote $\epsilon^* = \frac{\epsilon_0}{(d^*)^{C_{\boldsymbol{\alpha}}} \mathcal{R}^{d^*}}$. By taking
$$
\epsilon_{\mathrm{prod}} = \frac{\epsilon^*}{2}, ~ \epsilon_f = \frac{\epsilon^*}{2d^*C_3^{d^*-1}k},
$$
we ensure that $\epsilon_{\mathrm{s}_{\mathbf{n},\boldsymbol{\alpha}}} \leq \epsilon^*$. Note that $k = \mathcal{O}(\log \epsilon^{-1})$. By substituting $\epsilon_{\mathrm{prod}}$ and $\epsilon_f$ into the network parameters of $\mathrm{s}_{\mathrm{prod}}$ and $\mathrm{s}_{\mathrm{sum}}$, we obtain the network parameters of $\mathrm{s}_{\mathbf{n},\boldsymbol{\alpha}}$, which satisfy
$$
L = \mathcal{O}\left(\log^2\epsilon^{-1}_0 + \log^2C_0 + \log^2 \mathcal{R} + \log^4 \epsilon^{-1} \right), 
$$
$$
M = \mathcal{O}\left(\log \epsilon^{-1}(\log^3\epsilon^{-1}_0 + \log^3C_0 + \log^3 \mathcal{R} + \log^6 \epsilon^{-1}) \right),
$$
$$
J = \mathcal{O}\left(\log \epsilon^{-1} (\log^4\epsilon^{-1}_0 + \log^4C_0 + \log^4 \mathcal{R}  + \log^8 \epsilon^{-1} \right)), 
$$
$$
\kappa = \exp\left(\mathcal{O}\left(\log^2\epsilon^{-1}_0 + \log^2C_0 + \log^2 \mathcal{R} + \log^4 \epsilon^{-1} \right)\right).
$$
Finally, we can construct a ReLU neural network 
$$
\mathrm{s}_1(t,\mathbf{x}) = \sum_{\mathbf{n}\in[\mathcal{R}]^{d^*}}\sum_{\Vert\boldsymbol{\alpha}\Vert_1 < C_{\boldsymbol{\alpha}}} c_{\mathbf{n},\boldsymbol{\alpha}}\mathrm{s}_{\mathbf{n},\boldsymbol{\alpha}}(t,\mathbf{x})
$$
with network parameters
$$
L = \mathcal{O}\left(\log^2\epsilon^{-1}_0 + \log^2 C_0 + \log^2 \mathcal{R} + \log^4 \epsilon^{-1} \right), $$
$$
M = \mathcal{O}\left(\mathcal{R}^{d^*}\log \epsilon^{-1} (\log^3\epsilon^{-1}_0 + \log^3 C_0 + \log^3 \mathcal{R} + \log^6 \epsilon^{-1})\right),  $$
$$
J = \mathcal{O}\left(\mathcal{R}^{d^*}\log \epsilon^{-1} (\log^4\epsilon^{-1}_0 + \log^4 C_0 + \log^4 \mathcal{R} + \log^8 \epsilon^{-1}) \right), $$
$$
\kappa = \exp\left(\mathcal{O}\left(\log^2\epsilon^{-1}_0 + \log^2C_0 + \log^2 \mathcal{R} + \log^4 \epsilon^{-1} \right)\right).
$$
The approximation error between $\mathrm{s}_1(t,\mathbf{x})$ and \eqref{eq:diffused_local_poly1} satisfies
$$
\begin{aligned}
\left|\mathrm{s}_1(t,\mathbf{x}) - \sum_{\mathbf{n}\in[\mathcal{R}]^{d^*}}\sum_{\Vert\boldsymbol{\alpha}\Vert_1 < C_{\boldsymbol{\alpha}}}c_{\mathbf{n},\boldsymbol{\alpha}}\cdot\prod_{i=1}^{d^*}\sum_{l=0}^{k-1}f(t,x_i,n_i,\alpha_i,l) \right| &\leq \sum_{\mathbf{n}\in[\mathcal{R}]^{d^*}}\sum_{\Vert\boldsymbol{\alpha}\Vert_1 < C_{\boldsymbol{\alpha}}}|c_{\mathbf{n},\boldsymbol{\alpha}}|\epsilon_{\mathrm{s}_{\mathbf{n},\boldsymbol{\alpha}}} \\
& \lesssim \sum_{\mathbf{n}\in[\mathcal{R}]^{d^*}}\sum_{\Vert\boldsymbol{\alpha}\Vert_1 < C_{\boldsymbol{\alpha}}}\frac{1}{\boldsymbol{\alpha}!} \cdot \epsilon^* \\
& \lesssim (d^*)^{C_{\boldsymbol{\alpha}}} \mathcal{R}^{d^*} \cdot \frac{\epsilon_0}{(d^*)^{C_{\boldsymbol{\alpha}}} \mathcal{R}^{d^*}} \lesssim \epsilon_0.
\end{aligned}
$$
The proof is complete.
\end{proof}

Similarly, the $j$-th element of \eqref{eq:diffused_local_poly2} can be rewritten as
$$
[\eqref{eq:diffused_local_poly2}]_j = \sum_{\mathbf{n}\in[\mathcal{R}]^{d^*}}\sum_{\Vert\boldsymbol{\alpha}\Vert_1 < C_{\boldsymbol{\alpha}}}c_{\mathbf{n},\boldsymbol{\alpha}}
\left[\left(
\prod_{i=1, i\neq j}^{d^*}\sum_{l=0}^{k-1}f(t,x_i,n_i,\alpha_i,l)
\right) \cdot \left(\sum_{l=0}^{k-1}f_1(t,x_j,n_j,\alpha_j,l)\right)
\right],
$$
where
$$
f_1(t,x,n,\alpha,l):= \frac{1}{\sigma_t(2\pi)^{1/2}} \int_{a_{x}\cap \left(\frac{2(n-1)}{\mathcal{R}}-1, \frac{2n}{\mathcal{R}}-1 \right]} \mathbf{I}_{\{|y| \leq C_{\mathbf{y}}\}} \left(\frac{x-y}{\sigma_t}\right)
\left(\frac{y + 1}{2} - \frac{n}{\mathcal{R}}\right)^{\alpha} 
\frac{1}{l!}\left(-\frac{(x-y)^2}{2\sigma_t^2}\right)^l \mathrm{d}y.
$$
In the same way, we can choose a ReLU neural network $\mathbf{s}_2(t,\mathbf{x})\in\mathbb{R}^{d^*}$ to approximate \eqref{eq:diffused_local_poly2}. We omit the proof and give the following lemma.
\begin{lemma}\label{lem:approximate_poly2}
Given $\mathcal{R} \gg 1$, $C_0 > 0$, let $1 - T = \mathcal{R}^{-C_T}$, where $C_T > 0$ is a constant. For any $\epsilon_0 > 0$, there exist ReLU neural networks $\mathrm{s}_{2,j}\in\mathrm{NN}(L,M,J,\kappa)$, $1\leq j \leq d^*$, with
$$
L_j = \mathcal{O}\left(\log^2\epsilon^{-1}_0 + \log^2 C_0 + \log^2 \mathcal{R} + \log^4 \epsilon^{-1} \right), $$
$$
M_j = \mathcal{O}\left(\mathcal{R}^{d^*}\log \epsilon^{-1} (\log^3\epsilon^{-1}_0 + \log^3 C_0 + \log^3 \mathcal{R} + \log^6 \epsilon^{-1})\right),  $$
$$
J_j = \mathcal{O}\left(\mathcal{R}^{d^*}\log \epsilon^{-1} (\log^4\epsilon^{-1}_0 + \log^4 C_0 + \log^4 \mathcal{R} + \log^8 \epsilon^{-1}) \right), $$
$$
\kappa_j = \exp\left(\mathcal{O}\left(\log^2\epsilon^{-1}_0 + \log^2C_0 + \log^2 \mathcal{R} + \log^4 \epsilon^{-1} \right)\right),
$$
such that for any $\mathbf{x}\in[-C_0,C_0]^{d^*}$ and $t\in[\mathcal{R}^{-C_T}, 1]$, the following holds:
$$
\left|\mathrm{s}_{2,j}(t,\mathbf{x}) - \sum_{\mathbf{n}\in[\mathcal{R}]^{d^*}}\sum_{\Vert\boldsymbol{\alpha}\Vert_1 < C_{\boldsymbol{\alpha}}}c_{\mathbf{n},\boldsymbol{\alpha}}
\left[\left(
\prod_{i=1, i\neq j}^{d^*}\sum_{l=0}^{k-1}f(t,x_i,n_i,\alpha_i,l)
\right) \cdot \left(\sum_{l=0}^{k-1}f_1(t,x_j,n_j,\alpha_j,l)\right)
\right] \right| \lesssim \epsilon_0.
$$
\end{lemma}

Therefore, according to Lemma \ref{lem:approximate_poly2}, there exists a network $\mathbf{s}_2(t,\mathbf{x}) = [\mathrm{s}_{2,1}(t,\mathbf{x}), \cdots, \mathrm{s}_{2, d^*}(t,\mathbf{x})]^{\top} \in \mathbb{R}^{d^*}$ that approximates \eqref{eq:diffused_local_poly2}. 
The $\ell_2$-norm error between $\mathbf{s}_2(t,\mathbf{x})$ and \eqref{eq:diffused_local_poly2} is bounded by $\mathcal{O}(\epsilon_0)$.

\subsection{Error bound of approximating $\nabla\log q_t(\mathbf{x})$ with ReLU neural network}\label{subsec:c4}
In this subsection, we construct two ReLU neural networks to approximate the true score function $\nabla\log q_t(\mathbf{x})$ on different time interval. We first introduce the following lemma.

\begin{lemma} \label{lem:approximate_interval_1}
Let $\mathcal{R} \gg 1$. There exists a ReLU neural network $\mathbf{s}_{\mathrm{score}}^{(1)}\in\mathrm{NN}(L,M,J,\kappa)$ with
$$
L = \mathcal{O}(\log^4 \mathcal{R}), M = \mathcal{O}(\mathcal{R}^{d^*}\log^7 \mathcal{R}), J = \mathcal{O}(\mathcal{R}^{d^*}\log^9 \mathcal{R}), \kappa = \exp\left(\mathcal{O}(\log^4 \mathcal{R})\right)
$$
that satisfies
$$
\int_{\mathbb{R}^{d^*}} q_t(\mathbf{x})\Vert \mathbf{s}_{\mathrm{score}}^{(1)}(t,\mathbf{x}) - \nabla\log q_t(\mathbf{x})\Vert^2 \mathrm{d}\mathbf{x} \lesssim \frac{\mathcal{R}^{-2\beta}\log \mathcal{R}}{\sigma_t^2}, ~~ t\in [\mathcal{R}^{-C_T}, 3 \mathcal{R}^{-C_T}].
$$
Moreover, we can take $\mathbf{s}_{\mathrm{score}}^{(1)}$ satisfying $\Vert \mathbf{s}_{\mathrm{score}}^{(1)} (t, \cdot)\Vert_{\infty} \lesssim \frac{\sqrt{\log \mathcal{R}}}{\sigma_t}$. 
\end{lemma}
\begin{proof}
The approximation error can be decomposed into three terms:
$$
\begin{aligned}
& \int_{\mathbb{R}^{d^*}} q_t(\mathbf{x})\Vert \mathbf{s}_{\mathrm{score}}^{(1)}(t,\mathbf{x}) - \nabla\log q_t(\mathbf{x})\Vert^2 \mathrm{d}\mathbf{x} \\
= & \underbrace{ \int_{\Vert\mathbf{x}\Vert_\infty > 1 + C\sigma_t\sqrt{\log\epsilon^{-1}_1}} q_t(\mathbf{x}) \Vert \mathbf{s}_{\mathrm{score}}^{(1)}(t,\mathbf{x}) - \nabla\log q_t(\mathbf{x})\Vert^2 \mathrm{d}\mathbf{x}. }_{(\mathrm{I})} \\
= & \underbrace{\int_{\Vert \mathbf{x}\Vert_{\infty} \leq 1 + C\sigma_t\sqrt{\log\epsilon^{-1}_1}} q_t(\mathbf{x}) \mathbf{I}_{\{q_t(\mathbf{x}) < \epsilon_1\}} \Vert \mathbf{s}_{\mathrm{score}}^{(1)}(t,\mathbf{x}) - \nabla\log q_t(\mathbf{x})\Vert^2 \mathrm{d}\mathbf{x} }_{(\mathrm{II})} \\
+ & \underbrace{\int_{\Vert \mathbf{x}\Vert_{\infty} \leq 1 + C\sigma_t\sqrt{\log\epsilon^{-1}_1}} q_t(\mathbf{x}) \mathbf{I}_{\{q_t(\mathbf{x}) \geq \epsilon_1\}}\Vert \mathbf{s}_{\mathrm{score}}^{(1)}(t,\mathbf{x}) - \nabla\log q_t(\mathbf{x})\Vert^2 \mathrm{d}\mathbf{x} }_{(\mathrm{III})}. 
\end{aligned}
$$
According to Lemma \ref{lem:bound_for_clipping}, there exists a constant $C > 0$ such that for any $\epsilon_1 > 0$, 
\begin{equation}\label{eq:appendix_eq1}
\underbrace{ \int_{\Vert\mathbf{x}\Vert_\infty > 1 + C\sigma_t\sqrt{\log\epsilon^{-1}_1}} q_t(\mathbf{x}) \Vert \mathbf{s}_{\mathrm{score}}^{(1)}(t,\mathbf{x}) - \nabla\log q_t(\mathbf{x})\Vert^2 \mathrm{d}\mathbf{x} }_{(\mathrm{I})} \lesssim \frac{\epsilon_1}{\sigma_t} + \sigma_t\epsilon_1\Vert \mathbf{s}_{\mathrm{score}}^{(1)}(t,\cdot)\Vert_{\infty}^2.
\end{equation}
Since $\Vert\nabla\log q_t(\mathbf{x})\Vert$ is bounded by $\frac{C\sqrt{\log\epsilon^{-1}_1}}{\sigma_t}$  for $\Vert\mathbf{x}\Vert_{\infty} > 1 + C\sigma_t\sqrt{\log\epsilon^{-1}_1}$ according to Lemma \ref{lem:derivatives_boundness}, 
we can choose $\mathbf{s}_{\mathrm{score}}^{(1)}$ such  that $\Vert\mathbf{s}_{\mathrm{score}}^{(1)}(t,\cdot)\Vert_{\infty} \lesssim \frac{\sqrt{\log\epsilon^{-1}_1}}{\sigma_t}$. Therefore,  \eqref{eq:appendix_eq1} is bounded by $\frac{\epsilon_1\log\epsilon^{-1}_1}{\sigma_t^2}$. Taking $\epsilon_1 = \mathcal{R}^{-(2\beta + 1)}$, then \eqref{eq:appendix_eq1} $\lesssim \frac{\mathcal{R}^{-(2\beta + 1)}\log \mathcal{R}}{\sigma_t^2}$, which is smaller than $\frac{\mathcal{R}^{-2\beta}\log \mathcal{R}}{\sigma_t^2}$. Moreover, by Lemma \ref{lem:bound_for_clipping}, we can bound $(\mathrm{II})$ as follows: 
$$
\begin{aligned}
& ~~~~ \underbrace{ \int_{\Vert\mathbf{x}\Vert_{\infty} \leq 1 + C\sigma_t\sqrt{\log\epsilon^{-1}_1}} q_t(\mathbf{x})\mathbf{I}_{\{q_t(\mathbf{x}) < \epsilon_1\}} \Vert 
\mathbf{s}_{\mathrm{score}}^{(1)}(t,\mathbf{x}) - \nabla\log q_t(\mathbf{x})\Vert^2 \mathrm{d}\mathbf{x} }_{(\mathrm{II})} \\
& \lesssim \frac{\epsilon_1}{\sigma_t^2} \cdot (\log\epsilon^{-1}_1)^{\frac{d^*+2}{2}} + \Vert\mathbf{s}_{\mathrm{score}}^{(1)}(t,\cdot)\Vert_{\infty}^2 \epsilon_1 \cdot (\log\epsilon^{-1}_1)^{\frac{d^*}{2}} \\
& \lesssim \frac{\epsilon_1}{\sigma_t^2} \cdot (\log\epsilon^{-1}_1)^{\frac{d^*+2}{2}} \lesssim \frac{\mathcal{R}^{-(2\beta + 1)}\log^{\frac{d^*+2}{2}}\mathcal{R}}{\sigma_t^2},
\end{aligned}
$$
which is smaller than $\frac{\mathcal{R}^{-2\beta}\log \mathcal{R}}{\sigma_t^2}$ for sufficiently large $\mathcal{R}$. Thus, we can only focus on bounding the term (III). 

Since $\Vert \nabla\log q_t(\mathbf{x}) \Vert \lesssim \frac{\sqrt{\log \mathcal{R}}}{\sigma_t}$, there exists a constant $C_4 > 0$ such that $\Vert\nabla\log q_t(\mathbf{x})\Vert_{\infty} \leq \Vert \nabla\log q_t(\mathbf{x}) \Vert \leq \frac{C_4\sqrt{\log \mathcal{R}}}{\sigma_t}$. Define $$
\mathbf{h}^{\prime}(t,\mathbf{x}):= \max\left\{\min\left\{ \frac{\mathbf{h}_1(t,\mathbf{x})}{ g_1(t,\mathbf{x}) \vee \mathcal{R}^{-(2\beta + 1)}}, C_4\sqrt{\log \mathcal{R}} \right\}, -C_4\sqrt{\log \mathcal{R}}\right\}.
$$
We can decompose the term (III) as
$$
\begin{aligned}
& \underbrace{\int_{\Vert \mathbf{x}\Vert_{\infty} \leq 1 + C\sigma_t\sqrt{\log\epsilon^{-1}_1}} q_t(\mathbf{x}) \mathbf{I}_{\{q_t(\mathbf{x}) \geq \epsilon_1\}}\Vert \mathbf{s}_{\mathrm{score}}^{(1)}(t,\mathbf{x}) - \nabla\log q_t(\mathbf{x})\Vert^2 \mathrm{d}\mathbf{x} }_{(\mathrm{III})} \\
\lesssim & 
\underbrace{ \int_{\Vert \mathbf{x}\Vert_{\infty} \leq 1 + C\sigma_t\sqrt{\log\epsilon^{-1}_1}} q_t(\mathbf{x}) \mathbf{I}_{\{q_t(\mathbf{x}) \geq \epsilon_1\}} \left\Vert \mathbf{s}_{\mathrm{score}}^{(1)}(t,\mathbf{x}) - \frac{\mathbf{h}^{\prime}(t,\mathbf{x})}{\sigma_t} \right\Vert^2 \mathrm{d}\mathbf{x} }_{(A)} \\
+ & \underbrace{ \int_{\Vert \mathbf{x}\Vert_{\infty} \leq 1 + C\sigma_t\sqrt{\log\epsilon^{-1}_1}} q_t(\mathbf{x}) \mathbf{I}_{\{q_t(\mathbf{x}) \geq \epsilon_1\}} \left\Vert \frac{\mathbf{h}^{\prime}(t,\mathbf{x})}{\sigma_t}  - 
\nabla\log q_t(\mathbf{x})
\right\Vert^2 \mathrm{d}\mathbf{x} }_{(B)}.
\end{aligned}
$$
Now, we construct the ReLU neural network $\mathbf{s}_{\mathrm{score}}^{(1)}$ to bound the approximation error $(A)$. The construction is straightforward. We take $C_0 = 1 + C\sigma \sqrt{(2\beta+1)\log N} = \mathcal{O}(\sqrt{\log N})$. By Lemma \ref{lem:approximate_poly1} and Lemma \ref{lem:approximate_poly2}, for any $\epsilon_0 > 0$, we can construct two ReLU neural networks  $\mathrm{s}_3(t,\mathbf{x})$ and $\mathbf{s}_4(t,\mathbf{x})$ such that
$$
\begin{aligned}
|\mathrm{s}_3(t,\mathbf{x}) \vee \mathcal{R}^{-(2\beta + 1)} - g_1(t,\mathbf{x}) \vee \mathcal{R}^{-(2\beta + 1)}| & \leq |\mathrm{s}_3(t,\mathbf{x}) - g_1(t,\mathbf{x})| \\ 
& \leq |\mathrm{s}_3(t,\mathbf{x}) - g_3(t,\mathbf{x})| + |g_1(t,\mathbf{x}) - g_3(t,\mathbf{x})| \\
& \lesssim \epsilon_0 + \epsilon\log^{\frac{d^*}{2}}\epsilon^{-1},
\end{aligned}
$$
and
$$
\begin{aligned}
\Vert \mathbf{s}_4(t,\mathbf{x}) - \mathbf{h}_1(t,\mathbf{x})\Vert & \leq \Vert \mathbf{s}_4(t,\mathbf{x}) - \mathbf{h}_3(t,\mathbf{x})\Vert + \Vert \mathbf{h}_1(t,\mathbf{x}) - \mathbf{h}_3(t,\mathbf{x})\Vert
 \\
& \lesssim \epsilon_0 + \epsilon\log^{\frac{d^*+1}{2}}\epsilon^{-1}.
\end{aligned}
$$
The term $g_1(t,\mathbf{x}) \vee \mathcal{R}^{-(2\beta + 1)}$ can be written as 
$$
g_1(t,\mathbf{x}) \vee \mathcal{R}^{-(2\beta + 1)} = \mathrm{ReLU}(g_1(t,\mathbf{x}) - \mathcal{R}^{-(2\beta + 1)}) + \mathcal{R}^{-(2\beta + 1)},
$$ 
which can be approximated by
$$
\mathrm{s}_5(t,\mathbf{x}):= \mathrm{ReLU}(\mathrm{s}_3(t,\mathbf{x}) - \mathcal{R}^{-(2\beta + 1)}) + \mathcal{R}^{-(2\beta + 1)}.
$$
Thus, we can define $\mathbf{s}_6(t,\mathbf{x})$ as
$$
[\mathbf{s}_6(t,\mathbf{x})]_i:= \mathrm{s}_{\mathrm{clip}}\left(\mathrm{s}_{\mathrm{prod},1}\left([\mathbf{s}_4(t,\mathbf{x})]_i, \mathrm{s}_{\mathrm{rec},1}(\mathrm{s}_5(t,\mathbf{x}))\right), -C_4\sqrt{\log \mathcal{R}}, C_4\sqrt{\log \mathcal{R}} \right), 
~ 1\leq i \leq d^*,
$$
to approximate $[\mathbf{h}^{\prime}(t,\mathbf{x})]_i, ~ 1\leq i \leq d^*$. Subsequently, we can define $\mathbf{s}_{\mathrm{score}}^{(1)}(t,\mathbf{x})$ as
$$
[\mathbf{s}_{\mathrm{score}}^{(1)}(t,\mathbf{x})]_i :=  \mathrm{s}_{\mathrm{clip}}\left(
\mathrm{s}_{\mathrm{prod},2}\left([\mathbf{s}_6(t,\mathbf{x})]_i, \mathrm{s}_{\mathrm{rec},2}(\mathrm{s}_{\mathrm{root}}(\sigma^2 t)) \right),
-\frac{C_4\sqrt{\log \mathcal{R}}}{\sigma_t},
\frac{C_4\sqrt{\log \mathcal{R}}}{\sigma_t}\right), ~ 1 \leq i \leq d^*,
$$
to approximate $\frac{[\mathbf{h}^{\prime}(t,\mathbf{x})]_i}{\sigma_t}, 1\leq i \leq d^*$.

Now, we analyse the error bound between $[\mathbf{s}_{\mathrm{score}}^{(1)}]_i$ and $\frac{[\mathbf{h}^{\prime}(t,\mathbf{x})]_i}{\sigma_t}, 1\leq i \leq d^*$. The approximation error can be bounded by
$$
\epsilon_{\mathbf{s}_{\mathrm{score}}^{(1)},i} \leq \epsilon_{\mathrm{prod},2} + 2C_5\max\left\{
\epsilon_{\mathbf{s}_{6},i}, \left| \mathrm{s}_{\mathrm{rec},2}(\mathrm{s}_{\mathrm{root}}(\sigma_ t)) - \frac{1}{\sigma_t}
\right|
\right\},
$$
where
$
C_5 = \max\{\sup_t\sigma_t, \sup_{t,\mathbf{x}}[\mathbf{h}^{\prime}(t,\mathbf{x})]_i\} = \mathcal{O}(\sqrt{\log \mathcal{R}})$ and $\epsilon_{\mathbf{s}_6,i} = \sup_{t,\mathbf{x}}\left|[\mathbf{s}_6(t,\mathbf{x})]_i - [\mathbf{h}^{\prime}(t,\mathbf{x})]_i\right|$. $\epsilon_{\mathbf{s}_{6},i}$ can be bounded as follows:
$$
\epsilon_{\mathbf{s}_{6},i} \lesssim \epsilon_{\mathrm{prod},1} + 2C_6\max \left\{\epsilon_0 + \epsilon\log^{\frac{d^*+1}{2}}\epsilon^{-1}, \epsilon_{\mathrm{rec},1} + \frac{\epsilon_0 + \epsilon\log^{\frac{d^*}{2}}\epsilon^{-1}}{\epsilon_{\mathrm{rec},1}^2} \right\},
$$
where $C_6 \lesssim \max \left\{ 
\sup_{t,\mathbf{x}}|[\mathbf{h}_1(t,\mathbf{x})]_i|, \sup_{t,\mathbf{x}}|g_1(t,\mathbf{x}) \vee \mathcal{R}^{-(2\beta + 1)}|
\right\} = \mathcal{O}(1)$. Similarly, we have
$$
\left| \mathrm{s}_{\mathrm{rec},2}(\mathrm{s}_{\mathrm{root}}(\sigma^2 t)) - \frac{1}{\sigma_t}
\right| \leq \epsilon_{\mathrm{rec},2} + \frac{\sigma \epsilon_{\mathrm{root}}}{\epsilon_{\mathrm{rec},2}^2}.
$$
By taking
$$
\epsilon = \mathcal{R}^{-(3\beta + 2)}, ~ \epsilon_{\mathrm{prod},1} = \frac{\mathcal{R}^{-\beta}}{8C_5}, ~ \epsilon_{\mathrm{prod},2} = \frac{\mathcal{R}^{-\beta}}{2}, ~ \epsilon_{\mathrm{rec},1} = \frac{\mathcal{R}^{-\beta}}{32C_5 C_6}, ~ \epsilon_{\mathrm{rec},2} = \frac{\mathcal{R}^{-\beta}}{8C_5},
$$
and
$$
\epsilon_0 = \min\left\{
\frac{\mathcal{R}^{-\beta}}{16C_5 C_6}, \left(\frac{\mathcal{R}^{-\beta}}{32C_5C_6}\right)^3
\right\}, ~ \epsilon_{\mathrm{root}} = \frac{1}{\sigma} \cdot  \left(\frac{\mathcal{R}^{-\beta}}{8C_5}\right)^3,
$$
we obtain
$$
\epsilon_{\mathbf{s}_{6},i} \lesssim \frac{\mathcal{R}^{-\beta}}{C_5}, ~ \epsilon_{\mathbf{s}_{\mathrm{score}}^{(1)},i} \lesssim \mathcal{R}^{-\beta}.
$$
Subsequently, we can obtain the network parameters of $[\mathbf{s}_{\mathrm{score}}^{(1)}]_i$, $1\leq i \leq d^*$:
$$
L_i = \mathcal{O}(\log^4 \mathcal{R}), M_i = \mathcal{O}(\mathcal{R}^{d^*}\log^7 \mathcal{R}), J_i = \mathcal{O}(\mathcal{R}^{d^*}\log^9 \mathcal{R}), \kappa_i = \exp\left(\mathcal{O}(\log^4 \mathcal{R})\right).
$$
Combining $[\mathbf{s}_{\mathrm{score}}^{(1)}]_i$, $1\leq i \leq d^*$, and by Lemma \ref{lem:parallelization}, we construct the ReLU neural network 
$$
\mathbf{s}_{\mathrm{score}}^{(1)}:=\left[ [\mathbf{s}_{\mathrm{score}}^{(1)}]_1, [\mathbf{s}_{\mathrm{score}}^{(1)}]_2, \cdots, [\mathbf{s}_{\mathrm{score}}^{(1)}]_{d^*} \right]^{\top},
$$
with network parameters
$$
L = \mathcal{O}(\log^4 \mathcal{R}), M = \mathcal{O}(\mathcal{R}^{d^*}\log^7 \mathcal{R}), J = \mathcal{O}(\mathcal{R}^{d^*}\log^9 \mathcal{R}), \kappa = \exp\left(\mathcal{O}(\log^4 \mathcal{R})\right),
$$
and it satisfies $\Vert\mathbf{s}_{\mathrm{score}}^{(1)}(t,\cdot)\Vert_{\infty} \lesssim \frac{\sqrt{\log \mathcal{R}}}{\sigma_t}$. Thus, the approximation error $(A)$ is bounded by
\begin{equation}\label{eq:appendix_boundA}
(A) \lesssim d^*\mathcal{R}^{-2\beta} \lesssim \mathcal{R}^{-2\beta}.
\end{equation}

Next, we bound the approximation error $(B)$. Recall that $\Vert\mathbf{x}\Vert_{\infty} \leq 1 + C\sigma_t\sqrt{(2\beta + 1)\log \mathcal{R}} \leq 1 + C\sigma\sqrt{(2\beta + 1)\log \mathcal{R}}$ and $q_t(\mathbf{x}) \geq \mathcal{R}^{-(2\beta + 1)}$ hold. In this case, we have $\Vert \nabla\log q_t(\mathbf{x})\Vert \lesssim \frac{\sqrt{\log \mathcal{R}}}{\sigma_t}.$ We first consider the case $\mathbf{x}\in [-1,1]^{d^*}$. For $1 \leq i \leq d^*$, we have
$$
\begin{aligned}
& ~~~ \left|
\frac{[\mathbf{h}^{\prime}(t,\mathbf{x})]_i}{\sigma_t} - [\nabla\log q_t(\mathbf{x})]_i
\right| \\
&= \left|\frac{[\mathbf{h}^{\prime}(t,\mathbf{x})]_i}{\sigma_t}\right|
\left|1- \frac{[\nabla\log q_t(\mathbf{x})]_i}{[\mathbf{h}^{\prime}(t,\mathbf{x})]_i/\sigma_t}
\right| \\
& \lesssim \frac{\sqrt{\log \mathcal{R}}}{\sigma_t} \frac{|[\mathbf{h}^{\prime}(t,\mathbf{x})]_i - \sigma_t[\nabla\log q_t(\mathbf{x})]_i|}{|[\mathbf{h}^{\prime}(t,\mathbf{x})]_i|} \\
& \lesssim \frac{\sqrt{\log \mathcal{R}}}{\sigma_t} \frac{
\left| \frac{[\mathbf{h}_1(t,\mathbf{x})]_i}{g_1(t,\mathbf{x}) \vee \mathcal{R}^{-(2\beta + 1)}} - \frac{\sigma_t[\nabla q_t(\mathbf{x})]_i}{q_t(\mathbf{x})}\right|
}{\left|\frac{[\mathbf{h}_1(t,\mathbf{x})]_i}{g_1(t,\mathbf{x}) \vee \mathcal{R}^{-(2\beta + 1)}} \right|} \\
& \lesssim \frac{\sqrt{\log \mathcal{R}}}{\sigma_t} 
\left| 
\frac{
[\mathbf{h}_1(t,\mathbf{x})]_i q_t(\mathbf{x}) - \sigma_t[\nabla q_t(\mathbf{x})]_i (g_1(t,\mathbf{x}) \vee \mathcal{R}^{-(2\beta + 1)})
}{[\mathbf{h}_1(t,\mathbf{x})]_i q_t(\mathbf{x})}
\right| \\
& \lesssim \frac{\sqrt{\log \mathcal{R}}}{\sigma_t} 
\frac{
|[\mathbf{h}_1(t,\mathbf{x})]_i| |q_t(\mathbf{x}) - g_1(t,\mathbf{x}) \vee \mathcal{R}^{(-2\beta + 1)}| + g_1(t,\mathbf{x}) \vee \mathcal{R}^{-(2\beta + 1)} |[\mathbf{h}_1(t,\mathbf{x})]_i - \sigma_t[\nabla q_t(\mathbf{x})]_i|
}{|[\mathbf{h}_1(t,\mathbf{x})]_i| q_t(\mathbf{x})}.
\end{aligned}
$$
When $\mathbf{x}\in [-1,1]^{d^*}$, by Lemma \ref{lem:bound_for_density}, we have $q_t(\mathbf{x}) \gtrsim \mathcal{O}(1)$. Therefore,
$$
\left|
\frac{[\mathbf{h}^{\prime}(t,\mathbf{x})]_i}{\sigma_t} - [\nabla\log q_t(\mathbf{x})]_i
\right| \lesssim \frac{\sqrt{\log \mathcal{R}}}{\sigma_t} |q_t(\mathbf{x}) - g_1(t,\mathbf{x}) \vee \mathcal{R}^{-(2\beta + 1)}| + \frac{1}{\sigma_t}\cdot |[\mathbf{h}_1(t,\mathbf{x})]_i - \sigma_t[\nabla q_t(\mathbf{x})]_i|,
$$
which implies that
\begin{equation} \label{eq:appendix_eq2}
\left\Vert 
\frac{\mathbf{h}^{\prime}(t,\mathbf{x})}{\sigma_t} - \nabla\log q_t(\mathbf{x})
\right\Vert \lesssim \frac{\sqrt{\log \mathcal{R}}}{\sigma_t} \left(
| q_t(\mathbf{x}) - g_1(t,\mathbf{x})| + \Vert \mathbf{h}_1(t,\mathbf{x}) - \sigma_t \nabla q_t(\mathbf{x})\Vert
\right).
\end{equation}
Here,  we use $q_t(\mathbf{x}) \geq \mathcal{R}^{-(2\beta + 1)}$ and $|q_t(\mathbf{x}) \vee \mathcal{R}^{-(2\beta + 1)} - g_1(t,\mathbf{x}) \vee \mathcal{R}^{-(2\beta + 1)}| \leq |q_t(\mathbf{x}) - g_1(t,\mathbf{x})|$.

Now, we consider the case that $1 < \Vert\mathbf{x}\Vert_{\infty} \leq 1 + C\sigma_t\sqrt{(2\beta + 1)\log \mathcal{R}}$. Then, for $1 \leq i \leq d^*$, we have
$$
\begin{aligned}
&\left| 
\frac{[\mathbf{h}^{\prime}(t,\mathbf{x})]_i}{\sigma_t} - [\nabla\log q_t(\mathbf{x})]_i
\right|\\
&\leq \left| 
\frac{[\mathbf{h}(t,\mathbf{x})]_i}{\sigma_t(g_3(t,\mathbf{x}) \vee \mathcal{R}^{-(2\beta + 1)})} - [\nabla\log q_t(\mathbf{x})]_i
\right| \\
& \lesssim \frac{\left|
[\mathbf{h}_1(t,\mathbf{x})]_i - \sigma_t [\nabla q_t(\mathbf{x})]_i
\right|}{\sigma_t(g_1t,\mathbf{x}) \vee \mathcal{R}^{-(2\beta + 1)})}  + |[\nabla q_t(\mathbf{x})]_i|\cdot \left|\frac{1}{g_1(t,\mathbf{x}) \vee \mathcal{R}^{-(2\beta + 1)}} - \frac{1}{q_t(\mathbf{x})}\right|.
\end{aligned}
$$
On the right-hand side of the above inequality, the first term is bounded by
$$
\frac{\left|
[\mathbf{h}_1(t,\mathbf{x})]_i - \sigma_t [\nabla q_t(\mathbf{x})]_i
\right|}{\sigma_t(g_1(t,\mathbf{x}) \vee \mathcal{R}^{-(2\beta + 1)})} \lesssim \frac{\left|
[\mathbf{h}_1(t,\mathbf{x})]_i - \sigma_t [\nabla q_t(\mathbf{x})]_i
\right|}{\sigma_t \mathcal{R}^{-(2\beta + 1)}}.
$$
And the second term can be bounded by
$$
\begin{aligned}
|[\nabla q_t(\mathbf{x})]_i|\cdot \left|\frac{1}{g_1(t,\mathbf{x}) \vee \mathcal{R}^{-(2\beta + 1)}} - \frac{1}{q_t(\mathbf{x})}\right| &\lesssim \frac{\sqrt{\log \mathcal{R}}}{\sigma_t} \cdot \frac{|q_t(\mathbf{x}) - g_1(t,\mathbf{x}) \vee \mathcal{R}^{-(2\beta + 1)}|}{g_1(t,\mathbf{x}) \vee \mathcal{R}^{-(2\beta + 1)}} \\
& \lesssim \frac{\mathcal{R}^{2\beta + 1}\sqrt{\log \mathcal{R}}}{\sigma_t} |q_t(\mathbf{x}) - g_1(t,\mathbf{x})|.
\end{aligned}
$$
Therefore, we have
\begin{equation} \label{eq:appendix_eq3}
\left\Vert
\frac{\mathbf{h}^{\prime}(t,\mathbf{x})}{\sigma_t} - \nabla\log q_t(\mathbf{x})
\right\Vert \lesssim \frac{\mathcal{R}^{2\beta + 1}\sqrt{\log \mathcal{R}}}{\sigma_t} \left(
|q_t(\mathbf{x}) - g_1(t,\mathbf{x})| + \Vert \mathbf{h}_1(t,\mathbf{x}) - \sigma_t \nabla q_t(\mathbf{x})\Vert
\right).
\end{equation}

We bound the approximation error $(B)$. We first consider the case when $\Vert \mathbf{x}\Vert_{\infty} \leq 1$, then according to \eqref{eq:appendix_eq2}, we have
$$
\begin{aligned}
& ~~~ \int_{\Vert\mathbf{x}\Vert_{\infty} \leq 1} q_t(\mathbf{x}) \mathbf{I}_{\{q_t(\mathbf{x}) \geq \mathcal{R}^{-(2\beta + 1)}\}} \left\Vert
\frac{\mathbf{h}^{\prime}(t,\mathbf{x})}{\sigma_t} - \nabla\log q_t(\mathbf{x})
\right\Vert^2 \mathrm{d}\mathbf{x} \\
& \lesssim \int_{\Vert\mathbf{x}\Vert_{\infty} \leq 1} \frac{\log \mathcal{R}}{\sigma_t^2} \left(
| q_t(\mathbf{x}) - g_1(t,\mathbf{x})|^2 + \Vert \mathbf{h}_1(t,\mathbf{x}) - \sigma_t \nabla q_t(\mathbf{x})\Vert^2
\right) \mathrm{d}\mathbf{x} \\
& \lesssim \frac{\log \mathcal{R}}{\sigma_t^2} \int_{\Vert \mathbf{x} \Vert_{\infty} \leq 1} \Bigg(\left|
\int_{\mathbb{R}^{d^*}}\frac{1}{\sigma_t^{d^*}(2\pi)^{d^*/2}} [\widehat{p}_{data}^*(\mathbf{y})-p_{\mathcal{R}}(\mathbf{y})]\exp\left(-\frac{\Vert\mathbf{x}-\mathbf{y}\Vert^2}{2\sigma_t^2}\right)\mathrm{d}\mathbf{y} \right|^2 \\
& ~~~~~~~~~ + \left\Vert \int_{\mathbb{R}^{d^*}}\frac{\mathbf{x}-\mathbf{y}}{\sigma_t^{d^*+1}(2\pi)^{d^*/2}} [\widehat{p}_{data}^*(\mathbf{y})-p_{\mathcal{R}}(\mathbf{y})]\exp\left(-\frac{\Vert\mathbf{x}-\mathbf{y}\Vert^2}{2\sigma_t^2}\right)\mathrm{d}\mathbf{y} \right\Vert^2 \Bigg) \mathrm{d}\mathbf{x} \\
& \lesssim \frac{\log \mathcal{R}}{\sigma_t^2} \int_{\Vert \mathbf{x} \Vert_{\infty} \leq 1} \Bigg(
\int_{\mathbb{R}^{d^*}}\frac{1}{\sigma_t^{d^*}(2\pi)^{d^*/2}} |\widehat{p}_{data}^*(\mathbf{y})-p_{\mathcal{R}}(\mathbf{y})|^2\exp\left(-\frac{\Vert\mathbf{x}-\mathbf{y}\Vert^2}{2\sigma_t^2}\right)\mathrm{d}\mathbf{y} \\
& ~~~~~~~~~ +  \int_{\mathbb{R}^{d^*}}\frac{\Vert\mathbf{x}-\mathbf{y}\Vert^2}{\sigma_t^{d^*+1}(2\pi)^{d^*/2}} |\widehat{p}_{data}^*(\mathbf{y})-p_{\mathcal{R}}(\mathbf{y})|^2\exp\left(-\frac{\Vert\mathbf{x}-\mathbf{y}\Vert^2}{2\sigma_t^2}\right)\mathrm{d}\mathbf{y}  \Bigg) \mathrm{d}\mathbf{x}\\
& \lesssim \frac{\log \mathcal{R}}{\sigma_t^2} \int_{\mathbb{R}^{d^*}}|\widehat{p}_{data}^*(\mathbf{y}) - p_{\mathcal{R}}(\mathbf{y})|^2 \mathrm{d}\mathbf{y} \lesssim \frac{\mathcal{R}^{-2\beta}\log \mathcal{R}}{\sigma_t^2}.
\end{aligned}
$$
We then consider the case of $1 < \Vert\mathbf{x}\Vert_{\infty} \leq 1 + C\sigma_t\sqrt{(2\beta + 1)\log \mathcal{R}}$. Then, we have
$$
\begin{aligned}
& ~~~ \int_{1 < \Vert\mathbf{x}\Vert_{\infty} \leq 1 + C\sigma_t\sqrt{(2\beta + 1)\log \mathcal{R}}} q_t(\mathbf{x}) \mathbf{I}_{\{q_t(\mathbf{x}) \geq \mathcal{R}^{-(2\beta + 1)}\}} \left\Vert
\frac{\mathbf{h}^{\prime}(t,\mathbf{x})}{\sigma_t} - \nabla\log q_t(\mathbf{x})
\right\Vert^2 \mathrm{d}\mathbf{x} \\
& \lesssim \int_{1 < \Vert\mathbf{x}\Vert_{\infty} \leq 1 + C\sigma_t\sqrt{(2\beta + 1)\log \mathcal{R}}} \frac{\mathcal{R}^{4\beta + 2}\log \mathcal{R}}{\sigma_t^2} \left(
| q_t(\mathbf{x}) - g_1(t,\mathbf{x})|^2 + \Vert \mathbf{h}_1(t,\mathbf{x}) - \sigma_t \nabla q_t(\mathbf{x})\Vert^2
\right) \mathrm{d}\mathbf{x} \\
& \lesssim \frac{\mathcal{R}^{4\beta + 2}\log \mathcal{R}}{\sigma_t^2} \int_{1 < \Vert \mathbf{x} \Vert_{\infty} \leq 1 + C\sigma_t\sqrt{(2\beta + 1)\log \mathcal{R}}} \Bigg(\left|
\int_{\mathbb{R}^{d^*}}\frac{1}{\sigma_t^{d^*}(2\pi)^{d^*/2}} [\widehat{p}_{data}^*(\mathbf{y})-p_{\mathcal{R}}(\mathbf{y})]\exp\left(-\frac{\Vert\mathbf{x}-\mathbf{y}\Vert^2}{2\sigma_t^2}\right)\mathrm{d}\mathbf{y} \right|^2 \\
& ~~~~~~~~~ + \left\Vert \int_{\mathbb{R}^{d^*}}\frac{\mathbf{x}-\mathbf{y}}{\sigma_t^{d^*+1}(2\pi)^{d^*/2}} [\widehat{p}_{data}^*(\mathbf{y})-p_{\mathcal{R}}(\mathbf{y})]\exp\left(-\frac{\Vert\mathbf{x}-\mathbf{y}\Vert^2}{2\sigma_t^2}\right)\mathrm{d}\mathbf{y} \right\Vert^2 \Bigg) \mathrm{d}\mathbf{x} \\
& \lesssim \frac{\mathcal{R}^{4\beta + 2}\log \mathcal{R}}{\sigma_t^2} \int_{1 < \Vert \mathbf{x} \Vert_{\infty} \leq 1 + C\sigma_t\sqrt{(2\beta + 1)\log \mathcal{R}}} \Bigg(
\int_{\mathbb{R}^{d^*}}\frac{1}{\sigma_t^{d^*}(2\pi)^{d^*/2}} |\widehat{p}_{data}^*(\mathbf{y})-p_{\mathcal{R}}(\mathbf{y})|^2\exp\left(-\frac{\Vert\mathbf{x}-\mathbf{y}\Vert^2}{2\sigma_t^2}\right)\mathrm{d}\mathbf{y} \\
& ~~~~~~~~~ +  \int_{\mathbb{R}^{d^*}}\frac{\Vert\mathbf{x}-\mathbf{y}\Vert^2}{\sigma_t^{d^*+1}(2\pi)^{d^*/2}} |\widehat{p}_{data}^*(\mathbf{y})-p_{\mathcal{R}}(\mathbf{y})|^2\exp\left(-\frac{\Vert\mathbf{x}-\mathbf{y}\Vert^2}{2\sigma_t^2}\right)\mathrm{d}\mathbf{y}  \Bigg) \mathrm{d}\mathbf{x} 
\end{aligned}
$$
Taking $\epsilon = \mathcal{R}^{-(6\beta + 2)}$ in Lemma \ref{lem:integral_clipping} and replacing $\widehat{p}_{data}^*$ with $|\widehat{p}_{data}^* - p_{\mathcal{R}}|^2$, we have
$$
\begin{aligned}
& \int_{1 < \Vert \mathbf{x} \Vert_{\infty} \leq 1 + C\sigma_t\sqrt{(2\beta + 1)\log \mathcal{R}}}
\int_{\mathbb{R}^{d^*}}\frac{|\widehat{p}_{data}^*(\mathbf{y})-p_{\mathcal{R}}(\mathbf{y})|^2}{\sigma_t^{d^*}(2\pi)^{d^*/2}}\exp\left(-\frac{\Vert\mathbf{x}-\mathbf{y}\Vert^2}{2\sigma_t^2}\right)\mathrm{d}\mathbf{y} \mathrm{d}\mathbf{x} \\
\lesssim &\int_{1 < \Vert \mathbf{x} \Vert_{\infty} \leq 1 + C\sigma_t\sqrt{(2\beta + 1)\log \mathcal{R}}}
\int_{\Vert\mathbf{x} - \mathbf{y}\Vert_{\infty} \leq \mathcal{O}(1)\sigma_t\sqrt{\log \mathcal{R}}}\frac{|\widehat{p}_{data}^*(\mathbf{y})-p_{\mathcal{R}}(\mathbf{y})|^2}{\sigma_t^{d^*}(2\pi)^{d^*/2}} \exp\left(-\frac{\Vert\mathbf{x}-\mathbf{y}\Vert^2}{2\sigma_t^2}\right)\mathrm{d}\mathbf{y}\mathrm{d}\mathbf{x} \\
& ~~~~~~~~ + \mathcal{R}^{-(6\beta + 2)},
\end{aligned}
$$
and
$$
\begin{aligned}
& \int_{1 < \Vert \mathbf{x} \Vert_{\infty} \leq 1 + C\sigma_t\sqrt{(2\beta + 1)\log \mathcal{R}}}
\int_{\mathbb{R}^{d^*}}\frac{\Vert\mathbf{x} - \mathbf{y}\Vert^2 |\widehat{p}_{data}^*(\mathbf{y})-p_{\mathcal{R}}(\mathbf{y})|^2}{\sigma_t^{d^*}(2\pi)^{d^*/2}}\exp\left(-\frac{\Vert\mathbf{x}-\mathbf{y}\Vert^2}{2\sigma_t^2}\right)\mathrm{d}\mathbf{y}\mathrm{d}\mathbf{x} \\
\lesssim &\int_{1 < \Vert \mathbf{x} \Vert_{\infty} \leq 1 + C\sigma_t\sqrt{(2\beta + 1)\log \mathcal{R}}}
\int_{\Vert\mathbf{x} - \mathbf{y}\Vert_{\infty} \leq \mathcal{O}(1)\sigma_t\sqrt{\log \mathcal{R}}}\frac{\log \mathcal{R} |\widehat{p}_{data}^*(\mathbf{y})-p_{\mathcal{R}}(\mathbf{y})|^2}{\sigma_t^{d^*}(2\pi)^{d^*/2}} \exp\left(-\frac{\Vert\mathbf{x}-\mathbf{y}\Vert^2}{2\sigma_t^2}\right)\mathrm{d}\mathbf{y}\mathrm{d}\mathbf{x} \\
& ~~~~~~~~~~~ + \mathcal{R}^{-(6\beta + 2)},
\end{aligned}
$$
where we use 
$\Vert\mathbf{x} - \mathbf{y}\Vert_{\infty}^2 \lesssim \sigma_t^2\log \mathcal{R} \lesssim \log \mathcal{R}$, $t\in [\mathcal{R}^{-C_T}, 3\mathcal{R}^{-C_T}]$ and $1 + C\sigma_t\sqrt{(2\beta + 1)\log \mathcal{R}} = \mathcal{O}(1)$ for sufficiently large $\mathcal{R}$. When $1 < \Vert\mathbf{x}\Vert_{\infty} \leq 1 + C\sigma_t\sqrt{(2\beta + 1)\log \mathcal{R}} = \mathcal{O}(1)$ and $\Vert\mathbf{x}-\mathbf{y}\Vert_{\infty} \leq \mathcal{O}(1)\sigma_t\sqrt{\log \mathcal{R}}$, then $1-\mathcal{O}(1)\sigma_t\sqrt{\log \mathcal{R}} \leq \Vert\mathbf{y}\Vert_{\infty} \leq 1$. Since $t\in [\mathcal{R}^{-C_T}, 3\mathcal{R}^{-C_T}]$, $\mathcal{O}(1)\sigma_t\sqrt{\log \mathcal{R}} \leq a_0$ holds for sufficiently large $\mathcal{R}$.
Therefore, by Lemma \ref{lem:approximate_pdata2}, it holds that
$$
\begin{aligned}
& ~~~ \int_{1 < \Vert \mathbf{x} \Vert_{\infty} \leq 1 + C\sigma_t\sqrt{(2\beta + 1)\log \mathcal{R}}}
\int_{\Vert\mathbf{x} - \mathbf{y}\Vert_{\infty} \leq \mathcal{O}(1)\sigma_t\sqrt{\log \mathcal{R}}}\frac{|\widehat{p}_{data}^*(\mathbf{y})-p_{\mathcal{R}}(\mathbf{y})|^2}{\sigma_t^{d^*}(2\pi)^{d^*/2}} \exp\left(-\frac{\Vert\mathbf{x}-\mathbf{y}\Vert^2}{2\sigma_t^2}\right)\mathrm{d}\mathbf{y}\mathrm{d}\mathbf{x} \\
& \lesssim \int_{1 < \Vert \mathbf{x} \Vert_{\infty} \leq 1 + C\sigma_t\sqrt{(2\beta + 1)\log \mathcal{R}}}
\int_{1-a_0 < \Vert \mathbf{y}\Vert_{\infty} \leq 1}\frac{|\widehat{p}_{data}^*(\mathbf{y})-p_{\mathcal{R}}(\mathbf{y})|^2}{\sigma_t^{d^*}(2\pi)^{d^*/2}} \exp\left(-\frac{\Vert\mathbf{x}-\mathbf{y}\Vert^2}{2\sigma_t^2}\right)\mathrm{d}\mathbf{y}\mathrm{d}\mathbf{x} \\
& \lesssim \int_{1-a_0<\Vert\mathbf{y}\Vert_\infty \leq 1}|\widehat{p}_{data}^*(\mathbf{y}) - p_{\mathcal{R}}(\mathbf{y})|^2 \mathrm{d}\mathbf{y} \lesssim \mathcal{R}^{-(6\beta + 4)}.
\end{aligned}
$$
Thus, we have
\begin{equation}\label{eq:appendix_boundB}
\begin{aligned}
& ~~~ \int_{1 < \Vert\mathbf{x}\Vert_{\infty} \leq 1 + C\sigma_t\sqrt{(2\beta + 1)\log \mathcal{R}}} q_t(\mathbf{x}) \mathbf{I}_{\{q_t(\mathbf{x}) \geq \mathcal{R}^{-(2\beta + 1)}\}} \left\Vert
\frac{\mathbf{h}^{\prime}(t,\mathbf{x})}{\sigma_t} - \nabla\log q_t(\mathbf{x})
\right\Vert^2 \mathrm{d}\mathbf{x} \\
& \lesssim  \frac{\mathcal{R}^{4\beta + 2}\log^2 \mathcal{R}}{\sigma_t^2} \cdot \mathcal{R}^{-(6\beta + 4)} + \frac{\mathcal{R}^{4\beta + 2}\log \mathcal{R}}{\sigma_t^2} \cdot \mathcal{R}^{-(6\beta + 2)} \\
& \lesssim \frac{\mathcal{R}^{-2\beta}\log \mathcal{R}}{\sigma_t^2}.
\end{aligned}
\end{equation}
Combining \eqref{eq:appendix_boundA} and \eqref{eq:appendix_boundB}, we finally obtain (III) $\lesssim \frac{\mathcal{R}^{-2\beta}\log \mathcal{R}}{\sigma_t^2}$.
The proof is complete.
\end{proof}

Next, we approximate $\nabla\log q_t(\mathbf{x})$ via ReLU neural networks on $[2\mathcal{R}^{-C_T}, 1]$. For any $2\mathcal{R}^{-C_T} \leq t_0 < t \leq 1$, by the Markov property, we have
$$
q_t(\mathbf{x}) = \frac{1}{(2\pi)^{d^*/2}\sigma_{t-t_0}^{d^*}}\int_{\mathbb{R}^{d^*}} q_{t_0}(\mathbf{y}) \exp\left(-\frac{\Vert\mathbf{x} - \mathbf{y}\Vert^2}{2\sigma_{t-t_0}^2}\right) \mathrm{d}\mathbf{y}.
$$
We first approximate $q_{t_0}$ and give the following lemma.

\begin{lemma}\label{lem:approximate_pt0}
Let $\mathcal{R} \gg 1$. For any $t_0 \in [2\mathcal{R}^{-C_T}, 1]$, there exists a constant $C_7 > 0$ and a function $p_{\mathcal{R}}$ such that
$$
\left(\int_{\mathbb{R}^{d^*}} \left|q_{t_0}(\mathbf{x}) - p_{\mathcal{R}}(\mathbf{x}) \right|^2 \mathrm{d}\mathbf{x} \right)^{\frac{1}{2}} \lesssim \mathcal{R}^{-(3\beta + 1)},
$$
where $p_{\mathcal{R}}(\mathbf{x})$ has the following form:
$$
p_{\mathcal{R}}(\mathbf{x}) = t_0^{-\frac{k}{2}} \sum_{\mathbf{n}\in[\mathcal{R}]^{d^*}}\sum_{\Vert\boldsymbol{\alpha}\Vert_1\leq k}c_{\mathbf{n},\boldsymbol{\alpha}}p_{\mathbf{n},\boldsymbol{\alpha}}\left(\frac{\mathbf{x}}{2(1+C_7\sigma\sqrt{\log \mathcal{R}})} 
 + \frac{1}{2} \right) \mathbf{I}_{\{\Vert \mathbf{x} \Vert_\infty \leq 1 + C_7\sigma\sqrt{\log \mathcal{R}}\}},
$$
with $k = \lfloor \frac{C_T}{2} + (3\beta + 2)\rfloor + 1$.
\end{lemma}

\begin{proof}
Let $k \geq 1$. According to Lemma \ref{lem:derivatives_boundness}, for any $\mathbf{x}\in\mathbb{R}^{d^*}$, we have
$$
\Vert \partial_{x_{i_1}x_{i_2}\cdots x_{i_k}} q_{t_0} \Vert \lesssim \frac{1}{\sigma_{t_0}^k} \lesssim t_0^{-\frac{k}{2}}.
$$
Thus $q_{t_0} t_0^{\frac{k}{2}}$ is $k$-H\"older continuous on $\mathbb{R}^{d^*}$ with some constant $C_k$. By replacing $q_{t_0}$ with $q_{t_0}^2$ in Lemma \ref{lem:bound_for_clipping}, we can claim that there exists a constant $C_7$ such that $\left(\int_{\Vert\mathbf{x}\Vert_\infty > 1 + C_7\sigma\sqrt{\log \mathcal{R}}}q_{t_0}^2(\mathbf{x}) \mathrm{d}\mathbf{x} \right)^{\frac{1}{2}} \lesssim \mathcal{R}^{-(3\beta + 1)}$. Similar to the proof of Lemma \ref{lem:approximate_pdata1}, we let $$
f(\mathbf{x}) = t_0^{\frac{k}{2}}q_{t_0}\left((1 + C_7\sigma\sqrt{\log \mathcal{R}})(2\mathbf{x} -1) \right), ~ \mathbf{x} \in [0,1]^{d^*}.
$$ 
Then, $\Vert f \Vert_{\mathcal{H}([0,1]^{d^*})} \lesssim 2^k(1 + C_7\sigma\sqrt{\log \mathcal{R}})^k = \mathcal{O}(\log^{\frac{k}{2}} N)$. Thus, there exists $p_{\mathcal{R},0}$ such that
$$
\left| p_{\mathcal{R},0}(\mathbf{x}) - t_0^{\frac{k}{2}}q_{t_0}(\mathbf{x}) \right| \lesssim \mathcal{R}^{-k}\log^{\frac{k}{2}}\mathcal{R}, ~ \Vert\mathbf{x}\Vert_\infty \leq 1 + C_7\sigma\sqrt{\log \mathcal{R}},
$$
where 
$
p_{\mathcal{R},0}(\mathbf{x}) = \sum_{\mathbf{n}\in[\mathcal{R}]^{d^*}}\sum_{\Vert\boldsymbol{\alpha}\Vert_1\leq k}c_{\mathbf{n},\boldsymbol{\alpha}}p_{\mathbf{n},\boldsymbol{\alpha}}\left(\frac{\mathbf{x}}{2(1+C_7\sigma\sqrt{\log \mathcal{R}})} 
 + \frac{1}{2} \right) \mathbf{I}_{\{\Vert \mathbf{x} \Vert_\infty \leq 1 + C_7\sigma\sqrt{\log \mathcal{R}}\}}.
$
Next, we let $p_{\mathcal{R}}(\mathbf{x}) = t_0^{-\frac{k}{2}}p_{\mathcal{R},0}(\mathbf{x})$, then
$$
\left|p_{\mathcal{R}}(\mathbf{x}) - q_{t_0}(\mathbf{x})\right| \lesssim \mathcal{R}^{-\left(k-\frac{C_T}{2}\right)}\log^{\frac{k}{2}}\mathcal{R}.
$$
Let $k = \lfloor\frac{C_T}{2} + (3\beta + 2)\rfloor + 1 \geq \frac{C_T}{2} + 3\beta + 2$, then 
$$
\begin{aligned}
\left(\int_{\Vert\mathbf{x}\Vert_\infty \leq 1+C_7\sigma\sqrt{\log \mathcal{R}}}\left|p_{\mathcal{R}}(\mathbf{x}) - q_{t_0}(\mathbf{x})\right|^2 \mathrm{d}\mathbf{x} \right) ^{\frac{1}{2}} & \lesssim {\mathcal{R}}^{-\left(k - \frac{C_T}{2}\right)}\log^{\frac{k+d^*}{2}} \mathcal{R} \\
& \lesssim \mathcal{R}^{-(3\beta + 1)}
\end{aligned}
$$
for sufficient large $\mathcal{R}$. Subsequently, $p_{\mathcal{R}}$ satisfies that
$$
\begin{aligned}
&\left(\int_{\mathbb{R}^{d^*}} \left| q_{t_0}(\mathbf{x}) - p_{\mathcal{R}}(\mathbf{x}) \right|^2 \mathrm{d}\mathbf{x} \right)^{\frac{1}{2}}\\
& = \left(\int_{\Vert\mathbf{x}\Vert_\infty \leq 1+C_7\sigma\sqrt{\log \mathcal{R}}}\left|p_{\mathcal{R}}(\mathbf{x}) - q_{t_0}(\mathbf{x})\right|^2 \mathrm{d}\mathbf{x} + \int_{\Vert\mathbf{x}\Vert_\infty > 1 + C_7\sigma\sqrt{\log \mathcal{R}}}q_{t_0}^2(\mathbf{x}) \mathrm{d}\mathbf{x} \right)^{\frac{1}{2}} \\
&\lesssim {\mathcal{R}}^{-(3\beta + 1)}.
\end{aligned}
$$
The proof is complete.
\end{proof}

Based on Lemma \ref{lem:approximate_pt0}, we can construct a ReLU neural network to approximate $\nabla\log q_t(\mathbf{x})$ on $[2\mathcal{R}^{-C_T}, 1]$.

\begin{lemma} \label{lem:approximate_interval_2}
Let $N \gg 1$ and $t_0 = 2{\mathcal{R}}^{-C_T}$. There exists a ReLU neural network $\mathbf{s}_{\mathrm{score}}^{(2)}\in\mathrm{NN}(L,M,J,\kappa)$ with
$$
L = \mathcal{O}(\log^4 \mathcal{R}), M = \mathcal{O}(\mathcal{R}^{d^*}\log^7 \mathcal{R}), J = \mathcal{O}({\mathcal{R}}^{d^*}\log^9 \mathcal{R}), \kappa = \exp\left(\mathcal{O}(\log^4 \mathcal{R})\right)
$$
that satisfies
$$
\int_{\mathbb{R}^{d^*}} q_t(\mathbf{x})\Vert \mathbf{s}_{\mathrm{score}}^{(2)}(t,\mathbf{x}) - \nabla\log q_t(\mathbf{x})\Vert^2 \mathrm{d}\mathbf{x} \lesssim \frac{\mathcal{R}^{-2\beta}\log \mathcal{R}}{\sigma_t^2}, ~~ t\in [2{\mathcal{R}}^{-C_T}, 1].
$$
Moreover, we can take $\mathbf{s}_{\mathrm{score}}^{(2)}$ satisfying $\Vert \mathbf{s}_{\mathrm{score}}^{(2)} (t, \cdot)\Vert_{\infty} \lesssim \frac{\sqrt{\log \mathcal{R}}}{\sigma_t}$. 
\end{lemma}

\begin{proof}
The proof is similar to the proof of Lemma \ref{lem:approximate_interval_1}.
We also decompose the approximation error into three terms.
$$
\begin{aligned}
& \int_{\mathbb{R}^{d^*}} q_t(\mathbf{x})\Vert \mathbf{s}_{\mathrm{score}}^{(2)}(t,\mathbf{x}) - \nabla\log q_t(\mathbf{x})\Vert^2 \mathrm{d}\mathbf{x} \\
= & \underbrace{ \int_{\Vert\mathbf{x}\Vert_\infty > 1 + C\sigma_t\sqrt{\log\epsilon^{-1}_2}} q_t(\mathbf{x}) \Vert \mathbf{s}_{\mathrm{score}}^{(2)}(t,\mathbf{x}) - \nabla\log q_t(\mathbf{x})\Vert^2 \mathrm{d}\mathbf{x}. }_{(\mathrm{I})} \\
= & \underbrace{\int_{\Vert \mathbf{x}\Vert_{\infty} \leq 1 + C\sigma_t\sqrt{\log\epsilon^{-1}_2}} q_t(\mathbf{x}) \mathbf{I}_{\{q_t(\mathbf{x}) < \epsilon_2\}} \Vert \mathbf{s}_{\mathrm{score}}^{(2)}(t,\mathbf{x}) - \nabla\log q_t(\mathbf{x})\Vert^2 \mathrm{d}\mathbf{x} }_{(\mathrm{II})} \\
+ & \underbrace{\int_{\Vert \mathbf{x}\Vert_{\infty} \leq 1 + C\sigma_t\sqrt{\log\epsilon^{-1}_2}} q_t(\mathbf{x}) \mathbf{I}_{\{q_t(\mathbf{x}) \geq \epsilon_2\}}\Vert \mathbf{s}_{\mathrm{score}}^{(2)}(t,\mathbf{x}) - \nabla\log q_t(\mathbf{x})\Vert^2 \mathrm{d}\mathbf{x} }_{(\mathrm{III})}. 
\end{aligned}
$$

According to Lemma \ref{lem:bound_for_clipping}, there exists a constant $C > 0$ such that for any $\epsilon_2 > 0$, 
\begin{equation}\label{eq:appendix_eq11}
\underbrace{ \int_{\Vert\mathbf{x}\Vert_\infty > 1 + C\sigma_t\sqrt{\log\epsilon^{-1}_2}} q_t(\mathbf{x}) \Vert \mathbf{s}_{\mathrm{score}}^{(2)}(t,\mathbf{x}) - \nabla\log q_t(\mathbf{x})\Vert^2 \mathrm{d}\mathbf{x} }_{(\mathrm{I})} \lesssim \frac{\epsilon_2}{\sigma_t} + \sigma_t\epsilon_2\Vert \mathbf{s}_{\mathrm{score}}^{(2)}(t,\cdot)\Vert_{\infty}^2.
\end{equation}
Since $\Vert\nabla\log q_t(\mathbf{x})\Vert$ is bounded by $\frac{C\sqrt{\log\epsilon^{-1}_2}}{\sigma_t}$ in $\Vert\mathbf{x}\Vert_{\infty} > 1 + C\sigma_t\sqrt{\log\epsilon^{-1}_2}$ due to Lemma \ref{lem:derivatives_boundness}, $\mathbf{s}_{\mathrm{score}}^{(2)}$ can be taken so that $\Vert\mathbf{s}_{\mathrm{score}}^{(2)}(t,\cdot)\Vert_{\infty} \lesssim \frac{\sqrt{\log\epsilon^{-1}_2}}{\sigma_t}$. Therefore \eqref{eq:appendix_eq11} is bounded by $\frac{\epsilon_2\log\epsilon^{-1}_2}{\sigma_t^2}$. Taking $\epsilon_2 = {\mathcal{R}}^{-(2\beta + 1)}$, then \eqref{eq:appendix_eq11} $\lesssim \frac{{\mathcal{R}}^{-(2\beta + 1)}\log \mathcal{R}}{\sigma_t^2}$, which is smaller than $\frac{{\mathcal{R}}^{-2\beta}\log \mathcal{R}}{\sigma_t^2}$. Moreover, by Lemma \ref{lem:bound_for_clipping}, we can bound the error
$$
\begin{aligned}
& ~~~~ \underbrace{ \int_{\Vert\mathbf{x}\Vert_{\infty} \leq 1 + C\sigma_t\sqrt{\log\epsilon^{-1}_2}} q_t(\mathbf{x})\mathbf{I}_{\{q_t(\mathbf{x}) < \epsilon_2\}} \Vert 
\mathbf{s}_{\mathrm{score}}^{(2)}(t,\mathbf{x}) - \nabla\log q_t(\mathbf{x})\Vert^2 \mathrm{d}\mathbf{x} }_{(\mathrm{II})} \\
& \lesssim \frac{\epsilon_2}{\sigma_t^2} \cdot (\log\epsilon^{-1}_2)^{\frac{d^*+2}{2}} + \Vert \mathbf{s}_{\mathrm{score}}^{(2)}(t,\cdot)\Vert_{\infty}^2 \epsilon_2 \cdot (\log\epsilon^{-1}_2)^{\frac{d^*}{2}} \\
& \lesssim \frac{\epsilon_2}{\sigma_t^2} \cdot (\log\epsilon^{-1}_2)^{\frac{d^*+2}{2}} \lesssim \frac{{\mathcal{R}}^{-(2\beta + 1)}\log^{\frac{d^*+2}{2}}\mathcal{R}}{\sigma_t^2},
\end{aligned}
$$
which is smaller than $\frac{\mathcal{R}^{-2\beta}\log \mathcal{R}}{\sigma_t^2}$ for sufficiently large $\mathcal{R}$. Thus, we can only focus on bounding the term (III).  

We first take $C_7 \geq C\sqrt{2\beta + 1}$ in Lemma \ref{lem:approximate_pt0} and $C_0 = 1 + C_7\sigma\sqrt{\log \mathcal{R}} = \mathcal{O}(\sqrt{\log \mathcal{R}})$. 
Then, we replace $\frac{y + 1}{2} - \frac{n}{\mathcal{R}}$ with
$\frac{y}{2(1 + C_7\sigma\sqrt{\log \mathcal{R}})} + \frac{1}{2} - \frac{n}{\mathcal{R}}$ in $f(t,x,n,\alpha,l)$, 
and replace $\Vert\boldsymbol{\alpha}\Vert_1 < C_{\boldsymbol{\alpha}}$ and $C_{\mathbf{y}}$ with $\Vert\boldsymbol{\alpha}\Vert_1 \leq k=\mathcal{O}(1)$ and $1 + C_7\sigma\sqrt{\log \mathcal{R}}$.  
We also use the notations $g_1(t,\mathbf{x})$, $\mathbf{h}_1(t,\mathbf{x})$ and $\mathbf{h}^{\prime}(t,\mathbf{x})$, where
$$
g_1(t,\mathbf{x}) := \int_{\mathbb{R}^{d^*}}p_{\mathcal{R}}(\mathbf{y})\mathbf{I}_{\{\Vert\mathbf{y}\Vert_{\infty}\leq 1+C_7\sigma\sqrt{\log \mathcal{R}}\}}\cdot
\frac{1}{\sigma_t^{d^*}(2\pi)^{d^*/2}}\exp\left(-\frac{\Vert\mathbf{x} - \mathbf{y}\Vert^2}{2\sigma_t^2}\right) \mathrm{d}\mathbf{y},
$$
$$
\mathbf{h}_1(t,\mathbf{x}):= \int_{\mathbb{R}^{d^*}}p_{\mathcal{R}}(\mathbf{y})\mathbf{I}_{\{\Vert\mathbf{y}\Vert_{\infty}\leq 1+C_7\sigma\sqrt{\log \mathcal{R}}\}}\cdot
\frac{\mathbf{x}-\mathbf{y}}{\sigma_t^{d^*+1}(2\pi)^{d^*/2}}\exp\left(-\frac{\Vert\mathbf{x} - \mathbf{y}\Vert^2}{2\sigma_t^2}\right) \mathrm{d}\mathbf{y},
$$
$$
\mathbf{h}^{\prime}(t,\mathbf{x}):= \max\left\{\min\left\{ \frac{\mathbf{h}_1(t,\mathbf{x})}{ g_1(t,\mathbf{x}) \vee {\mathcal{R}}^{-(2\beta + 1)}}, \mathcal{O}(\sqrt{\log \mathcal{R}}) \right\}, -\mathcal{O}(\sqrt{\log \mathcal{R}})\right\}.
$$

Applying the argument of Lemma \ref{lem:approximate_interval_1}, we can construct a ReLU neural network $\mathbf{s}_{\mathrm{score}}^{(2)}$ with network parameters
$$
L = \mathcal{O}(\log^4 \mathcal{R}), M = \mathcal{O}({\mathcal{R}}^{d^*}\log^7 \mathcal{R}), J = \mathcal{O}({\mathcal{R}}^{d^*}\log^9 \mathcal{R}), \kappa = \exp\left(\mathcal{O}(\log^4 \mathcal{R})\right)
$$
that satisfies 
$$  
\left\Vert \mathbf{s}_{\mathrm{score}}^{(2)}(t,\mathbf{x}) - \frac{\mathbf{h}^{\prime}(t,\mathbf{x})}{\sigma_t} \right\Vert_\infty \lesssim {\mathcal{R}}^{-\beta}, ~ \Vert\mathbf{x}\Vert_{\infty} \leq 1 + C_7\sigma\sqrt{\log \mathcal{R}},
$$
and
$$
\Vert \mathbf{s}_{\mathrm{score}}^{(2)}(t, \cdot) \Vert_{\infty} \lesssim \frac{\sqrt{\log \mathcal{R}}}{\sigma_t}.
$$
Therefore, we have
$$
\int_{\Vert\mathbf{x}\Vert_\infty\leq 1+C_7\sigma\sqrt{\log \mathcal{R}}} q_t(\mathbf{x})\mathbf{I}_{\{q_t(\mathbf{x})\geq {\mathcal{R}}^{-(2\beta + 1)}\}} \left\Vert \mathbf{s}_{\mathrm{score}}^{(2)}(t,\mathbf{x}) - \frac{\mathbf{h}^{\prime}(t,\mathbf{x})}{\sigma_t} \right\Vert^2 \mathrm{d} \mathbf{x} \lesssim d^*{\mathcal{R}}^{-\beta} \lesssim {\mathcal{R}}^{-\beta}.
$$
Next, we consider bound the term
$$
\int_{\Vert\mathbf{x}\Vert_\infty\leq 1+C_7\sigma\sqrt{\log \mathcal{R}}} q_t(\mathbf{x})\mathbf{I}_{\{q_t(\mathbf{x})\geq {\mathcal{R}}^{-(2\beta + 1)}\}} \left\Vert \frac{\mathbf{h}^{\prime}(t,\mathbf{x})}{\sigma_t} - \nabla\log q_t(\mathbf{x})\right\Vert^2 \mathrm{d} \mathbf{x}.
$$
By \eqref{eq:appendix_eq3} in Lemma \ref{lem:approximate_interval_1},
$$
\left\Vert
\frac{\mathbf{h}^{\prime}(t,\mathbf{x})}{\sigma_t} - \nabla\log q_t(\mathbf{x})
\right\Vert \lesssim \frac{{\mathcal{R}}^{2\beta + 1}\sqrt{\log \mathcal{R}}}{\sigma_t} \left(
|q_t(\mathbf{x}) - g_1(t,\mathbf{x})| + \Vert \mathbf{h}_1(t,\mathbf{x}) - \sigma_t \nabla q_t(\mathbf{x})\Vert
\right).
$$
Then, we have
$$
\begin{aligned}
& ~~~~ \int_{\Vert\mathbf{x}\Vert_\infty\leq 1+C_7\sigma\sqrt{\log \mathcal{R}}} q_t(\mathbf{x})\mathbf{I}_{\{q_t(\mathbf{x})\geq {\mathcal{R}}^{-(2\beta + 1)}\}} \left\Vert \frac{\mathbf{h}^{\prime}(t,\mathbf{x})}{\sigma_t} - \nabla\log q_t(\mathbf{x})\right\Vert^2 \mathrm{d} \mathbf{x} \\
& \lesssim \frac{{\mathcal{R}}^{4\beta + 2}\log \mathcal{R}}{\sigma_t^2} \int_{ \Vert \mathbf{x} \Vert_{\infty} \leq 1 + C_7\sigma\sqrt{\log \mathcal{R}}} \Bigg(
\int_{\mathbb{R}^{d^*}}\frac{1}{\sigma_t^{d^*}(2\pi)^{d^*/2}} |q_{t_0}^*(\mathbf{y})-p_{\mathcal{R}}(\mathbf{y})|^2\exp\left(-\frac{\Vert\mathbf{x}-\mathbf{y}\Vert^2}{2\sigma_t^2}\right)\mathrm{d}\mathbf{y} \\
& ~~~~~~~~~ +  \int_{\mathbb{R}^{d^*}}\frac{\Vert\mathbf{x}-\mathbf{y}\Vert^2}{\sigma_t^{d^*+1}(2\pi)^{d^*/2}} |q_{t_0}^*(\mathbf{y})-p_{\mathcal{R}}(\mathbf{y})|^2\exp\left(-\frac{\Vert\mathbf{x}-\mathbf{y}\Vert^2}{2\sigma_t^2}\right)\mathrm{d}\mathbf{y}  \Bigg) \mathrm{d}\mathbf{x} \\
& \lesssim \frac{{\mathcal{R}}^{4\beta + 2}\log \mathcal{R}}{\sigma_t^2} \int_{\mathbb{R}^{d^*}} \left|q_{t_0}(\mathbf{y}) - p_{\mathcal{R}}(\mathbf{y})  \right|^2 \mathrm{d}\mathbf{y} \lesssim \frac{{\mathcal{R}}^{4\beta + 2}\log \mathcal{R}}{\sigma_t^2} \cdot {\mathcal{R}}^{-(6\beta + 2)} \lesssim \frac{{\mathcal{R}}^{-2\beta}\log \mathcal{R}}{\sigma_t^2}.
\end{aligned}
$$
Therefore, we finally obtain
$$
(\mathrm{III}) \lesssim \frac{{\mathcal{R}}^{-2\beta}\log \mathcal{R}}{\sigma_t^2},
$$
which implies that
$$
\int_{\mathbb{R}^{d^*}} q_t(\mathbf{x})\Vert \mathbf{s}_{\mathrm{score}}^{(2)}(t,\mathbf{x}) - \nabla\log q_t(\mathbf{x})\Vert^2 \mathrm{d}\mathbf{x} \lesssim \frac{{\mathcal{R}}^{-2\beta}\log \mathcal{R}}{\sigma_t^2}.
$$
The proof is complete.
\end{proof}

Combining Lemma \ref{lem:approximate_interval_1} and Lemma \ref{lem:approximate_interval_2}, we immediately obtain Lemma \ref{lem:approximation}.

\begin{proof}[Proof of Lemma \ref{lem:approximation}]
According to Lemma \ref{lem:approximate_interval_1} and Lemma \ref{lem:approximate_interval_2}, there exist two ReLU neural networks $\mathbf{s}_{\mathrm{score}}^{(1)}(t,\mathbf{x})$ and $\mathbf{s}_{\mathrm{score}}^{(2)}(t,\mathbf{x})$ that approximate the score function $\nabla\log q_t(\mathbf{x})$ on $[{\mathcal{R}}^{-C_T},3{\mathcal{R}}^{-C_T}]$ and $[2{\mathcal{R}}^{-C_T},1]$, respectively. 
Therefore, setting $t_1={\mathcal{R}}^{-C_T}$, $t_2=2{\mathcal{R}}^{-C_T}$, $s_1=3{\mathcal{R}}^{-C_T}$ and $s_2=1$ in Lemma \ref{lem:switching}, we can construct a ReLU neural network 
$$\mathbf{s}(t,\mathbf{x}) := \mathrm{s}_{\mathrm{switch},1}(t,t_2,s_1)\mathbf{s}_{\mathrm{score}}^{(1)}(t,\mathbf{x}) + \mathrm{s}_{\mathrm{switch},2}(t,t_2,s_1)\mathbf{s}_{\mathrm{score}}^{(2)}(t,\mathbf{x})
$$ 
with network parameters
$$
L = \mathcal{O}(\log^4 N), M = \mathcal{O}(N^{d^*}\log^7 N), J = \mathcal{O}(N^{d^*}\log^9 N), \kappa = \exp\left(\mathcal{O}(\log^4 N)\right)
$$
that approximates $\nabla\log q_t(\mathbf{x})$ with the approximation error 
$$
\begin{aligned}
& ~~~~ \int_{\mathbb{R}^{d^*}} q_t(\mathbf{x})\Vert \mathbf{s}(t,\mathbf{x}) - \nabla\log q_t(\mathbf{x})\Vert^2 \mathrm{d}\mathbf{x} \\
& \lesssim \mathrm{s}_{\mathrm{switch},1}^2(t,t_2,s_1)\int_{\mathbb{R}^{d^*}} q_t(\mathbf{x})\Vert \mathbf{s}_{\mathrm{score}}^{(1)}(t,\mathbf{x}) - \nabla\log q_t(\mathbf{x})\Vert^2 \mathrm{d}\mathbf{x} \\
& ~~~~~~~~~ + 
\mathrm{s}_{\mathrm{switch},2}^2(t,t_2,s_1)\int_{\mathbb{R}^{d^*}} q_t(\mathbf{x})\Vert \mathbf{s}_{\mathrm{score}}^{(2)}(t,\mathbf{x}) - \nabla\log q_t(\mathbf{x})\Vert^2 \mathrm{d}\mathbf{x} \\
& \lesssim \mathrm{s}_{\mathrm{switch},1}^2(t,t_2,s_1) \frac{{\mathcal{R}}^{-2\beta}\log \mathcal{R}}{\sigma_t^2} + \mathrm{s}_{\mathrm{switch},2}^2(t,t_2,s_1) \frac{{\mathcal{R}}^{-2\beta}\log \mathcal{R}}{\sigma_t^2} \\
& \lesssim \frac{{\mathcal{R}}^{-2\beta}\log \mathcal{R}}{\sigma_t^2}.
\end{aligned}
$$
Here, we use 
$0\leq \mathrm{s}_{\mathrm{switch},1}, \mathrm{s}_{\mathrm{switch},2} \leq 1$. The proof is complete.
\end{proof}

\section{Statistical Error}\label{sec:se}
In this section, we bound the statistical error and prove Lemma \ref{lem:statistical_error}. Then, combining Lemma \ref{lem:approximation} and Lemma \ref{lem:statistical_error}, we prove Theorem \ref{thm:generalization}. We begin by providing an upper bound for $\ell_{\mathbf{s}}(\widehat{\mathbf{x}})$.

\begin{paragraph}{Upper bound for $\ell_{\mathbf{s}}(\widehat{\mathbf{x}})$.}
Now we derive the upper bound for $\ell_{\mathbf{s}}(\widehat{\mathbf{x}})$. It holds that
$$
\mathbb{E}_{\mathbf{z}}\left\Vert \mathbf{s}(1-t,\widehat{\mathbf{x}} + \sigma\sqrt{1-t}\mathbf{z}) + \frac{\mathbf{z}}{\sigma\sqrt{1-t}}\right\Vert^2 \lesssim 
\frac{\log \mathcal{R}}{\sigma^2(1-t)} + \frac{\mathbb{E}\Vert\mathbf{z}\Vert^2}{\sigma^2(1-t)}
\lesssim \frac{\log \mathcal{R}}{1-t}.
$$
Thus, we have
$$
\ell_{\mathbf{s}}(\widehat{\mathbf{x}}) \lesssim \frac{\log^2 \mathcal{R}}{1-{\mathcal{R}}^{-C_T}} \lesssim \log^2 \mathcal{R}
$$
for sufficiently large $\mathcal{R}$.
\end{paragraph}

\begin{paragraph}{Lipschitz continuity for $\ell_{\mathbf{s}}(\widehat{\mathbf{x}})$.}
Now we derive the Lipschitz continuity for $\ell_{\mathbf{s}}(\widehat{\mathbf{x}})$. Note that we restrict ReLU neural networks into class $\mathcal{C}$. For any $\mathbf{s}_1$, $\mathbf{s}_2 \in \mathcal{C}$, by the construction structure of $\mathbf{s}_1$, $\mathbf{s}_2$,
we have
\begin{equation*}
\begin{aligned}
|\ell_{\mathbf{s}_1}(\widehat{\mathbf{x}}) - \ell_{\mathbf{s}_2}(\widehat{\mathbf{x}})|
&\leq\frac{1}{T}\int_{0}^{T}
\mathbb{E}_{\mathbf{z}}\Vert \mathbf{s}_1 - \mathbf{s}_2\Vert\left\Vert \mathbf{s}_1 + \mathbf{s}_2 + \frac{2\mathbf{z}}{\sigma\sqrt{1-t}}\right\Vert \mathrm{d}t\\
&\leq\frac{1}{T}\int_{0}^{T}\left(\mathbb{E}_{\mathbf{z}}\Vert \mathbf{s}_1 - \mathbf{s}_2 \Vert^2\right)^{\frac{1}{2}}\left(\mathbb{E}_{\mathbf{z}}\left\Vert \mathbf{s}_1 + \mathbf{s}_2 + \frac{2\mathbf{z}}{\sigma
\sqrt{1-t}}\right\Vert^2\right)^{\frac{1}{2}}\mathrm{d}t\\
&\lesssim \frac{1}{T}\int_{0}^{T}\left(\mathbb{E}_{\mathbf{z}}\Vert \mathbf{s}_1 - \mathbf{s}_2 \Vert^2\right)^{\frac{1}{2}}\left(\frac{\log \mathcal{R}}{\sigma^2(1-t)} + \frac{d^*}{\sigma^2(1-t)}\right)^{\frac{1}{2}}\mathrm{d}t\\
&\lesssim \left(\frac{1}{T}\int_{0}^{T}\mathbb{E}_{\mathbf{z}}\left\Vert \mathbf{s}_1 - \mathbf{s}_2 \right\Vert^2 \mathrm{d}t\right)^{\frac{1}{2}}\left(\log \mathcal{R} \cdot \frac{1}{T}\int_{0}^{T}\frac{1}{1-t}\mathrm{d}t\right)^{\frac{1}{2}}\\
&\lesssim  \log \mathcal{R} \cdot \left(\frac{1}{T}\int_{0}^{T}\mathbb{E}_{\mathbf{z}}\Vert \mathbf{s}_1 - \mathbf{s}_2 \Vert^2 \mathrm{d}t\right)^{\frac{1}{2}}\\
&\lesssim  \log \mathcal{R} \cdot
\Vert \mathbf{s}_1 - \mathbf{s}_2 \Vert_{L^{\infty}([0,T]\times\mathbb{R}^{d^*})}\\
& \lesssim \log \mathcal{R} \cdot \Vert \mathbf{s}_1 - \mathbf{s}_2 \Vert_{L^{\infty}([0,1-{\mathcal{R}}^{-C_T}]\times[-\mathcal{O}(1)\sqrt{\log \mathcal{R}}, \mathcal{O}(1) \sqrt{\log \mathcal{R}}]^{d^*})}.
\end{aligned}
\end{equation*}
\end{paragraph}

\begin{paragraph}{Covering number evaluation.} 
The covering number of the neural network class is evaluated as follows:
\begin{equation*}
    \begin{aligned}
    \log\mathcal{N}\left(\mathcal{C}, \delta,\Vert\cdot\Vert_{L^{\infty}([0,T]\times\mathbb{R}^{d^*}}\right) 
    &= \log\mathcal{N}\left(\mathcal{C}, \delta,\Vert\cdot\Vert_{L^{\infty}([0,1-{\mathcal{R}}^{-C_T}]\times[-\mathcal{O}(1)\sqrt{\log \mathcal{R}}, \mathcal{O}(1) \sqrt{\log \mathcal{R}}]^{d^*})}\right)\\
    &\lesssim JL\log{\left(\frac{LM\kappa (1-{\mathcal{R}}^{-C_T} \vee \mathcal{O}(1)\sqrt{\log \mathcal{R}})}{\delta}\right)}\\
    & \lesssim {\mathcal{R}}^{d^*}\log^{13}\mathcal{R}\left(\log^4 \mathcal{R} + \log \frac{1}{\delta} \right),
    \end{aligned}
\end{equation*} 
where we used $
L = \mathcal{O}(\log^4 \mathcal{R}), M = \mathcal{O}({\mathcal{R}}^{d^*}\log^7 \mathcal{R}), J = \mathcal{O}(\mathcal{R}^{d^*}\log^9 \mathcal{R}), \kappa = \exp\left(\mathcal{O}(\log^4 \mathcal{R})\right)$. 
The details of this derivation
can be found in \cite[Lemma 5.3]{CJLZ2022nonparametric}.
\end{paragraph}

\begin{proof}[Proof of Lemma \ref{lem:statistical_error}]
Let $\ell(\mathbf{s},\widehat{\mathbf{x}}) = \ell_{\mathbf{s}}(\widehat{\mathbf{x}}) - \ell_{\mathbf{s}^*} (\widehat{\mathbf{x}})$ and $\widehat{\mathcal{X}}^{\prime} = \{\widehat{\mathbf{x}}_1^{\prime}, \cdots, \widehat{\mathbf{x}}_n^{\prime}\}$ be an independent copy of $\widehat{\mathcal{X}}$, 
then for any $\mathbf{s}_1$, $\mathbf{s}_2\in\mathcal{C}$, there exists a constant $C_8 > 0$ such that
$$
|\ell(\mathbf{s}_1,\widehat{\mathbf{x}}) - \ell(\mathbf{s}_2,\widehat{\mathbf{x})}| = |\ell_{\mathbf{s}_1}(\widehat{\mathbf{x}}) - \ell_{\mathbf{s}_2}(\widehat{\mathbf{x}})|\leq C_8\log \mathcal{R} \Vert \mathbf{s}_1 - \mathbf{s}_2 \Vert_{L^{\infty}([0,T]\times\mathbb{R}^{d^*})}. 
$$
We first estimate $\mathbb{E}_{\widehat{\mathcal{X}},\mathcal{T},\mathcal{Z}}\left(\mathcal{L}(\widehat{\mathbf{s}}) - 2\overline{\mathcal{L}}_{\widehat{\mathcal{X}}}(\widehat{\mathbf{s}}) + \mathcal{L}(\mathbf{s}^*)\right)$. 
It follows that
\begin{equation*}
    \begin{aligned}
&\mathbb{E}_{\widehat{\mathcal{X}},\mathcal{T},\mathcal{Z}}\left(\mathcal{L}(\widehat{\mathbf{s}}) - 2\overline{\mathcal{L}}_{\widehat{\mathcal{X}}}(\widehat{\mathbf{s}}) + \mathcal{L}(\mathbf{s}^*)\right)\\ 
&= \mathbb{E}_{\widehat{\mathcal{X}},\mathcal{T},\mathcal{Z}}\left(\mathbb{E}_{\widehat{\mathcal{X}}^{\prime}}\left[\frac{1}{n}\sum_{i=1}^{n}\left(\ell_{\widehat{\mathbf{s}}}(\widehat{\mathbf{x}}_i^{\prime}) - \ell_{\mathbf{s}^*}(\widehat{\mathbf{x}}_i^{\prime})\right)\right]-\frac{2}{n}\sum_{i=1}^{n}\left(\ell_{\widehat{\mathbf{s}}}(\widehat{\mathbf{x}}_i) - \ell_{\mathbf{s}^*}(\widehat{\mathbf{x}}_i)\right)
        \right)\\
        &=\mathbb{E}_{\widehat{\mathcal{X}},\mathcal{T},\mathcal{Z}}\left[\frac{1}{n}\sum_{i=1}^{n}G(\widehat{\mathbf{s}},\widehat{\mathbf{x}}_i)\right],
    \end{aligned}
\end{equation*}
where
$
G(\widehat{\mathbf{s}},\widehat{\mathbf{x}}) = \mathbb{E}_{\widehat{\mathcal{X}}^{\prime}}\left[\ell(\widehat{\mathbf{s}},\widehat{\mathbf{x}}^{\prime}) - 2\ell(\widehat{\mathbf{s}},\widehat{\mathbf{x}})\right].
$

Let $\mathcal{C}_{\delta}$ be the $\delta$-covering of $\mathcal{C}$ with minimum cardinality $\mathcal{N}_\delta:=\mathcal{N}(\mathrm{NN},\delta,\Vert\cdot\Vert_{L^{\infty}([0,T]\times\mathbb{R}^{d^*})})$, then for any $\mathbf{s}\in\mathcal{C}$, there exists a $\mathbf{s}_{\delta}\in\mathcal{C}_{\delta}$ such that 
$$
|\ell(\mathbf{s},\widehat{\mathbf{x}}) - \ell(\mathbf{s}_{\delta},\widehat{\mathbf{x}})|\leq C_8\log \mathcal{R} \Vert \mathbf{s} - \mathbf{s}_{\delta}\Vert_{L^{\infty}([0,T]\times\mathbb{R}^{d^*})}\leq C_8\delta\log \mathcal{R}.
$$
Therefore, 
$$
G(\widehat{\mathbf{s}},\widehat{\mathbf{x}})\leq G(\mathbf{s}_{\delta},\widehat{\mathbf{x}}) + 3C_8\delta\log \mathcal{R}.
$$
Since $|\ell(\mathbf{s}_{\delta},\widehat{\mathbf{x}}_i)|\lesssim \log^2 \mathcal{R}$, there exists a constant $C_9 > 0$ such that
$|\ell(\mathbf{s}_{\delta},\widehat{\mathbf{x}}_i)| \leq C_9\log^2 \mathcal{R}$.
We have that $|\ell(\mathbf{s}_{\delta}, \widehat{\mathbf{x}}_i) - \mathbb{E}\ell(\mathbf{s}_{\delta},\widehat{\mathbf{x}}_i)|\leq 2C_9\log^2 \mathcal{R}$. Let $V^2 = \mathrm{Var}[\ell(\mathbf{s}_{\delta},\widehat{\mathbf{x}}_i)]$, then 
\begin{equation*}
    \begin{aligned}
V^2&\leq\mathbb{E}_{\widehat{\mathcal{X}}}[\ell(\mathbf{s}_{\delta},\widehat{\mathbf{x}}_i)]^2\\
    &\leq C_8^2\log^2 \mathcal{R} \mathbb{E}_{\widehat{\mathcal{X}}}\left[\ell_{\mathbf{s}_{\delta}}(\widehat{\mathbf{x}}_i) - \ell_{\mathbf{s}^*}(\widehat{\mathbf{x}}_i)\right]\\
    &= C_8^2\log^2 \mathcal{R} \mathbb{E}_{\widehat{\mathcal{X}}}[\ell(\mathbf{s}_{\delta}, \widehat{\mathbf{x}}_i)].
    \end{aligned}
\end{equation*}
We obtain
$$
\mathbb{E}_{\widehat{\mathcal{X}}}[\ell(\mathbf{s}_{\delta}, \widehat{\mathbf{x}}_i)]\geq\frac{V^2}{C_8^2\log^2 \mathcal{R}}.
$$
By Bernstein's inequality, we have
\begin{equation*}
    \begin{aligned}
    &~\mathbb{P}_{\widehat{\mathcal{X}},\mathcal{T},\mathcal{Z}}\left[\frac{1}{n}\sum_{i=1}^{n}G(\mathbf{s}_{\delta},\widehat{\mathbf{x}}_i) > t\right]\\
    =&~\mathbb{P}_{\widehat{\mathcal{X}},\mathcal{T},\mathcal{Z}}\left(\mathbb{E}_{\widehat{\mathcal{X}}^{\prime}}\left[\frac{1}{n}\sum_{i=1}^{n}\ell(\mathbf{s}_{\delta}, \widehat{\mathbf{x}}_i^{\prime})\right] - \frac{1}{n}\sum_{i=1}^{n}\ell(\mathbf{s}_{\delta},\widehat{\mathbf{x}}_i) > \frac{t}{2} + \mathbb{E}_{{\widehat{\mathcal{X}}}^{\prime}}\left[\frac{1}{2n}\sum_{i=1}^{n}\ell(\mathbf{s}_{\delta}, \widehat{\mathbf{x}}_i^{\prime})\right]\right)\\
    =&~\mathbb{P}_{\widehat{\mathcal{X}},\mathcal{T},\mathcal{Z}}\left(\mathbb{E}_{\widehat{\mathcal{X}}}\left[\frac{1}{n}\sum_{i=1}^{n}\ell(\mathbf{s}_{\delta}, \widehat{\mathbf{x}}_i)\right] - \frac{1}{n}\sum_{i=1}^{n}\ell(\mathbf{s}_{\delta},\widehat{\mathbf{x}}_i) > \frac{t}{2} + \mathbb{E}_{\widehat{\mathcal{X}}}\left[\frac{1}{2n}\sum_{i=1}^{n}\ell(\mathbf{s}_{\delta}, \widehat{\mathbf{x}}_i)\right]\right)\\
    \leq&~\mathbb{P}_{\widehat{\mathcal{X}},\mathcal{T},\mathcal{Z}}\left(\mathbb{E}_{\widehat{\mathcal{X}}}\left[\frac{1}{n}\sum_{i=1}^{n}\ell(\mathbf{s}_{\delta}, \widehat{\mathbf{x}}_i)\right] - \frac{1}{n}\sum_{i=1}^{n}\ell(\mathbf{s}_{\delta},\widehat{\mathbf{x}}_i) > \frac{t}{2} + \frac{V^2}{2C_8^2\log^2 \mathcal{R}}\right)\\
    \leq&~\exp\left(-\frac{nu^2}{2V^2 + \frac{4uC_9\log^2 \mathcal{R}}{3}}\right)\\
    \leq&\exp\left(-\frac{nt}{8\left(C_8^2 + \frac{C_9}{3}\right)\log^2 \mathcal{R}}\right),
    \end{aligned}
\end{equation*}
where $u = \frac{t}{2} + \frac{V^2}{2C_8^2\log^2 \mathcal{R}}$, and we use $u\geq\frac{t}{2}$ and $V^2\leq 2uC_8^2\log^2 \mathcal{R} $. 
Hence,  for any $t > 3C_8\delta\log \mathcal{R}$, we have
\begin{equation*}
    \begin{aligned}
        \mathbb{P}_{\widehat{\mathcal{X}},\mathcal{T},\mathcal{Z}}\left[\frac{1}{n}\sum_{i=1}^{n}G(\widehat{\mathbf{s}}, \widehat{\mathbf{x}}_i) > t\right]
        &\leq\mathbb{P}_{\widehat{\mathcal{X}},\mathcal{T},\mathcal{Z}}\left[\mathop{\mathrm{sup}}_{\mathbf{s}\in\mathcal{C}}\frac{1}{n}\sum_{i=1}^{n}G(\mathbf{s}, \widehat{\mathbf{x}}_i) > t\right]\\
        &\leq\mathbb{P}_{\widehat{\mathcal{X}},\mathcal{T},\mathcal{Z}}\left[\mathop{\mathrm{max}}_{\mathbf{s}_{\delta}\in \mathcal{C}_{\delta}}\frac{1}{n}\sum_{i=1}^{n}G(\mathbf{s}_{\delta}, \widehat{\mathbf{x}}_i) > t- 3C_8\delta\log \mathcal{R}\right]\\
        &\leq\mathcal{N}_{\delta}\mathop{\mathrm{max}}_{\mathbf{s}_{\delta}\in \mathcal{C}_{\delta}}\mathbb{P}_{\widehat{\mathcal{X}},\mathcal{T},\mathcal{Z}}\left[\frac{1}{n}\sum_{i=1}^{n}G(\mathbf{s}_{\delta}, \widehat{\mathbf{x}}_i) > t- 3C_8\delta\log \mathcal{R}\right]\\
        &\leq\mathcal{N}_{\delta}\exp\left(-\frac{n(t-3C_8\delta\log \mathcal{R})}{8\left(C_8^2 + \frac{C_9}{3}\right)\log^2 \mathcal{R}}\right).
    \end{aligned}
\end{equation*}
By setting $a=\left[3C_8\log \mathcal{R} + 8\left(C_8^2 + \frac{C_9}{3} \right)\log^2 \mathcal{R}\log\mathcal{N}_{\delta} \right] \delta$ and $\delta = \frac{1}{n}$, then we obtain
\begin{equation}\label{eq:statsitical_error1}
    \begin{aligned}
        \mathbb{E}_{\widehat{\mathcal{X}},\mathcal{T},\mathcal{Z}}\left[\frac{1}{n}\sum_{i=1}^{n}G(\widehat{\mathbf{s}},\widehat{\mathbf{x}}_i)\right]
        &\leq a + \mathcal{N}_{\delta}\int_{a}^{\infty}\exp\left(-\frac{n(t-3C_8\delta\log \mathcal{R})}{8\left(C_8^2 + \frac{C_9}{3}\right)\log^2 \mathcal{R}}\right)\mathrm{d}t\\
        & \leq \left[3C_8\log \mathcal{R} + 8\left(C_8^2 + \frac{C_9}{3} \right)\log^2 \mathcal{R}\log\mathcal{N}_{\delta} \right] \delta + \frac{8\left(C_8^2 + \frac{C_9}{3} \right)\log^2 \mathcal{R}}{n} \\
        & \lesssim {\mathcal{R}}^{d^*}\log^{15} \mathcal{R}\left(\log^4 \mathcal{R} + \log\frac{1}{\delta}\right) \delta + \frac{\log^2 \mathcal{R}}{n} \\
        & \lesssim \frac{{\mathcal{R}}^{d^*}\log^{15} \mathcal{R}\left(\log^4 \mathcal{R} + \log n\right)}{n}.
    \end{aligned}
\end{equation}

Next, we estimate $\mathbb{E}_{\widehat{\mathcal{X}},\mathcal{T},\mathcal{Z}}\left(\overline{\mathcal{L}}_{\widehat{\mathcal{X}}}(\widehat{\mathbf{s}}) - \widehat{\mathcal{L}}_{\widehat{\mathcal{X}},\mathcal{T},\mathcal{Z}}(\widehat{\mathbf{s}})\right)$.
Recall that
$$
\overline{\mathcal{L}}_{\widehat{\mathcal{X}}}(\widehat{\mathbf{s}}) - \widehat{\mathcal{L}}_{\widehat{\mathcal{X}},\mathcal{T},\mathcal{Z}}(\widehat{\mathbf{s}}) = \frac{1}{n}\sum_{i=1}^{n}\left(\ell_{\widehat{\mathbf{s}}}(\widehat{\mathbf{x}}_i) - \widehat{\ell}_{\widehat{\mathbf{s}}}(\widehat{\mathbf{x}}_i)\right).
$$
We decompose $\frac{1}{n}\sum_{i=1}^{n}\left(\ell_{\widehat{\mathbf{s}}}(\widehat{\mathbf{x}}_i) - \widehat{\ell}_{\widehat{\mathbf{s}}}(\widehat{\mathbf{x}}_i)\right)$ into 
the following three terms:
$$
\underbrace{\frac{1}{n}\sum_{i=1}^{n}(\ell_{\widehat{\mathbf{s}}}(\widehat{\mathbf{x}}_i) - \ell_{\widehat{\mathbf{s}}}^{\mathrm{trunc}}(\widehat{\mathbf{x}}_i)
)}_{(A)} +\underbrace{\frac{1}{n}\sum_{i=1}^{n}(\ell_{\widehat{\mathbf{s}}}^{\mathrm{trunc}}(\widehat{\mathbf{x}}_i) - \widehat{\ell}_{\widehat{\mathbf{s}}}^{\mathrm{trunc}}(\widehat{\mathbf{x}}_i)
)}_{(B)}
+\underbrace{\frac{1}{n}\sum_{i=1}^{n}
(\widehat{\ell}_{\widehat{\mathbf{s}}}^{\mathrm{trunc}}(\widehat{\mathbf{x}}_i) - \widehat{\ell}_{\widehat{\mathbf{s}}}(\widehat{\mathbf{x}}_i)
)}_{(C)},
$$
where 
$$
\ell_{\widehat{\mathbf{s}}}^{\mathrm{trunc}}(\widehat{\mathbf{x}}_i) = \mathbb{E}_{\mathbf{z}}\left(\frac{1}{T}\int_{0}^{T}\left\Vert\widehat{\mathbf{s}}(1-t,\widehat{\mathbf{x}}_i + \sigma\sqrt{1-t}\mathbf{z}) + \frac{\mathbf{z}}{\sigma\sqrt{1-t}}\right\Vert^2\mathrm{d}t\cdot\mathbf{I}_{\{\Vert\mathbf{z}\Vert_{\infty} \leq r\}}\right),
$$
and
$$
\widehat{\ell}_{\widehat{\mathbf{s}}}^{\mathrm{trunc}}(\widehat{\mathbf{x}}_i) = \frac{1}{m}\sum_{j=1}^{m}\left\Vert\widehat{\mathbf{s}}(1-t_j,\widehat{\mathbf{x}}_i + \sigma\sqrt{1-t_j}\mathbf{z}_j) + \frac{\mathbf{z}_j}{\sigma\sqrt{1-t_j}}\right\Vert^2\mathbf{I}_{\{\Vert\mathbf{z}_j\Vert_{\infty} \leq r\}}.
$$
We estimate these three terms separately.
Firstly, 
\begin{equation*}
    \begin{aligned}
        (A) &= \frac{1}{n}\sum_{i=1}^{n}\mathbb{E}_{\mathbf{z}}\left(\frac{1}{T}\int_{0}^{T}\left\Vert\widehat{\mathbf{s}}(1-t,\widehat{\mathbf{x}}_i + \sigma\sqrt{1-t}\mathbf{z}) + \frac{\mathbf{z}}{\sigma\sqrt{1-t}}\right\Vert^2\mathrm{d}t\cdot\mathbf{I}_{\{\Vert\mathbf{z}\Vert_{\infty} > r\}}\right)\\
        &\lesssim \left(\log \mathcal{R}\mathbb{P}(\Vert\mathbf{z}\Vert_{\infty} > r) + \mathbb{E}\left(
        \Vert\mathbf{z}\Vert^2\mathbf{I}_{\{\Vert\mathbf{z}\Vert_{\infty}>r\}}\right)\right) \cdot \frac{1}{T}\int_{0}^{T}\frac{1}{\sigma^2(1-t)}\mathrm{d}t\\
        &\lesssim \log \mathcal{R}\left(\log \mathcal{R} \mathbb{P}(\Vert\mathbf{z}\Vert_{\infty} > r) + [\mathbb{E}
        (\Vert\mathbf{z}\Vert^4)]^{\frac{1}{2}}\mathbb{P}({\Vert\mathbf{z}\Vert_{\infty}>r})^{\frac{1}{2}} \right)\\
        &\lesssim \log^2 \mathcal{R} \mathbb{P}({\Vert\mathbf{z}\Vert_{\infty}>r})^{\frac{1}{2}}\\
        &\lesssim \log^2 \mathcal{R}\exp\left(-\frac{r^2}{4}\right).
    \end{aligned}
\end{equation*}
Therefore, there exists a constant $C_{10} > 0$ such that
$$
\mathbb{E}_{\widehat{\mathcal{X}},\mathcal{T},\mathcal{Z}}\left(\frac{1}{n}\sum_{i=1}^{n}\left(\ell_{\widehat{\mathbf{s}}}(\widehat{\mathbf{x}}_i) - \ell_{\widehat{\mathbf{s}}}^{\mathrm{trunc}}(\widehat{\mathbf{x}}_i)\right)\right) \leq C_{10}\log^2 \mathcal{R}\exp\left(-\frac{r^2}{4}\right).
$$
Let $h_{\mathbf{s}}(t,\widehat{\mathbf{x}}_i,\mathbf{z}) = \Vert \mathbf{s}(1-t,\widehat{\mathbf{x}}_i + \sigma\sqrt{1-t}\mathbf{z}) + \frac{\mathbf{z}}{\sigma\sqrt{1-t}}\Vert^2\mathbf{I}_{\{\Vert\mathbf{z}\Vert_{\infty}\leq r\}}$, then there exists a constant $C_{11} > 0$ such that
$$
0 \leq h_{\mathbf{s}}(t,\widehat{\mathbf{x}}_i,\mathbf{z})\lesssim \frac{r^2 + \log \mathcal{R}}{1-T} \leq C_{11} {\mathcal{R}}^{C_T}(r^2 + \log \mathcal{R}):= E_{\mathcal{R}}(r).
$$
For any $\delta_1 > 0$, $\mathbf{s}\in\mathrm{NN}$, there exists a $\mathbf{s}_{\delta_1}\in \mathcal{C}_{\delta_1}$ and a constant $C_{12} > 0$ such that
\begin{equation*}
\begin{aligned}
|h_{\mathbf{s}}(t,\widehat{\mathbf{x}}_i,\mathbf{z}) - h_{\mathbf{s}_{\delta_1}}(t,\widehat{\mathbf{x}}_i, \mathbf{z})|
&\lesssim \delta\left\Vert \mathbf{s} + \mathbf{s}_{\delta_1} + \frac{2\mathbf{z}}{\sigma\sqrt{1-t}}\right\Vert\mathbf{I}_{\{\Vert\mathbf{z}\Vert_{\infty}\leq r\}}\\
& \lesssim \frac{r + \sqrt{\log \mathcal{R}}}{\sqrt{1-T}} \cdot \delta_1 \\
&\leq C_{12} {\mathcal{R}}^{\frac{C_T}{2}}(r + \sqrt{\log \mathcal{R}})\delta_1.
\end{aligned}
\end{equation*}
Then, for fixed $\widehat{\mathbf{x}}_i$, we have
\begin{equation*}
\begin{aligned}
&\frac{1}{T}\int_{0}^{T}\mathbb{E}_{\mathbf{z}}h_{\widehat{\mathbf{s}}}(t,\widehat{\mathbf{x}}_i, \mathbf{z})\mathrm{d}t - \frac{1}{m}\sum_{j=1}^{m}h_{\widehat{\mathbf{s}}}(t_j,\mathbf{x}_i,\mathbf{z}_j)\\
\leq&\mathop{\mathrm{sup}}_{\mathbf{s}\in\mathcal{C}}\left(\frac{1}{T}\int_{0}^{T}\mathbb{E}_{\mathbf{z}}h_{\mathbf{s}}(t,\widehat{\mathbf{x}}_i, \mathbf{z})\mathrm{d}t - \frac{1}{m}\sum_{j=1}^{m}h_{\mathbf{s}}(t_j,\widehat{\mathbf{x}}_i,\mathbf{z}_j)\right)\\
\leq&\mathop{\mathrm{max}}_{\mathbf{s}_{\delta}\in \mathcal{C}_{\delta_1}}\left(\frac{1}{T}\int_{0}^{T}\mathbb{E}_{\mathbf{z}}h_{\mathbf{s}_{\delta_1}}(t,\mathbf{x}_i, \mathbf{z})\mathrm{d}t - \frac{1}{m}\sum_{j=1}^{m}h_{\mathbf{s}_{\delta_1}}(t_j,\mathbf{x}_i,\mathbf{z}_j)\right) + 2C_{12}{\mathcal{R}}^{\frac{C_T}{2}}(r + \sqrt{\log \mathcal{R}})\delta_1.
\end{aligned}
\end{equation*}
Let $b = 2C_{12}{\mathcal{R}}^{\frac{C_T}{2}}(r + \sqrt{\log \mathcal{R}})\delta_1$. For $t > b$,  using Hoeffding's inequality implies
$$
\begin{aligned}
&~\mathbb{P}_{\mathcal{T},\mathcal{Z}}\left(\frac{1}{T}\int_{0}^{T}\mathbb{E}_{\mathbf{z}}h_{\widehat{\mathbf{s}}}(t,\widehat{\mathbf{x}}_i, \mathbf{z})\mathrm{d}t - \frac{1}{m}\sum_{j=1}^{m}h_{\widehat{\mathbf{s}}}(t_j,\widehat{\mathbf{x}}_i,\mathbf{z}_j) > t
\right)\\
\leq&~\mathbb{P}_{\mathcal{T},\mathcal{Z}}\left(\mathop{\mathrm{max}}_{\mathbf{s}_{\delta_1}\in \mathcal{C}_{\delta_1}}\left(\frac{1}{T}\int_{0}^{T}\mathbb{E}_{\mathbf{z}}h_{\mathbf{s}_{\delta_1}}(t,\widehat{\mathbf{x}}_i, \mathbf{z})\mathrm{d}t - \frac{1}{m}\sum_{j=1}^{m}h_{\mathbf{s}_{\delta_1}}(t_j,\widehat{\mathbf{x}}_i,\mathbf{z}_j)\right) > t - b
\right)\\
\leq&~\mathcal{N}_{\delta_1}\mathop{\mathrm{max}}_{\mathbf{s}_{\delta_1}\in \mathcal{C}_{\delta_1}}\mathbb{P}_{\mathcal{T},\mathcal{Z}}\left(\frac{1}{T}\int_{0}^{T}\mathbb{E}_{\mathbf{z}}h_{\mathbf{s}_{\delta_1}}(t,\widehat{\mathbf{x}}_i, \mathbf{z})\mathrm{d}t - \frac{1}{m}\sum_{j=1}^{m}h_{\mathbf{s}_{\delta_1}}(t_j,\widehat{\mathbf{x}}_i,\mathbf{z}_j) > t - b
\right)\\
\leq&~\mathcal{N}_{\delta_1}\exp{\left(-\frac{2m(t-b)^2}{E_{\mathcal{R}}^2(r)}\right)}.
\end{aligned}
$$
Therefore, by taking  expectation  over 
$\mathcal{T},\mathcal{Z}$,  for any 
$c_0 > 0$, we deduce that $(B)$ satisfies
$$
\begin{aligned}
\mathbb{E}_{\mathcal{T},\mathcal{Z}}\left(\ell_{\widehat{\mathbf{s}}}^{\mathrm{trunc}}(\widehat{\mathbf{x}}_i) - \widehat{\ell}_{\widehat{\mathbf{s}}}^{\mathrm{trunc}}(\widehat{\mathbf{x}}_i)\right)  &=\int_{0}^{+\infty}\mathbb{P}_{\mathcal{T},\mathcal{Z}}\left(
\ell_{\widehat{\mathbf{s}}}^{\mathrm{trunc}}(\widehat{\mathbf{x}}_i) - \widehat{\ell}_{\widehat{\mathbf{s}}}^{\mathrm{trunc}}(\widehat{\mathbf{x}}_i) > t
\right)\mathrm{d}t\\
&\leq b + c_0 + \mathcal{N}_{\delta_1}\int_{c_0}^{+\infty}\exp{\left(-\frac{2mt^2}{E_{\mathcal{R}}^2(r)}\right)}\mathrm{d}t\\
&\leq b + c_0 + \frac{\sqrt{\pi}}{2}\mathcal{N}_{\delta_1}\exp{\left(-\frac{2mc_0^2}{E_{\mathcal{R}}^2(r)}\right)}\frac{E_{\mathcal{R}}(r)}{\sqrt{2m}}.
\end{aligned}
$$
Thus, we have
$$
\begin{aligned}
&~\mathbb{E}_{\widehat{\mathcal{X}},\mathcal{T},\mathcal{Z}}\left(\frac{1}{n}\sum_{i=1}^{n}\left(\ell_{\widehat{\mathbf{s}}}^{\mathrm{trunc}}(\widehat{\mathbf{x}}_i) - \widehat{\ell}_{\widehat{\mathbf{s}}}^{\mathrm{trunc}}(\widehat{\mathbf{x}}_i)\right)\right) \\
=&~\frac{1}{n}\sum_{i=1}^{n}\mathbb{E}_{\widehat{\mathcal{X}}}\left[\mathbb{E}_{\mathcal{T},\mathcal{Z}}\left(\ell_{\widehat{\mathbf{s}}}^{\mathrm{trunc}}(\widehat{\mathbf{x}}_i) - \widehat{\ell}_{\widehat{\mathbf{s}}}^{\mathrm{trunc}}(\widehat{\mathbf{x}}_i)\right)\right]\\
\leq&~b + c_0 + \frac{\sqrt{\pi}}{2}\mathcal{N}_{\delta_1}\exp{\left(-\frac{2mc_0^2}{E_{\mathcal{R}}^2(r)}\right)}\frac{E_{\mathcal{R}}(r)}{\sqrt{2m}}.
\end{aligned}
$$

The last term can be expressed as
$$
(C) = -\frac{1}{mn}\sum_{i=1}^{n}\sum_{j=1}^{m}\left\Vert\widehat{\mathbf{s}}(1-t_j,\widehat{\mathbf{x}}_i + \sigma\sqrt{1-t_j}\mathbf{z}_j) + \frac{\mathbf{z}_j}{\sigma\sqrt{1-t_j}}\right\Vert^2\mathbf{I}_{\{\Vert\mathbf{z}_j\Vert_{\infty} > r\}}\leq 0,
$$
which implies
$$
\mathbb{E}_{\widehat{\mathcal{X}},\mathcal{T},\mathcal{Z}}\left(
\frac{1}{n}\sum_{i=1}^{n}\left(\widehat{\ell}_{\widehat{\mathbf{s}}}^{\mathrm{trunc}}(\widehat{\mathbf{x}}_i) - \widehat{\ell}_{\widehat{\mathbf{s}}}(\widehat{\mathbf{x}}_i)\right)
\right)\leq 0.
$$
Combining the above inequalities, we have
\begin{equation*}
\mathbb{E}_{\widehat{\mathcal{X}},\mathcal{T},\mathcal{Z}}\left[\frac{1}{n}\sum_{i=1}^{n}\left(\ell_{\widehat{\mathbf{s}}}(\widehat{\mathbf{x}}_i) - \widehat{\ell}_{\widehat{\mathbf{s}}}(\widehat{\mathbf{x}}_i)\right)\right]
\leq b_0 + c_0 + \frac{\sqrt{\pi}}{2}\mathcal{N}_{\delta_1}\exp{\left(-\frac{2mc_0^2}{E_{\mathcal{R}}^2(r)}\right)}\frac{E_{\mathcal{R}}(r)}{\sqrt{2m}},
\end{equation*}
where $b_0 = C_{10}\log^2 \mathcal{R}\exp\left(-\frac{r^2}{4}\right) + 2C_{12}{\mathcal{R}}^{\frac{C_T}{2}}(r + \sqrt{\log \mathcal{R}})\delta_1$. 
By setting $c_0 = E_{\mathcal{R}}(r)\sqrt{\frac{\log{\mathcal{N}_{\delta_1}}}{2m}}$, $r = 2\sqrt{\log{m}}$, and $\delta_1 = \frac{1}{m}$, we obtain
\begin{equation}\label{eq:statsitical_error2}
\begin{aligned}
& ~~~~ \mathbb{E}_{\widehat{\mathcal{X}},\mathcal{T}, \mathcal{Z}}\left[\frac{1}{n}\sum_{i=1}^{n}\left(\ell_{\widehat{\mathbf{s}}}(\widehat{\mathbf{x}}_i) - \widehat{\ell}_{\widehat{\mathbf{s}}}(\widehat{\mathbf{x}}_i)\right)\right]\\
& \leq b_0 + E_{\mathcal{R}}(r)\cdot\frac{\sqrt{\log\mathcal{N}_{1/m}} + 1}{\sqrt{2m}}\\
& \lesssim \frac{\log^2 \mathcal{R} + {\mathcal{R}}^{\frac{C_T}{2}}(\sqrt{\log m} + \sqrt{\log \mathcal{R}})}{m} + {\mathcal{R}}^{C_T}(\log m + \log \mathcal{R}) \cdot \frac{{\mathcal{R}}^{\frac{d^*}{2}}\log^{\frac{13}{2}}\mathcal{R}(\log^2 \mathcal{R} + \sqrt{\log m})}{\sqrt{m}} \\
& \lesssim {\mathcal{R}}^{C_T}(\log m + \log \mathcal{R}) \cdot \frac{{\mathcal{R}}^{\frac{d^*}{2}}\log^{\frac{13}{2}}\mathcal{R}(\log^2 \mathcal{R} + \sqrt{\log m})}{\sqrt{m}}.
\end{aligned}
\end{equation}
The proof is complete.
\end{proof}

Combining Lemma \ref{lem:approximation} and Lemma \ref{lem:statistical_error}, we can prove Theorem \ref{thm:generalization}.

\begin{proof}[Proof of Theorem \ref{thm:generalization}]
By choosing $\mathcal{R} = \lfloor n^{\frac{1}{d^* + 2\beta}} \rfloor + 1 \lesssim n^{\frac{1}{d^* + 2\beta}}$ in Lemma \ref{lem:approximation}, the approximation error can be bounded as
$$
\begin{aligned}
\inf_{\mathbf{s}\in\mathcal{C}}\left(\mathcal{L}(\mathbf{s}) - \mathcal{L}(\mathbf{s}^*)\right) 
&\leq \frac{1}{T}\int_{0}^{T}\left(\int_{\mathbf{x}\sim q_{1-t}(\mathbf{x})}\Vert \mathbf{s}(1-t,\mathbf{x}) - \nabla\log q_{1-t}(\mathbf{x}) \Vert^2 \mathrm{d}\mathbf{x}\right) \mathrm{d} t \\
& \lesssim \frac{1}{T}\int_0^T\frac{{\mathcal{R}}^{-2\beta}\log \mathcal{R}}{\sigma^2(1-t)} \mathrm{d} t \lesssim {\mathcal{R}}^{-2\beta} \log^2 \mathcal{R} \lesssim n^{-\frac{2\beta}{d^* + 2\beta}} \log^2 n.
\end{aligned}
$$
Substituting $\mathcal{R} = \lfloor n^{\frac{1}{d^* + 2\beta}} \rfloor + 1 \lesssim n^{\frac{1}{d^* + 2\beta}}$, $m = n^{\frac{d^* + 8\beta}{d^* + 2\beta}}$ and $C_T = 2\beta$ into \eqref{eq:statsitical_error1} and \eqref{eq:statsitical_error2}, we have
$$
\eqref{eq:statsitical_error1} \lesssim n^{-\frac{2\beta}{d^* + 2\beta}}\log^{19}n, ~~ \eqref{eq:statsitical_error2} \lesssim n^{-\frac{2\beta}{d^* + 2\beta}}\log^{\frac{19}{2}}n,
$$
which implies that the statistical error can be bounded as
$$
\mathbb{E}_{\widehat{\mathcal{X}},\mathcal{T},\mathcal{Z}}\left(\mathcal{L}(\widehat{\mathbf{s}}) - 2\overline{\mathcal{L}}_{\widehat{\mathcal{X}}}(\widehat{\mathbf{s}}) + \mathcal{L}(\mathbf{s}^{*})\right) + 2\mathbb{E}_{\widehat{\mathcal{X}},\mathcal{T},\mathcal{Z}}\left(\overline{\mathcal{L}}_{\widehat{\mathcal{X}}}(\widehat{\mathbf{s}}) - \widehat{\mathcal{L}}_{\widehat{\mathcal{X}},\mathcal{T},\mathcal{Z}}(\widehat{\mathbf{s}})\right) \lesssim n^{-\frac{2\beta}{d^* + 2\beta}}\log^{19}n.
$$
Thus, we finally obtain
$$
\mathbb{E}_{\widehat{\mathcal{X}},\mathcal{T},\mathcal{Z}}\left(\frac{1}{T}\int_{0}^{T}\mathbb{E}_{\mathbf{x}_t}\Vert\widehat{\mathbf{s}}(1-t,\mathbf{x}_t) - \nabla\log q_{1-t}(\mathbf{x}_t)\Vert^2 \mathrm{d}t\right) \lesssim n^{-\frac{2\beta}{d^* + 2\beta}} \log^{19}n.
$$
The proof is complete.
\end{proof}

\section{Bound $\mathbb{E}_{\mathcal{X}, \mathcal{Y},\mathcal{T},\mathcal{Z}}\left[W_2(\widetilde{\pi}_T^L, \widehat{p}_{data}^*)\right]$}\label{sec:bb}
In this section, we bound $\mathbb{E}_{\mathcal{X}, \mathcal{Y},\mathcal{T},\mathcal{Z}}[W_2(\widetilde{\pi}_T^L, \widehat{p}_{data}^*)]$. We can decompose $\mathbb{E}_{\mathcal{X}, \mathcal{Y},\mathcal{T},\mathcal{Z}}\left[W_2(\widetilde{\pi}_T^L, \widehat{p}_{data}^*)\right]$ into two terms:
$$
\mathbb{E}_{\mathcal{X},\mathcal{Y},\mathcal{T},\mathcal{Z}}[W_2({\widetilde{\pi}_T^L}, \widehat{p}^*_{data})]
 \leq 
  \mathbb{E}_{\mathcal{X},\mathcal{Y},\mathcal{T},\mathcal{Z}}[W_2( \widetilde{\pi}_T^L, \pi_T^L)] + 
  \mathbb{E}_{\mathcal{Y}}[W_2(\pi_T^L, \widehat{p}_{data}^*)]
$$
In the following two subsections, we bound $\mathbb{E}_{\mathcal{X},\mathcal{Y},\mathcal{T},\mathcal{Z}}[W_2( \widetilde{\pi}_T^L, \pi_T^L)]$ and $\mathbb{E}_{\mathcal{Y}}[W_2(\pi_T^L, \widehat{p}_{data}^*)]$ separately.
\subsection{Bound $\mathbb{E}_{\mathcal{Y}}[W_2(\pi_T^L, \widehat{p}^{*}_{data})]$} 
In this subsection, we bound the term $\mathbb{E}_{\mathcal{Y}}[W_2(\pi_T^L, \widehat{p}^{*}_{data})]$
and prove Lemma \ref{lem:early_stopping}.

\begin{proof}[Proof of Lemma \ref{lem:early_stopping}]
$\mathbb{E}_{\mathcal{Y}}[W_2(\pi_T^L, \widehat{p}^{*}_{data})]$ can be decomposed into following two terms:
$$
\mathbb{E}_{\mathcal{Y}}[W_2(\pi_T^L, \widehat{p}^{*}_{data})] \leq \mathbb{E}_{\mathcal{Y}}[W_2(\pi_T^L, \pi_T)] + \mathbb{E}_{\mathcal{Y}}[W_2(\pi_T, \widehat{p}^{*}_{data})].
$$
The first term $\mathbb{E}_{\mathcal{Y}}[W_2(\pi_T^L, \pi_T)]$ satisfies that
$$
\begin{aligned}
\mathbb{E}_{\mathcal{Y}}[W_2(\pi_T^L, \pi_T)] 
&\leq \mathbb{E}_{\mathcal{Y}}\left(\mathbb{E}\Vert \mathbf{x}_T - \mathbf{x}_T\mathbf{I}_{\{\Vert\mathbf{x}_T\Vert_{\infty} \leq L\}}\Vert^2\right)^{\frac{1}{2}} \\
& \leq \mathbb{E}_{\mathcal{Y}}\left(
\mathbb{E}\Vert\mathbf{x}_T\Vert^2\mathbf{I}_{\{\Vert \mathbf{x}_T \Vert_\infty > L\}}
\right)^{\frac{1}{2}} \\
& \leq \mathbb{E}_{\mathcal{Y}} \left( \mathbb{E}\Vert\mathbf{x}_T\Vert^4 \cdot \mathbb{P}(\Vert \mathbf{x}_T\Vert_\infty > L)
\right)^{\frac{1}{4}}.
\end{aligned}
$$
Now, we consider the process 
$$
\bar{\mathbf{x}}_t = \bar{\mathbf{x}}_0 + \sigma\bar{\mathbf{w}}_t,~ \bar{\mathbf{x}}_0\sim {\widehat{p}}^{*}_{data}(\mathbf{x}),~ \bar{\mathbf{x}}_1\sim q(\sigma,\mathbf{x}),
$$
where $\bar{\mathbf{w}}_t$ is a standard Brownian motion. Then $\bar{\mathbf{x}}_t\sim q(\sqrt{t}\sigma, \mathbf{x})$ and $\bar{\mathbf{x}}_{1-t}$ satisfies the reverse SDE
$$
\mathrm{d}\mathbf{x}_t = \sigma^2\nabla\log{q}_{1-t}(\mathbf{x}_t)\mathrm{d}t + \sigma \mathrm{d}\mathbf{w}_t.
$$
Thus, $\bar{\mathbf{x}}_{1-t}\overset{d}{=}\mathbf{x}_{t}$, which implies that
$$
\mathbb{E}\Vert\mathbf{x}_T\Vert^4 = \mathbb{E}\Vert \bar{\mathbf{x}}_0 + \sigma\bar{\mathbf{w}}_{1-T}\Vert^4 \lesssim 1 + \sigma^4(1-T)^2d^*(d^* + 2) = \mathcal{O}(1)
$$
and
$$
\begin{aligned}
\mathbb{P}\left(\Vert\mathbf{x}_T\Vert_\infty > L\right) & = \mathbb{P}\left(\Vert \bar{\mathbf{x}}_0 + \sigma\bar{\mathbf{w}}_{1-T}\Vert_\infty > L\right) \\
& \leq 2d^*\exp\left(-\frac{(L-1)^2}{2\sigma^2(1-T)}\right)
\end{aligned}
$$
Using $e^{-x} \leq \frac{1}{x}$ for $x > 0$, we have
$$
\mathbb{E}_{\mathcal{Y}}[W_2(\pi_T^L, \pi_T)]\lesssim \left[\exp\left(-\frac{(L-1)^2}{4\sigma^2(1-T)}\right)\right]^{\frac{1}{2}}  \lesssim \sqrt{1-T} \lesssim n^{-\frac{\beta}{d^* + 2\beta}}.
$$
The second term satisfies that
$$
\mathbb{E}_{\mathcal{Y}}[W_2(\pi_T, \widehat{p}_{data}^*)]  \leq \mathbb{E}_{\mathcal{Y}}\left(
\mathbb{E}\Vert \mathbf{x}_T - \mathbf{x}_1 \Vert^2
\right)^{\frac{1}{2}} 
\leq \mathbb{E}_{\mathcal{Y}}\left(\Vert \sigma\bar{\mathbf{w}}_{1-T}\Vert^2\right)^{\frac{1}{2}} \lesssim \sqrt{1-T} \lesssim n^{-\frac{\beta}{d^* + 2\beta}}.
$$
Therefore, combining the above two inequalities, we finally obtain
\begin{equation}\label{eq:sub_bound1}
\mathbb{E}_{\mathcal{Y}}[W_2(\pi_T^L, \widehat{p}^{*}_{data})] \lesssim n^{-\frac{\beta}{d^* + 2\beta}}.
\end{equation}
The proof is complete.
\end{proof}

\subsection{Bound $\mathbb{E}_{\mathcal{X}, \mathcal{Y},\mathcal{T},\mathcal{Z}}[W_2(\widetilde{\pi}_T^L, \pi_T^L)]$
}
In this subsection, we bound the term $\mathbb{E}_{\mathcal{X}, \mathcal{Y},\mathcal{T},\mathcal{Z}}[W_2(\widetilde{\pi}_T^L, \pi_T^L)]$. Since $\widetilde{\mathbf{x}}_T^L$ and $\mathbf{x}_T^L$ are bounded by $L$, $\widetilde{\pi}_T^L$ and $\pi_T^L$ are supported on bounded region $[-L, L]^{d^*}$. Then, we obtain the inequality $\mathbb{E}_{\mathcal{X}, \mathcal{Y},\mathcal{T},\mathcal{Z}}[W_2(\widetilde{\pi}_T^L, \pi_T^L)] \lesssim \mathbb{E}_{\mathcal{X}, \mathcal{Y},\mathcal{T},\mathcal{Z}}[\mathrm{TV}(\widetilde{\pi}_T^L, \pi_T^L)] \lesssim \mathbb{E}_{\mathcal{X}, \mathcal{Y},\mathcal{T},\mathcal{Z}}[\mathrm{TV}(\widetilde{\pi}_T, \pi_T)]$. Therefore, we only need to bound $\mathbb{E}_{\mathcal{X}, \mathcal{Y},\mathcal{T},\mathcal{Z}}[\mathrm{TV}(\widetilde{\pi}_T, \pi_T)]$. We first introduce the following lemma.

\begin{lemma}[Proposition D.1 in \cite{oko2023diffusion}]
\label{lem:Girsanov}
Let $p_0$ be any probability distribution, and $\mathbf{z}=(\mathbf{z}_t)_{t\in [0,T]}$, $\mathbf{z}^{\prime}=(\mathbf{z}^{\prime}_t)_{t\in [0,T]}$ be two different processes satisfying
$$
\begin{aligned}
& \mathrm{d}\mathbf{z}_t = \mathbf{b}(t,\mathbf{z}_t)\mathrm{d}t + \sigma(t)\mathrm{d}\mathbf{w}_t, ~ \mathbf{z}_0 \sim p_0, \\
& \mathrm{d}\mathbf{z}^{\prime}_t = \mathbf{b}^{\prime}(t,\mathbf{z}^{\prime}_t) \mathrm{d}t + \sigma(t)\mathrm{d}\mathbf{w}_t, ~ \mathbf{z}^{\prime}_0 \sim p_0.
\end{aligned}
$$
We define the distributions of $\mathbf{z}_t$ and $\mathbf{z}^{\prime}_t$ as $p_t$ and $p^{\prime}_t$, and the path measures of $\mathbf{z}$ and $\mathbf{z}^{\prime}$ as $\mathbb{P}$ and $\mathbb{P}^{\prime}$, respectively. Suppose that the Novikov's condition holds, i.e.,
\begin{equation}\label{eq:novikov}
\mathbb{E}_{\mathbb{P}}\left[
\exp\left(\frac{1}{2}\int_0^T \frac{\Vert \mathbf{b}(t,\mathbf{z}_t) - \mathbf{b}^{\prime}(t,\mathbf{z}_t)\Vert^2}{\sigma^2(t)}\mathrm{d}t \right)
\right] < +\infty.
\end{equation}
Then, the Radon-Nikodym derivative of $\mathbb{P}^{\prime}$ with respect to $\mathbb{P}$ is
$$
\frac{\mathrm{d}\mathbb{P}^{\prime}}{\mathrm{d}\mathbb{P}} = \exp\left(-\int_{0}^T\frac{\mathbf{b}(t,\mathbf{z}_t) - \mathbf{b}^{\prime}(t,\mathbf{z}_t)}{\sigma(t)}\mathrm{d}\mathbf{w}_t -\int_0^T \frac{\Vert \mathbf{b}(t,\mathbf{z}_t) - \mathbf{b}^{\prime}(t,\mathbf{z}_t)\Vert^2}{2\sigma^2(t)}\mathrm{d}t\right),
$$
and therefore we have 
$$
\mathrm{KL}(p_T|p_T^{\prime}) \leq \mathrm{KL}(\mathbb{P}|\mathbb{P}^{\prime}) = \mathbb{E}_{\mathbb{P}} \left[\frac{1}{2}\int_0^T\frac{\Vert \mathbf{b}(t,\mathbf{z}_t) - \mathbf{b}^{\prime}(t,\mathbf{z}_t)\Vert^2}{\sigma^2(t)} \mathrm{d}t \right].
$$

\end{lemma}

Based on Lemma \ref{lem:Girsanov}, we can prove Theorem \ref{thm:sampling_error}.

\begin{proof}[Proof of Theorem \ref{thm:sampling_error}]
For any $t\in[0, T]$, $\mathbf{x}_t\sim q_{1-t}(\mathbf{x})$ and $\Vert\widehat{\mathbf{s}}(1-t,\mathbf{x})\Vert$, $\Vert\nabla\log q_{1-t}(\mathbf{x})\Vert \lesssim \frac{\sqrt{\log \mathcal{R}}}{\sigma_{1-t}} \lesssim \frac{\sqrt{\log n}}{\sqrt{1-T}}$.  Therefore, it holds that
$$
\begin{aligned}
&\mathbb{E}_{\mathbb{P}} \left[\exp\left(\sum_{i=0}^{K-1} \int_{t_i}^{t_{i+1}} \frac{\Vert \widehat{\mathbf{s}}(1-t_i, \mathbf{x}_{t_i}) - \nabla\log q_{1-t}(\mathbf{x}_t)\Vert^2}{2\sigma^2} \mathrm{d}t \right)\right]\\
& \lesssim \exp\left(\mathcal{O}\left(\frac{T\log n}{1-T}\right)\right) \\
& \lesssim \exp\left(\mathcal{O}\left(n^{\frac{2\beta}{d^* + 2\beta}} \log n\right)\right) < +\infty.
\end{aligned}
$$
We know that condition \eqref{eq:novikov} is satisfied. Therefore, by Lemma \ref{lem:Girsanov} and inequality $\mathrm{TV}(\widetilde{\pi}_T, \pi_T) = \mathrm{TV}(\pi_T, \widetilde{\pi}_T) \lesssim \sqrt{\mathrm{KL}(\pi_T, \widetilde{\pi}_T)}$, we obtain
$$
\begin{aligned}
\mathbb{E}_{\mathcal{X}, \mathcal{Y},\mathcal{T},\mathcal{Z}}[\mathrm{TV}^2(\widetilde{\pi}_T, \pi_T)] & \lesssim \mathbb{E}_{\mathcal{X}, \mathcal{Y},\mathcal{T},\mathcal{Z}} \left[\mathbb{E}_{\mathbb{P}}\left(\sum_{i=0}^{K-1}\int_{t_i}^{t_{i + 1}} \Vert \widehat{\mathbf{s}}(1-t_i, \mathbf{x}_{t_i}) - \nabla\log q_{1-t}(\mathbf{x}_t)\Vert^2 \mathrm{d}t \right) \right] \\
& \lesssim \mathbb{E}_{\mathcal{X}, \mathcal{Y},\mathcal{T},\mathcal{Z}} \left[\sum_{i=0}^{K-1}\int_{t_i}^{t_{i + 1}} \mathbb{E}_{\mathbb{P}} \Vert \widehat{\mathbf{s}}(1-t_i, \mathbf{x}_{t_i}) - \nabla\log q_{1-t}(\mathbf{x}_t)\Vert^2 \mathrm{d}t  \right].
\end{aligned}
$$
Moreover, the term $\mathbb{E}_{\mathcal{X}, \mathcal{Y},\mathcal{T},\mathcal{Z}}[\mathrm{TV}^2(\widetilde{\pi}_T, \pi_T)]$ can be bounded as
$$
\begin{aligned}
\mathbb{E}_{\mathcal{X}, \mathcal{Y},\mathcal{T},\mathcal{Z}}[\mathrm{TV}^2(\widetilde{\pi}_T, \pi_T)] &\lesssim \underbrace{\mathbb{E}_{\mathcal{X}, \mathcal{Y},\mathcal{T},\mathcal{Z}} \left[\sum_{i=0}^{K-1}\int_{t_i}^{t_{i + 1}} \mathbb{E}_{\mathbb{P}} \Vert \widehat{\mathbf{s}}(1-t_i, \mathbf{x}_{t_i}) - \nabla\log q_{1-t_i}(\mathbf{x}_{t_i})\Vert^2 \mathrm{d}t  \right]}_{\mathrm{(I)}} \\
& ~~~ + \underbrace{\mathbb{E}_{\mathcal{X}, \mathcal{Y},\mathcal{T},\mathcal{Z}} \left[\sum_{i=0}^{K-1}\int_{t_i}^{t_{i + 1}} \mathbb{E}_{\mathbb{P}} \Vert  \nabla\log q_{1-t_i}(\mathbf{x}_{t_i}) - \nabla\log q_{1-t_i}(\mathbf{x}_t)\Vert^2 \mathrm{d}t  \right]}_{\mathrm{(II)}} \\
& ~~~ + \underbrace{\mathbb{E}_{\mathcal{X}, \mathcal{Y},\mathcal{T},\mathcal{Z}} \left[\sum_{i=0}^{K-1}\int_{t_i}^{t_{i + 1}} \mathbb{E}_{\mathbb{P}} \Vert \nabla\log q_{1-t_i}(\mathbf{x}_t) - \nabla\log q_{1-t}(\mathbf{x}_t)\Vert^2 \mathrm{d}t  \right]}_{\mathrm{(III)}}.
\end{aligned}
$$
We bound the above three terms separately. The first term yields that
$$
\mathrm{(I)} = \mathbb{E}_{\widehat{\mathcal{X}}, \mathcal{T},\mathcal{Z}} \left[\sum_{i=0}^{K-1} \mathbb{E}_{\mathbb{P}} \Vert \widehat{\mathbf{s}}(1-t_i, \mathbf{x}_{t_i}) - \nabla\log q_{1-t_i}(\mathbf{x}_{t_i})\Vert^2 (t_{i+1} - t_i)  \right].
$$
Therefore, for any $\epsilon > 0$, there exists $\Delta_{\epsilon} > 0$ such that if $\max_{0\leq i \leq K-1}
(t_{i + 1} - t_i) \leq \Delta_{\epsilon}$, then it holds that
$$
\left|\mathrm{(I)} - \mathbb{E}_{\widehat{\mathcal{X}}, \mathcal{T},\mathcal{Z}}\left[
\int_0^T\mathbb{E}_{\mathbb{P}}\Vert \widehat{\mathbf{s}}(1-t, \mathbf{x}_t) - \nabla\log q_{1-t}(\mathbf{x}_t)\Vert^2 \mathrm{d}t
\right] \right| \leq \epsilon.
$$
Taking $\epsilon = n^{-\frac{2\beta}{d^* + 2\beta}}\log^{19}n$, then there exists $\Delta_n > 0$ such that if $\max_{0\leq i \leq K-1}(t_{i+1} - t_i) \leq \Delta_n$, then
$$
(\mathrm{I}) \lesssim n^{-\frac{2\beta}{d^* + 2\beta}}\log^{19}n.
$$

Next, we bound the second term $\mathrm{(II)}$. For any $\epsilon > 0$, there exists a constant $C > 0$ such that
$$
\begin{aligned}
& ~~~~  \mathbb{E}_{\mathbb{P}} \Vert  \nabla\log q_{1-t_i}(\mathbf{x}_{t_i}) - \nabla\log q_{1-t_i}(\mathbf{x}_t)\Vert^2 \\
& \lesssim \mathbb{E}_{\mathbb{P}} \Vert  \nabla\log q_{1-t_i}(\mathbf{x}_{t_i}) - \nabla\log q_{1-t_i}(\mathbf{x}_t)\Vert^2\mathbf{I}_{\{\Vert\mathbf{x}_{t_i}\Vert_{\infty} > 1 + C\sigma_{1-t_i}\sqrt{\log\epsilon^{-1}}\}} \\
& ~~~~ + \mathbb{E}_{\mathbb{P}} \Vert  \nabla\log q_{1-t_i}(\mathbf{x}_{t_i}) - \nabla\log q_{1-t_i}(\mathbf{x}_t)\Vert^2\mathbf{I}_{\{\Vert\mathbf{x}_{t}\Vert_{\infty} > 1 + C\sigma_{1-t}\sqrt{\log\epsilon^{-1}}\}} \\
& ~~~~ +  \mathbb{E}_{\mathbb{P}} \Vert  \nabla\log q_{1-t_i}(\mathbf{x}_{t_i}) - \nabla\log q_{1-t_i}(\mathbf{x}_t)\Vert^2\mathbf{I}_{\{\Vert\mathbf{x}_{t_i}\Vert_{\infty} \leq 1 + C\sigma_{1-t_i}\sqrt{\log\epsilon^{-1}}, \Vert\mathbf{x}_{t}\Vert_{\infty} \leq 1 + C\sigma_{1-t}\sqrt{\log\epsilon^{-1}}\}} \\
& \lesssim \frac{\log n}{1-t_i} \mathbb{P}\left( \Vert\mathbf{x}_{t_i}\Vert_{\infty} > 1 + C\sigma_{1-t_i}\sqrt{\log\epsilon^{-1}}\right) + \frac{\log n}{1-t_i} \mathbb{P}\left( \Vert\mathbf{x}_{t}\Vert_{\infty} > 1 + C\sigma_{1-t}\sqrt{\log\epsilon^{-1}}\right) \\
& ~~~~ + \mathbb{E}_{\mathbb{P}} \Vert  \nabla\log q_{1-t_i}(\mathbf{x}_{t_i}) - \nabla\log q_{1-t_i}(\mathbf{x}_t)\Vert^2\mathbf{I}_{\{\Vert\mathbf{x}_{t_i}\Vert_{\infty} \leq 1 + C\sigma_{1-t_i}\sqrt{\log\epsilon^{-1}}, \Vert\mathbf{x}_{t}\Vert_{\infty} \leq 1 + C\sigma_{1-t}\sqrt{\log\epsilon^{-1}}\}}.
\end{aligned}
$$
By Lemma \ref{lem:bound_for_clipping}, we have
$$
\mathbb{P}\left( \Vert\mathbf{x}_{t_i}\Vert_{\infty} > 1 + C\sigma_{1-t_i}\sqrt{\log\epsilon^{-1}}\right) \lesssim \sigma_{1-t_i}\epsilon, ~ \mathbb{P}\left( \Vert\mathbf{x}_{t}\Vert_{\infty} > 1 + C\sigma_{1-t}\sqrt{\log\epsilon^{-1}}\right) \lesssim \sigma_{1-t}\epsilon.
$$
According to Lagrange's Mean Value Theorem and Lemma \ref{lem:derivatives_boundness}, there exists $\mathbf{x}^{\prime}$ between $\mathbf{x}_{t_i}$ and $\mathbf{x}_t$ such that
$$
\begin{aligned}
& ~~~ \Vert \nabla\log q_{1-t_i}(\mathbf{x}_{t_i}) - \nabla\log q_{1-t_i}(\mathbf{x}_t)\Vert \mathbf{I}_{\{\Vert\mathbf{x}_{t_i}\Vert_{\infty} \leq 1 + C\sigma_{1-t_i}\sqrt{\log\epsilon^{-1}}, \Vert\mathbf{x}_{t}\Vert_{\infty} \leq 1 + C\sigma_{1-t}\sqrt{\log\epsilon^{-1}}\}} \\
& = \Vert \partial_{\mathbf{x}}\nabla\log q_{1-t_i}(\mathbf{x}^{\prime})\Vert \Vert\mathbf{x}_{t_i} - \mathbf{x}_t\Vert \mathbf{I}_{\{\Vert\mathbf{x}_{t_i}\Vert_{\infty} \leq 1 + C\sigma_{1-t_i}\sqrt{\log\epsilon^{-1}}, \Vert\mathbf{x}_{t}\Vert_{\infty} \leq 1 + C\sigma_{1-t}\sqrt{\log\epsilon^{-1}}\}} \\
& \lesssim \frac{\log\epsilon^{-1}}{\sigma^2_{1-t_i}} \cdot \Vert \mathbf{x}_{t_i} - \mathbf{x}_t\Vert.
\end{aligned}
$$
Then, we have
$$
\begin{aligned}
& ~~~~ \mathbb{E}_{\mathbb{P}}\Vert \nabla\log q_{1-t_i}(\mathbf{x}_{t_i}) - \nabla\log q_{1-t_i}(\mathbf{x}_t)\Vert^2 \mathbf{I}_{\{\Vert\mathbf{x}_{t_i}\Vert_{\infty} \leq 1 + C\sigma_{1-t_i}\sqrt{\log\epsilon^{-1}}, \Vert\mathbf{x}_{t}\Vert_{\infty} \leq 1 + C\sigma_{1-t}\sqrt{\log\epsilon^{-1}}\}} \\
& \lesssim \frac{\log^2\epsilon^{-1}}{\sigma^4_{1-t_i}} \mathbb{E}_{\mathbb{P}}\Vert \mathbf{x}_{t_i} - \mathbf{x}_t \Vert^2 \lesssim \frac{\log^2\epsilon^{-1}}{(1-t_i)^2} \cdot (t-t_i).
\end{aligned}
$$
Therefore, we obtain
$$
\begin{aligned}
& ~~~~ \sum_{i=0}^{K-1}\int_{t_i}^{t_{i + 1}}\mathbb{E}_{\mathbb{P}} \Vert  \nabla\log q_{1-t_i}(\mathbf{x}_{t_i}) - \nabla\log q_{1-t_i}(\mathbf{x}_t)\Vert^2 \mathrm{d}t  \\
& \lesssim \sum_{i=0}^{K-1}\left[\frac{(t_{i+1} - t_i)\epsilon\log n}{\sqrt{1-t_i}} + \frac{(t_{i + 1} - t_i)\epsilon\log n\sqrt{1-t_i}}{1-t_i} + \frac{\log^2\epsilon^{-1}}{(1-t_i)^2} \cdot (t_{i+1}-t_i)^2 \right] \\
& \lesssim \frac{T\epsilon\log n}{\sqrt{1-T}} + \frac{T\log^2\epsilon^{-1}}{(1-T)^2}\max_{0\leq i\leq K-1}(t_{i+1}-t_i) \\
& \lesssim \epsilon n^{\frac{\beta}{d^* + 2\beta}} \log n + n^{\frac{4\beta}{d^* + 2\beta}} \log^2\epsilon^{-1} \max_{0 \leq i \leq K-1}(t_{i+1}-t_i).
\end{aligned}
$$
Taking $\epsilon = n^{-\frac{3\beta}{d^* + 2\beta}}$ and $\max_{0 \leq i \leq K-1}(t_{i+1} - t_i) = \mathcal{O}(n^{-\frac{6\beta}{d^* + 2\beta}})$, we obtain
$$
\mathrm{(II)} \lesssim n^{-\frac{2\beta}{d^* + 2\beta}} \log^2 n. 
$$

Finally, we bound the last term $\mathrm{(III)}$. By Lemma \ref{lem:bound_for_clipping}, for any $\epsilon > 0$, there exists a constant $C > 0$ such that
$$
\begin{aligned} & ~~~~\mathbb{E}_{\mathbb{P}}\Vert\nabla\log q_{1-t_i}(\mathbf{x}_t) - \nabla\log q_{1-t}(\mathbf{x}_t)\Vert^2 \\ & \lesssim \mathbb{E}_{\mathbb{P}}\Vert\nabla\log q_{1-t_i}(\mathbf{x}_t) - \nabla\log q_{1-t}(\mathbf{x}_t)\Vert^2\mathbf{I}_{\{\Vert\mathbf{x}_t\Vert_\infty \leq 1 + C\sigma_{1-t}\sqrt{\log\epsilon^{-1}}\}} 
\\
& ~~~~ +\mathbb{E}_{\mathbb{P}}\Vert\nabla\log q_{1-t_i}(\mathbf{x}_t) - \nabla\log q_{1-t}(\mathbf{x}_t)\Vert^2\mathbf{I}_{\{\Vert\mathbf{x}_t\Vert_\infty > 1 + C\sigma_{1-t}\sqrt{\log\epsilon^{-1}}\}} \\
& \lesssim \mathbb{E}_{\mathbb{P}}\Vert\nabla\log q_{1-t_i}(\mathbf{x}_t) - \nabla\log q_{1-t}(\mathbf{x}_t)\Vert^2\mathbf{I}_{\{\Vert\mathbf{x}_t\Vert_\infty \leq 1 + C\sigma_{1-t}\sqrt{\log\epsilon^{-1}}\}} \\ & ~~~~ + \left(\frac{\log n}{1-t_i} + \frac{\log n}{1-t}\right) \mathbb{P}\left(\Vert\mathbf{x}_t\Vert_\infty > 1 + C\sigma_{1-t}\sqrt{\log \epsilon^{-1}}\right) \\
& \lesssim \mathbb{E}_{\mathbb{P}}\Vert\nabla\log q_{1-t_i}(\mathbf{x}_t) - \nabla\log q_{1-t}(\mathbf{x}_t)\Vert^2\mathbf{I}_{\{\Vert\mathbf{x}_t\Vert_\infty \leq 1 + C\sigma_{1-t}\sqrt{\log\epsilon^{-1}}\}}\\
& ~~~~ + \left(\frac{\log n}{1-t_i} + \frac{\log n}{1-t}\right)\sqrt{1-t}\cdot\epsilon.
\end{aligned}
$$
According to Lagrange’s Mean Value Theorem and Lemma \ref{lem:derivatives_boundness}, there exists $t^{\prime}$ between $t_i$ and $t$ such that
$$
\begin{aligned}
& ~~~~ \mathbb{E}_{\mathbb{P}}\Vert\nabla\log q_{1-t_i}(\mathbf{x}_t) - \nabla\log q_{1-t}(\mathbf{x}_t)\Vert^2\mathbf{I}_{\{\Vert\mathbf{x}_t\Vert_\infty \leq 1 + C\sigma_{1-t}\sqrt{\log\epsilon^{-1}}\}} \\ & = \mathbb{E}_{\mathbb{P}} \Vert\nabla\log q_{1-t^{\prime}}(\mathbf{x}_t)\mathbf{I}_{\{\Vert\mathbf{x}_t\Vert_\infty \leq 1 + C\sigma_{1-t}\sqrt{\log\epsilon^{-1}}\}}(\mathbf{x}_t)\Vert \cdot (t-t_i)^2 \\
& \lesssim \frac{1}{(1-t^{\prime})^{\frac{3}{2}}}\cdot\log^{\frac{3}{2}}\epsilon^{-1} \cdot (t - t_i)^2 \\
& \lesssim \frac{1}{(1-T)^{\frac{3}{2}}}\cdot\log^{\frac{3}{2}}\epsilon^{-1} \cdot (t - t_i)^2.
\end{aligned}
$$
Therefore, we have
$$
\begin{aligned}
& ~~~~ \sum_{i=0}^{K-1}\int_{t_i}^{t_{i + 1}}\mathbb{E}_{\mathbb{P}}\Vert\nabla\log q_{1-t_i}(\mathbf{x}_t) - \nabla\log q_{1-t}(\mathbf{x}_t)\Vert^2 \mathrm{d}t \\
& \lesssim \sum_{i=0}^{K-1}\left[\frac{(t_{i + 1} - t_i)^3}{(1-T)^{\frac{3}{2}}}\cdot\log^{\frac{3}{2}}\epsilon^{-1} + \left(\frac{1}{\sqrt{
1-t_i}} + \frac{1}{\sqrt{1-t_{i + 1}}}\right)\epsilon\log n \right] \\
& \lesssim \frac{T}{(1-T)^{\frac{3}{2}}}\max_{0\leq i \leq K-1}(t_{i + 1} - t_i)\cdot \log^{\frac{3}{2}}\epsilon^{-1} + \frac{T}{\sqrt{1-T}}\cdot\epsilon\log n \\
& \lesssim n^{\frac{3\beta}{d^* + 2\beta}}\max_{0\leq i \leq K-1}(t_{i + 1} - t_i)\cdot \log^{\frac{3}{2}}\epsilon^{-1} + \epsilon\cdot n^{\frac{\beta}{d^* + 2\beta}}\log n.
\end{aligned}
$$
Taking $\epsilon = n^{-\frac{3\beta}{d^* + 2\beta}}$ and $\max_{0\leq i \leq K-1}(t_{i+1} - t_i) = \mathcal{O}(n^{-\frac{5\beta}{d^* + 2\beta}})$, we can bound the term $\mathrm{(III)}$ as
$$
\mathrm{(III)} \lesssim n^{-\frac{2\beta}{d^* + 2\beta}} \log^{\frac{3}{2}}n.
$$
Finally, taking $\max_{0\leq i \leq K-1}(t_{i+1}-t_i) = \mathcal{O}\left(\min\{\Delta_n, n^{-\frac{6\beta}{d^* + 2\beta}}\}\right)$, we can bound $\mathbb{E}_{\mathcal{X},\mathcal{Y},\mathcal{T},\mathcal{Z}}[\mathrm{TV}^2(\widetilde{\pi}_T,\pi_T)]$ as
$$
\mathbb{E}_{\mathcal{X},\mathcal{Y},\mathcal{T}, \mathcal{Z}}[\mathrm{TV}^2(\widetilde{\pi}_T,\pi_T)] \lesssim \mathrm{(I)} + \mathrm{(II)} + \mathrm{(III)} \lesssim n^{-\frac{2\beta}{d^* + 2\beta}}\log^{19}n,
$$
which implies that
\begin{equation} \label{eq:sub_bound2}
\mathbb{E}_{\mathcal{X}, \mathcal{Y},\mathcal{T},\mathcal{Z}}[W_2(\widetilde{\pi}_T^L, \pi_T^L)] \lesssim n^{-\frac{\beta}{d^* + 2\beta}} \log^{\frac{19}{2}}n.
\end{equation}
The proof is complete.
\end{proof}

Based on Theorem \ref{thm:sampling_error} and Lemma \ref{lem:early_stopping}, we can bound $\mathbb{E}_{\mathcal{X}, \mathcal{Y},\mathcal{T},\mathcal{Z}}[W_2(\widetilde{\pi}_T^L, \widehat{p}_{data}^*)]$.
\begin{proof}[Proof of Theorem \ref{thm:oracle_inequality}]
Combining \eqref{eq:sub_bound1} and \eqref{eq:sub_bound2}, we obtain
$$
\begin{aligned}
\mathbb{E}_{\mathcal{X},\mathcal{Y},\mathcal{T},\mathcal{Z}}[W_2({\widetilde{\pi}_T^L}, \widehat{p}^*_{data})]
& \leq 
\mathbb{E}_{\mathcal{X},\mathcal{Y},\mathcal{T},\mathcal{Z}}[W_2( \widetilde{\pi}_T^L, \pi_T^L)] + 
\mathbb{E}_{\mathcal{Y}}[W_2(\pi_T^L, \widehat{p}_{data}^*)] \\
& \lesssim n^{-\frac{\beta}{d^* + 2\beta}}\log^{\frac{19}{2}}n.
\end{aligned}
$$
The proof is complete.
\end{proof}

\section{Main Result}\label{sec:appG}
In this section, we prove our main result 
(Theorem \ref{thm:main_result}).
\begin{proof}[Proof of Theorem \ref{thm:main_result}]
In our framework, we have the following total error decomposition:
\begin{align*}
&~~~~\mathbb{E}_{\mathcal{X},\mathcal{Y},\mathcal{T},\mathcal{Z}} [W_2(\widehat{\boldsymbol{D}}_{\#}\widetilde{\pi}_T^L, p_{data})] \\
&\leq \mathbb{E}_{\mathcal{X},\mathcal{Y}, \mathcal{T}, \mathcal{Z}} [W_2(\widehat{\boldsymbol{D}}_\#\widetilde{\pi}_T^L, (\widehat{\boldsymbol{D}} \circ \widehat{\boldsymbol{E}})_\#p_{data})] + \mathbb{E}_{\mathcal{Y}} [W_2((\widehat{\boldsymbol{D}} \circ \widehat{\boldsymbol{E}})_\#p_{data}, (\widehat{\boldsymbol{D}} \circ \widehat{\boldsymbol{E}})_\#\widetilde{p}_{data})] \\
&~~~~ + \mathbb{E}_{\mathcal{Y}} [W_2((\widehat{\boldsymbol{D}} \circ \widehat{\boldsymbol{E}})_\#\widetilde{p}_{data}, \widetilde{p}_{data})] + W_2(\widetilde{p}_{data},{p_{data}}) \\
&\leq \gamma_{\boldsymbol{D}}\mathbb{E}_{\mathcal{X},\mathcal{Y},\mathcal{T},\mathcal{Z}} [W_2(\widetilde{\pi}_T^L, \widehat{p}^{*}_{data})] + (\gamma_{\boldsymbol{D}}\gamma_{\boldsymbol{E}}+1)W_2(\widetilde{p}_{data}, p_{data}) + \mathbb{E}_{\mathcal{Y}} [\mathcal{H}(\widehat{\boldsymbol{\boldsymbol{E}}}, \widehat{\boldsymbol{D}})]^{1/2},
\end{align*}
where the second inequality follows from $$\mathbb{E}_{\mathcal{Y}} [W_2((\widehat{\boldsymbol{D}} \circ \widehat{\boldsymbol{E}})_\#\widetilde{p}_{data}, \widetilde{p}_{data})] \leq \mathbb{E}_{\mathcal{Y}} [\mathcal{H}(\widehat{\boldsymbol{E}}, \widehat{\boldsymbol{D}})^{1/2}] \leq \mathbb{E}_{\mathcal{Y}} [\mathcal{H}(\widehat{\boldsymbol{E}}, \widehat{\boldsymbol{D}})]^{1/2}.
$$ 
Therefore, under Assumptions \ref{ass:bounded_support}-\ref{ass: distribution drift}, we obtain
$$    
\mathbb{E}_{\mathcal{X},\mathcal{Y},\mathcal{T},\mathcal{Z}}[W_2(\widehat{\boldsymbol{D}}_{\#}\widetilde{\pi}_T^L, p_{data})] = \widetilde{\mathcal{O}}\left(
n^{-\frac{\beta}{d^* + 2\beta}} + \epsilon_{p_{data},\widetilde{p}_{data}} + \mathcal{M}^{-\frac{1}{2(d + 2)}} + \delta_0^{1/2}
\right).
$$
Moreover, if $\mathcal{M}>n^{\frac{2\beta(d+2)}{d^*+2\beta}}$, then
$$
\mathbb{E}_{\mathcal{X},\mathcal{Y}, \mathcal{T}, \mathcal{Z}}[W_2(\widehat{\boldsymbol{D}}_{\#}\widetilde{\pi}_T^L, p_{data})] = \widetilde{\mathcal{O}}\left(
n^{-\frac{\beta}{d^{*} + 2\beta}} +  \epsilon_{p_{data},\widetilde{p}_{data}} + \delta_0^{1/2}
\right).
$$
The proof is complete.
\end{proof}

\section{Auxiliary Lemmas}
\subsection{Several High-Probability Bounds}\label{sec:shb}
Following \cite{oko2023diffusion},
we provide several high-probability bounds in this section.  In the following, 
we denote $q_t(\mathbf{x}):= q(\sqrt{t}\sigma, \mathbf{x})$, and $\sigma_t := \sqrt{t}\sigma$.

\subsubsection{Bounds on $q_t(\mathbf{x})$}

In this section, we give the upper and lower bounds 
on $q_t(\mathbf{x})$.

\begin{lemma}\label{lem:bound_for_density}
For any $\mathbf{x}\in\mathbb{R}^{d^*}$, the following upper and lower bounds on $q_t(\mathbf{x})$ hold: 
\begin{equation}\label{eq:upper_lower_bound_q}
\exp\left(-\frac{d^*(\Vert \mathbf{x} \Vert_{\infty} - 1)_{+}^{2}}{\sigma_t^2}\right) \lesssim q_t(\mathbf{x}) \lesssim \exp\left(-\frac{(\Vert \mathbf{x} \Vert_{\infty} - 1)_{+}^{2}}{2\sigma_t^2}\right).
\end{equation}
\end{lemma}

\begin{proof}
This proof can be divided into two cases: $\mathbf{x}\in [-1, 1]^{d^*}$ and  $\mathbf{x}\notin [-1,1]^{d^*}$.\\
\textbf{
Case I ($\mathbf{x}\in [-1, 1]^{d^*}$): 
}
Given that $\widehat{p}_{data}^{*} \leq C_u$, we have
$$
\begin{aligned}
q_t(\mathbf{x}) &= \int_{\mathbb{R}^{d^*}}\frac{1}{\sigma_t^{d^*}(2\pi)^{d^*/2}} \widehat{p}_{data}^{*}(\mathbf{y})\exp\left(-\frac{\Vert \mathbf{x} - \mathbf{y}\Vert^2}{2\sigma_t^2}\right)\mathrm{d}\mathbf{y}\\
&\leq C_u \int_{\mathbb{R}^{d^*}}\frac{\mathbf{I}_{\{\mathbf{y}\in [-1,1]^{d^*}\}}}{\sigma_t^{d^*}(2\pi)^{d^*/2}} \exp\left(-\frac{\Vert \mathbf{x} - \mathbf{y}\Vert^2}{2\sigma_t^2}\right)\mathrm{d}\mathbf{y}\\
&\leq \frac{C_u 2^{d^*}}{\sigma_t^{d^*}(2\pi)^{d^*/2}}.
\end{aligned}
$$
Additionally, we have
$$
q_t(\mathbf{x}) \leq C_u \int_{\mathbb{R}^{d^*}}\frac{1}{\sigma_t^{d^*}(2\pi)^{d^*/2}} \exp\left(-\frac{\Vert \mathbf{x} - \mathbf{y}\Vert^2}{2\sigma_t^2}\right)\mathrm{d}\mathbf{y} \leq C_u.
$$
Thus, $q_t(\mathbf{x})$ is bounded by $\min\left\{\frac{C_u2^{d^*}}{\sigma_t^{d^*}(2\pi)^{d^*/2}}, C_u\right\}$, which is a constant that depends on $C_u$, $\sigma$, and $d^*$.

The lower bound can be derived as follows:
$$
q_t(\mathbf{x}) \geq C_l\int_{\mathbb{R}^{d^*}}\frac{\mathbf{I}_{\{\mathbf{y}\in[-1,1]^{d^*}\}}}{\sigma_t^{d^*}(2\pi)^{d^*/2}}\exp\left(-\frac{\Vert\mathbf{x}-\mathbf{y}\Vert^2}{2\sigma_t^2}\right)\mathrm{d}\mathbf{y}.
$$
Let $\mathbf{z} = \frac{\mathbf{x} - \mathbf{y}}{\sigma_t}$, then we have
$$
\begin{aligned}
q_t(\mathbf{x}) &\geq C_l\int_{\mathbb{R}^{d^*}}\frac{\mathbf{I}_{\left\{\mathbf{z}\in\left[\frac{\mathbf{x} - 1}{\sigma_t},\frac{\mathbf{x} + 1}{\sigma_t}\right]\right\}}}{(2\pi)^{d^*/2}}\exp\left(-\frac{\Vert\mathbf{z}\Vert^2}{2}\right)\mathrm{d}\mathbf{z} \\
&\geq C_l\int_{\mathbb{R}^{d^*}}\frac{\mathbf{I}_{\left\{\mathbf{z}\in\left[\frac{\mathbf{x} - 1}{\sigma},\frac{\mathbf{x} + 1}{\sigma}\right]\right\}}}{(2\pi)^{d^*/2}}\exp\left(-\frac{\Vert\mathbf{z}\Vert^2}{2}\right)\mathrm{d}\mathbf{z} \\
&\geq \frac{C_l2^{d^*}}{\sigma^{d^*}(2\pi)^{d^*/2}}\exp\left(-\frac{2d^*}{\sigma^2}\right),
\end{aligned}
$$
which is a constant that depends on $C_l$, $\sigma$ and $d^*$.

\noindent
\textbf{
Case II ($\mathbf{x}\notin [-1,1]^{d^*}$): 
}
Let $r = \frac{\Vert \mathbf{x} \Vert_{\infty} - 1}{\sigma_t}$ and $|x_{i^*}| = \Vert \mathbf{x} \Vert_{\infty}$. For $i\neq i^*$, if $x_i > 1$, then
$$
\begin{aligned}
\int_{\mathbb{R}}\frac{\mathbf{I}_{\{y_i\in[-1,1]\}}}{\sigma_t(2\pi)^{1/2}}\exp\left(-\frac{(x_i - y_i)^2}{2\sigma_t^2}\right)\mathrm{d}y_i
&\leq \int_{-1}^{1}\frac{1}{\sigma_t(2\pi)^{1/2}}\exp\left(-\frac{(1-y_i)^2}{2\sigma_t^2}\right)\mathrm{d}y_i \\
&\leq \int_{0}^{\infty}\frac{1}{\sigma_t(2\pi)^{1/2}}\exp\left(-\frac{y_i^2}{2\sigma_t^2}\right)\mathrm{d}y_i \leq \frac{1}{2}. 
\end{aligned}
$$
If $x_i < -1$, then
$$
\begin{aligned}
\int_{\mathbb{R}}\frac{\mathbf{I}_{\{y_i\in[-1,1]\}}}{\sigma_t(2\pi)^{1/2}}\exp\left(-\frac{(x_i - y_i)^2}{2\sigma_t^2}\right)\mathrm{d}y_i
&\leq \int_{-1}^{1}\frac{1}{\sigma_t(2\pi)^{1/2}}\exp\left(-\frac{(-1-y_i)^2}{2\sigma_t^2}\right)\mathrm{d}y_i \\
&\leq \int_{0}^{\infty}\frac{1}{\sigma_t(2\pi)^{1/2}}\exp\left(-\frac{y_i^2}{2\sigma_t^2}\right)\mathrm{d}y_i \leq \frac{1}{2}.
\end{aligned}
$$
Therefore, we have
$$
q_t(\mathbf{x}) \leq \frac{C_u}{2^{d^*-1}}\int_{-1}^{1}\frac{1}{\sigma_t(2\pi)^{1/2}}\exp\left(-\frac{(x_{i^*} - y_{i^*})^2}{2\sigma_t^2}\right)\mathrm{d}y_{i^*}.
$$
Let $z_{i^*} = \frac{x_{i^*} - y_{i^*}}{\sqrt{2}\sigma_t}$, then
$$
q_t(\mathbf{x}) \leq \frac{C_u}{2^{d^*-1}}\int_{\frac{x_{i^*} - 1}{\sqrt{2}\sigma_t}}^{\frac{x_{i^*} + 1}{\sqrt{2}\sigma_t}}\frac{1}{\sqrt{\pi}}\exp(-z_{i^*}^2)\mathrm{d}z_{i^*}.
$$
If $x_{i^*} > 1$, then
$$
q_t(\mathbf{x})\leq \frac{C_u}{2^{d^*-1}}\int_{\frac{\Vert\mathbf{x}\Vert_{\infty} - 1}{\sqrt{2}\sigma_t}}^{\infty}\frac{1}{\sqrt{\pi}}\exp(-z_{i^*}^2)\mathrm{d}z_{i^*}\lesssim \exp\left(-\frac{(\Vert \mathbf{x} \Vert_{\infty} - 1)^2}{2\sigma_t^2}\right).
$$
If $x_i^* < -1$, $x_i^* = -\Vert \mathbf{x}\Vert_{\infty}$, then
$$
\begin{aligned}
q_t(\mathbf{x}) &\leq \frac{C_u}{2^{d^*-1}}\int_{-\frac{x_{i^*} + 1}{\sqrt{2}\sigma_t}}^{-\frac{x_{i^*} - 1}{\sqrt{2}\sigma_t}}\frac{1}{\sqrt{\pi}}\exp(-z_{i^*}^2)\mathrm{d}z_{i^*}\\
&\leq \frac{C_u}{2^{d^*-1}}\int_{\frac{\Vert\mathbf{x}\Vert_{\infty} - 1}{\sqrt{2}\sigma_t}}^{\infty}\frac{1}{\sqrt{\pi}}\exp(-z_{i^*}^2)\mathrm{d}z_{i^*}\\
&\lesssim \exp\left(-\frac{(\Vert \mathbf{x} \Vert_{\infty} - 1)^2}{2\sigma_t^2}\right).
\end{aligned}
$$

On the other hand, for $1\leq i\leq d^*$, let $z_i = \frac{x_i-y_i}{\sqrt{2}\sigma_t}$, then we have
$$
g(x_i) := \int_{-1}^{1}\frac{1}{\sigma_t(2\pi)^{1/2}}\exp\left(-\frac{(x_i - y_i)^2}{2\sigma_t^2}\right)\mathrm{d}y_i = \int_{\frac{x_{i} - 1}{\sqrt{2}\sigma_t}}^{\frac{x_{i} + 1}{\sqrt{2}\sigma_t}}\frac{1}{\sqrt{\pi}}\exp(-z_{i}^2)\mathrm{d}z_{i}.
$$
If $x_i > 1$, then $g(x_i)$ is a decreasing function. Therefore,
\begin{align*}
g(x_i)\geq g(\Vert\mathbf{x}\Vert_{\infty}) &= \int_{\frac{\Vert\mathbf{x}\Vert_{\infty} - 1}{\sqrt{2}\sigma_t}}^{\frac{\Vert\mathbf{x}\Vert_{\infty} + 1}{\sqrt{2}\sigma_t}}\frac{1}{\sqrt{\pi}}\exp(-z_{i}^2)\mathrm{d}z_{i} \\
&\geq \int_{\frac{r}{\sqrt{2}}}^{\frac{r}{\sqrt{2}} + \frac{\sqrt{2}}{\sigma}}\frac{1}{\sqrt{\pi}}\exp(-z_i^2)\mathrm{d}z_i \\
&\geq \frac{\sqrt{2}}{\sigma\sqrt{\pi}}\exp\left(-\left(r^2 + \frac{4}{\sigma^2}\right)\right)\\
&\gtrsim\exp\left(-\frac{(\Vert\mathbf{x}\Vert_{\infty} - 1)^2}{\sigma_t^2}\right).
\end{align*}
If $x_i<-1$, then
$$
g(x_i) = \int_{\frac{-x_{i} - 1}{\sqrt{2}\sigma_t}}^{\frac{-x_{i} + 1}{\sqrt{2}\sigma_t}}\frac{1}{\sqrt{\pi}}\exp(-z_{i}^2)\mathrm{d}z_{i}.
$$
It is easy to check that $g(x_i)$ is a increasing function. Therefore,
$$
g(x_i)\geq g(-\Vert\mathbf{x}\Vert_{\infty}) = \int_{\frac{\Vert\mathbf{x}\Vert_{\infty} - 1}{\sqrt{2}\sigma_t}}^{\frac{\Vert\mathbf{x}\Vert_{\infty} + 1}{\sqrt{2}\sigma_t}}\frac{1}{\sqrt{\pi}}\exp(-z_{i}^2)\mathrm{d}z_{i}
\gtrsim\exp\left(-\frac{(\Vert\mathbf{x}\Vert_{\infty} - 1)^2}{\sigma_t^2}\right).
$$
The above discussion implies that
$$
q_t(\mathbf{x}) \gtrsim \exp\left(-\frac{d^*(\Vert\mathbf{x}\Vert_{\infty} - 1)^2}{\sigma_t^2}\right).
$$

Based on the above discussion of the two cases,
we obtain
$$
\exp\left(-\frac{d^*(\Vert \mathbf{x} \Vert_{\infty} - 1)_{+}^{2}}{\sigma_t^2}\right) \lesssim q_t(\mathbf{x}) \lesssim \exp\left(-\frac{(\Vert \mathbf{x} \Vert_{\infty} - 1)_{+}^{2}}{2\sigma_t^2}\right).
$$
The proof is complete.
\end{proof}

\subsubsection{Bounds on the derivatives of $q_t(\mathbf{x})$ and $\nabla\log q_t(\mathbf{x})$}
In this section, we give the upper bounds on the derivatives of $q_t(\mathbf{x})$ and $\nabla\log q_t(\mathbf{x})$. We first state the following lemma.
\begin{lemma}[Integral Clipping]\label{lem:integral_clipping} Let $\mathbf{x}\in\mathbb{R}^{d^*}$ and $\boldsymbol{\alpha}\in\mathbb{N}^{d^*}$. For any $0 < \epsilon < 1$, there exists a constant $C > 0$ such that
$$
\begin{aligned}
&\Bigg|
\int_{\mathbb{R}^{d^*}}\prod_{i=1}^{d^*}\left(\frac{x_i-y_i}{\sigma_t}\right)^{\alpha_i}\frac{1}{\sigma_t^{d^*}(2\pi)^{d^*/2}}\widehat{p}_{data}^*(\mathbf{y})\exp\left(-\frac{\Vert\mathbf{x}-\mathbf{y}\Vert}{2\sigma_t^2}\right)\mathrm{d}\mathbf{y} \\ & ~~~~~~~~~ -  \int_{A_{\mathbf{x}}}\prod_{i=1}^{d^*}\left(\frac{x_i-y_i}{\sigma_t}\right)^{\alpha_i}\frac{1}{\sigma_t^{d^*}(2\pi)^{d^*/2}}\widehat{p}_{data}^*(\mathbf{y})\exp\left(-\frac{\Vert\mathbf{x}-\mathbf{y}\Vert}{2\sigma_t^2}\right)\mathrm{d}\mathbf{y}
\Bigg| \lesssim \epsilon,
\end{aligned}
$$
where $A_{\mathbf{x}} = \prod_{i=1}^{d^*}a_{i,\mathbf{x}}$ with $a_{i,\mathbf{x}} = [x_i - C\sigma_t\sqrt{\log\epsilon^{-1}}, x_i + C\sigma_t\sqrt{\log\epsilon^{-1}}].$ 
    
\end{lemma}

\begin{proof}
It follows that
$$
\begin{aligned}
&\Bigg|
\int_{\mathbb{R}^{d^*}\backslash A_{\mathbf{x}}} \prod_{i=1}^{d^*}\left(\frac{x_i-y_i}{\sigma_t}\right)^{\alpha_i}\frac{1}{\sigma_t^{d^*}(2\pi)^{d^*/2}}\widehat{p}_{data}^*(\mathbf{y})\exp\left(-\frac{\Vert\mathbf{x}-\mathbf{y}\Vert}{2\sigma_t^2}\right)\mathrm{d}\mathbf{y}
\Bigg| \\
\leq & \frac{C_u}{\sigma_t^{d^*}(2\pi)^{d^*/2}}\int_{\mathbb{R}^{d^*}\backslash A_{\mathbf{x}}}\prod_{i=1}^{d^*} \left|\frac{x_i-y_i}{\sigma_t}\right|^{\alpha_i} \mathbf{I}_{\{\Vert\mathbf{y}\Vert_\infty\leq 1\}}\exp\left(-\frac{\Vert\mathbf{x}-\mathbf{y}\Vert^2}{2\sigma_t^2}\right)\mathrm{d}\mathbf{y} \\
\leq & \frac{C_u}{\sigma_t^{d^*}(2\pi)^{d^*/2}} \cdot \sum_{j=1}^{d^*}
\int_{\mathbb{R}\times\cdots\times\mathbb{R}\times (\mathbb{R}\backslash a_{j,\mathbf{x}})\times\mathbb{R}\times\cdots\times\mathbb{R}} \prod_{i=1}^{d^*} \left|\frac{x_i-y_i}{\sigma_t}\right|^{\alpha_i} \mathbf{I}_{\{|y_i|\leq 1\}}\exp\left(-\frac{\Vert\mathbf{x}-\mathbf{y}\Vert^2}{2\sigma_t^2}\right)\mathrm{d}\mathbf{y} \\
\leq & C_u\cdot\sum_{j=1}^{d^*} \Bigg[ \left(\prod_{i=1, i\neq j}^{d^*} \frac{1}{\sigma_t\sqrt{2\pi}} \int_{\mathbb{R}}\left|\frac{x_i-y_i}{\sigma_t}\right|^{\alpha_i} \mathbf{I}_{\{|y_i|\leq 1\}} \exp\left(-\frac{(x_i-y_i)^2}{2\sigma_t^2}\right) \mathrm{d}y_i
\right) \\
& ~~~~~~~~~~~~~~~~~~~~~~~~~~ \cdot  \frac{1}{\sigma_t\sqrt{2\pi}} \int_{\mathbb{R}\backslash a_{j,\mathbf{x}}}\left|\frac{x_j-y_j}{\sigma_t}\right|^{\alpha_j} \mathbf{I}_{\{|y_j|\leq 1\}} \exp\left(-\frac{(x_j-y_j)^2}{2\sigma_t^2}\right) \mathrm{d}y_j \Bigg].
\end{aligned}
$$
For $i \neq j$, let $z_i = \frac{x_i - y_i}{\sigma_t}$.
Then, it holds that
$$
\frac{1}{\sigma_t\sqrt{2\pi}}\int_{\mathbb{R}}\left|\frac{x_i-y_i}{\sigma_t}\right|^{\alpha_i} \mathbf{I}_{\{|y_i|\leq 1\}} \exp\left(-\frac{(x_i-y_i)^2}{2\sigma_t^2}\right) \mathrm{d}y_i \leq \frac{1}{\sqrt{2\pi}}\int_{\mathbb{R}} |z_i|^{\alpha_i}\exp\left(-\frac{z_i^2}{2}\right)\mathrm{d}z_i,
$$
which is upper bounded by $\mathcal{O}(1)$. 
For $i=j$, we have
$$
\begin{aligned}
&~\frac{1}{\sigma_t\sqrt{2\pi}}\int_{\mathbb{R}\backslash a_{j,\mathbf{x}}}\left|\frac{x_j-y_j}{\sigma_t}\right|^{\alpha_j} \mathbf{I}_{\{|y_j|\leq 1\}} \exp\left(-\frac{(x_j-y_j)^2}{2\sigma_t^2}\right) \mathrm{d}y_j \\ \leq &~\frac{1}{\sqrt{2\pi}}\int_{|z_j|\geq C\sqrt{\log\epsilon^{-1}}} |z_j|^{\alpha_j}\exp\left(-\frac{z_j^2}{2}\right)\mathrm{d}z_j \\
\leq &~ \frac{2}{\sqrt{2\pi}}\int_{z_j\geq C\sqrt{\log\epsilon^{-1}}} z_j^{\alpha_j}\exp\left(-\frac{z_j^2}{2}\right)\mathrm{d}z_j \\
\lesssim &~ \epsilon^{\frac{C}{2}}(\log\epsilon^{-1})^{\frac{\alpha_j}{2}}.
\end{aligned}
$$
Taking $C \geq 4$, we obtain
$$
\frac{1}{\sigma_t\sqrt{2\pi}}\int_{\mathbb{R}\backslash a_{j,\mathbf{x}}}\left|\frac{x_j-y_j}{\sigma_t}\right|^{\alpha_j} \mathbf{I}_{\{|y_j|\leq 1\}} \exp\left(-\frac{(x_j-y_j)^2}{2\sigma_t^2}\right) \mathrm{d}y_j \lesssim \epsilon,
$$
which implies that
$$
\Bigg|
\int_{\mathbb{R}^{d^*}\backslash A_{\mathbf{x}}} \prod_{i=1}^{d^*}\left(\frac{x_i-y_i}{\sigma_t}\right)^{\alpha_i}\frac{1}{\sigma_t^{d^*}(2\pi)^{d^*/2}}\widehat{p}_{data}^*(\mathbf{y})\exp\left(-\frac{\Vert\mathbf{x}-\mathbf{y}\Vert}{2\sigma_t^2}\right)\mathrm{d}\mathbf{y}
\Bigg| \lesssim \sum_{j=1}^{d^*} (\mathcal{O}(1))^{d^*-1}\cdot\epsilon \lesssim \epsilon.
$$
The proof is complete.
\end{proof}

\begin{lemma}[Boundedness of Derivatives]\label{lem:derivatives_boundness}
For any $k\in\mathbb{N}^{+}$, the following upper bounds hold:
\begin{equation}\label{eq:derivative_q}
|\partial_{x_{i_1}}\partial_{x_{i_2}}\cdots\partial_{x_{i_k}}q_t(\mathbf{x})| \lesssim \frac{1}{\sigma_t^k},
\end{equation}
and 
\begin{equation}\label{eq:bound_score}
\Vert \nabla\log q_t(\mathbf{x}) \Vert \lesssim \frac{1}{\sigma_t}\cdot\left(\frac{(\Vert\mathbf{x}\Vert_{\infty} - 1)_{+}}{\sigma_t} \vee 1\right).
\end{equation}
Moreover, we have 
\begin{equation}\label{eq:derivative_x_score}
\Vert\partial_{x_i}\nabla\log q_t(\mathbf{x})\Vert \lesssim \frac{1}{\sigma_t^2}\cdot\left(\frac{(\Vert\mathbf{x}\Vert_{\infty} - 1)_{+}^2}{\sigma_t^2} \vee 1\right), ~~~ 1\leq i\leq d^*,
\end{equation}
and 
\begin{equation}\label{eq:derivative_t_score}
\Vert\partial_t\nabla\log q_t(\mathbf{x})\Vert \lesssim \frac{\partial_t\sigma_t}{\sigma_t^2}\cdot \left(\frac{(\Vert\mathbf{x}\Vert_{\infty} - 1)_{+}^2}{\sigma_t^2} \vee 1\right)^{\frac{3}{2}}.
\end{equation}
\end{lemma}
\begin{proof}
We first prove \eqref{eq:derivative_q}. Let $f_1(\mathbf{x}) = q_t(\mathbf{x})$. For a multi-index $\boldsymbol{\alpha}\in\mathbb{N}^{d^*}$, we denote $f_1^{(\boldsymbol{\alpha})}(\mathbf{x}):= \partial_{x_1}^{\alpha_1}\partial_{x_2}^{\alpha_2}\cdots\partial_{x_{d^*}}^{\alpha_{d^*}}f_1(\mathbf{x})$. For $\boldsymbol{\alpha}\in\mathbb{N}^{d^*}$, we define $B_{\boldsymbol{\alpha}}:=\{\mathbf{u}\in\mathbb{N}^{d^*}|u_i\leq \alpha_i ~ (i=1,2,\cdots,d^*)\}$. Then, it holds that 
$$
\partial_{x_1}^{\alpha_1}\partial_{x_2}^{\alpha_2}\cdots\partial_{x_{d^*}}^{\alpha_{d^*}} e^{-\Vert\mathbf{x}\Vert^2/2} = \sum_{\mathbf{u}\in B_{\boldsymbol{\alpha}}}C_{\mathbf{u}}\partial_{x_1}^{u_1}\partial_{x_2}^{u_2}\cdots\partial_{x_{d^*}}^{u_{d^*}} e^{-\Vert\mathbf{x}\Vert^2/2}
$$
with some constants $C_{\mathbf{u}}$. Then, we can write $f_1^{(\mathbf{s})}(\mathbf{x})$ as
$$
f_1^{(\boldsymbol{\alpha})}(\mathbf{x}) = \frac{1}{\sigma_t^{\sum_{i=1}^{d^*}\alpha_i}}\cdot\sum_{\mathbf{u}\in B_{\boldsymbol{\alpha}}}C_{\mathbf{u}}\int\prod_{i=1}^{d^*}\left(\frac{x_i-y_i}{\sigma_t}\right)^{u_i}\frac{\widehat{p}_{data}^*(\mathbf{y})}{\sigma_t^{d^*}(2\pi)^{d^*/2}}\exp\left(-\frac{\Vert\mathbf{x}-\mathbf{y}\Vert^2}{2\sigma_t^2}\right)\mathrm{d}\mathbf{y}.
$$
Taking $\alpha_{i_1},\cdots,\alpha_{i_k}=1$ and the others as 0, i.e, $\sum_{i=1}^{d^*}\alpha_i = k$, and since $\widehat{p}_{data}^*(\mathbf{y})\leq C_u$, then the term 
$$
\sum_{\mathbf{u}\in B_{\boldsymbol{\alpha}}}C_{\mathbf{u}}\int\prod_{i=1}^{d^*}\left|\frac{x_i-y_i}{\sigma_t}\right|^{u_i}\frac{\widehat{p}_{data}^*(\mathbf{y})}{\sigma_t^{d^*}(2\pi)^{d^*/2}}\exp\left(-\frac{\Vert\mathbf{x}-\mathbf{y}\Vert^2}{2\sigma_t^2}\right)\mathrm{d}\mathbf{y}
$$
is bounded by a constant, which implies that
$$
|\partial_{x_{i_1}}\partial_{x_{i_2}}\cdots\partial_{x_{i_k}}q_t(\mathbf{x})| \lesssim \frac{1}{\sigma_t^k}.
$$

Next, we prove \eqref{eq:bound_score} and \eqref{eq:derivative_x_score}. For convenience, we focus on the first coordinate of $\nabla\log q_t(\mathbf{x})$, and all  other coordinates of $\nabla\log q_t(\mathbf{x})$ are bounded in the same manner. Let $f_2(\mathbf{x}):= [\sigma_t\nabla q_t(\mathbf{x})]_1$. 
Then, we have
$$
[\nabla\log q_t(\mathbf{x})]_1 = \frac{1}{\sigma_t}\cdot\frac{f_2(\mathbf{x})}{f_1(\mathbf{x})}, ~ [\partial_{x_i}\nabla\log q_t(\mathbf{x})]_1 = \frac{1}{\sigma_t}\cdot\left(\frac{\partial_{x_i}f_2(\mathbf{x})}{f_1(\mathbf{x})} - \frac{f_2(\mathbf{x})(\partial_{x_i}f_1(\mathbf{x}))}{f_1^2(\mathbf{x})}\right).
$$
Moreover,
$$
\frac{f_2(\mathbf{x})}{f_1(\mathbf{x})} = \frac{
-\int\left(\frac{x_1-y_1}{\sigma_t}\right)\cdot\frac{1}{\sigma_t^{d^*}(2\pi)^{d^*/2}}\widehat{p}_{data}^*(\mathbf{y})\exp\left(-\frac{\Vert\mathbf{x}-\mathbf{y}\Vert}{2\sigma_t^2}\right)\mathrm{d}\mathbf{y}}
{\int\frac{1}{\sigma_t^{d^*}(2\pi)^{d^*/2}}\widehat{p}_{data}^*(\mathbf{y})\exp\left(-\frac{\Vert\mathbf{x}-\mathbf{y}\Vert}{2\sigma_t^2}\right)\mathrm{d}\mathbf{y}},
$$
$$
\frac{\partial_{x_i}f_1(\mathbf{x})}{f_1(\mathbf{x})} = \frac{1}{\sigma_t}\cdot\frac{
-\int\left(\frac{x_i-y_i}{\sigma_t}\right)\cdot\frac{1}{\sigma_t^{d^*}(2\pi)^{d^*/2}}\widehat{p}_{data}^*(\mathbf{y})\exp\left(-\frac{\Vert\mathbf{x}-\mathbf{y}\Vert}{2\sigma_t^2}\right)\mathrm{d}\mathbf{y}}
{\int\frac{1}{\sigma_t^{d^*}(2\pi)^{d^*/2}}\widehat{p}_{data}^*(\mathbf{y})\exp\left(-\frac{\Vert\mathbf{x}-\mathbf{y}\Vert}{2\sigma_t^2}\right)\mathrm{d}\mathbf{y}},
$$
$$
\frac{\partial_{x_i}f_2(\mathbf{x})}{f_1(\mathbf{x})} = -\frac{1}{\sigma_t} \cdot \frac{
-\int \left(\mathbf{I}_{\{i=1\}} - \frac{x_1-y_1}{\sigma_t}\cdot\frac{x_i-y_i}{\sigma_t}\right)\cdot\frac{1}{\sigma_t^{d^*}(2\pi)^{d^*/2}}\widehat{p}_{data}^*(\mathbf{y})\exp\left(-\frac{\Vert\mathbf{x}-\mathbf{y}\Vert}{2\sigma_t^2}\right)\mathrm{d}\mathbf{y}}
{\int\frac{1}{\sigma_t^{d^*}(2\pi)^{d^*/2}}\widehat{p}_{data}^*(\mathbf{y})\exp\left(-\frac{\Vert\mathbf{x}-\mathbf{y}\Vert}{2\sigma_t^2}\right)\mathrm{d}\mathbf{y}}.
$$
The integral terms in the three equations above can be expressed in the following unified form:
\begin{equation}\label{eq:eq_unified}
\frac{
\int\prod_{i=1}^{d^*}\left(\frac{x_i-y_i}{\sigma_t}\right)^{\alpha_i}\frac{1}{\sigma_t^{d^*}(2\pi)^{d^*/2}}\widehat{p}_{data}^*(\mathbf{y})\exp\left(-\frac{\Vert\mathbf{x}-\mathbf{y}\Vert}{2\sigma_t^2}\right)\mathrm{d}\mathbf{y}}
{\int\frac{1}{\sigma_t^{d^*}(2\pi)^{d^*/2}}\widehat{p}_{data}^*(\mathbf{y})\exp\left(-\frac{\Vert\mathbf{x}-\mathbf{y}\Vert}{2\sigma_t^2}\right)\mathrm{d}\mathbf{y}}
\end{equation}
with $\sum_{i=1}^{d^*}\alpha_i\leq 2$.
According to Lemma \ref{lem:integral_clipping}, for any $0 < \epsilon < 1$, there exists a constant $C > 0$ such that
$$
\begin{aligned}
&\Bigg|
\int_{\mathbb{R}^{d^*}}\prod_{i=1}^{d^*}\left(\frac{x_i-y_i}{\sigma_t}\right)^{\alpha_i}\frac{1}{\sigma_t^{d^*}(2\pi)^{d^*/2}}\widehat{p}_{data}^*(\mathbf{y})\exp\left(-\frac{\Vert\mathbf{x}-\mathbf{y}\Vert}{2\sigma_t^2}\right)\mathrm{d}\mathbf{y} \\ & ~~~~~~~~~ -  \int_{A_{\mathbf{x}}}\prod_{i=1}^{d^*}\left(\frac{x_i-y_i}{\sigma_t}\right)^{\alpha_i}\frac{1}{\sigma_t^{d^*}(2\pi)^{d^*/2}}\widehat{p}_{data}^*(\mathbf{y})\exp\left(-\frac{\Vert\mathbf{x}-\mathbf{y}\Vert}{2\sigma_t^2}\right)\mathrm{d}\mathbf{y}
\Bigg| \lesssim \epsilon,
\end{aligned}
$$
where $A_{\mathbf{x}} = \prod_{i=1}^{d^*}a_{i,\mathbf{x}}$ with $a_{i,\mathbf{x}} = [x_i - C\sigma_t\sqrt{\log\epsilon^{-1}}, x_i + C\sigma_t\sqrt{\log\epsilon^{-1}}].$ Therefore, when $q_t(\mathbf{x})\geq \epsilon$,
$$
\begin{aligned}
|\eqref{eq:eq_unified}| &\leq \frac{
\int_{A_{\mathbf{x}}}\prod_{i=1}^{d^*}\left|\frac{x_i-y_i}{\sigma_t}\right|^{\alpha_i}\frac{1}{\sigma_t^{d^*}(2\pi)^{d^*/2}}\widehat{p}_{data}^*(\mathbf{y})\exp\left(-\frac{\Vert\mathbf{x}-\mathbf{y}\Vert}{2\sigma_t^2}\right)\mathrm{d}\mathbf{y}}
{\int_{A_{\mathbf{x}}}\frac{1}{\sigma_t^{d^*}(2\pi)^{d^*/2}}\widehat{p}_{data}^*(\mathbf{y})\exp\left(-\frac{\Vert\mathbf{x}-\mathbf{y}\Vert}{2\sigma_t^2}\right)\mathrm{d}\mathbf{y}} + \mathcal{O}(1)\\
&\leq\max_{\mathbf{y}\in A_{\mathbf{x}}}\left[\prod_{i=1}^{d^*}\left|\frac{x_i-y_i}{\sigma_t}\right|^{\alpha_i}\right] + \mathcal{O}(1)  \\
&\leq\left(C^2\log\epsilon^{-1}\right)^{\frac{\sum_{i=1}^{d^*}\alpha_i}{2}} + \mathcal{O}(1) \\
&\lesssim \left(\log\epsilon^{-1}\right)^{\frac{\sum_{i=1}^{d^*}\alpha_i}{2}},
\end{aligned}
$$
which implies that
$$
\left| \frac{f_2(\mathbf{x})}{f_1(\mathbf{x})} \right| \lesssim \sqrt{\log\epsilon^{-1}},~
\left| \frac{\partial_{x_i}f_1(\mathbf{x})}{f_1(\mathbf{x})} \right| \lesssim \frac{\sqrt{\log\epsilon^{-1}}}{\sigma_t},~
\left| \frac{\partial_{x_i}f_2(\mathbf{x})}{f_1(\mathbf{x})} \right| \lesssim \frac{\log\epsilon^{-1}}{\sigma_t}.
$$
Then, we obtain
$$
\Vert\nabla\log q_t(\mathbf{x})\Vert \lesssim \frac{\sqrt{\log\epsilon^{-1}}}{\sigma_t}, ~ \Vert\partial_{x_i}\nabla\log q_t(\mathbf{x})\Vert \lesssim \frac{\log\epsilon^{-1}}{\sigma_t^2}.
$$
By Lemma \ref{lem:bound_for_density}, there exists a constant $0 < C_{l,1} < 1$ such that $q_t(\mathbf{x}) \geq C_{l,1}\exp\left(-\frac{d^*(\Vert\Vert_{\infty}-1)_{+}^2}{\sigma_t^2}\right)$.
Replacing $\epsilon$ with $C_{l,1}\exp\left(-\frac{d^*(\Vert\Vert_{\infty}-1)_{+}^2}{\sigma_t^2}\right)$, we have
$$
\Vert \nabla\log q_t(\mathbf{x}) \Vert \lesssim \frac{1}{\sigma_t}\cdot\left(\frac{(\Vert\mathbf{x}\Vert_{\infty} - 1)_{+}}{\sigma_t} \vee 1\right)
$$
and
$$
\Vert\partial_{x_i}\nabla\log q_t(\mathbf{x})\Vert \lesssim \frac{1}{\sigma_t^2}\cdot\left(\frac{(\Vert\mathbf{x}\Vert_{\infty} - 1)_{+}^2}{\sigma_t^2} \vee 1\right), ~~~ 1\leq i\leq d^*.
$$

Finally, we prove \eqref{eq:derivative_t_score}. By a simple calculation, we have
$$
\begin{aligned}
&\partial_t[\nabla\log q_t(\mathbf{x})]_1\\
&= \partial_t\left(\frac{1}{\sigma_t}\cdot\frac{f_2(\mathbf{x})}{f_1(\mathbf{x})}\right) = \left(\partial_t\frac{1}{\sigma_t}\right)\frac{f_2(\mathbf{x})}{f_1(\mathbf{x})} - \frac{1}{\sigma_t} \cdot \frac{\partial_t f_1(\mathbf{x})}{f_1(\mathbf{x})} \cdot \frac{f_2(\mathbf{x})}{f_1(\mathbf{x})} + \frac{1}{\sigma_t} \cdot \frac{\partial_t f_2(\mathbf{x})}{f_1(\mathbf{x})}\\
&= -\frac{\partial_t\sigma_t}{\sigma_t}\cdot[\nabla\log q_t(\mathbf{x})]_1 - \frac{1}{\sigma_t}\cdot
\frac{
\int\frac{\partial_t\sigma_t\left(\Vert\mathbf{x} - \mathbf{y}\Vert^2\sigma_t^{-3}-d\sigma_t^{-1}\right)}{\sigma_t^{d^*}(2\pi)^{d^*/2}}\widehat{p}_{data}^*(\mathbf{y})\exp\left(-\frac{\Vert\mathbf{x}-\mathbf{y}\Vert}{2\sigma_t^2}\right)\mathrm{d}\mathbf{y}}
{\int\frac{1}{\sigma_t^{d^*}(2\pi)^{d^*/2}}\widehat{p}_{data}^*(\mathbf{y})\exp\left(-\frac{\Vert\mathbf{x}-\mathbf{y}\Vert}{2\sigma_t^2}\right)\mathrm{d}\mathbf{y}} \cdot \frac{f_2(\mathbf{x})}{f_1(\mathbf{x})} \\
&~~~-\frac{1}{\sigma_t} \cdot \frac{
\int\frac{1}{\sigma_t^{d^*+1}(2\pi)^{d^*/2}}(x_1-y_1)\widehat{p}_{data}^*(\mathbf{y})\exp\left(-\frac{\Vert\mathbf{x}-\mathbf{y}\Vert}{2\sigma_t^2}\right)\partial_t\sigma_t[\Vert\mathbf{x} - \mathbf{y}\Vert^2\sigma_t^{-3} - (d^* + 1)\sigma_t^{-1}]\mathrm{d}\mathbf{y}}
{\int\frac{1}{\sigma_t^{d^*}(2\pi)^{d^*/2}}\widehat{p}_{data}^*(\mathbf{y})\exp\left(-\frac{\Vert\mathbf{x}-\mathbf{y}\Vert}{2\sigma_t^2}\right)\mathrm{d}\mathbf{y}}.
\end{aligned}
$$
According to \eqref{eq:eq_unified}, we obtain
$$
\Vert\partial_t\nabla\log q_t(\mathbf{x})\Vert \lesssim \frac{\partial_t\sigma_t}{\sigma_t^2}\cdot \left(\frac{(\Vert\mathbf{x}\Vert_{\infty} - 1)_{+}^2}{\sigma_t^2} \vee 1\right)^{\frac{3}{2}}.
$$
The proof is complete.
\end{proof}

\begin{lemma}[Error Bounds for Clipping]\label{lem:bound_for_clipping} Let $0 < \epsilon < 1$. For all $0 < t \leq 1$, there exists a constant $C > 0$ such that 
\begin{equation}\label{eq:clipping_bound_1}
\int_{\Vert\mathbf{x}\Vert_{\infty}\geq 1 + C\sigma_t\sqrt{\log\epsilon^{-1}}} q_t(\mathbf{x})\Vert\nabla\log q_t(\mathbf{x})\Vert^2\mathrm{d}\mathbf{x} \lesssim \frac{\epsilon}{\sigma_t},
\end{equation}

\begin{equation}\label{eq:clipping_bound_2}
\int_{\Vert\mathbf{x}\Vert_{\infty}\geq 1 + C\sigma_t\sqrt{\log\epsilon^{-1}}} q_t(\mathbf{x})\mathrm{d}\mathbf{x} \lesssim \sigma_t\epsilon.
\end{equation}
Moreover, for $\Vert\mathbf{x}\Vert_{\infty} \leq 1 + C\sigma_t\sqrt{\log\epsilon^{-1}}$, it holds that
\begin{equation}\label{eq:clipping_bound_3}
\int_{\Vert\mathbf{x}\Vert_{\infty}\leq 1 + C\sigma_t\sqrt{\log\epsilon^{-1}}} q_t(\mathbf{x})\mathbf{I}_{\{q_t(\mathbf{x})\leq\epsilon\}}\Vert\nabla\log q_t(\mathbf{x})\Vert^2\mathrm{d}\mathbf{x} \lesssim \frac{\epsilon}{\sigma_t^2}\cdot(\log\epsilon^{-1})^{\frac{d^*+2}{2}},
\end{equation}

\begin{equation}\label{eq:clipping_bound_4}
\int_{\Vert\mathbf{x}\Vert_{\infty}\leq 1 + C\sigma_t\sqrt{\log\epsilon^{-1}}} q_t(\mathbf{x})\mathbf{I}_{\{q_t(\mathbf{x})\leq\epsilon\}}\mathrm{d}\mathbf{x} \lesssim \epsilon\cdot(\log\epsilon^{-1})^{\frac{d^*}{2}}.
\end{equation}
\end{lemma}

\begin{proof}
We first prove \eqref{eq:clipping_bound_1}. For $\Vert\mathbf{x}\Vert_\infty \geq 1 + C\sigma_t\sqrt{\log\epsilon^{-1}}$, according to Lemma \ref{lem:bound_for_density} and Lemma \ref{lem:derivatives_boundness}, we have
$$
q_t(\mathbf{x})\Vert\nabla\log q_t(\mathbf{x})\Vert^2 \lesssim \frac{1}{\sigma_t^2}\exp\left(-\frac{(\Vert\mathbf{x}\Vert_\infty - 1)^2}{2\sigma_t^2}\right) \cdot \frac{(\Vert\mathbf{x}\Vert_\infty - 1)^2}{\sigma_t^2}.
$$
Without loss of generality, we set $|x_1| = \Vert\mathbf{x}\Vert$. Then, we have
$$
\begin{aligned}
&\int_{\Vert\mathbf{x}\Vert_{\infty}\geq 1 + C\sigma_t\sqrt{\log\epsilon^{-1}}} q_t(\mathbf{x})\Vert\nabla\log q_t(\mathbf{x})\Vert^2\mathrm{d}\mathbf{x} \\
\lesssim &\int_{|x_1|\geq 1 + C\sigma_t\sqrt{\log\epsilon^{-1}}}\int_{|x_2|\leq |x_1|}\cdots\int_{|x_{d^*}|\leq |x_1|}\frac{1}{\sigma_t^2}\exp\left(-\frac{(|x_1| - 1)^2}{2\sigma_t^2}\right) \cdot \frac{(|x_1| - 1)^2}{\sigma_t^2}\mathrm{d}x_2\cdots\mathrm{d}x_{d^*}\mathrm{d}x_1\\
\lesssim &\int_{|x_1|\geq 1 + C\sigma_t\sqrt{\log\epsilon^{-1}}}\frac{1}{\sigma_t^2}\exp\left(-\frac{(|x_1| - 1)^2}{2\sigma_t^2}\right) \cdot \frac{(|x_1| - 1)^2}{\sigma_t^2}\cdot |x_1|^{d^*-1}\mathrm{d}x_1\\
\lesssim &\int_{x_1\geq 1 + C\sigma_t\sqrt{\log\epsilon^{-1}}}\frac{1}{\sigma_t^2}\exp\left(-\frac{(x_1 - 1)^2}{2\sigma_t^2}\right) \cdot \frac{(x_1 - 1)^2}{\sigma_t^2}\cdot x_1^{d^*-1}\mathrm{d}x_1\\
\lesssim &\int_{C\sqrt{\log\epsilon^{-1}}}^{\infty}\frac{1}{\sigma_t}\exp\left(-\frac{r^2}{2}\right)\cdot r^2(1+r\sigma_t)^{d^*-1}\mathrm{d}r\\
\lesssim & \int_{C\sqrt{\log\epsilon^{-1}}}^{\infty}\frac{1}{\sigma_t}\exp\left(-\frac{r^2}{2}\right)\cdot r^{d^*+1}\mathrm{d}r \\
\lesssim & \frac{1}{\sigma_t}\epsilon^{\frac{C^2}{2}}\cdot(\log\epsilon^{-1})^{\frac{d^*}{2}},
\end{aligned}
$$
where we let $r = \frac{x_1 - 1}{\sigma_t}$. Taking $C \geq 2$, it holds that
$$
\int_{\Vert\mathbf{x}\Vert_{\infty}\geq 1 + C\sigma_t\sqrt{\log\epsilon^{-1}}} q_t(\mathbf{x})\Vert\nabla\log q_t(\mathbf{x})\Vert^2\mathrm{d}\mathbf{x} \lesssim \frac{1}{\sigma_t}\epsilon^2\cdot(\log\epsilon^{-1})^{\frac{d^*}{2}} \lesssim \frac{\epsilon}{\sigma_t}.
$$
Additionally,
\eqref{eq:clipping_bound_2} can be derived in the same manner,  that is,
$$
\begin{aligned}
\int_{\Vert\mathbf{x}\Vert_{\infty}\geq 1 + C\sigma_t\sqrt{\log\epsilon^{-1}}} q_t(\mathbf{x})\mathrm{d}\mathbf{x} 
& \lesssim  \int_{C\sqrt{\log\epsilon^{-1}}}^{\infty}\sigma_t\exp\left(-\frac{r^2}{2}\right)\cdot r^{d^*-1}\mathrm{d}r \\
& \lesssim \sigma_t\epsilon.
\end{aligned}
$$

Now, we consider the second part of this lemma. For $\Vert\mathbf{x}\Vert_\infty \leq 1 + C\sigma_t\sqrt{\log\epsilon^{-1}}$, by Lemma \ref{lem:bound_for_density} and Lemma \ref{lem:derivatives_boundness},  we have
$$
q_t(\mathbf{x})\mathbf{I}_{\{q_t(\mathbf{x})\leq\epsilon\}}\Vert\nabla\log q_t(\mathbf{x})\Vert^2 \lesssim \epsilon\cdot\frac{1}{\sigma_t^2}\log\epsilon^{-1}.
$$
Therefore, 
$$
\begin{aligned}
\int_{\Vert\mathbf{x}\Vert_{\infty}\leq 1 + C\sigma_t\sqrt{\log\epsilon^{-1}}} q_t(\mathbf{x})\mathbf{I}_{\{q_t(\mathbf{x})\leq\epsilon\}}\Vert\nabla\log q_t(\mathbf{x})\Vert^2\mathrm{d}\mathbf{x} &\lesssim \left(1 + C\sigma_t\sqrt{\log\epsilon^{-1}}\right)^{d^*} \cdot \frac{\epsilon}{\sigma_t^2}\log\epsilon^{-1} \\
&\lesssim \frac{\epsilon}{\sigma_t^2}\cdot(\log\epsilon^{-1})^{\frac{d^*+2}{2}}.
\end{aligned}
$$
In the same manner, we also have
$$
\begin{aligned}
\int_{\Vert\mathbf{x}\Vert_{\infty}\leq 1 + C\sigma_t\sqrt{\log\epsilon^{-1}}} q_t(\mathbf{x})\mathbf{I}_{\{q_t(\mathbf{x})\leq\epsilon\}}\mathrm{d}\mathbf{x} &\lesssim \left(1 + C\sigma_t\sqrt{\log\epsilon^{-1}}\right)^{d^*} \cdot \epsilon \\ 
&\lesssim \epsilon\cdot(\log\epsilon^{-1})^{\frac{d^*}{2}}.
\end{aligned}
$$
The proof is complete.
\end{proof}
\subsection{Auxiliary Lemmas on ReLU Network Approximation}\label{sec:app}
In this section, we summarize existing results and fundamental tools for function approximation using neural networks. See \cite{nakada2020adaptive,oko2023diffusion,Fu2024UnveilCD} for more details.

\subsubsection{Construction of a large ReLU network}

\begin{lemma}[Neural Network Concatenation]\label{lem:concatenation}
For a series of ReLU networks $\mathbf{s}_1: \mathbb{R}^{d_1}\rightarrow\mathbb{R}^{d_2}$, $\mathbf{s}_2:\mathbb{R}^{d_2}\rightarrow\mathbb{R}^{d_3}$, $\cdots$, $\mathbf{s}_k:\mathbb{R}^{d_k}\rightarrow\mathbb{R}^{d_{k+1}}$ with $\mathbf{s}_i\in\mathrm{NN}(L_i,M_i,J_i,\kappa_i)~(i=1,2,\cdots,k)$, there exists a neural network $\mathbf{s}\in\mathrm{NN}(L,M,J,\kappa)$ satisfying $\mathbf{s}(\mathbf{x}) = \mathbf{s}_k\circ\mathbf{s}_{k-1}\circ\cdots\circ\mathbf{s}_1(\mathbf{x})$ for all $\mathbf{x}\in\mathbb{R}^{d_1}$, with
$$
L = \sum_{i=1}^{k}L_i, ~~ M\leq 2\sum_{i=1}^{k}M_i, ~~ J\leq 2\sum_{i=1}^{k}J_i, ~~\text{and} ~ \kappa\leq\max_{1\leq i \leq k}\kappa_i.
$$

\end{lemma}

\begin{lemma}[Identity Function]\label{lem:identity_function}
Given $d\in\mathbb{N}_{+}$ and $L\geq 2$, there exists a neural network $\mathbf{s}_{\mathrm{Id}, L} \in \mathrm{NN}(L,M,J,\kappa)$ that realizes $d$-dimensional identity function $\mathbf{s}_{\mathrm{Id},L}(\mathbf{x}) = \mathbf{x}$, $\mathbf{x}\in\mathbb{R}^{d}$. Here 
$$
M = 2d, ~~ J = 2dL, ~~ \kappa = 1.
$$
\end{lemma}

\begin{lemma}[Neural Network Parallelization]\label{lem:parallelization}
For a series of ReLU networks $\mathbf{s}_i: \mathbb{R}^{d_i}\rightarrow\mathbb{R}^{d_i^{\prime}}$ with $\mathbf{s}_i\in\mathrm{NN}(L_i,M_i,J_i,\kappa_i)~(i=1,2,\cdots,k)$, there exists a neural network $\phi\in\mathrm{NN}(L,M,J,\kappa)$ satisfying $\mathbf{s}(\mathbf{x}) = [\mathbf{s}_1^{\top}(\mathbf{x}_1), \mathbf{s}_2^{\top}(\mathbf{x}_2),\cdots,\mathbf{s}_{k}^{\top}(\mathbf{x}_k)]:\rightarrow \mathbb{R}^{d_1+d_2+\cdots+d_k}\rightarrow\mathbb{R}^{d_1^{\prime} + d_2^{\prime} + \cdots + d_k^{\prime}}$ for all $\mathbf{x} = (\mathbf{x}_1^{\top},\mathbf{x}_2^{\top},\cdots,\mathbf{x}_k^{\top})^{\top}\in\mathbb{R}^{d_1+d_2+\cdots+d_k}$($\mathbf{x}_i$ can be shared), with
$$
\begin{aligned}
L = L, ~~ M\leq 2\sum_{i=1}^{k}M_i, ~~ J\leq 2\sum_{i=1}^{k}J_i, ~~\text{and} ~ \kappa\leq\max_{1\leq i \leq k}\kappa_i ~~ (\text{when} ~ L = L_i ~ \text{holds  for all} ~ i), \\
L = \max_{1\leq i\leq k} L_i, ~~ M \leq 2\sum_{i=1}^{k} M_i, ~~ J \leq 2\sum_{i=1}^{k} (J_i + Ld_i^{\prime}),~~ \text{and} ~ \kappa\leq \max{\{\max_{1\leq i\leq k}\kappa_i, 1\}} ~ (\text{otherwise}). 
\end{aligned}
$$    
Moreover, for $\mathbf{x}_1 = \mathbf{x}_2 = \cdots = \mathbf{x}_k = \mathbf{x}\in\mathbb{R}^d$ and $d_1^{\prime} = d_2^{\prime} = \cdots = d_k^{\prime} = d^{\prime}$, there exists a neural network $\mathbf{s}_{\mathrm{sum}}\in\mathrm{NN}(L,M,J,\kappa)$ that realizes $\mathbf{s}_{\mathrm{sum}}(\mathbf{x}) = \sum_{i=1}^{k}\mathbf{s}_i(\mathbf{x})$ with
$$
L = \max_{1\leq i\leq k} L_i + 1, ~~ M \leq 4\sum_{i=1}^{k} M_i, ~~ J \leq 4\sum_{i=1}^{k} (J_i + Ld_i^{\prime}) + 2M,~~ \text{and} ~ \kappa\leq \max{\{\max_{1\leq i\leq k}\kappa_i, 1\}}.
$$
\end{lemma}

\subsubsection{Approximation of basic functions with ReLU network}

\begin{lemma}[Approximating the Multiple Products]\label{lem:product}
Let $d\geq 2$, $C\geq 1$. For any $\epsilon > 0$, there exists a neural network $\mathrm{s}_{\mathrm{prod}}\in\mathrm{NN}(L,M,J,\kappa)$ with $L = \mathcal{O}(\log d(\log d + \log\epsilon^{-1} + d\log C))$, $M = 48d$, $J = \mathcal{O}(d(\log d + \log\epsilon^{-1} + d\log C)$, $\kappa = C^d$ such that
$$
\left|\mathrm{s}_{\mathrm{prod}}(x_1^{\prime},x_2^{\prime},\cdots,x_d^{\prime}) - \prod_{i=1}^d x_i\right| \leq \epsilon + dC^{d-1}\epsilon_0,
$$
for all $\mathbf{x}\in[-C,C]^d$ and $\mathbf{x}^{\prime}\in\mathbb{R}^d$ with $\Vert\mathbf{x} - \mathbf{x}^{\prime}\Vert_{\infty} \leq \epsilon_0$. Moreover, $|\mathrm{s}_{\mathrm{prod}}(\mathbf{x}^{\prime})|\leq C^d$ for all $\mathbf{x}^{\prime}\in\mathbb{R}^d$, and $\mathrm{s}_{\mathrm{prod}}(x_1^{\prime},x_2^{\prime},\cdots,x_d^{\prime}) = 0$ if at least one of $x_i^{\prime}$ is 0.
\end{lemma}

\begin{remark}
We note that some of $x_i$, $x_j$ $(i\neq j)$ can be shared. For $\prod_{i=1}^{I}x_i^{u_i}$ with $u_i\in\mathbb{N}_{+} (i=1,2,\cdots,I)$ and $\sum_{i=1}^{I}u_i = d$, there exists a neural network satisfying the same bounds as above.    
\end{remark}

\begin{lemma}[Approximating the Reciprocal Function]\label{lem:reciprocal}
For any $0 < \epsilon < 1$, there exists a neural network $\mathrm{s}_{\mathrm{rec}}\in\mathrm{NN}(L,M,J,\kappa)$ with $L = \mathcal{O}(\log^2\epsilon^{-1})$, $M = \mathcal{O}(\log^3\epsilon^{-1})$, $J = \mathcal{O}(\log^4\epsilon^{-1})$ and $\kappa = \mathcal{O}(\epsilon^{-2})$ such that
$$
\left|\mathrm{s}_{\mathrm{rec}}(x^{\prime}) - \frac{1}{x}\right| \leq \epsilon + \frac{|x^{\prime} - x|}{\epsilon^2},
$$
for all $x \in [\epsilon, \epsilon^{-1}]$ and $x^{\prime}\in\mathbb{R}$.
\end{lemma}

\begin{lemma}[Approximating the Square Root Function]\label{lem:square_root}
For any $0 < \epsilon < 1$, there exists a neural network $\mathrm{s}_{\mathrm{root}}\in\mathrm{NN}(L,M,J,\kappa)$ with $L = \mathcal{O}(\log^2\epsilon^{-1})$, $M = \mathcal{O}(\log^3\epsilon^{-1})$, $J = \mathcal{O}(\log^4\epsilon^{-1})$ and $\kappa = \mathcal{O}(\epsilon^{-1})$ such that
$$
|\mathrm{s}_{\mathrm{root}}(x^{\prime}) - \sqrt{x}| \leq \epsilon + \frac{x^{\prime} - x}{\sqrt{\epsilon}},
$$
for all $x\in[\epsilon, \epsilon^{-1}]$ and $x^{\prime}\in\mathbb{R}$.
\end{lemma}

\subsubsection{Clipping and switching functions}
\begin{lemma}[Clipping Function]
\label{lem:clipping}
For any $\mathbf{a}$, $\mathbf{b}\in\mathbb{R}^d$ with $a_i \leq b_i$ 
$(i=1,2,\cdots,d)$, there exists a clipping function $\mathrm{s}_{\mathrm{clip}}(\mathbf{x}, \mathbf{a},\mathbf{b})\in\mathrm{NN}(L,M,J,\kappa)$ with
$$
L = 2, ~ M = 2d, ~ J = 7d, ~ \kappa = \left(\max_{1\leq i \leq d}\max\{|a_i|, |b_i|\}\right) \vee 1, 
$$
such that
$$
\mathrm{s}_{\mathrm{clip}}(\mathbf{x}, \mathbf{a}, \mathbf{b})_i = \min\{b_i, \max\{x_i, a_i\}\} ~~ (i=1,2,\cdots,d).
$$
When $a_i=c_{min}$ and $b_i=c_{max}$ for all $i$, 
we sometimes denote $\mathrm{s}_{\mathrm{clip}}(\mathbf{x},\mathbf{a},\mathbf{b})$ as $\mathrm{s}_{\mathrm{clip}}(\mathbf{x},c_{min},c_{max})$ using scalar values $c_{min}$ and $c_{max}$.
\end{lemma}

\begin{lemma}[Switching Function]\label{lem:switching}
Let $t_1 < t_2 < s_1 < s_2$, and $f(t,\mathbf{x})$ be a scalar-valued function (for a vector-valued function, we just apply this coordinate-wise). Assume that $|\phi_1(t,\mathbf{x}) - f(t,\mathbf{x})| \leq \epsilon$ on $[t_1, s_1]$ and $|\phi_2(t,\mathbf{x}) - f(t,\mathbf{x})| \leq \epsilon$ on $[t_2, s_2]$. Then, there exist two neural networks $\mathrm{s}_{\mathrm{switch},1}(t, t_2, s_1)$ and $\mathrm{s}_{\mathrm{switch},2}(t, t_2, s_1)\in\mathrm{NN}(L,M,J,\kappa)$ with 
$$
L = 3, ~~ M = 2, ~~ S = 8, ~ \text{and}~ \kappa = \max\{s_1, (s_1 - t_2)^{-1}\}
$$
such that
$$
|\mathrm{s}_{\mathrm{switch},1}(t,t_2,s_1)\phi_1(t,\mathbf{x}) + \mathrm{s}_{\mathrm{switch},2}(t,t_2,s_1)\phi_2(t,\mathbf{x}) - f(t,\mathbf{x})| \leq \epsilon
$$
holds for any $t\in[t_1, s_2]$, where
$$
\mathrm{s}_{\mathrm{switch},1}(t,t_2,s_1) = \frac{1}{s_1 - t_2}\mathrm{ReLU}\left(s_1 - \mathrm{s}_{\mathrm{clip}}(t,t_2,s_1)\right),
$$
$$
\mathrm{s}_{\mathrm{switch},2}(t,t_2,s_1) = \frac{1}{s_1 - t_2}\mathrm{ReLU}\left( \mathrm{s}_{\mathrm{clip}}(t,t_2,s_1)-t_2\right).
$$
\end{lemma}

\bibliographystyle{alpha}
\bibliography{biblio}

\end{document}